\documentclass[10pt]{article}
\ifdefined\XeTeXgenerateactualtext
\fi

\usepackage[margin=1in]{geometry}
\usepackage{mathptmx}
\usepackage[T1]{fontenc}
\usepackage{microtype}
\usepackage{graphicx}
\usepackage{xcolor}
\usepackage{booktabs}
\usepackage{amsmath,amssymb,amsthm}
\usepackage[numbers,sort&compress]{natbib}
\usepackage[hidelinks]{hyperref}
\hypersetup{
  pdfauthor={Zhicheng Du and Lan Ma},
  pdftitle={What Simple Readouts Miss in Continuous Language Diffusion},
  pdfsubject={Computation and Language; Machine Learning},
  pdfkeywords={continuous language diffusion, affine readouts, residual errors, matched repair, decoder interfaces, output selection}
}
\usepackage{caption}
\usepackage{subcaption}
\usepackage[section]{placeins}
\usepackage{float}
\usepackage{enumitem}
\usepackage{array}
\usepackage{etoolbox}
\AtBeginEnvironment{figure}{\setlength{\parfillskip}{0pt plus 1fil}}
\AtBeginEnvironment{table}{\setlength{\parfillskip}{0pt plus 1fil}}
\AtBeginEnvironment{thebibliography}{\setlength{\parfillskip}{0pt plus 1fil}}

\title{\vspace{-1.5em}What Simple Readouts Miss\\in Continuous Language Diffusion}
\author{%
  Zhicheng Du, Lan Ma\thanks{Corresponding author.}\\[0.4em]
  {\normalsize Tsinghua Shenzhen International Graduate School, Tsinghua University}\\[0.25em]
  {\small \href{mailto:duzc21@mails.tsinghua.edu.cn}{\texttt{duzc21@mails.tsinghua.edu.cn}},
  \href{mailto:malan@sz.tsinghua.edu.cn}{\texttt{malan@sz.tsinghua.edu.cn}}}%
}
\date{November 7, 2026}

\newcommand{\secref}[1]{Section~\hyperref[#1]{\ref*{#1}}}
\newcommand{\figref}[1]{Figure~\hyperref[#1]{\ref*{#1}}}
\newcommand{\tabref}[1]{Table~\hyperref[#1]{\ref*{#1}}}
\newcommand{\thmref}[1]{Theorem~\hyperref[#1]{\ref*{#1}}}
\newcommand{\propref}[1]{Proposition~\hyperref[#1]{\ref*{#1}}}
\newcommand{\corref}[1]{Corollary~\hyperref[#1]{\ref*{#1}}}
\newcommand{\pentry}{\tau_{\mathrm{persist}}}
\newcommand{\parhead}[1]{\par\noindent\textbf{#1.}\hspace{0.5em}}
\newcommand{\leadin}[1]{\textbf{#1}\space}
\newcommand{\axishead}[1]{{\bfseries #1}}
\definecolor{ELFTakeawayBlue}{HTML}{29465B}
\newcommand{\takeawaybox}[4]{%
  \par\addvspace{7pt}%
  \begingroup
  \setlength{\fboxsep}{0pt}%
  \setlength{\fboxrule}{0.45pt}%
  \noindent\fcolorbox{ELFTakeawayBlue}{white}{%
    \begin{minipage}{\dimexpr\linewidth-2\fboxrule\relax}
      \setlength{\parindent}{0pt}%
      \setlength{\parskip}{0pt}%
      \phantomsection\label{#4}%
      \colorbox{ELFTakeawayBlue}{%
        \hspace*{8pt}%
        \parbox{\dimexpr\linewidth-16pt\relax}{%
          \vspace{2pt}\small\bfseries\color{white}%
          \strut #1: #2\strut\par\vspace{2pt}}%
        \hspace*{8pt}}%
      \par\nointerlineskip\kern6pt
      \noindent\hspace*{8pt}%
      \begin{minipage}{\dimexpr\linewidth-16pt\relax}
        \small\raggedright\setlength{\parfillskip}{0pt plus 1fil}%
        #3\par
      \end{minipage}%
      \par\kern6pt
    \end{minipage}}%
  \par\endgroup
  \addvspace{7pt}%
}

\newtheorem{proposition}{Proposition}
\newtheorem{theorem}{Theorem}
\newtheorem{corollary}{Corollary}

\begin{document}
\maketitle

\begin{abstract}

How much of a learned decoder can a simple readout preserve? On generated ELF and AURORA states, token-wise affine maps recover $97.17\%$ and $95.74\%$ of native lexical decisions. Contextual computation has already produced these states; only the terminal readout is simplified. The remaining disagreements matter beyond their count: equal-count rare-first and low-native-margin-first repairs have opposing pooled evaluator effects in both interfaces. On conditional summarization, rare-first improves reference similarity, while source consistency remains unresolved. Score accounting locates part of the benefit on unchanged text despite increased unigram cost, without identifying rarity as a cause. Repeatable repair effects do not ensure prediction transfer: a frozen predictor improves on two ELF checkpoints but fails on a distilled student. A separate task-fitted selector improves BERTScore over random selection by $0.261$ points on 256 fresh articles at a $50\%$ native-call budget; all-native decoding remains higher-scoring and acceleration is unestablished. Compressed and directional interface tests further distinguish fitting-state coverage, native decisions, and soft inputs. Together, the results motivate testing a simplification at three levels: preserved decisions, output consequences, and computational savings. Native decisions are fidelity targets, not semantic truth.

\end{abstract}

\section{Introduction}
\label{sec:introduction}

Continuous language generation first produces a state and then turns it into tokens. In the released Embedded Language Flows (ELF)~\citep{elf2026} checkpoints studied here, a learned flow generates states in a frozen T5~\citep{t52020} representation space, and the shared contextual network performs one final decoding pass. We simplify that final operation to a token-wise affine map. This controlled change asks a concrete question: \emph{what does a simple readout preserve, and what is lost when it replaces the native decoder?} A ``miss'' means a disagreement on the same state. The states already contain context, and the native decisions are fidelity targets rather than semantic truth.

The first result makes this simplification plausible. Affine maps recover $97.17\%$ of native lexical decisions on newly generated ELF states and $95.74\%$ at an independently trained AURORA-LM interface. Yet ELF readout PPL remains $26.27$--$26.33$, versus native $24.71$ under the same output boundary. Most decisions survive, but agreement alone does not explain the residual score gap. The paper follows that gap from the decisions that are missed, through their measured consequences, to whether those consequences can guide a useful choice.

The decisive comparison holds the number of corrections fixed. On the same outputs, rare-first repair improves the pooled external score relative to random repair, while low-native-margin-first repair worsens it. This ordering recurs across ELF and AURORA, two evaluators, and conditional summarization (\secref{sec:prospective-readout}--\secref{sec:task-transfer}). Thus two repaired outputs can have the same native agreement and different measured utility. Post hoc controls also find contextual-score benefits on unchanged text despite increased cost under a fixed unigram reference. They test a frequency-based explanation without identifying rarity as the cause; source-consistency benefit on XSum remains unresolved.

A useful average repair effect does not automatically predict which new sequence will benefit. A frozen composition-aware predictor improves on ELF-B and ELF-M, fails on a distilled student, and has unresolved benefit on AURORA. We therefore test an actionable signal separately: a task-fitted model selects whole outputs for native decoding using only affine-derived test features. On 256 disjoint XSum articles and two new seeds per architecture, its $50\%$ native-call policy improves BERTScore over random selection by $0.261$ points. All-native decoding remains higher-scoring, and measured stage costs do not establish acceleration (\secref{sec:native-free-routing}). The distinction is between error incidence, correction value, and a usable selection signal.

The additional model families test where this argument stops applying. LangFlow shows that fitting data and an architectural vocabulary shortcut affect what a probe recovers. CALM changes the state unit to a compressed chunk and exposes a large gap between supplied-state and prior-state readout fidelity under the tested fitting budget. LLaCO changes the downstream object: a soft embedding, not the adapter's most likely token, enters the contextual backbone. Cola-DLM and bit-level CoBit provide complementary corruption and state-unit boundary checks. These are different tests with different native targets, not repetitions of one architecture comparison (\tabref{tab:cross-family-evidence}).

Together, they motivate a simple discipline: identify the information used by the next computation, test its preservation on the states actually encountered, and measure the consequences and cost of replacing it. The local margin argument, context interventions, and training controls explain why latent proximity, high agreement, and clean reconstruction cannot substitute for those tests. They do not identify the full causal origin of the learned geometry. \figref{fig:overview} shows the operational roles; \secref{sec:cross-family-takeaways} condenses the findings into four takeaways for future interfaces.

\begin{figure}[!htbp]
  \centering
  \includegraphics[width=\linewidth]{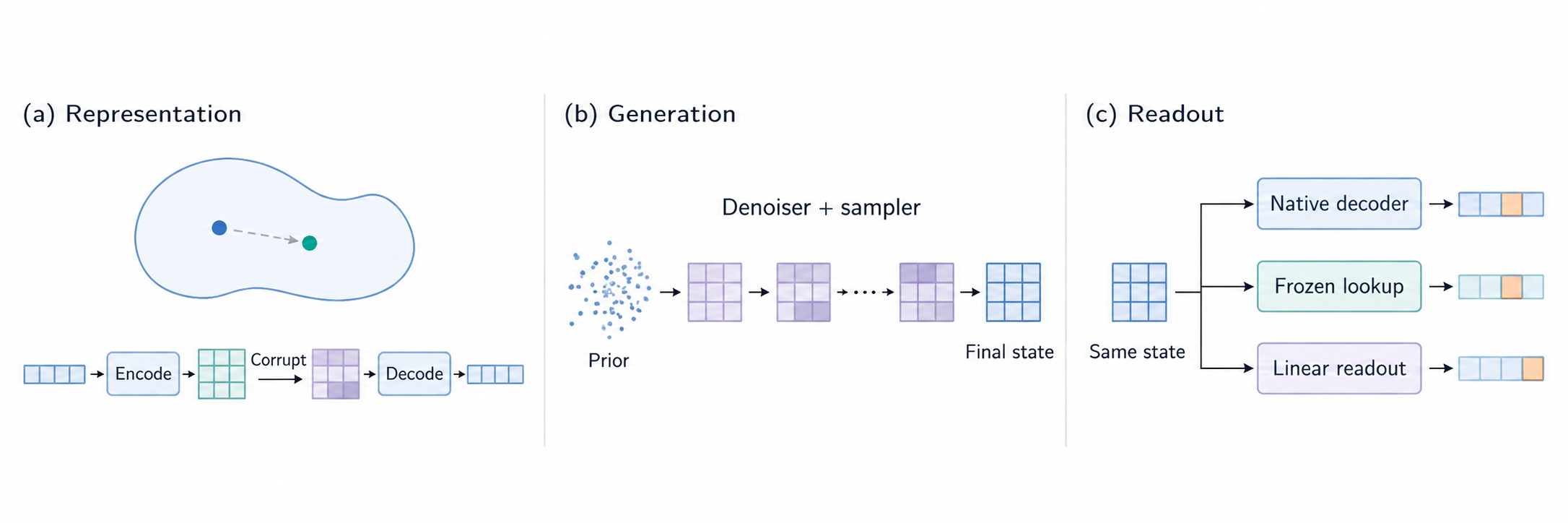}
  \caption{Conceptual overview of representation, transport, and decoding. (a) Corruption acts on continuous states after encoding; the native decoder tests their recoverability. (b) The denoiser and sampler must transport prior states toward useful decoder-readable regions. (c) Frozen lookup and a fitted linear readout are alternative probes of the same generated states, compared with the complete contextual decoder. These are operational roles, not independent trained modules. Regions, trajectories, and token strips are illustrative, not measured geometry or quantitative results. }
  \label{fig:overview}
\end{figure}

% Keep the contribution list together after the introductory overview.
\clearpage
\parhead{Contributions}

\begin{enumerate}[leftmargin=1.4em,itemsep=0.15em,topsep=0.15em]
\item \textbf{A controlled simplification with a consequential residual.} Repeated common-test fits in ELF and an independent AURORA interface establish high affine recoverability. Equal-count repairs then show that the unrecovered decisions differ in measured consequence, with replication across evaluators and a conditional task.
\item \textbf{Tests that separate explanation from prediction.} Score accounting, confidence controls, and context interventions examine alternatives to a frequency-only account. Frozen prediction failures distinguish repeatable repair effects from transferable forecasts; stability-proxy counterexamples delimit the local margin argument. Exploratory and prospective evidence remain separate.
\item \textbf{A bounded practical use and transferable evaluation questions.} Affine-only test features improve whole-output selection on fresh articles at a fixed native-call budget. State-distribution, compressed-interface, directional-adapter, and timing controls show what must additionally be verified before claiming a general or efficient replacement.
\end{enumerate}

\parhead{Reading the evidence}
The short reading path is \secref{sec:prospective-readout} for the central matched tests, \secref{sec:cross-family-takeaways} for the synthesis, and \secref{sec:conclusion} for the resulting claim. \tabref{tab:cross-family-evidence} maps each model family to the question it answers. Readers investigating mechanisms can then follow \secref{sec:diagnostic}--\secref{sec:experiments} for representation, stability, trajectory, and training controls. The appendices retain protocols, every reported condition, and null or adverse results. The reading path separates the main argument from its supporting depth without strengthening the claims in either.

\parhead{Code and reproducibility}
Code: \href{https://github.com/dodoxxb/continuous-language-readouts}{project GitHub repository}; selected readout checkpoints: \href{https://github.com/dodoxxb/continuous-language-readouts/releases}{Releases page}. Offline checks recompute saved measurements without rerunning generation, training, or neural evaluation.

\section{Related Work}
\label{sec:related-work}

Recent masked-DLM analyses compare accuracy--compute frontiers and find that confidence-based remasking gains depend on block length and sampling settings~\citep{frkovic2026remasking}; per-position confidence also cannot in general determine joint dependencies under independent discrete updates~\citep{webb2026confidence}. Those assumptions differ from our terminal continuous-state readout, so we do not transfer their conclusions as a theorem about ELF. They motivate distinguishing confidence, dependence, and the consequences of a correction. Predictive-State Flows instead trains a block-causal decoder to resolve token choices from generated predictive states~\citep{he2026predictivestates}. That allocation of generative roles is a trained architecture, whereas our interventions examine existing decoder-facing endpoints.

\parhead{Residual fidelity and generative interventions}
GRAFT studies underproduction of rare words relative to human text and intervenes during generation~\citep{yang2026graft}. Our residual tail instead consists of disagreements between readouts of the same generated state. Restoring native decisions need not restore human lexical coverage. SERF changes training through noisy-to-clean self-distillation~\citep{kouzelis2026serf}, while Rectified Language Flow co-designs embedding geometry, endpoint commitment, and transport~\citep{su2026relf}. Their improvements do not conflict with the limits of our fixed-checkpoint substitutions. In masked DLMs, dSpark selects commitments for their effect on remaining predictions and distills the schedule~\citep{myrzakhan2026dspark}; this related distinction between confidence and utility concerns an earlier intervention and a different state space. Our matched terminal repairs and separate whole-output selection test address the consequences of replacing an existing decoder.

\parhead{Continuous generation and the native interface}
Diffusion-LM, DiffuSeq, SSD-LM, and latent language diffusion establish continuous or latent routes to text generation~\citep{diffusionlm2022,diffuseq2022,ssdlm2022,latentdiffusionlanguage2022}; masked models instead operate on token or mask states~\citep{d3pm2021,mdlm2024,llada2025}. ELF combines flow matching with noisy decoder training over frozen T5 contextual representations~\citep{elf2026,t52020,flowmatching2022}. We take this released architecture and its training principle as the starting point. Our contribution is to test what simple readouts preserve on generated states, whether their error composition predicts a conditional text-score gap on new checkpoints, and which local decoder measurements fail to establish stable or useful generation. These are diagnostic questions rather than a new continuous generator or an explanation of why T5 pretraining produces the observed geometry.

\parhead{Decoder and representation design}
CoDAR replaces final token projection with a contextual autoregressive decoder~\citep{codar2026}. Cola-DLM learns a compressed Text VAE and a block-causal latent prior~\citep{cola2026}, while LDLM jointly trains the encoder, prior, and decoder~\citep{ldlm2026}. AURORA-LM preserves a high-capacity decoder-facing target while restricting the noisy-input pathway~\citep{aurora2026}. These designs motivate evaluating target compatibility separately from the learned path into that target. We do so with interventions within ELF, a frozen readout-and-repair replication in AURORA, and native-interface boundary checks in Cola-DLM, LangFlow, and CoBit~\citep{langflow2026,bitstream2026}. The latter use different state units and objectives, so they delimit applicability rather than provide a matched cross-architecture ranking. Cola2 co-adapts semantic span compression and discrete generation~\citep{guo2026cola2}; our original Cola-DLM boundary checks do not evaluate Cola2. Sigma~\citep{yang2026sigma} learns a low-dimensional embedding geometry for blockwise continuous diffusion at 3B and 8B scale; we do not test readout transfer to that interface. A successful ELF readout also does not imply that a compressed latent model should recover a particular teacher token sequence.

\parhead{Compression and directional competition}
CALM compresses a four-token chunk into one continuous vector and predicts such vectors autoregressively~\citep{shao2025calm}. This changes both the state unit and the generative factorization, so its decoder is a boundary case for our diffusion study. LLaCO uses directional token competition to motivate spherical corruption and semantic-retention scheduling, together with a learned adapter and time-conditioned backbone~\citep{jia2026llaco}. Its geometric identity measure is distinct from the complete contextual model's final token decision. Our added tests separate compressed-latent readout fidelity, directional identity, and adapter approximation (\secref{app:new-interfaces}); neither a new schedule nor a different representation is covered by our earlier fixed-checkpoint negative results.

\parhead{Margins, temporal structure, and commitment}
Margin-over-sensitivity robustness bounds are standard in classifier analysis~\citep{hein2017formal}; the local bound here applies that argument to a contextual decoder. Our empirical question is whether a tractable first-order score predicts held-out flips beyond phase, raw depth, and perturbation amplitude. The observed proxy-radius violations make the distinction between a diagnostic and a certificate explicit. Temporal analyses of LLaDA trace token commitment and linguistic information during denoising~\citep{temporal2026}. Prophet, Just on Time, and Clock Diffusion use confidence or token-dependent progress for adaptive inference~\citep{prophet2025,jot2026,schiff2026clock}. We instead measure when the existing continuous trajectory becomes readable and how simple readouts behave afterward. PACT applies a sampled-token likelihood test to committed sets in masked diffusion, with a progress fallback~\citep{li2026pact}; that output criterion is distinct from perturbation stability of a continuous decoder. Endpoint-referenced commitment remains retrospective, and our monitored early-exit probe does not establish a wall-clock speedup.

\parhead{Transport and task success}
Progressive distillation reduces the number of sampling steps by training a student~\citep{salimans2022progressive}; our few-step comparisons audit the released ELF teacher and student rather than reproduce their training. Representation-space MMD post-trains generators including ELF using contextual feature distributions~\citep{drobyshevskiy2026mmd}; this training intervention addresses generation, while our endpoint substitutions hold the generator fixed. TIBF continues training with token-identity Brier supervision and a probability-weighted embedding interface~\citep{jiang2026tibf}; our endpoint substitutions instead keep the generator fixed. DeltaFlow changes the ELF transport architecture and temporal regularization~\citep{deltaflow2026}. ELF v3 also reports a separately task-trained GSM8K result~\citep{elf2026}; our conditional readout tests address XSum, not that reasoning setting. CEDR combines teacher representations, asynchronous denoising, and task-specific training for conditional reasoning~\citep{cheng2026cedr}. This is a different training and evaluation setting: local token stability in our unconditional audit does not establish reasoning ability. Decoder-dose and fixed-state decoder-swap controls address the narrower distinction between handed-state recovery and the states reached during generation.

\parhead{Evaluation beyond generated-text perplexity}
In masked diffusion, finite-precision categorical sampling can improve PPL while reducing diversity~\citep{zheng2025masked}. Work on continuous DLMs identifies a dominant self-conditioning direction associated with repetition~\citep{repetitionattractor2026}. Our evidence comes from different controlled failures: covariance and PCA controls, collapsed decoder objectives, and readouts that change external PPL while losing native decisions. We therefore report PPL with entropy, repetition, distributional comparison, and native-token agreement; MAUVE supplies one such distributional metric~\citep{mauve2023}. None alone establishes useful language generation. \secref{app:extended-related-work} retains the broader method comparisons, and \tabref{tab:related-interface} summarizes their interface assumptions.

\section{Preliminaries: Representation, Transport, and Decoder}
\label{sec:preliminaries}

ELF is a useful testbed because it separates the continuous state space from the final token decision. For a token sequence $y$, a frozen T5 encoder produces contextual clean states $x=E(y)$. Flow matching trains a denoiser to recover $x$ from Gaussian-corrupted states,
\[
z_t = t x + (1-t)\epsilon,\qquad \epsilon \sim \mathcal{N}(0,I).
\]
At inference time, the sampler starts from noise, repeatedly predicts clean states $\hat{x}_t$, and maps the final state to token logits through ELF's native decoding mode. In the released PyTorch checkpoints, this is not a bare unembedding matrix: the full shared contextual network processes the input in decode mode before its nonlinear projection and unembedding head. We use \emph{native decoder} for this complete contextual mapping, not merely its final projection or a separately trained module; the distinction follows the released implementation rather than treating the paper's compact $x$-prediction notation as a literal linear head.

We use three decoder-facing objects throughout the paper. The \emph{native decoder} is the trained ELF readout just described, or the corresponding trained decoder for another state space such as Cola-DLM. The \emph{token-embedding lookup} used by ZSBD is different: it is a nearest-neighbor rule in the frozen T5 token-embedding table with no newly trained readout and is used only as a geometry probe. This lookup is not prior-free: it uses the labeled T5 token table tied to ELF's own interface. The \emph{MDP readout} is a learned token-wise linear classifier trained after generation to imitate the native decoder's argmax labels on generated final states. Keeping these objects separate is important: BGEE monitors the native decoder margin, ZSBD tests token-embedding alignment, and MDP measures how linearly recoverable the native decision boundary is on the generated-state manifold.

This interface view matters for our diagnostics. A predicted clean state can be close to the target under MSE while still lying outside the decoder's decision basin. Conversely, a state can be easy to decode while preserving little diversity or order information. Continuous language generation is therefore a question of interface compatibility: the representation must carry linguistic structure, the trajectory must transport noisy states into the decoder basin, and the decoder must assign stable token decisions once the trajectory arrives. We first test the relationship between endpoint readability and a conditional text score. The subsequent diagnostic protocol decomposes interface compatibility into five measurable axes for the supporting experiments.

\section{What Simple Readouts Miss and Why It Matters}
\label{sec:prospective-readout}

The common-test study reaches $97.17\%$ native agreement with a token-wise affine readout, while numeric and low-frequency disagreements remain concentrated (\secref{app:confirmatory-readout}). We examine the consequences of these misses by restoring equally many native decisions under different selection rules, then ask whether error composition predicts the score gap on new conditions.

\parhead{Equal agreement, different consequences}
The repair intervention separates error count from error selection. For the prespecified readout seed 17, each sequence receives exactly $\lfloor e/2\rfloor$ native corrections, where $e$ is its number of disagreements. We compare rare-first, numeric-first, low-native-margin-first, and three uniform random selections. Consequently every repaired output has exactly the same native agreement within its sequence. The prespecified rare-first contrast is $+0.0096$ nats per evaluator token (97.5\% interval $[+0.0079,+0.0112]$; random minus rare-first). It supports a score advantage for this targeted repair in the pooled test, with the other policies and condition-specific results reported separately. In contrast, low-native-margin-first repair is worse than random by $0.0181$ nats in the pooled GPT-2-Large test; numeric-first repair shows no clear advantage. The B32 rare-first interval includes zero, so the pooled benefit does not establish a benefit in every condition. An error category can be enriched among disagreements without being the best category to correct for this evaluator. The protocol and outcomes are retained regardless of direction; no policy is selected on this test. This intervention identifies the conditional score effect of these output edits, not the effect of retraining a decoder or changing transport.

A secondary OPT-125M evaluation, specified before inspecting the primary scoring outcomes, gives a pooled rare-first contrast of $+0.0087$ nats (descriptive 95\% interval $[+0.0072,+0.0102]$) on the exact same outputs. Low-native-margin-first repair remains worse than random, while the numeric-first interval again includes zero. This checks evaluator sensitivity without adding a confirmatory primary test.

\parhead{From repair effects to prediction}
Average repair effects do not imply accurate sequence-level forecasts. The separately prespecified prediction test asks whether composition improves on total disagreement alone. Before generating the test states, we freeze both the readouts and a ridge regression on 512 previously inspected ELF-B sequences. The baseline predicts excess per-sequence GPT-2-Large NLL from total disagreement, native sequence length, and native NLL. The augmented predictor adds the rates of rare, numeric, and low-native-margin disagreements, using the same ridge penalty and no test tuning.

We evaluate 256 new sequences per generation seed, with two seeds for each of ELF-B at 32 steps, ELF-M at 32 steps, and the released distilled student at 8 steps: 1,536 sequences in total. Three fixed ELF-B readouts transfer without refitting. The frozen maps retain mean native agreement of $97.23\%$, $96.02\%$, and $92.59\%$ on B32, M32, and PD8 respectively, separating bulk transfer from exact interface invariance. The frozen tail-aware predictor reduces mean absolute prediction error by $18.2\%$ in the pooled test. The multiplicity-adjusted interval supports this aggregate improvement. The predictor improves on ELF-B and ELF-M but worsens MAE by $7.0\%$ on the distilled student, delimiting its transfer. The measurements and all condition-specific results appear in \figref{fig:prospective-tail} and \tabref{tab:prospective-prediction}. These predictors deliberately use native oracle information to test an explanation; they are not proposed as inference-time quality gates.

\begin{figure}[!htbp]
\centering
\includegraphics[width=\linewidth]{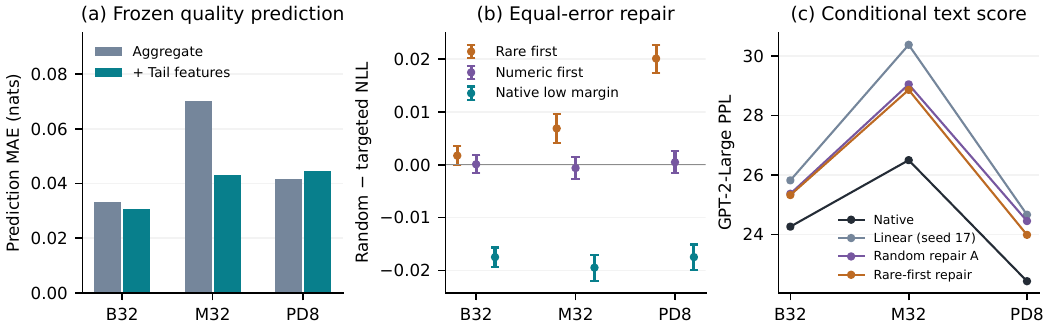}
\caption{Prospective readout-quality tests on new states, with predictors and readouts frozen before generation. B32 and M32 use 32-step sampling; PD8 uses the released 8-step student. (a) Mean absolute error predicting the per-sequence excess evaluator NLL of a linear readout over native decoding, averaged over all three fixed readouts. (b) Paired NLL differences between the mean of three random repair arms and targeted repair; positive values favor targeted repair. Every repair arm corrects exactly half the errors, rounded down, within each sequence. Bars are conditional 95\% paired sequence-bootstrap intervals. (c) Corpus PPL under the common native boundary; the plotted random arm A is fixed in advance, with all three reported in \tabref{tab:prospective-ppl}. Oracle repairs diagnose a text-score effect, not a deployable routing policy or factual correctness.}
\label{fig:prospective-tail}
\end{figure}

\parhead{Why context still matters}
A complementary intervention tests a different boundary of this simplification. Keeping eight focal vectors per sequence exactly fixed, we perturb only their context in 384 reused endpoints. At donor amplitude $0.25$, the native decoder changes $1.07\%$, $2.44\%$, and $2.25\%$ of the protected B32, M32, and PD8 decisions, while the affine readout is unchanged. The full amplitude sweep, norm-matched Gaussian control, and blockwise margin argument appear in \secref{app:context-boundary}. Thus high endpoint agreement does not make the native mapping context independent. These synthetic stress tests do not establish the cause of the earlier freezing failure or measure task quality.

The ELF tests therefore separate a recoverable bulk from unequal correction consequences and condition-dependent prediction. The next test asks whether the same distinctions survive a different native interface, before extending the outcome beyond unconditional language-model scores.

\subsection{An Independent-Interface Replication}
\label{sec:aurora-replication}

An ELF-only result could reflect its shared decoding network or representation interface. To test the scope of the finding, we repeat the frozen prediction and equal-count repair protocol on the released AURORA-LM-S OpenWebText checkpoint~\citep{aurora2026}. Its separately trained text autoencoder provides a different decoder-facing representation. The native decoder starts from shared learned queries, reads the latent sequence through six causal contextual layers, and applies a tied GPT-2 token projection. It has no separate affine latent-to-logit bypass. The pretrained backbone supplies the frozen token table rather than a contextual Transformer in this autoencoder. We fit new token-wise affine readouts in its native 768-dimensional coordinates; no ELF readout weights are transferred. The inputs are generated states before the complete native decoder. We use 2,048 training, 256 validation, and 512 calibration sequences, and freeze all three selected readouts and their ridge predictors before generating 1,536 test sequences at 16, 32, and 64 sampling steps per 16-token block. These are new seeds and sampling conditions within one independently released architecture, not independent denoiser training replicas.

Pooled lexical agreement to the native decoder is $95.74\%$ across the three fitted readouts on all new test states. In the pooled AURORA quality test, the composition-aware predictor reduces MAE by a point estimate of $1.9\%$ (baseline $0.0569$; augmented $0.0558$). The paired baseline-minus-augmented difference is $+0.0011 [-0.0011, +0.0034]$ nats, using the prespecified 97.5\% interval. Thus this test does not establish a prediction gain. At exactly the same number of corrections within each sequence, mean random repair minus rare-first repair is $+0.0112 [+0.0086, +0.0139]$ nats (97.5\% interval), which supports a rare-first advantage. Low-native-margin-first has contrast $-0.0129 [-0.0149, -0.0109]$ (descriptive 95\% interval). Both repair directions persist at every tested step count under both GPT-2-Large and OPT-125M, with all corresponding descriptive intervals excluding zero. Numeric-first has no clear pooled advantage. All step-specific outcomes, both evaluators, every readout fit, and the training-frequency sensitivity are retained in \secref{app:aurora-replication}; neither feature nor repair policy is selected on the new test.

The replicated result is therefore a distinction between corrections: rare-first and low-margin-first edits with identical native agreement have opposite external-score effects relative to random repair. This average intervention ordering does not entail that three composition features improve a frozen prediction of each sequence's score gap. The AURORA test supports the former while leaving the latter unresolved.

The comparison concerns conditional evaluator scores under a shared native output boundary. Native EOS separators are fixed across arms; repairs do not feed back into the generated latent trajectory. Training-pool sizes, tokenizers, and sequence horizons differ from ELF, so the replication tests the diagnostic protocol rather than attributing numerical differences solely to architecture. The ELF prospective study uses 1,024 source positions; this audit uses 512. A frequency-normalized sensitivity matches the earlier ELF rare-token frequency cutoff without altering the primary count-based test. These oracle interventions neither establish factual correctness nor supply a deployable repair policy.

An additional regression audit tests whether altered effective regularization explains the prediction result: repeated baseline columns add no information but weaken the ridge penalty, while unpenalized fits remove it. These post-protocol controls retain the pooled ELF gain and its distilled-student failure, and do not establish a gain on AURORA. Thus simple feature duplication does not account for the ELF improvement, but removing this alternative does not rescue its generalization (\secref{app:aurora-replication}).

\subsection{Transfer to Conditional Summarization}
\label{sec:task-transfer}

Replicating a repair ordering across interfaces does not yet establish its relevance to a conditional task. Language-model likelihood is an indirect outcome, so we transfer the three existing OWT affine readouts for each interface, without fitting or selecting them again, to the corresponding released ELF-B-XSum and AURORA-XSum checkpoints. Both architectures receive the same 1,024 held-out XSum news articles~\citep{narayan-etal-2018-dont}, with two generation seeds per article and model: 4,096 generations in total. References never enter generation. The six repair policies again make exactly the same number of native-token corrections within each generated summary; all samples remain in the primary analysis, including summaries with no editable budget. ROUGE~\citep{lin-2004-rouge}, BERTScore~\citep{bertscore2020}, and SummaC-ZS~\citep{laban-etal-2022-summac} test lexical overlap, reference similarity, and source consistency. These automatic proxies do not establish human preference or factual truth.

The transferred readouts recover $88.99\%$ of native lexical decisions on ELF-XSum and $90.13\%$ on AURORA-XSum, averaging the three fits by their lexical denominators. Agreement decreases relative to the corresponding OWT tests, while checkpoint and task change together. This zero-refit result is not an upper bound on a task-trained affine readout. For rare-first minus mean random repair, the pooled BERTScore F1 contrast is $+0.081 [+0.060, +0.101]$ points on a $100\times$ scale and the SummaC-ZS contrast is $+0.271 [-0.018, +0.561]$ points on a $100\times$ scale. These are the two prespecified primary contrasts, each with a 97.5\% paired-document interval. The pooled test finds a positive effect on reference similarity and an unresolved effect on source consistency. The BERTScore effect is small in absolute terms. Rare-first and low-margin-first have opposite BERTScore effects relative to random in both architectures, with the latter negative; numeric-first remains unresolved. These results concern correction ordering under automatic proxies, not a large application-level gain. Architecture-specific results, every policy, all three readouts, output-boundary controls, and eligibility counts are retained in \secref{app:task-transfer}.

The test changes both the denoiser checkpoint and the generation task while preserving the fitted readout and intervention rule within each interface. It extends the diagnostic to reference- and source-based outcomes; it does not isolate architecture from tokenizer or training differences. Native decisions remain oracle repair targets. A separate affine arm uses its own stopping boundary. These existing lexical readouts excluded EOS training targets; that arm records the behavior of using the frozen probes unmodified and does not establish an intrinsic architectural limit on stopping.

\takeawaybox{Finding}{Equal agreement, different consequences}{%
Across ELF and AURORA, equal-count native repairs have different evaluator consequences. Rare-first repair improves XSum reference similarity, while its source-consistency effect remains unresolved.%
}{takeaway:repair}

\subsection{Distinguishing Residual Consequences from Prediction Calibration}
\label{sec:deep-diagnosis}

The repair tests show that misses differ in measured consequence, but leave the reason open. A rare-token error need not be intrinsically more consequential, and an autoregressive evaluator can favor some substitutions for reasons unrelated to factual correctness. We therefore perform post hoc diagnostics on the saved tests, retaining their original primary analyses. The paired XSum NLL advantage of rare-first over random is $0.0425$ nats/token. Exact score accounting assigns $0.0220\,[0.0157,0.0284]$ to aligned unchanged evaluator tokens. Against a fixed smoothed OWT unigram reference, rare-first instead increases lexical surprisal by $0.0275\,[0.0247,0.0303]$ nats/token. The observed contextual-NLL benefit is thus not explained by a preference for frequent words under this specific reference prior. It includes consequences for the evaluator's assessment of subsequent unchanged text, beyond the directly replaced token scores.

Equal native-token corrections do not guarantee equal evaluator-token edits. Requiring all six policies to produce the same evaluator-token length and the same number of changed evaluator-token IDs leaves 2,619 generations from 958 documents. The pooled rare-first gain is $0.0178\,[0.0088,0.0270]$, including an unchanged-position term of $0.0135\,[0.0068,0.0205]$. The latter remains positive in each architecture, whereas the ELF total interval includes zero. These output-dependent subsets and descriptive 95\% intervals qualify, rather than replace, the original test. Frequency also co-varies with native confidence and source/reference occurrence, so the audit does not identify rarity itself as a causal feature (\secref{sec:deep-repair-accounting}).

Prediction transfer presents a different issue. On the distilled student, adding composition features changes mean signed prediction bias from $+0.00261$ to $+0.02330$ nats, even as mean per-readout Spearman correlation rises descriptively from $0.122$ to $0.266$. An exact squared-error decomposition attributes the observed MSE deterioration to increased squared bias, partly offset by an uncertain centered-variance decrease. AURORA does not exhibit the same large level error; its small incremental prediction gain remains unresolved, including under the existing relative-frequency control. A calibration-only predictor enforcing zero excess NLL at zero disagreement also fails to improve transfer. Thus an average repair ordering, a ranking association, and a calibrated sequence-level forecast are distinct claims (\secref{app:deep-prediction-diagnosis}).

\takeawaybox{Finding}{Repair effects do not guarantee prediction transfer}{%
The frozen composition-aware predictor improves on ELF-B/M, fails on the distilled student, and has unresolved benefit on AURORA. Repeated repair effects therefore require a separate prediction-transfer test.%
}{takeaway:prediction}

An additional equal-budget intervention asks which \emph{known} disagreements are worth correcting. On the same XSum outputs, restoring high-native-margin decisions improves pooled BERTScore by $0.1603$ points and NLL by $0.1073$ nats relative to random repair, with positive descriptive 95\% intervals in each architecture. A subsequently recorded comparison also favors high-margin over rare-first repair ($+0.0796$ BERTScore points and $+0.0648$ nats). Adjacent-margin pairs favor their lower-count member in the pooled comparison, but retain margin and source-overlap imbalance; ELF's NLL interval includes zero. These results separate disagreement incidence from correction consequence. They do not establish rarity as an independent cause or validate an affine-only correction policy (\secref{app:deep-margin-policies}; \figref{fig:deep-residual-diagnostics}).

An equal-parameter readout intervention tests whether a small amount of neighboring context supplies a practical remedy. We add rank-eight residuals using the focal state, the preceding four full-canvas states, or shuffled context, with OWT-only fitting and validation. Across two fitting seeds and both interfaces, the tested context branch does not reliably improve on the unchanged affine head. AURORA's first XSum collapse is separately traced to inverse-variance amplification of a near-constant training coordinate. A disclosed global-scale follow-up removes the collapse but does not establish a context benefit. This separates adaptation instability from the usefulness of the tested context feature; it neither rules out richer decoders nor supplies a new generation algorithm (\secref{sec:context_adapter_followup}).

\subsection{Choosing Native Decoding Before Seeing Its Output}
\label{sec:native-free-routing}

The preceding analyses separate correction consequences from prediction calibration, but their oracle repairs still require the native output.
We therefore ask a different practical question: can affine-only endpoint
features select entire outputs for native decoding at a fixed call budget?
The previously examined XSum cohort becomes development data for separate
ELF and AURORA ridge risk models. Eight features summarize affine confidence,
predicted token counts and types, and stopping/canvas information; neither
native test outputs nor external scores enter the selection. The OWT readouts
remain unchanged. We freeze the models, features, and policies before generating
two new seeds on 256 disjoint articles, yielding 1,024 new generations. Each
endpoint uses its own stopping rule. Within each architecture/seed group, the
policy selects 128 of 256 outputs with the largest predicted native-minus-affine
BERTScore gain. Both endpoints are computed to evaluate this counterfactual
selection; the experiment does not execute an accelerated pipeline.

At this fixed 50\% native-call budget, the primary pooled BERTScore improvement
over the exact random-policy expectation is $+0.2613\,[0.1395,0.3813]$ points
on the $100\times$ scale. Descriptive family intervals are positive for ELF
($+0.3006\,[0.1705,0.4408]$) and AURORA ($+0.2220\,[0.0251,0.4161]$).
The pooled routed score is $88.215$, compared with $87.954$ for random selection,
$86.760$ for all-affine, and $89.147$ for all-native decoding. Thus informed
selection improves the measured reference-similarity tradeoff at this call
budget, while all-native decoding retains the higher score. These are
whole-output choices, not the fixed-boundary token corrections studied above.

The boundaries are informative. Pooled NLL benefit is unresolved
($+0.0252\,[-0.0155,0.0661]$ nats/token), with an ELF point estimate below zero.
The pooled BERTScore advantage over the stopping/canvas control is
$+0.2519\,[0.0958,0.4139]$, but the AURORA-specific interval includes zero.
That control is constant on ELF and uses only canvas length on AURORA because
the other development features are constant; it cannot exclude indirect
native-length prediction by richer features. All fixed alternatives are
reported in \secref{app:native-free-routing}, including three-fit disagreement
and low affine confidence. This positive fixed-budget ranking test does not
establish absolute risk calibration, repair the earlier cross-condition NLL
prediction failures, or solve EOS learning for the existing lexical probes.

Stage timing also limits the systems interpretation. Mean single-head
readout/feature/risk intervals are $5.25$ and $5.35$ ms for ELF and AURORA,
against native-terminal-forward means of $10.37$ and $9.66$ ms, inside complete
sampling runs averaging $1.27$ and $2.90$ s. These shared-GPU measurements do
not establish net savings after routing overhead. The result supports a
test-time selection signal under the declared task and call budget, not a
general generation-quality or end-to-end speedup claim.

\takeawaybox{Finding}{Selection value and system savings differ}{%
At the tested $50\%$ native-call budget, affine-only features improve reference similarity on fresh XSum articles over random selection. All-native decoding still scores higher; selection overhead prevents an end-to-end speedup claim.%
}{takeaway:selection}

\section{What Local Readability Does Not Establish}
\label{sec:diagnostic}

The terminal-readout tests leave two tempting inferences unchecked: that easily recovered tokens make contextual computation unnecessary, and that a readable endpoint specifies a cheap path to useful text. The following diagnostics define the quantities needed to test those inferences. The margin argument supplies a standard sufficient condition, while the experiments test proxies and failure cases. This supporting analysis does not identify the complete causal origin of the generated state geometry.

\subsection{Five Axes of Readiness}
\label{sec:readiness-axes}

Given the interface above, we evaluate the representation-decoder interface along five axes. For a clean representation $x$ and ELF-style corruption
\[
z_t = t x + (1-t)\epsilon,\qquad \epsilon \sim \mathcal{N}(0,I),
\]
we use the following tests. These axes are not assumed to be independent. In practice they can trade off against one another: denoisability and decoder compatibility are connected through margins, and semantic recoverability often changes together with order sensitivity. The protocol identifies which axis becomes the bottleneck for a proposed state space rather than reducing the axes to a single score.

\axishead{Denoisability} asks whether $x$ can be recovered from $z_t$ under simple denoisers. We fit scalar, diagonal, low-rank, and full linear maps and report normalized MSE (NMSE) and cosine similarity. The denoisers are intentionally weak: if a finding requires a large Transformer, it cannot isolate the state-space geometry.

\axishead{Semantic recoverability} asks whether linguistic attributes remain accessible after corruption and recovery. We use token recovery, semantic cosine similarity, and lightweight attribute probes. These probes are not meant to solve language understanding; they test whether a denoised state still carries information that a simple classifier can read.

\axishead{Order sensitivity} asks whether contextual states distinguish real text from word-shuffled controls. This axis is important because a token embedding matrix can be easy to decode while ignoring contextual order. Order perturbations test how the chosen representation and decoder respond to broken contextual structure; an order-insensitive representation can remain token-decodable, and shuffled text need not become wholly unreadable.

\axishead{Decoder compatibility} asks whether the state remains safely decodable by its intended native decoder. We use \emph{basin} operationally, not as an undefined geometric metaphor. For a fixed output position $\ell$, let $g(h)\in\mathbb{R}^{V}$ denote its logits from the complete native contextual decoder, with $h=\operatorname{vec}(H)\in\mathbb{R}^{n}$, $n=Ld$, for the full input sequence. The position index is suppressed when fixed. For token $i$, define the token margin
\[
m_i(h)=g_i(h)-\max_{j\neq i}g_j(h),
\]
and the margin-$\tau$ decoder basin for token $i$ as the super-level set
\[
\mathcal{B}_{i,\tau}(g)=\left\{h\in\mathbb{R}^{n}\;:\;m_i(h)\geq\tau\right\}.
\]
For a decoded sequence $y=(y_1,\ldots,y_L)$, we summarize basin entry by lower-tail statistics of $\{m^{(\ell)}_{y_\ell}(\operatorname{vec}(H))\}_{\ell=1}^L$, especially the 10th-percentile margin. Unless explicitly qualified, \emph{decoder basin} in this paper refers to this operational native-decoder margin super-level set measured along generated trajectories. This definition does not assume the set is convex, bounded, or shared across architectures; it measures logit depth, not a geometric radius. Unless a reference token is explicitly specified, a reported top-two margin uses the current native winner. Margin with respect to a ground-truth or final token is a distinct, signed quantity. In handed-clean corruption sweeps it is the ground-truth-token margin and can be negative. The five-epoch corruption audit pools valid token margins before taking the 10th percentile, whereas generated-state endpoint tables average per-sequence 10th percentiles. These summaries share the $M_{10}$ notation but are not identically aggregated or directly interchangeable. We measure token agreement to the native decoder, decoder entropy, margin under latent corruption, and calibration against small learned decoders. Nearest-neighbor token lookup is reported separately as a frozen-embedding alignment probe with no newly trained readout, not as the decoder-basin definition.

The later probes are not alternative definitions of the same mathematical set. BGEE measures stability timing by monitoring this margin super-level set. ZSBD measures token-embedding Voronoi proximity by asking whether states in or near the decoder basin align with labeled T5 token vectors. MDP measures linear recoverability by asking whether the native decoder's labels are locally linearly recoverable on the generated-state manifold. These probes are related because they are all decoder-facing, but they are not equivalent. We use \emph{basin navigation} as shorthand for trajectories entering a high-margin decoder-basin region, while reporting the other two probes as geometric shadows of that entry in the token-embedding Voronoi space and on the linear-readout manifold.

This margin-based basin is also high-dimensional and anisotropic. The sufficient radius in \thmref{thm:decoder-margin} is local and direction-agnostic only because it uses a worst-case Lipschitz constant. Directional reverse-basin navigation (RBN) later shows that random, sentiment, PCA, and decoder-gradient directions have very different effective sensitivities. Throughout the empirical discussion, ``basin width'' refers only to token retention under a declared corruption sweep, not to raw margin depth or a certified spherical radius. A larger raw margin alone cannot establish a wider decision region across checkpoints. A positive margin threshold defines a scale-dependent super-level set, whereas the token decision region at threshold zero is invariant to positive logit scaling.

\axishead{Trajectory reliability} asks whether the denoising trajectory exposes signals correlated with its final outcome. For ELF we record self-conditioning delta, agreement with zero-self-conditioned predictions, decoder entropy, and effective-rank proxies over time. These signals are correlated with final sample quality and then perturbed through shuffled, delayed, clipped, and reversed-time controls. We call them reliability signals rather than certificates: they are useful diagnostics under the reported metrics, not formal guarantees and not a substitute for task or human evaluation.

\parhead{Protocol}
The protocol follows a fixed order. First, remove capacity by fitting simple denoisers on candidate spaces and controls. Second, add decoder-facing tests, because a low-MSE state may still cross token boundaries. Third, audit the real model trajectory and test whether internal signals predict final quality. Fourth, only after the above signals are present, apply a minimal inference intervention. This order prevents a sampler result from being mistaken for a mechanism.

\parhead{Controls}
Each axis uses controls that preserve one property while breaking another:
\begin{itemize}[leftmargin=1.4em,itemsep=0.12em,topsep=0.12em]
\item \emph{Covariance-matched Gaussian} preserves first- and second-order statistics while removing linguistic content.
\item \emph{Whitened contextual embeddings} test whether recovery depends on anisotropy rather than semantic organization.
\item \emph{Token-shuffled contextual embeddings} preserve the multiset of words while breaking order.
\item \emph{Sequence-shuffled embeddings} preserve marginal embedding statistics while mismatching text and representation.
\item \emph{Random T5 controls} preserve architecture and dimensionality while removing pretrained geometry.
\end{itemize}
These controls are essential because many plausible explanations are too broad: a result that vanishes when the easy statistical property is preserved but the linguistic property is removed is evidence that rules out the preserved statistical property as a sufficient explanation.

\parhead{Readiness scores}
We avoid collapsing the axes into a single score in the main text, but internally each axis can be written as a normalized diagnostic. Denoisability is $1-\mathrm{NMSE}$ under a fixed denoiser family. Recoverability is a probe or nearest-neighbor recovery score relative to clean states. Order sensitivity is the gap between real and shuffled inputs. Decoder compatibility is the fraction of tokens whose native-decoder margin remains positive under corruption. Trajectory reliability is the absolute Spearman correlation between an internal trajectory statistic and final quality under permutation controls. Reporting the vector is more informative than reporting the mean, because the failure modes are qualitatively different and a collapsed score can hide which specific axis has broken.

\subsection{Why MSE Is Not Enough}
\label{sec:mse-not-enough}

Our linear diagnostic is simple on purpose: it separates geometry from model capacity. That same simplicity exposes the central failure mode: lower MSE need not preserve linguistic information.

\begin{proposition}[Linear MSE denoising is second-order]
\label{prop:linear-mse}
Let $x$ and $\epsilon$ be zero-mean independent random vectors with covariances $\Sigma_x$ and $I$. For $z_t=t x+(1-t)\epsilon$ with invertible covariance (in particular, $0\leq t<1$), the population-optimal linear MSE denoiser $A^\star z_t$ is
\[
A^\star = t \Sigma_x \left(t^2\Sigma_x + (1-t)^2 I \right)^{-1}.
\]
\end{proposition}

\noindent The proof follows from the normal equations:
$A^\star=\mathbb{E}[xz_t^\top]\mathbb{E}[z_tz_t^\top]^{-1}$, with
$\mathbb{E}[xz_t^\top]=t\Sigma_x$ and
$\mathbb{E}[z_tz_t^\top]=t^2\Sigma_x+(1-t)^2 I$.
If the covariance is singular, the Moore--Penrose pseudoinverse gives a minimum-norm solution; this includes singular $\Sigma_x$ at $t=1$. Thus a space can be easy to linearly denoise because of its covariance spectrum, even if the recovered state contains little linguistic information. This is the reason our diagnostic always pairs MSE with semantic, order, and decoder tests.

\takeawaybox{Theory takeaway}{Low MSE is not sufficient evidence}{%
Under the stated assumptions, optimal linear MSE denoising depends only on covariance. Low MSE alone therefore cannot distinguish linguistic structure from a covariance-matched control.%
}{takeaway:mse}

\subsection{Decoder Compatibility as a Basin Property}
\label{sec:decoder-basin-theory}

Token decoding is a margin problem. Let a decoder choose token $i$ when logit $g_i(h)$ exceeds all alternatives. A perturbation $\delta$ preserves the token whenever
\[
g_i(h+\delta)-g_j(h+\delta) > 0\quad \forall j\ne i.
\]
For a locally linear decoder this margin is controlled by $(w_i-w_j)^\top \delta$; for a nonlinear decoder the same idea holds locally through the Jacobian. This gives a simple bound.

\begin{theorem}[Ball-localized decoder margin bound]
\label{thm:decoder-margin}
For a fixed output position, let $g:\mathbb{R}^{n}\to\mathbb{R}^{V}$ be the full-context logit map and let $i$ be its unique winner at $h$. Set $m=\min_{j\ne i}(g_i(h)-g_j(h))>0$. Suppose every pairwise margin $g_i-g_j$ is $L$-Lipschitz on the closed ball $B(h,R)$, where $R>0$ and $L>0$. Define $r=\min\{R,m/L\}$. Then $\|\delta\|_2<r$ preserves token $i$, and
\[
\Pr[\mathrm{decode}(h+\delta)\ne i]\leq\Pr[\|\delta\|_2\geq r].
\]
For $\delta\sim\mathcal{N}(0,\sigma^2 I_n)$, $\sigma>0$, the upper bound is $\Pr[\chi_n^2\geq(r/\sigma)^2]$.
\end{theorem}
\begin{proof}
Inside $B(h,R)$, $(g_i-g_j)(h+\delta)\geq(g_i-g_j)(h)-L\|\delta\|_2\geq m-L\|\delta\|_2>0$ whenever $\|\delta\|_2<r$. Thus every rival remains strictly below $i$. A flip is contained in the complementary norm event, proving the probability bound. The Gaussian expression follows from $\|\delta\|_2^2/\sigma^2\sim\chi_n^2$. Equality may permit a tie and is not covered by the strict preservation claim.
\end{proof}
This is a standard local margin--sensitivity argument~\citep{hein2017formal}, specialized to the contextual decoder; it is not a new general robustness theorem. The radius cap is essential. For example, $q(x)=1+x-100x^2$ is $3$-Lipschitz on $[-0.01,0.01]$, but $q(0.2)=-2.8$ despite $0.2<q(0)/3$. A constant valid only on that small neighborhood cannot license the larger perturbation.

\begin{corollary}[Positive logit scaling preserves decisions]
\label{cor:scale}
Replacing $g$ by $ag$, $a>0$, preserves every argmax and token-flip set under the same tie-breaking rule, while multiplying margins and Lipschitz constants by $a$. For fixed rivals with differentiable nonzero gradients, the first-order proxy $\rho_i=\min_j (g_i-g_j)/\|\nabla(g_i-g_j)\|_2$ is unchanged.
\end{corollary}
\begin{proof}
All pairwise differences and their gradients are multiplied by the same positive scalar. Their signs and the displayed ratios are therefore unchanged.
\end{proof}
For sequence-wide preservation under a full-state Frobenius norm, the sufficient bound uses the minimum over all valid-position margins and uniform pairwise Lipschitz bounds on the same ball. The 10th percentile $M_{10}$ does not certify the remaining lower tail and is not a whole-sequence certificate. We do not estimate certified $(R,L)$ pairs. The position-block form in \propref{prop:context-margin} makes explicit why fixing one latent does not control changes elsewhere in a contextual decoder. Our empirical $\rho_i$ uses only two top rivals and is a local first-order predictor, not an exact or vocabulary-wide radius (\secref{sec:calibration}).

\thmref{thm:decoder-margin} explains why decoder compatibility is not reducible to MSE. Two states with the same latent error can have different recovery probabilities if their margins or local Jacobians differ. For a fixed reference token, positive signed margin and token retention are equivalent apart from ties; plotting both is a consistency check, not independent evidence for the mechanism. The substantive empirical tests are held-out perturbations, prediction within a fixed phase, and matched interventions. The current-winner top-two margin is a different quantity and is not a calibrated flip probability (\secref{sec:calibration}). The bound is intentionally local and sufficient rather than a tight decoder theory: we do not estimate the full local Lipschitz spectrum of ELF's nonlinear decoder, and the isotropic Gaussian form ignores the anisotropy measured later. Its role is qualitative and diagnostic. It identifies the variables a decoder-basin analysis must measure---margin, local sensitivity, and perturbation scale---rather than providing a calibrated probability model for ELF's contextual decoder.

A small Euclidean error can therefore cross a decoder boundary, while a learned decoder can obtain low PPL by moving to a low-entropy high-frequency region. Decoder compatibility is therefore a joint basin-and-calibration property, not merely an aggregate reconstruction loss.

\takeawaybox{Theory takeaway}{Stability requires a valid sensitivity bound}{%
A sufficient stability radius requires margin and a sensitivity bound valid over the stated neighborhood. Raw margin and the measured first-order proxy do not supply a certified radius.%
}{takeaway:margin}

\parhead{From local basin to interface transition}
The margin bound is pointwise, but a diffusion trajectory is dynamic. We distinguish margin-basin entry from token commitment. For any decoder-facing signal $q(\tau)$ measured along normalized denoising phase $\tau\in[0,1]$, define the crossing phase
\[
\tau_q(a)=\inf\{\tau:q(\tau)\geq a\}.
\]
In particular, the sequence-level margin crossing is $\tau_M(\theta)=\inf\{\tau:M_{10}(\tau)\geq\theta\}$. On a recorded phase grid $0=\tau_0<\cdots<\tau_K\leq1$, let $a_{k\ell}$ be the recorded native token decision at position $\ell$ and phase $k$, and let $y^\star_\ell$ be the reference token obtained by decoding the final sampler state. The final predicted-clean readout need not exactly equal this reference. Let $u_{k\ell}$ denote the unmasked current-winner top-two logit margin. We separately define
\begin{align}
\tau^{(\ell)}_{\mathrm{persist}}
&=\inf\{\tau_k:a_{j\ell}=y^\star_\ell\ \text{for every recorded }j\geq k\},\\
\tau^{(\ell)}_{\mathrm{first}}
&=\inf\{\tau_k:a_{k\ell}=y^\star_\ell\},\\
\tau^{(\ell)}_{M,\mathrm{current}}(\theta)
&=\inf\{\tau_k:u_{k\ell}\geq\theta\}.
\end{align}
An unreached crossing has value $+\infty$. Persistent current-margin entry requires $u_{j\ell}\geq\theta$ at every subsequent recorded phase; it does not additionally require agreement with $y^\star_\ell$. A signed final-token margin $m^{(\ell)}_{y^\star_\ell}$ is a different observable. A token can first match its final identity, change again, persist later, and only afterward attain a deep margin; these events are not interchangeable.

The sampler and progressive-distillation tables report $\pentry$ using the historical audit's exact convention: predicted-clean readouts are recorded at step fraction $k/N$, the final raw sampler-state decode is appended at phase $1$, and persistent-entry phases are averaged over all output positions, including post-EOS positions, and then over samples. Appending the reference makes every entry finite. The separate 512-sample token-tail audit records only predicted-clean readouts at phases $k/(N-1)$, applies EOS masking to token decisions, and restricts analysis to non-padding positions in the final reference. It reports reached fractions and reached-only entry times with shard aggregation. Consequently its $99.96\%$ native-entry coverage need not be $100\%$, and its entry times are not numerically interchangeable with $\pentry$. Its Margin-8 criterion uses the persistent current-winner margin defined above, not a signed final-token margin.

For an agreement-like signal, such as ZSBD agreement to the native decoder, we define a transition width
\[
W_q(a,b)=\tau_q(b)-\tau_q(a),\qquad a<b,
\]
When both crossings are finite, an empirical entry sharpness is $S_q(a,b)=(b-a)/W_q(a,b)$ for positive width. For margin signals, the same definition applies with thresholds such as $m=2$ and $m=8$. These quantities do not assert a thermodynamic phase transition. They turn basin entry into measurable timing variables: when token alignment begins, how long the competition region lasts, and how abruptly the trajectory becomes decoder-stable.

\parhead{A rank-expansion model of pre-entry}
The pre-entry phase needs a different observable because the native decoder margin is still near zero. Let $\widehat{X}_\tau\in\mathbb{R}^{L\times d}$ be the predicted clean sequence at phase $\tau$, and let
\[
R_e(\widehat{X}_\tau)=\exp\!\left(-\sum_k p_k(\tau)\log p_k(\tau)\right),\qquad
p_k(\tau)=\frac{s_k(\widehat{X}_\tau)}{\sum_j s_j(\widehat{X}_\tau)}
\]
be the entropy effective rank used in the pre-entry audit, for nonzero sequence matrices. This is a within-sequence singular-spectrum statistic. In a local vectorized model,
\[
\operatorname{vec}(\widehat X_\tau)\approx b_\tau+\mathcal J_\tau\operatorname{vec}(Z_\tau),
\]
isotropic input noise gives variable-output covariance $\mathcal J_\tau\mathcal J_\tau^\top$. Its rank is bounded by $\operatorname{rank}(\mathcal J_\tau)$, but neither this ensemble covariance nor the full Jacobian rank equals $R_e(\widehat X_\tau)$. Low sequence rank therefore does not establish a low-rank Jacobian or a low-dimensional token manifold. The finite-difference pilot in \secref{app:extended-discussion} measures response-matrix and concatenated-feature ranks, not the full Jacobian spectrum. The distinction works in both directions: choose a sequence matrix $M$ of rank $\min\{L,d\}$ with equal nonzero singular values. For any nonzero vector $v$, $J=\operatorname{vec}(M)v^\top$ has rank one, yet reshaping $Jz$ gives $(v^\top z)M$, with maximal matrix effective rank $\min\{L,d\}$ whenever $v^\top z\ne0$.

A candidate observable for early subspace change is innovation energy relative to a fixed-dimensional dominant subspace, not an additive effective-rank law. Choose $1\leq r<d$ in advance, let $P_{s,r}^\perp\in\mathbb R^{d\times d}$ project onto the complement of the leading $r$ right-singular directions of $\widehat X_s$, and set $\Delta_s=\widehat X_{s+1}-\widehat X_s$. For a nonzero increment define
\[
\eta_s(r)=\frac{\|\Delta_s P_{s,r}^\perp\|_F^2}{\|\Delta_s\|_F^2}\in[0,1].
\]
The choice of $r$ and any singular-value tie convention are part of the diagnostic. Projecting outside the entire row space would instead give zero for a full-column-rank sequence matrix, even if its entropy effective rank is small. This is a proposed diagnostic, not a quantity estimated by the reported rank curves. Effective rank can change through reweighting existing singular directions; even a new direction need not increase entropy rank. The observed approximately linear early growth is an empirical fit, not a derived subspace-accumulation law.

A separate hypothesis is that changes matter to decoding when they align with token-separating directions. A feature-space projector $P_D\in\mathbb R^{d\times d}$ would yield an alignment statistic $\|\widehat X_\tau P_D\|_F^2/\|\widehat X_\tau\|_F^2$. Such a feature projector is only a simplifying hypothesis: the full contextual decoder Jacobian acts on all $Ld$ coordinates. The existing alignment and Jacobian pilots constrain this explanation but do not identify a necessary critical rank or derive a rank-to-margin law. We retain pre-entry, competition, and locked regions as measured phase descriptions. SC interventions locate where conditioning matters under the tested protocol; they do not prove that rank growth causes basin entry. No reverse-path hysteresis is measured.

A separate illustrative threshold model can parameterize transition width and tail persistence. Near a token's basin entry, write its relevant margin as
\[
m_i(\tau)=a_i(\tau-\tau_i)-b_i+\epsilon_i(\tau),
\]
where $a_i$ is the local margin-growth rate, $b_i$ is a token- or position-specific offset, and $\epsilon_i$ is residual stochasticity from the sampler and context. Then the agreement curve is a threshold probability, $A_i(\tau)=\Pr[m_i(\tau)>0]$. Assuming $a_i>0$ and a phase-independent residual distribution of scale $s_i>0$, the width of the transition from agreement level $q_1$ to $q_2$ is proportional to $s_i/a_i$; for a logistic residual, $W_i(q_1,q_2)=s_i\{\mathrm{logit}(q_2)-\mathrm{logit}(q_1)\}/a_i$. Larger or better trained denoisers can shift entry earlier by lowering $\tau_i$, and they may sharpen entry when $a_i$ increases relative to residual scale. The latter is an empirical question, not a built-in prediction: the coarse transition-grid supplement later shows earlier entry without a monotone transition-width law. Within this illustrative model, persistent tail tokens can arise from a mixture component with large $b_i$, small $a_i$, or high $s_i$; sample-level early exit is therefore governed by an order statistic over token entry times rather than by the median token. This connects four empirical facts in the paper: rank expands before signed final-reference margins turn positive, SC is causally useful in the competition window, high-margin crossings move earlier with same-interface denoiser scale, and rare/numeric/subword tokens dominate the late residual tail.

\begin{proposition}[Corrupted-latent cross-entropy applies local margin pressure]
\label{prop:ce-margin-pressure}
For cross-entropy $\ell=-\log p_i$ with $p_k=\operatorname{softmax}(g)_k$, the logit derivative is $\partial\ell/\partial g_k=p_k-\mathbf1\{k=i\}$. If logits themselves are independent optimization variables, a negative-gradient step raises $g_i-g_j$ by $\alpha(1-p_i+p_j)>0$ for every rival $j$ and step size $\alpha>0$. Corruption training applies this objective at sampled perturbed states. Shared-parameter network updates need not increase every sample's margin, cover an entire neighborhood, or match states visited during generation.

\end{proposition}

\begin{proposition}[Dispersion increases a prototype decoder basin]
\label{prop:dispersion-basin}
Consider a nearest-prototype decoder with token prototypes $\{\mu_k\}$ and logits $g_k(h)=-\|h-\mu_k\|_2^2$. If $h$ is assigned to token $i$, then the margin against $j$ is
\[
g_i(h)-g_j(h)=\|\mu_j-\mu_i\|_2^2-2(h-\mu_i)^\top(\mu_j-\mu_i).
\]
Writing $d_{ij}=\|\mu_j-\mu_i\|_2$, for $\|h-\mu_i\|_2\leq r_0$ the pairwise margin is at least $d_{ij}^2-2r_0d_{ij}$ and its Lipschitz constant is $2d_{ij}$. If all $d_{ij}>2r_0$, the corresponding sufficient radius is at least $\min_{j\ne i}d_{ij}/2-r_0$. At $h=\mu_i$ the distance to the nearest decision boundary is exactly $\min_{j\ne i}d_{ij}/2$. Increasing the minimum separation increases these bounds while the state remains within the stated prototype neighborhood.
\end{proposition}

This simplified prototype is not intended as a literal model of ELF's nonlinear decoder. It serves as a geometric bridge: dispersing token-conditioned representations can widen local decision basins when the representation remains near the correct prototype, while collapsed representations make small latent errors more likely to cross token boundaries. The connection to the margin bound suggests a training principle: decoder branches should apply margin-increasing pressure on plausible noisy latents, not only reconstruct clean ones. ELF's mixed denoising/decoding training and LDLM's decoder-input noise~\citep{ldlm2026} can both be read in this light. The open question is how to widen the basin without encouraging low-entropy collapse.

\parhead{SDE noise as basin exploration}
For a fixed local ball, increasing isotropic perturbation scale weakens the sufficient retention bound in \thmref{thm:decoder-margin}. This motivates, but does not prove, an exploration--exploitation interpretation of the full SDE trajectory. SDE gamma also changes the visited states and their subsequent drift; the local theorem alone therefore predicts neither the true flip rate along that path nor its generation quality. This is why we pair gamma and guidance sweeps with entropy, repetition, and MAUVE rather than relying on PPL as the sole quality indicator.

\subsection{Trajectory Reliability Signals}
\label{sec:trajectory-reliability-protocol}

We audit the released ELF checkpoints through the official PyTorch implementation, without modifying checkpoint weights or source behavior. During sampling we record self-conditioning (SC) delta, agreement with zero-self-conditioned predictions, decoder entropy, layer-similarity proxies, and effective-rank proxies. For each generated sample we then compute GPT-2-Large~\citep{gpt22019} PPL and correlate internal trajectory summaries with final quality. These signals are not used to train a new model; they test whether the model exposes when its own trajectory is reliable.

\parhead{Signal controls}
Following the logic of mechanistic diagnostic studies, we do not stop at a positive correlation. We perturb the proposed signal in three ways. A \emph{delayed} signal uses the same values but shifts them to the wrong time region. A \emph{clipped} signal removes its dynamic range. A \emph{random} signal preserves the marginal distribution but breaks the sample-trajectory association. We also compare front-loaded and back-loaded schedules to test whether the useful region is temporally localized.

\parhead{Why not learn a policy directly}
It is natural to ask whether trajectory signals could be fed into a learned policy that adjusts guidance strength or step allocation. We avoid this path on purpose. A learned policy would conflate three separate questions: whether the signal carries information, whether the policy class is adequate, and whether the reward is well-specified. Our approach separates them. We first verify that the signal predicts quality, then test localization with a hand-designed schedule. Learning an optimal policy is a natural next step, but it belongs after the diagnostic evidence has been established.

\subsection{Basin-Navigation Audit}
\label{sec:basin-navigation-audit}

Trajectory-reliability correlations say that a trajectory exposes useful signals, but not what the signal means. We therefore add a decoder-facing trajectory audit. At every SDE step, we decode the predicted clean latent $\hat{x}_t$ through the native ELF decoder and record top-1-vs-top-2 margins, decoder entropy, and the token-level self-conditioning disagreement
\[
\Delta_t(i)=\frac{\|\hat{x}^{\mathrm{SC}}_{t,i}-\hat{x}^{0}_{t,i}\|_2}{\|\hat{x}^{\mathrm{SC}}_{t,i}\|_2+\epsilon}.
\]
At selected steps we also compare the intermediate argmax tokens with the final decoded sequence and compute the margin of those final tokens under the intermediate decoder logits. This distinguishes three questions: whether an intermediate state is readable by the decoder, whether it agrees with the final sequence, and whether it lies inside the final-token margin basin. The audit is still training-free; it is a mechanism probe for the existing checkpoint.

\section{Controls on Representation, Trajectories, and Training}
\label{sec:experiments}

\parhead{Experiment roadmap}
We begin with the frozen readout-quality test (\secref{sec:prospective-readout}), which directly tests the inferential step from a structured error tail to output quality. The supporting experiments then address three questions. \textbf{Which properties make states readable?} Covariance, shuffle, PCA, semantic, and decoder controls separate candidate measurements (\secref{sec:denoise-recover-decouple}--\secref{sec:order-sensitivity}). \textbf{What does local reliability predict about generation?} Trajectory, perturbation, decoder-dose, and sampler controls test the limits of confidence and handed-state recovery (\secref{sec:trajectory-reliability-results}--\secref{sec:interface-phase-diagram}). \textbf{How far can decoding be simplified?} The complete historical and common-test readout studies, long-form checks, and external systems delimit transfer and residual errors (\secref{sec:three-minimal-probes}--\secref{sec:ruled-out-explanations}). Detailed evidence and negative controls remain in the appendix. \tabref{tab:scale-summary} retains the scale of the supporting studies; the new prospective test is specified separately in \secref{app:prospective-readout}.

\parhead{Models and data}
We evaluate the public ELF-B, ELF-M, and ELF-L checkpoints on OpenWebText (OWT)~\citep{openwebtext2019} generation. ELF-B and ELF-M provide the main diagnostic coverage; ELF-L is used for scale confirmations with smaller batches. We additionally audit the official PyTorch ELF-B teacher and released progressive-distillation student over 1/2/4/8/16/32 steps. We use T5~\citep{t52020} contextual embeddings, T5 token embeddings, random T5 controls, BERT~\citep{bert2019}, RoBERTa~\citep{roberta2019}, GPT-2~\citep{gpt22019}, Gaussian/covariance controls, and Cola-DLM VAE latents for representation diagnostics. For the pretraining-readiness audit, we additionally test local Google/T5-family checkpoints including T5-base/large, mT5~\citep{mt52021}, ByT5~\citep{byt52022}, Switch-base-8~\citep{switch2022}, and T5Gemma variants~\citep{t5gemma2025,t5gemma2_2025}. These extended controls are pair-compatibility checks against the ELF interface, not a benchmark ranking of encoder quality. Yelp Polarity and AG News~\citep{zhang2015charcnn} provide semantic and lightweight OOD corruption probes. GPT-2-Large~\citep{gpt22019} is used for ELF-internal Gen. PPL comparability and related DLM reporting conventions, not because it is a complete evaluator of 2026-level coherence or reasoning. OWT reference samples are used for JS divergence and MAUVE~\citep{mauve2023}.

\parhead{Why these controls}
The controls are chosen to isolate different claims. Gaussian and covariance-matched Gaussian controls test whether MSE recovery is explained by second-order statistics alone. Token-shuffled and sequence-shuffled contextual embeddings test whether contextual order matters beyond token identity. T5 token embeddings test a decoder-friendly but non-contextual interface. BERT, RoBERTa, and GPT-2 test whether the phenomenon is specific to T5-style span-corruption pretraining or appears across encoder and decoder families. The additional T5-family controls separate denoiser scale from interface scale: changing the denoiser while keeping the T5-small interface fixed is not the same intervention as replacing the encoder/decoder pair with a different dimensionality, vocabulary, or normalization. Cola-DLM tests a learned VAE latent interface rather than a frozen pretrained embedding interface. \tabref{tab:encoder-readiness-decomposition} and \tabref{tab:negative-control-summary} collect the resulting decomposition after the main experiments, separating contextuality, lexical geometry, scale mismatch, and decoder compatibility from covariance and shuffle confounds.

\parhead{Metrics}
We report Gen. PPL because it is the standard ELF metric, but we never treat it as a complete objective. We pair it with unigram entropy, distinct/repetition statistics~\citep{distinct2016}, JS divergence to reference token distributions, MAUVE~\citep{mauve2023}, token agreement, decoder logit entropy, and margin recovery under latent corruption. For the long-form boundary audit, we additionally compute adjacent- and first-to-last sentence cosine with a sentence-embedding model~\citep{sentencebert2019}. This multi-metric design is necessary because many of our negative results are metric failures: a method can improve one scalar while degrading the property the scalar was meant to measure.

\parhead{Hardware and scope}
The original representation, trajectory, and probe suite runs on a workstation with four 24\,GB GPUs. Final sampler, progressive-distillation, OOD, 50k-step controls, and two matched five-epoch training pairs run on a server with four 80\,GB NVIDIA A800 GPUs; batches are increased only within a condition, while compared conditions preserve the same evaluation batch, sample count, step grid, and metric code. The expanded four-dose grid includes a later $p=0.8$ continuation and closure evaluation on an RTX 5090; resumed data order and execution changes limit exact causal attribution (\secref{app:v3-protocol}). The public ELF checkpoints are not retrained. This setup reflects the central premise: most diagnostics are cheaper than large-scale model training, while targeted training controls are used where fixed checkpoints cannot separate transport from compatibility.

Sample counts differ by experimental role. Mechanism audits use hundreds of generated sequences with per-token or per-step statistics; trajectory-reliability and schedule experiments use multiple 1024-sample seeds; MDP uses up to 32k generated final states because it estimates a readout on the generated-state manifold. The sampler and teacher/student transport controls use three seeds with 256 sequences per seed, and the official PD reproduction uses 1000 samples per step count. In the original paired boundary audit, each condition uses 512 sequences for every corruption-basin audit and 256 paired-randomness generations at every epoch checkpoint. The closure grid uses 256 generations at epochs 1--4 and 1,024 at epoch 5 under its separately fixed global time grid; calibration and decoder swaps have their own explicitly stated sequence counts. Each subsection reports its own sample count, and \tabref{tab:scale-summary} separates fixed-checkpoint diagnostics, minimal probes, training controls, and external boundaries.

\begin{table}[!htbp]
\centering
\footnotesize
\caption{Experiment scale and claim level. Sample counts differ because the experiments answer different questions: mechanism audits need per-step or per-token statistics, readout probes need many generated final states, and external systems are used only as boundary diagnostics.}
\label{tab:scale-summary}
\setlength{\tabcolsep}{3pt}
\begin{tabular}{>{\raggedright\arraybackslash}p{0.18\linewidth}>{\raggedright\arraybackslash}p{0.30\linewidth}>{\raggedright\arraybackslash}p{0.23\linewidth}>{\raggedright\arraybackslash}p{0.17\linewidth}}
\toprule
Claim level & Main evidence & Scale reported & Intended use \\
\midrule
Representation readiness & T5 contextual/token, BERT, RoBERTa, GPT-2, Gaussian controls, Cola VAE latents, and PCA rank sweep & held-out OWT latent subsets; five PCA ranks; semantic probes on Yelp Polarity/AG News & diagnose candidate state spaces before relying on a denoiser \\
Trajectory reliability & trajectory signals correlated with GPT-2-Large log PPL; shuffled and time-region controls & ELF-B/M: three 1024-sample seeds; ELF-L: two 512-sample seeds & test whether checkpoints expose reliability signals \\
Trajectory and intervention audit & step-wise margin, entropy, self-conditioning delta, final-token agreement, pre-entry tensor capture, early zero-SC causal check, and shard standard errors of the mean (SEM) & ELF-B 32-step SDE: 512 samples; B/M/L and ODE/SDE confirmations: 512 samples per run; pre-entry tensor capture: two 32-sample repeats; early zero-SC audit: 256 matched-noise samples & fixed-checkpoint intervention evidence, not a training-seed scaling law \\
Compatibility vs. transport & official/CPS/ODE sampler counterfactual; released PD teacher/student; decoder-only, mixed, and staged objectives & sampler and teacher/student: three seeds $\times$ 256 samples; PD reproduction: 1000 samples per step; two training seeds $\times$ four decoder doses $\times$ five epoch checkpoints & test whether handed-state compatibility is sufficient and whether few-step quality tracks entry \\
Stability calibration and decoder swaps & random/top-rival perturbations; fixed-state full-context decoder swaps & 11 calibration conditions; 22 direction cells; 32 epoch-2 decoder-swap cells & test local flip prediction and bound pure readout-failure explanations \\
Minimal probes and stress tests & BGEE, ZSBD, MDP, RBN, same-start path selection, and long-form topic boundary & BGEE: 512 ELF-B and 256 M/L confirmations; ZSBD: 4k B/M/L plus 8k/16k ELF-B checks; MDP: 1k--32k ELF-B fitting-budget comparisons and 1k/4k M/L confirmations; RBN: 1024 ELF-B plus 256 M/L cross-scale confirmations; topic audit: 1000 texts per main group plus 512 ODE32 samples & measure timing, token alignment, linear recoverability, decoder-basin robustness, and discourse-level boundary cases \\
Training ablations & decoder-noise, decoder-branch frequency, margin-loss, long decoder-only/mixed, staged, and matched five-epoch variants in the official PyTorch ELF-B implementation & 5k/10k runs on four RTX 3090 GPUs; targeted 50k and five-epoch controls on A800, with a later RTX 5090 continuation; 512--1024-sample basin sweeps & mechanism-side training evidence, not a final generation recipe \\
External boundaries & Cola-DLM VAE/DiT, LangFlow, and CoBit checks & Cola DiT: 512 samples per guidance value; CoBit: structured proxy plus 512 real-code subset; LangFlow: small trajectory audits & show interface questions beyond ELF, not benchmark ranking \\
\bottomrule
\end{tabular}
\end{table}

\parhead{Evaluation order}
The experimental order mirrors the diagnostic logic of the paper. We first test representations without the ELF denoiser. We then audit ELF trajectories to establish the interface phase diagram, including an explicit pre-entry tensor-capture check. We then extract an operational phase-transition model from the trajectory observables. We then test a minimal schedule and three basin probes. We finally test decoder calibration, Cola-DLM boundary behavior, and short training-side basin widening. If any earlier stage failed, the later probe would be unmotivated. This ordering is what distinguishes the study from a collection of ablations.

\subsection{Denoisability and Recoverability Decouple}
\label{sec:denoise-recover-decouple}

\figref{fig:mse} shows the central diagnostic result. Denoisability and linguistic recoverability are not the same. Color separates the state-space family: blue points are language-model embedding spaces, while orange points are Cola-DLM VAE latents. Circle area encodes token recovery after denoising at $t=0.5$, not model size; the size legend is shown at the right of the plot. Some spaces are easy under MSE but weak under semantic or decoder tests; other spaces preserve semantics while being poorly matched to a particular decoder. This explains why ``continuous'' alone is not enough for language diffusion.

The most useful reading is not as a ranking of encoders, but as a rejection of an overly broad hypothesis. If continuous diffusion worked merely because the state distribution is smooth or Gaussian-like, then covariance-matched controls should behave like contextual T5. They do not. Likewise, bag-of-words token recovery does not establish order-sensitive contextual information under these probes. This is a limitation of that audit outcome, not an impossibility result for token-embedding diffusion with learned positional structure. The controls motivate checking denoisability and the linguistic distinctions required by the chosen interface together.

\begin{figure}[!htbp]
  \centering
  \includegraphics[width=.82\linewidth]{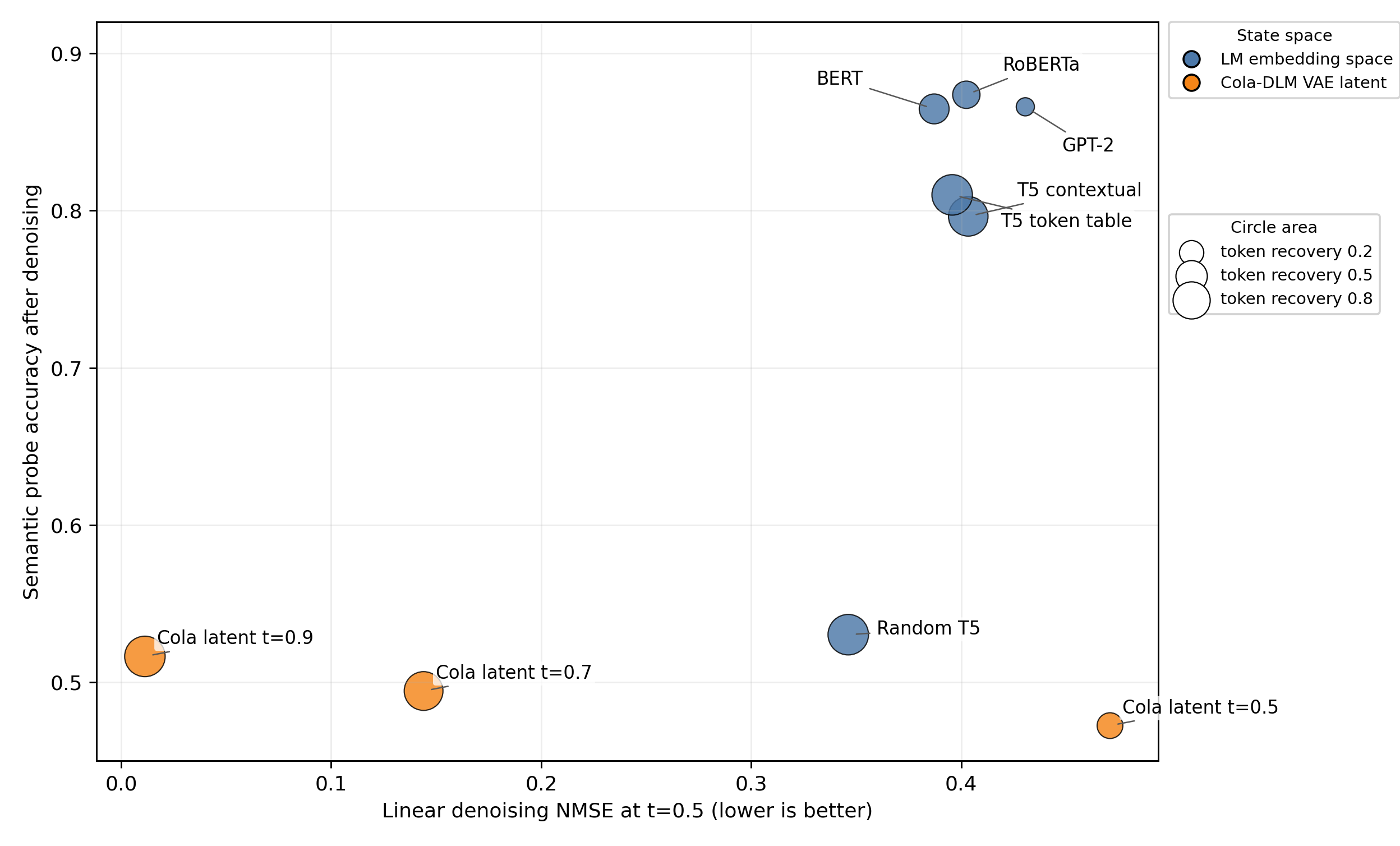}
  \caption{MSE can be misleading: denoisability and semantic recovery decouple. Color indicates state-space family (blue: language-model embeddings; orange: Cola-DLM VAE latents), and circle area indicates token recovery after denoising at $t=0.5$, not model size. The size legend maps circle area to representative recovery values. The diagnostic must include linguistic and decoder-facing axes.}
  \label{fig:mse}
\end{figure}

\subsection{Controlled Degradation: PCA Bottlenecks Break the Interface}
\label{sec:pca-negative}

A diagnostic protocol is more convincing if it flags an interface problem before one invests in a new training run. We therefore construct a controlled degradation inside the same T5 contextual family. We fit PCA on T5 contextual latents, project the latents to rank $r$, and project back to the native 512-dimensional ELF interface. This preserves the official denoiser and decoder shapes, but removes low-variance directions that may be important for order or decoder margins.

This is a narrow controlled test. PCA is not meant to simulate a nonlinear VAE or a learned tokenizer; it is an intervention inside the same ELF/T5 coordinate system. Its value is that the diagnostic axes are sensitive to interface degradation while the checkpoint, decoder, and text distribution remain fixed. A protocol that catches this failure is better positioned to flag problematic state spaces encountered in practice.

High-rank PCA should preserve denoisability and decoder compatibility. Low-rank PCA may still preserve a large amount of variance and a nontrivial latent cosine, but should lose decoder margin. If the bottleneck is inserted only at the final endpoint, the decoder may partially tolerate it. If the bottleneck is imposed at every denoising step, the trajectory itself should move into a different operating region.

\figref{fig:pca-negative} confirms the protocol's sensitivity to a degraded interface. Rank 256 keeps $85.3\%$ of variance and high decoder agreement. Rank 128 keeps $64.8\%$ of variance and still preserves most native decoder decisions. Rank 64 is the first clear boundary: agreement falls to $0.938$, the 10th percentile margin falls to $1.90$, and generation begins to move along a repetition frontier. Rank 32 is the failure flagged by the diagnostic: agreement falls to $0.478$ and the median margin becomes negative. When rank 32 is applied at every step, GPT-2-Large PPL becomes misleadingly low ($7.53$), but unigram entropy drops to $3.87$, distinct-2 drops to $0.074$, and bigram repetition rises to $0.926$. The failure is diagnostic: low PPL can hide a collapsed state-space interface.

\figref{fig:pca-order} shows that the same bottleneck also changes the order-sensitivity axis. Under full word shuffling, full-rank T5 pooled representations fall to cosine $0.777$, while rank-32 PCA remains at $0.966$. In other words, the bottleneck makes the representation more order-invariant under the pooled diagnostic: it smooths away directions that carry contextual order differences. This is the desired controlled failure. A lower-dimensional projection can preserve enough variance to look geometrically benign while suppressing linguistic distinctions that the interface needs.

\begin{figure}[!htbp]
  \centering
  \includegraphics[width=\linewidth]{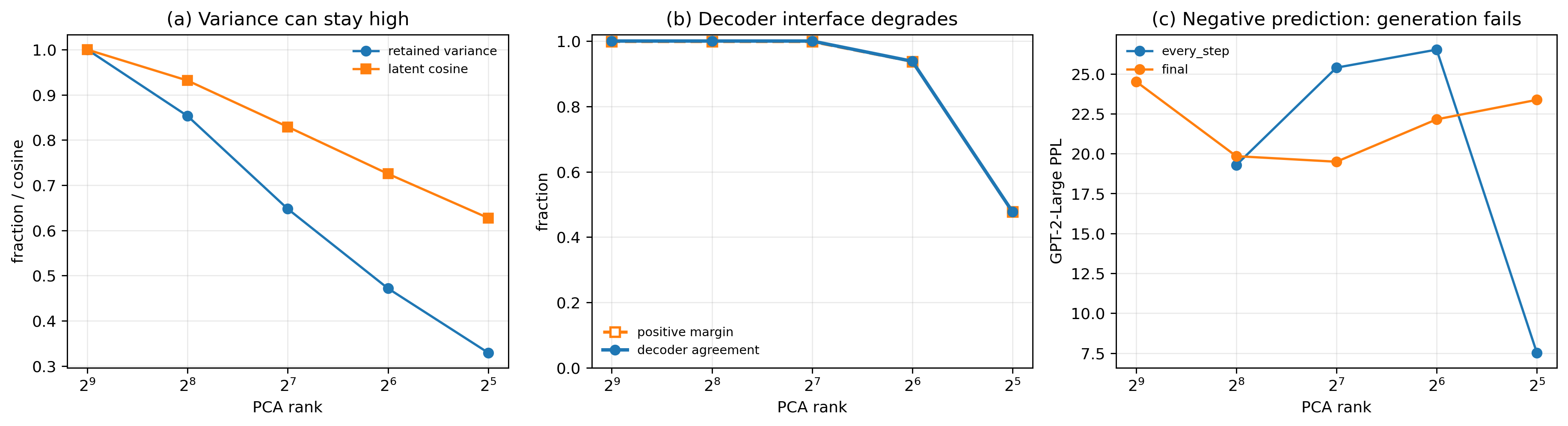}
  \caption{Controlled interface degradation through a PCA bottleneck. Left: retained variance and latent cosine remain high at moderate ranks. Middle: decoder agreement and margin degrade as the interface rank is reduced. Right: imposing the bottleneck along the generation trajectory yields low-PPL but low-diversity text. The every-step point at rank $2^9$ is omitted because full-rank PCA is the identity map and duplicates the no-bottleneck baseline. The diagnostic flags the degraded interface even when variance appears adequate.}
  \label{fig:pca-negative}
\end{figure}

\begin{figure}[!htbp]
  \centering
  \includegraphics[width=.92\linewidth]{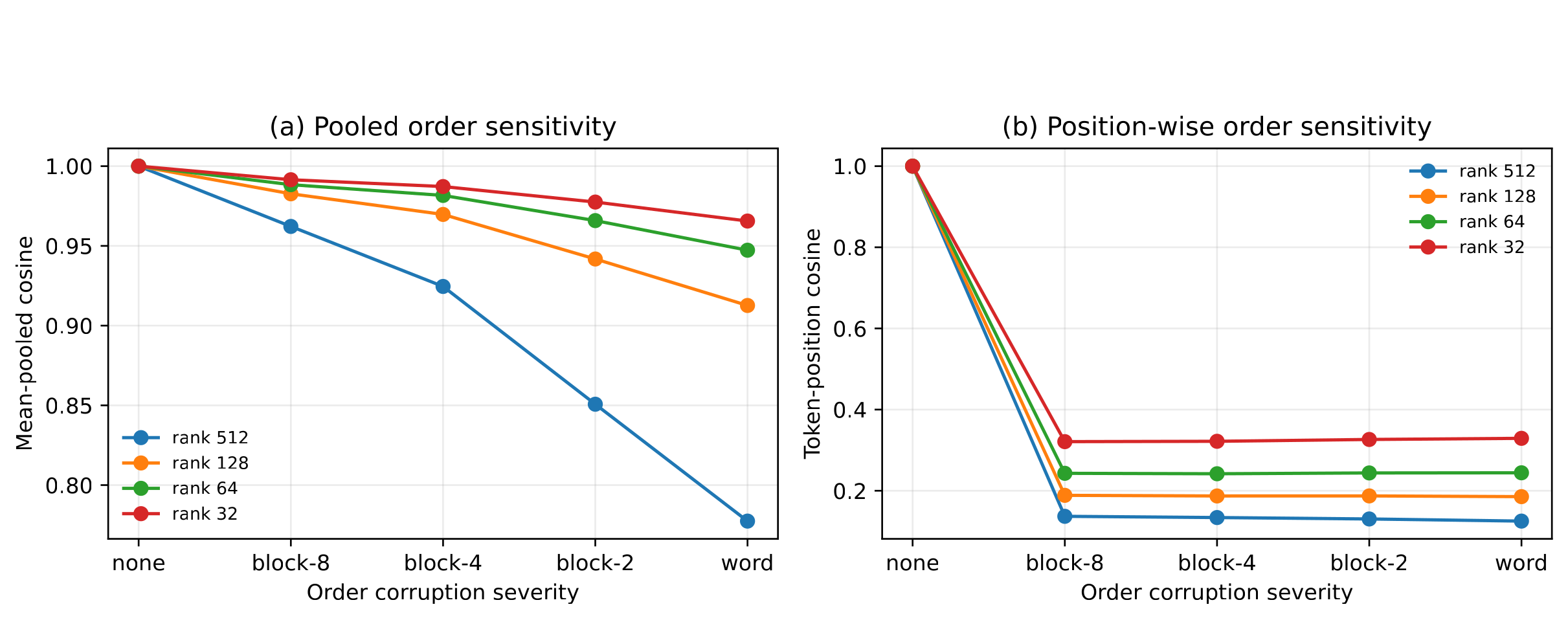}
  \caption{PCA bottlenecks reduce order sensitivity. Left: mean-pooled representations become less sensitive to stronger order corruption as rank decreases. Right: position-wise cosine shows the same effect at token locations. Low-rank variance preservation can therefore erase contextual order directions even when denoisability remains high.}
  \label{fig:pca-order}
\end{figure}

These PCA interventions delete coordinates from an already trained decoder-facing representation. They do not test a learned input bottleneck trained jointly with a denoiser, and should not be generalized to that different architectural intervention.

\subsection{Order Sensitivity Is Graded}
\label{sec:order-sensitivity}

The original order-sensitivity diagnostic used a binary real-vs-shuffled comparison. We strengthen it with graded block shuffling: keep words intact but shuffle blocks of length 8, 4, 2, or 1. The result is not merely that contextual encoders notice order. The shape of the degradation curve reveals what kind of order each state space preserves.

\figref{fig:order-curve} gives the graded order curve. T5 contextual embeddings degrade smoothly from cosine $0.906$ under block-8 shuffling to $0.729$ under full word shuffling. BERT shows a similar but slightly stronger degradation, from $0.890$ to $0.724$. T5 token embeddings, by contrast, remain nearly flat around $0.90$ in mean-pool cosine because they largely preserve a bag-of-words signal; their first-token cosine collapses, but the pooled representation does not encode contextual order in the same way. GPT-2 and RoBERTa pooled states are more invariant under this particular pooling metric, which is itself informative: order sensitivity should be measured with a representation and pooling choice that matches the proposed decoder interface.

This graded curve strengthens the interface claim. A state space can be easy to denoise and easy to decode while still failing to expose the word-order distinctions a language generator needs. Diffusion-readiness is therefore not a scalar smoothness property that can be verified by inspecting the covariance spectrum in isolation.

\begin{figure}[!htbp]
  \centering
  \includegraphics[width=.82\linewidth]{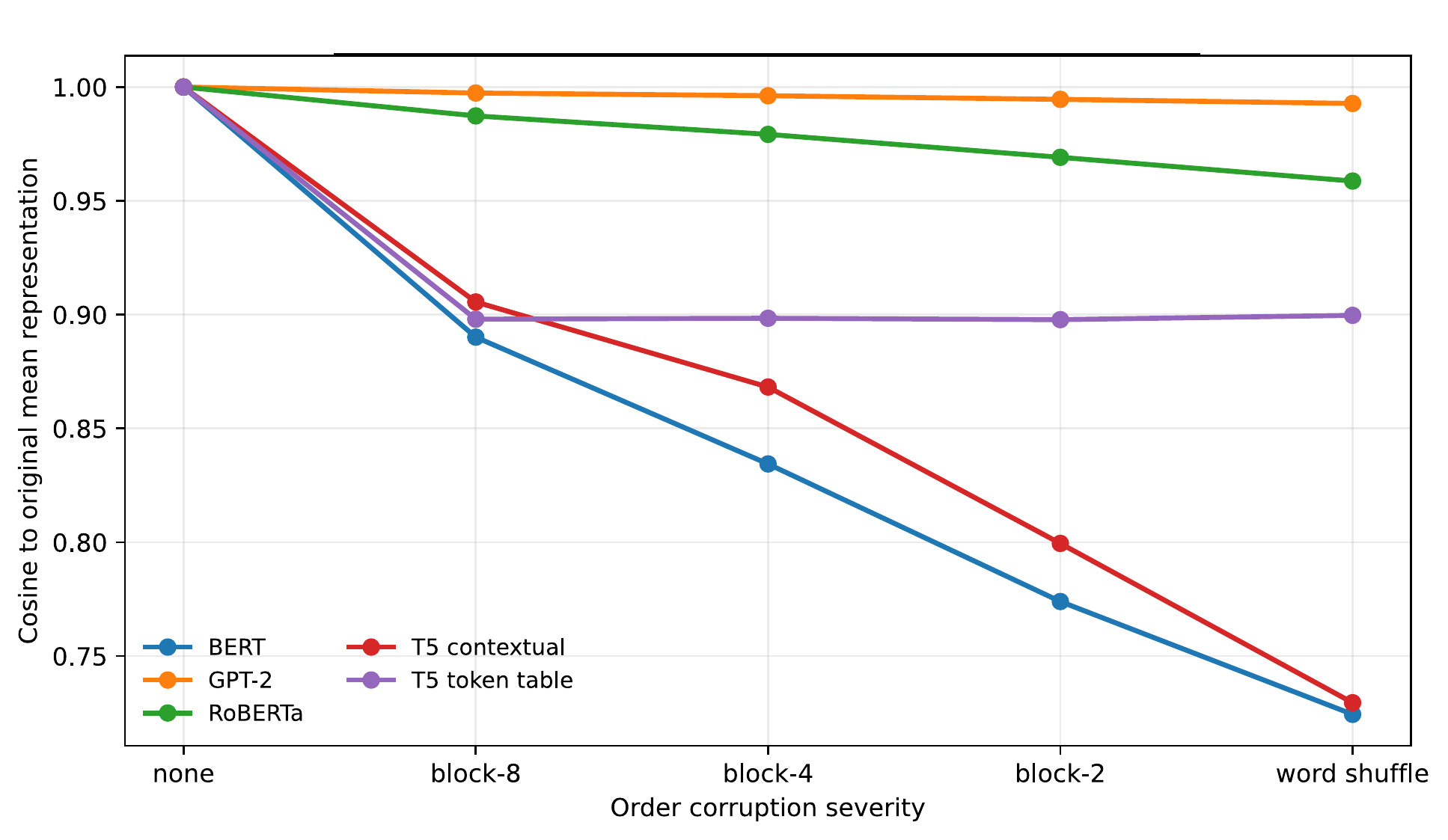}
  \caption{Order sensitivity is graded. The horizontal axis increases order corruption from none to full word shuffle; the vertical axis is cosine similarity to the original mean representation, so lower values mean stronger order sensitivity. Contextual T5 and BERT degrade progressively, whereas T5 token embeddings remain nearly flat under mean pooling, exposing a bag-of-words failure mode.}
  \label{fig:order-curve}
\end{figure}

\subsection{ELF Trajectories Expose Reliability Signals}
\label{sec:trajectory-reliability-results}

Across three 1024-sample seeds, ELF's internal trajectory signals are predictive under the standard ELF GPT-2-Large PPL evaluator. For ELF-B fixed SC=3, the peak relative self-conditioning change has Spearman correlation $-0.417 \pm 0.009$ with per-sample log PPL, while the minimum agreement with the zero-SC prediction has $0.411 \pm 0.009$. For ELF-M, the corresponding correlations are weaker but stable: $-0.253 \pm 0.017$ and $0.239 \pm 0.016$. Shuffled-signal permutation tests give $p \leq 0.005$ for these core signals. A two-seed ELF-L confirmation with 512 samples per seed and 64 SDE steps shows that reliability signals persist at 652M parameters, but the strongest region shifts. For seed 0, the mid-trajectory mean SC change has $\rho=0.359$ and the mid-trajectory zero-SC agreement has $\rho=-0.356$ with log PPL. For seed 1, the same mid-trajectory signals remain significant with $\rho=0.224$ and $\rho=-0.225$, while the full-trajectory SC-change peak reaches $\rho=-0.336$. All three ELF-L core correlations have permutation $p=0.002$. These are not formal certificates of text quality; they are trajectory statistics that remain informative under the reported automatic metrics and signal controls.

We use per-sample GPT-2-Large PPL here because it is the most sensitive cheap scalar for within-checkpoint trajectory correlations. This does not contradict the PPL critique above. The claim in this subsection is narrow: under a fixed checkpoint and a fixed sampler family, PPL is a useful \emph{relative proxy} for ranking samples from the same operating regime, and internal signals correlate with the standard evaluator used by ELF. It is not a sufficient cross-model, cross-sampler, or cross-temperature quality objective. Later calibration and boundary experiments therefore pair PPL with entropy, repetition, MAUVE, JS divergence, decoder agreement, and margin precisely because PPL alone can validate the wrong operating region.

The sign of the correlation is less important than its existence and localization. ELF-B and ELF-M expose reliable extrema over the full trajectory, whereas ELF-L exposes a clearer mid-trajectory average. This is consistent with a scale-dependent schedule: larger denoisers may stabilize earlier or distribute correction differently across time, but they still reveal reliability before final decoding. This observation argues against a purely final-decoder explanation of generation quality, in which the denoiser merely minimizes latent error and the decoder resolves all linguistic structure at the last step.

\figref{fig:trajectory-reliability} shows the main ELF-B/M reliability curves: a per-sample scatter makes one signal concrete, and a correlation summary shows that several trajectory statistics predict the standard GPT-2-Large PPL evaluator. \figref{fig:elfl-scale} summarizes the two-seed ELF-L check with the shifted mid-trajectory signals.

\begin{figure}[!htbp]
  \centering
  \includegraphics[width=\linewidth]{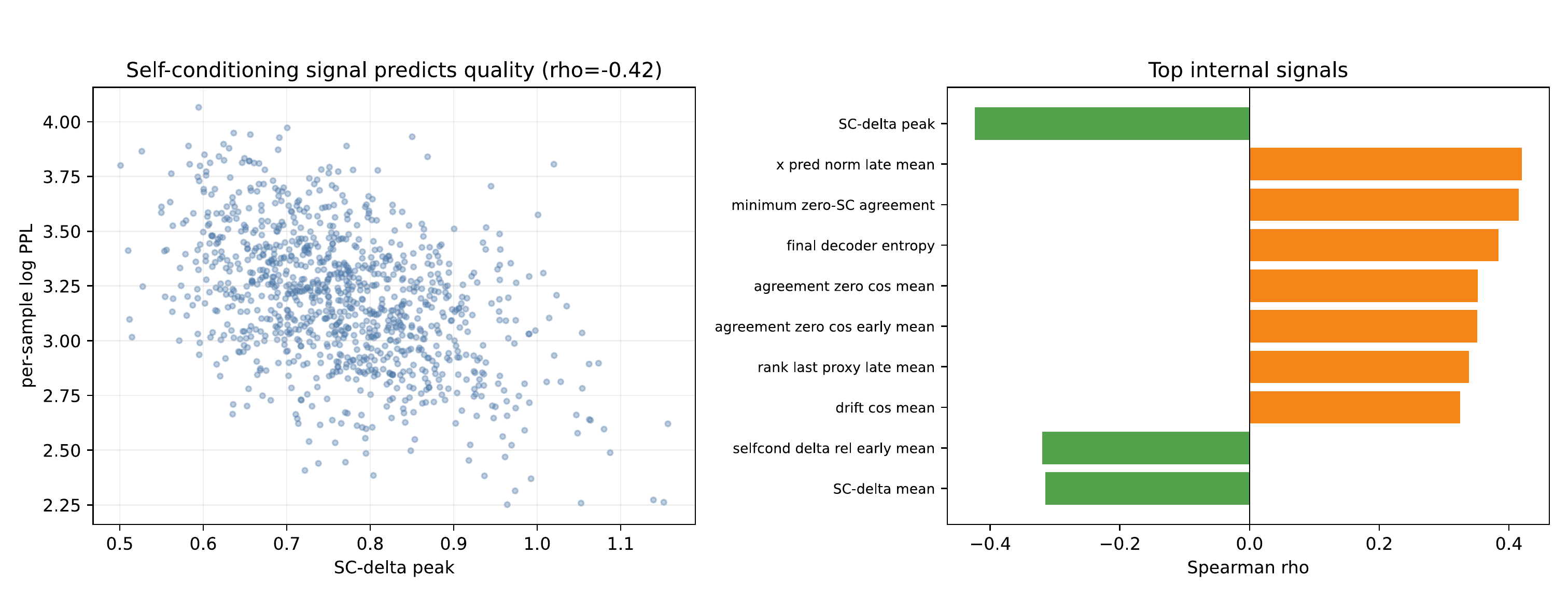}
  \caption{ELF internal signals predict the standard PPL evaluator. Left: each point is one generated sample; the x-axis is the peak relative self-conditioning change along the trajectory, and the y-axis is per-sample GPT-2-Large log PPL. Lower PPL tends to occur when this internal change is larger, giving a negative Spearman correlation. Right: the x-axis is Spearman $\rho$ between each trajectory statistic and log PPL; the y-axis lists candidate statistics. Multiple internal signals correlate with the evaluator before external scoring, motivating the later basin-navigation audit.}
  \label{fig:trajectory-reliability}
\end{figure}

\begin{figure}[!htbp]
  \centering
  \includegraphics[width=.95\linewidth]{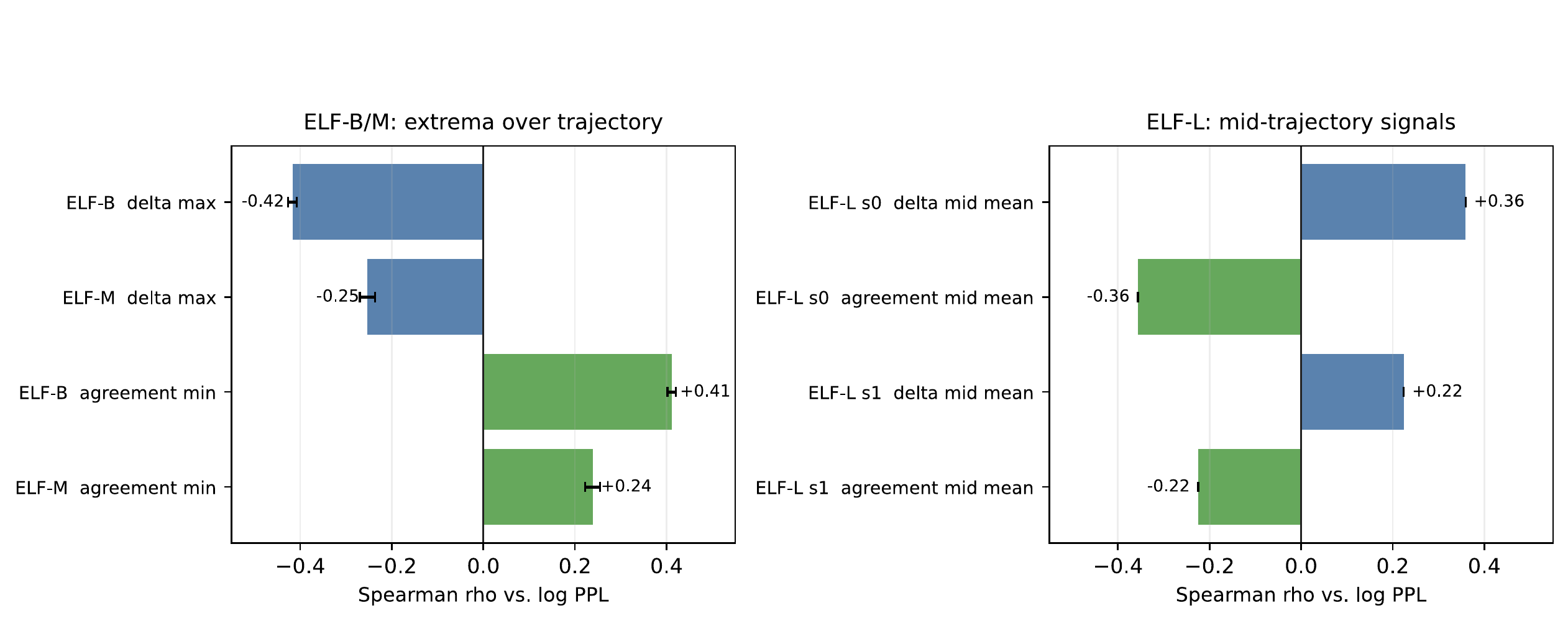}
  \caption{ELF-L trajectory reliability at larger scale. The x-axis in both panels is Spearman $\rho$ with per-sample log PPL; bars to the left mean the statistic is larger for lower-PPL samples, and bars to the right mean it is larger for higher-PPL samples. The left panel repeats the strongest whole-trajectory extrema from ELF-B/M, while the right panel reports two independent ELF-L seeds using mid-trajectory summaries. The correlations persist, but the strongest time region shifts, so the signal is phase-dependent rather than a universal confidence scalar.}
  \label{fig:elfl-scale}
\end{figure}

\subsection{Trajectory Evidence for Entry into Readable Regions}
\label{sec:basin-navigation-results}

The correlation result above leaves a mechanistic ambiguity. Do these reliability signals simply correlate with final PPL, or do they reflect how the trajectory approaches the decoder interface? We answer this by measuring decoder margins at every denoising step. For each step, we decode the predicted clean latent $\hat{x}_t$ through the native decoder and record top-1-vs-top-2 margins, decoder entropy, and self-conditioning disagreement. We then run a smaller selected-step audit that asks whether intermediate predictions already agree with the final decoded tokens.

\figref{fig:basin-navigation} gives a compact picture. This single-checkpoint audit uses raw denoising steps to show the actual SDE32 chronology; the cross-scale and sampler-comparison plots below use normalized phase so that 32- and 64-step trajectories can be compared. The 10th-percentile decoder margin of $\hat{x}_t$ grows from $0.049$ at step 0 to $12.30$ at step 31, while decoder entropy falls from $6.18$ nats to $0.018$ nats. Self-conditioning disagreement is not monotone. It rises from zero, peaks around steps 16--18 at about $0.47$, and then falls. The delta-margin correlation is also phase-dependent: it is near zero early, becomes positive in the middle region where disagreement peaks, and then becomes strongly negative late (mean Spearman $\rho=-0.31$ after step 24). Thus self-conditioning disagreement is not a universal confidence score. It becomes a margin proxy only once the trajectory has approached the decoder basin. A post-hoc GPT-2-Large evaluation on the same 512 samples connects the mechanism back to quality: maximum mid-trajectory self-conditioning delta correlates with log PPL at $\rho=-0.499$, early 10th-percentile margin mean at $\rho=-0.445$, and late decoder-entropy minimum at $\rho=0.409$.

We use ``compare'' as a compact description of this statistical pattern, not as a directly observed internal algorithm. The measurements are consistent with a comparison-then-commitment transition: candidate clean predictions differ most in the middle, while final-token margins and agreement appear later. They do not by themselves reveal the exact computation performed inside the denoiser.

\leadin{Basin-entry timing.} The selected-step audit makes the basin interpretation sharper. At step 0, the margin of the final decoded tokens under the intermediate logits is strongly negative (p10 $=-19.4$), and agreement with the final decode is only $1.0\%$. At step 20, agreement has risen to $76.2\%$ but the final-token margin is still negative (p10 $=-10.4$). The boundary is crossed around step 24: final-token p10 margin becomes positive ($0.76$), and agreement reaches $89.5\%$. By step 31, final-token p10 margin is $12.57$ and agreement is $99.96\%$. Intermediate decodes become diverse earlier than this: decoded sample entropy reaches about $5.0$ by step 16, but high-margin final compatibility arrives later. This separates \emph{readability} from \emph{basin entry}.

This is the empirical counterpart of \thmref{thm:decoder-margin}. As the trajectory enters the decoder basin, the relevant margin-to-sensitivity ratio should increase, making residual perturbations less likely to change tokens. We do not estimate the full local sensitivity spectrum here, but the measured part of the prediction is direct: lower-tail final-token margin moves from a large negative value to a large positive value, while final-token agreement rises from $1.0\%$ to $99.96\%$. This is the mechanism layer on top of the diagnostic layer. The first layer says what fails: MSE, PPL, and clean reconstruction can each validate the wrong thing. The second layer says what the audited successful trajectories are consistent with: the denoiser transports predictions through a phase of proposal, self-conditioning disagreement is statistically consistent with candidate revision, and the final phase enters a high-margin decoder-readable basin.

\begin{figure}[!htbp]
  \centering
  \includegraphics[width=\linewidth]{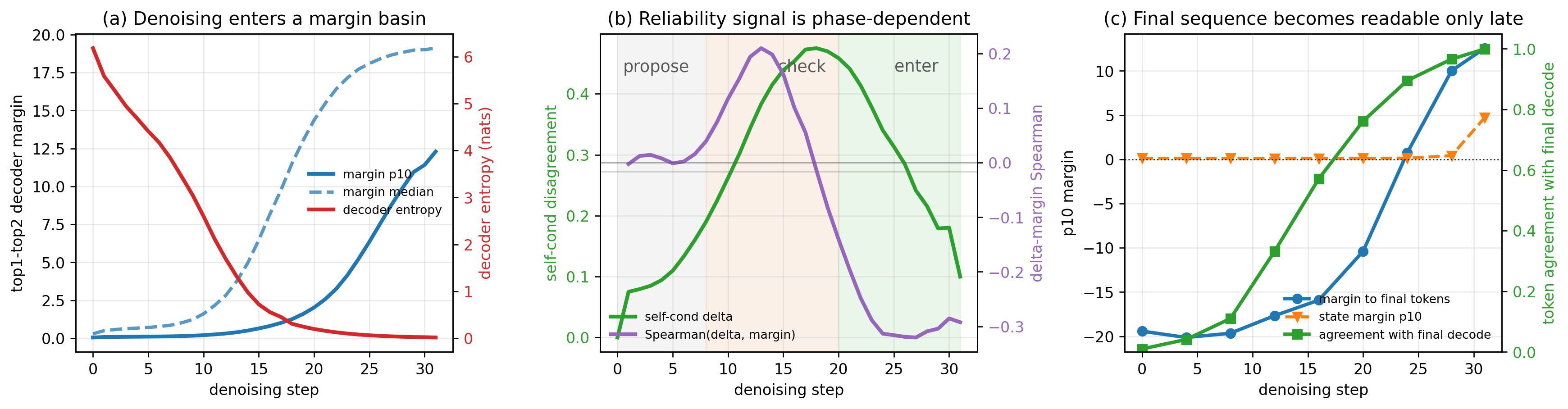}
  \caption{Denoising as basin navigation. (a) Decoder margins increase while decoder entropy decreases. (b) Self-conditioning disagreement is phase-dependent: it peaks in the middle region and becomes a negative margin proxy only late. (c) The final sequence enters a positive-margin basin late in the trajectory; readability and final-token compatibility are not the same.}
  \label{fig:basin-navigation}
\end{figure}

\subsection{Fixed-Phase, Interventional, and Training-Side Evidence}
\label{sec:fixed-phase-evidence}

Margin necessarily changes with denoising time, so a monotone trajectory plot alone cannot show that it contains information beyond ``the sampler is closer to the end.'' We therefore add a fixed-phase nested-prediction audit on the selected-step ELF-B trajectory. At each recorded step, we predict eventual native-token agreement from current decoder entropy, self-conditioning change, and relative step drift, using deterministic sample-level five-fold out-of-fold prediction. We then add the current 10th-percentile native margin and measure the change in out-of-fold $R^2$. The lower-tail margin adds information at the same phase rather than only across time: the largest gain is $\Delta R^2=0.183$ at step 16 (normalized phase $0.52$), and positive gains persist at steps 12, 20, 24, and 28. At the final step, agreement is already saturated (standard deviation $5.8\times10^{-4}$), so no signal can materially improve prediction. Partial rank correlations after residualizing the same controls are significant from steps 12 through 28, peaking at $\rho=0.520$ at step 24. These $64$-sample selected-step results are incremental-association evidence, not causal estimates, but they reject the simplest explanation that margin is useful only because it indexes time.

A second fixed-operating-point test connects the lower tail to external quality. On $128$ samples from one fixed ELF-B sampler, self-conditioning late/peak statistics alone give out-of-fold $R^2=0.061$ for per-sample GPT-2-Large log PPL. Adding final margin p10 raises this to $0.179$ ($\Delta R^2=0.118$; partial Spearman $\rho=-0.230$, $p=0.009$). Adding mean margin instead reduces out-of-fold performance ($\Delta R^2=-0.020$; partial $\rho=0.007$). This contrast is diagnostically useful: the fragile lower tail, rather than average decoder confidence, carries the incremental sample-quality signal.

The remaining two panels in \figref{fig:fixed-phase-causal-training} provide interventions at different levels. In matched-noise trajectories, a step-16 projection toward the labeled token-embedding basin raises final margin p10 by $0.085\pm0.045$, an anti-basin direction lowers it by $0.186\pm0.064$, and a matched random direction is near zero ($0.009\pm0.024$), over four $128$-sample shards. The corresponding PPL audit moves in the same signed direction (projection $-0.49$, anti-basin $+0.49$, random $-0.08$). At training time, the official-like decoder-input noise scale $5$ retains $94.4\%$ of clean token recovery at corruption level $t=0.6$, compared with $47.0\%$ for clean-only decoder training; scale $10$ is weaker, so the effect is not monotone regularization. Together, fixed-phase prediction, trajectory intervention, and training intervention support an operational basin variable while retaining their distinct evidential scope.

\begin{figure}[!htbp]
  \centering
  \includegraphics[width=\linewidth]{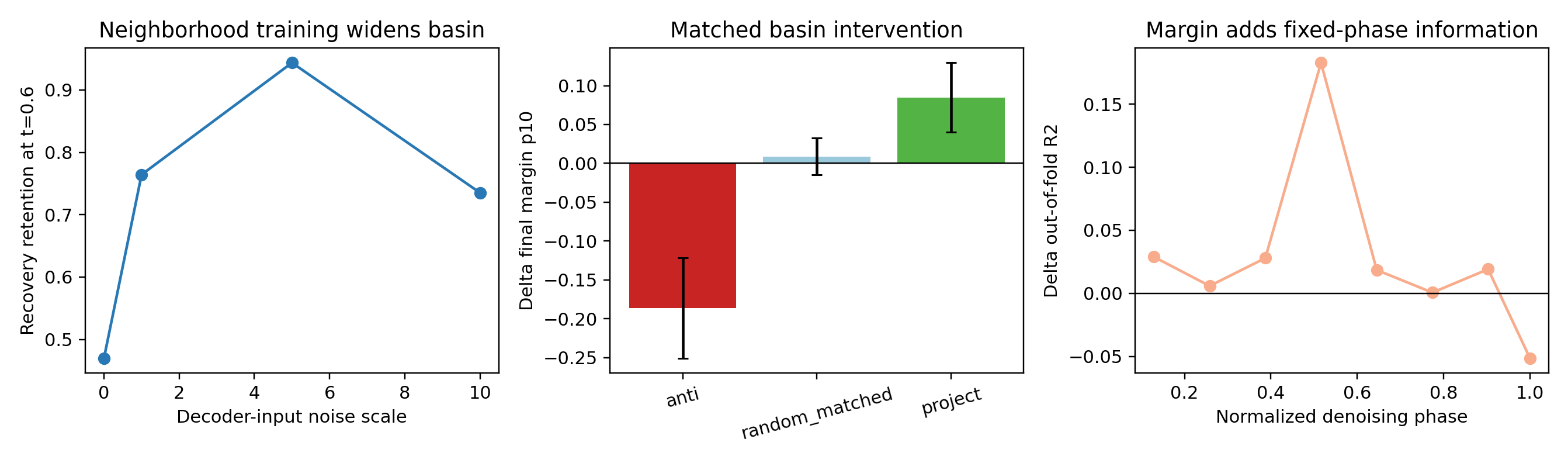}
  \caption{Evidence that lower-tail decoder compatibility is more than a clock. Left: decoder-input-noise training changes recovery retention at matched latent corruption, with an intermediate noise scale producing the strongest measured corruption retention. Middle: matched basin, anti-basin, and random trajectory interventions move the final margin tail in signed directions; bars show SEM over four $128$-sample shards. Right: at a fixed recorded phase, current margin p10 adds out-of-fold predictive information for eventual token agreement beyond entropy, self-conditioning change, and step drift. The panels combine training, interventional, and incremental-association evidence; none alone is presented as a universal causal theory.}
  \label{fig:fixed-phase-causal-training}
\end{figure}

\subsection{Margin depth is not basin radius}
\label{sec:calibration}

\corref{cor:scale} motivates testing whether sensitivity-normalized scores predict held-out token flips better than raw margin. We audit the released checkpoint on synthetic interface states, all eight matched epoch-2 models on their generated states, and the two $p=0.8$ epoch-5 endpoints as follow-up checks. Four fixed relative-RMS perturbation magnitudes are applied along random and top-rival gradient directions. The held-out sets contain 16 sequences for the released model (576 perturbations per direction) and eight per retrained checkpoint (192 per direction). Radius-normalized AUROC exceeds raw-margin AUROC in all nine testable gradient cells ($0.806$--$1.000$ versus $0.423$--$0.859$); the directional predictor is useful but does not universally rank second. These are local predictive probes, not global stability certificates. Five of eleven random-direction cells have no flips; two gradient cells have only flips. We retain these undefined AUROCs and report flip counts and sequence-cluster uncertainty rather than tune the perturbation magnitudes (\secref{app:calibration}). Positive logit rescaling preserves all tested argmax decisions while scaling margins exactly.

The pooled radius score uses perturbation magnitude, whereas raw margin does not. To isolate the additional ranking information from sensitivity, we therefore compare it with the amplitude-aware baseline $\|\Delta\|_F/m$ on the same saved perturbations (\secref{app:calibration-controls}). Radius normalization improves all nine testable pooled gradient-direction point estimates by $0.004$--$0.075$ AUROC, although four paired sequence-cluster intervals include zero. At fixed relative-RMS magnitude it wins in 23 of 25 testable gradient cells and loses in two; random-direction comparisons show no uniform advantage. Thus the evidence favors a local sensitivity-and-scale diagnostic under the declared stress directions, rather than a universally superior robustness score.

A stricter reanalysis compares only flip/nonflip pairs at the same corruption phase and relative perturbation scale (\secref{app:calibration-falsifiability}). The radius score retains higher point estimates in all nine testable gradient conditions, by $0.010$--$0.478$ conditional AUROC; four paired intervals include zero. Random directions yield three wins, two ties, and one loss, all with intervals including zero. This controls phase and scale mixing, but the gradient direction remains coupled to the first-order geometry. A separate falsification check finds 27 flips among 198 gradient perturbations with $\|\Delta\|_F<\rho$, including an uncorrupted generated-state example. These correlated observations refute a certified-radius interpretation of the proxy; they do not estimate a population failure probability or contradict the neighborhood-uniform Lipschitz theorem.

\takeawaybox{Finding}{The first-order proxy is not a certificate}{%
Observed token flips occur below the empirical first-order radius. This rules out a certificate interpretation of that proxy, even when its ranking is informative in the gradient-aligned tests.%
}{takeaway:calibration}

\subsection{Decoder Basins Are Targets, Not Transport}
\label{sec:compatibility-transport}

The operational account would be incomplete if strong handed-state decoder compatibility alone explained generation. We therefore run three component-level controls: change the sampler while holding the checkpoint fixed; compare the released many-step teacher with its progressive-distillation (PD) student; and change the decoder/flow objective during training. These controls ask whether basin entry, noise-coefficient preservation, and generation quality move together. They do not. The decoder-only endpoint is an expected limiting failure because its flow branch receives no training; it does not by itself identify a new transport mechanism. The noncollapsed interior, fixed-checkpoint sampler substitution, and fixed-state decoder swaps provide the less trivial contrasts. Shared-weight training and distillation do not isolate a decoder-frozen effect.

\tabref{tab:transport-controls} reports the two fixed-checkpoint controls over three seeds. In Panel A, the official ELF-B SDE has the best PPL, although its measured stochastic update has mean noise-coefficient standard-deviation ratio $0.968$. Replacing it with CPS~\citep{coeffpreserving2025} restores this ratio exactly to $1.000$, yet PPL worsens from $23.85\pm0.17$ to $43.03\pm1.79$ in every paired seed. ODE is similar in quality ($45.31\pm0.31$). Persistent-token entry $\pentry$ remains close ($0.876$, $0.873$, and $0.851$), so commitment/readability alone cannot rank the transports. The result does not establish that coefficient preservation is generally harmful; it shows that a post-hoc noise-coefficient correction is not interchangeable with the stochastic path on which this checkpoint was trained and evaluated. Sampler and checkpoint form one transport system.

Panel B supplies the complementary positive control. ELF v2 progressively halves the target budget through 16/8/4/2/1-step students while keeping a fixed 64-step teacher; its final student preserves the cross-entropy decode branch and replaces the denoising target with the teacher's interval displacement~\citep{elf2026,salimans2022progressive}. At eight inference steps, this released PD student moves $\pentry$ from $0.957\pm0.006$ to $0.668\pm0.029$, increases the final lower-tail margin from $0.67$ to $7.39$, and improves PPL from $89.90\pm9.99$ to $21.87\pm0.18$ while retaining entropy $5.03$. At 32 steps, the teacher narrows the final-margin and PPL gaps, but its persistent-token commitment remains later ($0.858$ versus $0.504$); the commitment-time gap does not narrow. An independent 1000-sample reproduction follows the released student's 1/2/4/8/16/32-step curve (measured PPL $133.55/73.36/33.23/22.88/22.16/21.14$ versus $136.10/68.25/34.33/23.18/22.12/21.32$ reported by ELF v2); the largest discrepancy is at two steps and is retained rather than tuned away. Thus the transport-targeted curriculum commits tokens earlier and reaches a much deeper final decoder basin under a small step budget; because shared weights and the CE branch are still updated, we do not claim a pure denoiser-only intervention.

\begin{table}[!htbp]
\centering
\small
\setlength{\tabcolsep}{4pt}
\caption{Compatibility is not transport. Mean $\pm$ standard deviation over three seeds, 256 sequences per seed. $\pentry$ is mean persistent-token entry (lower is earlier), not a margin threshold; $M_{10}$ is the final 10th-percentile native-decoder margin. The noise-coefficient standard-deviation ratio is an algebraic update diagnostic, not a measured guarantee of the learned state distribution, and is defined only for the stochastic sampler audit.}
\label{tab:transport-controls}
\begin{tabular}{lrrrrr}
\toprule
Condition & PPL $\downarrow$ & Entropy $\uparrow$ & $\pentry\downarrow$ & $M_{10}$ $\uparrow$ & Coeff. ratio \\
\midrule
\multicolumn{6}{l}{\emph{A. Fixed ELF-B checkpoint, 32 steps}} \\
Official SDE & $23.85\pm0.17$ & $5.155\pm0.008$ & $0.876\pm0.006$ & $12.27\pm0.33$ & $0.968\pm0.001$ \\
CPS SDE & $43.03\pm1.79$ & $5.307\pm0.005$ & $0.873\pm0.006$ & $9.99\pm0.60$ & $1.000\pm0.000$ \\
ODE & $45.31\pm0.31$ & $5.308\pm0.003$ & $0.851\pm0.007$ & $11.28\pm0.47$ & -- \\
\midrule
\multicolumn{6}{l}{\emph{B. Released teacher and PD student}} \\
Teacher, 8 steps & $89.90\pm9.99$ & $5.350\pm0.029$ & $0.957\pm0.006$ & $0.67\pm0.06$ & -- \\
PD student, 8 steps & $21.87\pm0.18$ & $5.031\pm0.011$ & $0.668\pm0.029$ & $7.39\pm0.15$ & -- \\
Teacher, 32 steps & $23.05\pm1.73$ & $5.113\pm0.010$ & $0.858\pm0.017$ & $10.26\pm1.58$ & -- \\
PD student, 32 steps & $20.98\pm0.21$ & $5.031\pm0.011$ & $0.504\pm0.053$ & $11.25\pm0.64$ & -- \\
\bottomrule
\end{tabular}
\end{table}

Short training controls first make the separation explicit. After 5k steps, a decoder-only objective has stronger corruption recovery than the official-like mixed objective both in-domain ($0.979$ versus $0.939$ at $t=0.6$) and on AG News ($0.832$ versus $0.726$) and Yelp Polarity ($0.948$ versus $0.899$), yet it receives no flow-loss examples and produces no nonempty generations. At 50k steps it collapses to essentially one sequence (unique fraction $1/256$, entropy approximately zero, PPL $1.02$). The matched full-training control in \figref{fig:matched-fulltrain} extends the short-horizon result to five epochs and two independent initialization seeds. For seeds 42 and 43, we train one $p=0.2$ mixed and one $p=1.0$ decoder-only model from shared initialization and nominal data-order settings for five OWT passes; the audited configuration files differ only in branch probability and output metadata. Configuration matching does not certify identical realized minibatch histories after resumptions (\secref{app:v3-protocol}). The decoder-only corruption-basin $M_{10}$ exceeds its paired mixed control at all ten epoch checkpoints, but its generated final-state $M_{10}$ is lower at all ten. At the final epoch, mean native/AG News/Yelp corruption $M_{10}$ is $8.92/4.13/7.07$ for decoder-only versus $6.78/-1.92/3.27$ for mixed. Nevertheless, mixed generations retain entropy $5.05/5.18$, unique fraction $1.0/1.0$, and PPL $34.85/36.84$ across the two seeds; decoder-only yields one low-entropy repeated regime (entropy $1.16$, unique fraction $0.531$, PPL $2.88$) and one run with no nonempty text. Its deceptively low finite PPL is another collapse artifact. Full per-seed results and protocol details are in \secref{app:transport-controls}. The conclusion is the central boundary of this paper: decoder compatibility defines a destination, whereas the denoising objective, sampler, and distillation curriculum determine whether useful probability mass is transported there.

\begin{figure}[!htbp]
  \centering
  \includegraphics[width=0.86\linewidth]{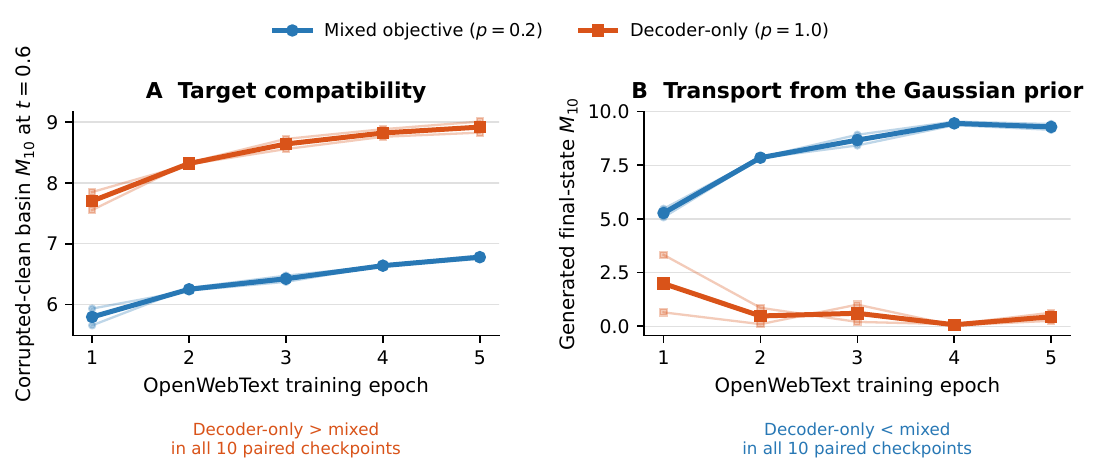}
  \caption{Matched five-epoch separation of target compatibility and transport. Two independently initialized pairs have configuration files differing only in decoder-branch probability ($p=0.2$ versus $p=1.0$) and output metadata; resumed execution need not preserve bitwise data order. Thin lines are seeds and thick lines their mean. (a) Decoder-only training has the higher corrupted-clean lower-tail margin at every paired checkpoint. (b) From the Gaussian prior, the same models have much weaker generated final-state margins because no flow objective connects the prior to that target. Basin and end-to-end audits use 512 and 256 paired-randomness samples, respectively.}
  \label{fig:matched-fulltrain}
\end{figure}

\FloatBarrier
\subsection{The interior of the dose grid}
\label{sec:dose-grid}

The boundary pair leaves the interior unexplored. We extend a budget- and configuration-matched protocol to the full grid $p\in\{0.2,0.5,0.8,1.0\}$ (two seeds, five epochs, shared starting weights within each seed, optimizer settings, effective batch, and horizon; \secref{app:v3-protocol}), where $p$ is the per-step probability that the decoder-corruption branch replaces a flow step: a checkpoint at dose $p$ and epoch $e$ has expected flow exposure $(1-p)\,e$ and decoder exposure $p\,e$.

Handed-corruption margin depth increases with decoder dose in both seeds at epoch 5: seed means at $t=0.6$ are $6.77,7.71,8.50,8.91$ across the four doses. This is a signed target-margin trend; recovery is already near saturation ($99.4$--$99.7\%$), and depth is not calibrated basin width. Under fixed-grid cfg3 sampling, five-epoch Gen-PPL is $26.94/26.72$ for $p=0.5$, $29.72/30.44$ for $p=0.2$, and $33.80/29.21$ for $p=0.8$ (seeds 42/43). Thus $p=0.5$ has the lowest PPL among noncollapsed endpoints in both seeds, but it has more bigram repetition than $p=0.2$ in both and higher MAUVE in only one. The interior result is a quality tradeoff, not a universal optimum (\tabref{tab:dose-endpoints}).

Equal expected flow exposure does not resolve the comparison by itself. Both $p=0.5$ at epoch 2 and $p=0.8$ at epoch 5 have one expected flow pass, but different decoder exposure, learning-rate stage, and optimization histories. Their relative PPL ordering reverses across seeds, including on the paired 256-example prefix (\secref{app:dose-grid}); extra decoder exposure therefore supplies no consistent improvement in this comparison.

A complementary crossed-decoder control holds the generated states fixed. At epoch 2, all four native decoders are applied to all four source generators, separately for each seed (32 cells). On noncollapsed sources, off-diagonal native-token agreement is $93.8$--$97.7\%$. Decoding the $p=1$ source with $D_{0.2}$ can restore nonempty output, but entropy remains $0.009/0.539$; replacing the readout does not restore normal diversity. Conversely, $D_1$ reads $p=0.2$ source states with $97.4/97.3\%$ agreement. This bounds a pure readout-failure explanation without asserting that the shared network contains independent modules (\secref{app:crossed-readout}).

\subsection{Interface Phase Diagram}
\label{sec:interface-phase-diagram}

The basin-navigation plot above is a trajectory audit; the same measurements can be summarized as an empirical interface phase diagram. The horizontal coordinate is normalized denoising phase. The vertical coordinates are decoder-facing reliability signals: native margin, ZSBD agreement, final-token agreement, entropy, and token-wise entry. We use three region names only as diagnostic shorthand. In the \emph{pre-entry} region, token-embedding alignment is weak and final-token agreement is low. In the \emph{competition} region, candidate basin labels become readable but are not yet stable; ZSBD agreement rises rapidly and token positions enter asynchronously. In the \emph{locked} region, the lower-tail native margin is high and most token decisions persist.

\leadin{Operational phase-transition model.} Let $\phi\in[0,1]$ denote normalized denoising phase. We summarize a trajectory by four order parameters: lower-tail native margin $M_{10}(\phi)$, final-token agreement $A(\phi)$, ZSBD/native agreement $Z(\phi)$, and self-conditioning disagreement $D(\phi)$. A crossing phase is
\[
\tau_M(\theta)=\inf\{\phi: M_{10}(\phi)\ge \theta\},
\]
with analogous crossings for $A$ or $Z$; a transition width is the phase interval needed for an agreement curve to move between two fixed levels, such as $Z=10\%$ and $Z=90\%$. This gives three testable predictions rather than a metaphor: stronger models or longer sampling should move high-margin crossings earlier; sample-level stopping should be limited by an asynchronous token tail; and pre-entry can show coherent rank growth before high-margin readability appears. The pre-entry tensor-capture audit (\secref{app:basin-stress-tests}) adds a fourth observable: the effective rank of predicted clean states rises while the noisy state remains high-rank, and consecutive predicted-clean update directions are coherent ($\langle\cos\rangle>0.75$) even though the lower-tail margin stays below $1$. The matched early zero-SC audit then shows that this chaotic-phase rank growth is not caused by the first few SC feedback updates. Directional bootstrapping therefore operates before margin-based basin entry, while SC has its largest measured quality impact later in the competition phase. The measurements below test these predictions. We do not claim a thermodynamic phase transition; the phase language is an operational model for decoder-facing observables.

\figref{fig:interface-phase-diagram} makes this dynamic view explicit. For ELF-B SDE32, ZSBD agreement rises from $10\%$ to $90\%$ over a width of $0.635$ normalized phase units, crossing the midpoint at phase $0.460$ and reaching $90\%$ at phase $0.858$. ODE32 has a similar ZSBD transition width ($0.631$) but reaches Margin-8 later (phase $0.923$, rather than around $0.85$ for SDE32 under the interpolated summaries). The cross-scale panel reports discrete recorded-phase landmarks from the main audit, so its Margin-8 labels should be read as visual trajectory summaries rather than the canonical $\tau_M(8)$ estimates used for trend fitting. For numerical crossing claims, the condition-matched supplement in \tabref{tab:p5-tau-scaling} recomputes $\tau_M(8)$ by shard-level linear interpolation: SDE32 moves from $0.849\pm0.007$ on ELF-B to $0.736\pm0.031$ on ELF-M and $0.657\pm0.036$ on ELF-L, while SDE64 moves from $0.726\pm0.024$ to $0.667\pm0.012$ and $0.582\pm0.017$. Thus, within ELF's shared frozen T5-small interface, denoiser scale shifts high-margin entry earlier under matched SDE settings. The same trend is consistent with BGEE: the main diagnostic Margin-12 gates shown in the phase-diagram figure save about $17.1\%$, $23.0\%$, and $27.3\%$ of denoising steps (number of function evaluations, NFEs), while the deterministic held-out evaluation split in \tabref{tab:basin-report} gives $16.6\%$, $23.4\%$, and $27.6\%$. A small transition-grid supplement in \tabref{tab:phase90-transition-grid} fills the missing B-SDE64 and M/L-SDE32 cells at coarse sample scale. It supports earlier entry as the stable scale/compute signal, while showing that ZSBD transition width itself is not monotone. We therefore treat the exponent as a reproducible same-interface denoiser-scale trend rather than a universal scaling law or a clean sharpness-scaling law.

The token-wise panel explains why sample-level early exit has a ceiling. In the 512-sample audit, stable Margin-8 entry is reached by $94.7\%$ of token positions, with shard-p90 phase $0.819$, while persistent native-token entry reaches $99.96\%$. The hard tail is structured rather than uniform: numeric tokens reach stable Margin-8 only $79.8\%$ of the time, and rare tokens reach it $81.4\%$ of the time. Thus the interface transition is not a single global event. It is a mostly shared sequence-level transition with an asynchronous token tail. This phase-diagram view supplies a temporal description supporting the readout and local-stability tests. A static readability observation says that some final states are easy for a decoder and others are not. Our evidence asks the stronger time-resolved question: how does an existing ELF trajectory enter the decoder basin, how does the entry shift with scale and sampler, and which token classes remain outside the early locked region? The supported claim is dynamic and diagnostic: basin entry is a measurable transition summarized by crossing phases and transition widths, not merely a final-state label.

\begin{figure}[!htbp]
  \centering
  \includegraphics[width=\linewidth]{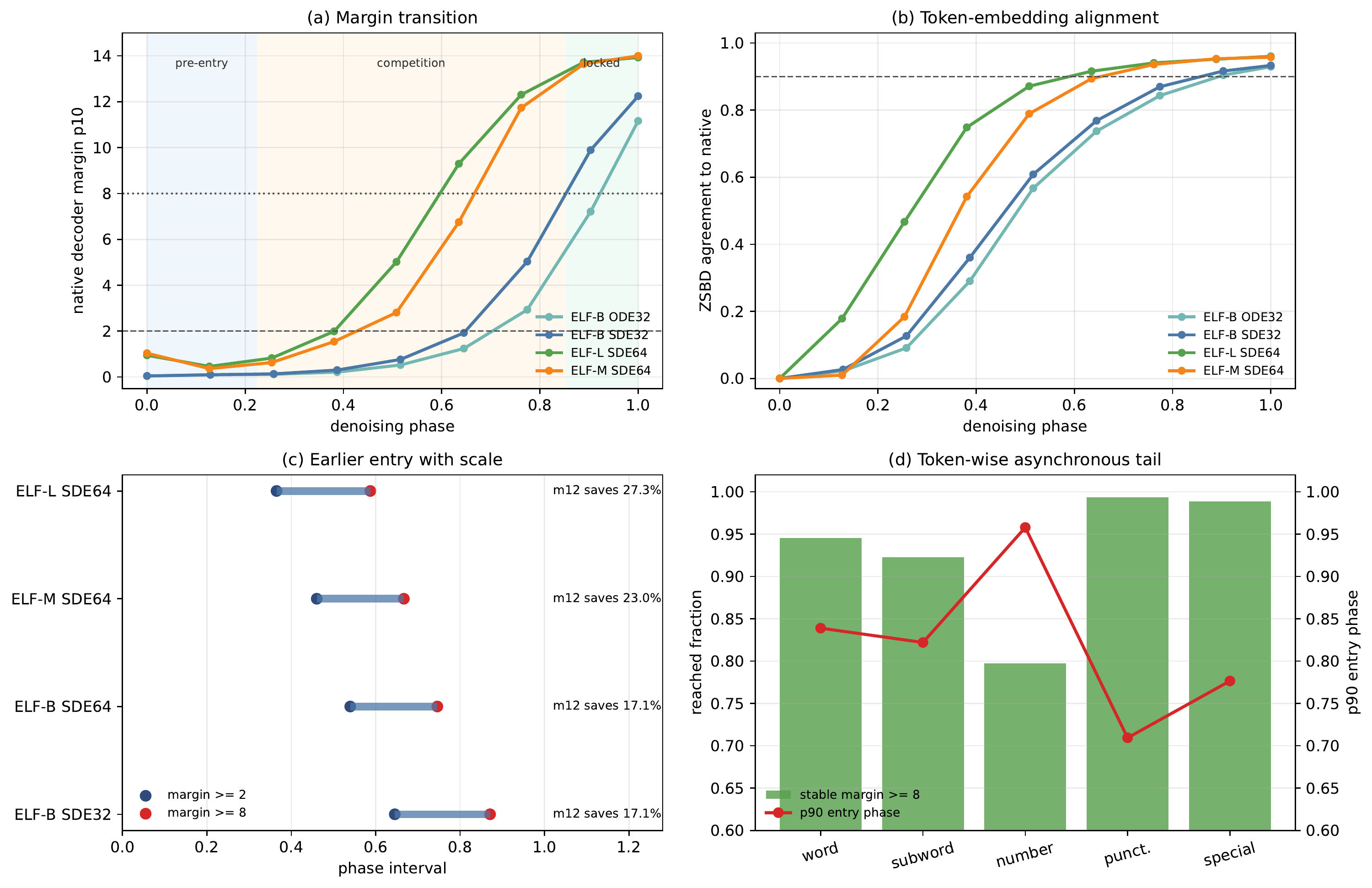}
  \caption{Interface phase diagram for ELF basin entry. (a) Native decoder margins trace pre-entry, competition, and locked regions. The region names are empirical diagnostics rather than thermodynamic phases. (b) ZSBD agreement follows the same transition, rising from weak token alignment to a high-agreement locked region. (c) Cross-scale entry shifts earlier under stronger models and longer SDE sampling; panel labels use illustrative recorded-phase landmarks from the main audit, while interpolated $\tau_M(8)$ estimates are reported in \tabref{tab:p5-tau-scaling} and \tabref{tab:phase90-transition-grid}. Text at right reports conservative BGEE Margin-12 savings from the main diagnostic audit; held-out evaluation gate savings are summarized in \tabref{tab:basin-report}. (d) Token-wise entry is asynchronous: most tokens lock early, but numeric and subword-like tail positions remain harder.}
  \label{fig:interface-phase-diagram}
\end{figure}

\subsection{Scale, Sampler, and Architecture Checks}
\label{sec:scale-sampler-architecture}

\leadin{Robustness checks.} A reproducible trajectory pattern should survive obvious perturbations. We therefore rerun the basin-navigation audit across ELF-B, ELF-M, and ELF-L; compare SDE32, SDE64, and ODE32 on ELF-B; and run analogous boundary probes on LangFlow~\citep{langflow2026} and CoBit~\citep{bitstream2026}. These checks are not meant to make the paper a benchmark suite. They test whether the margin-basin story is an artifact of one checkpoint, one sampler, or one specific decoder implementation.

\figref{fig:crossscale-basin} shows the cross-scale result. All ELF scales increase the decoder margin over denoising time, but the phase at which a strong basin is reached shifts with scale. For ELF-B SDE32, the 10th-percentile margin crosses 2.0 around phase $0.65$ and 8.0 around phase $0.87$. With 64 SDE steps, the same model reaches these thresholds earlier (about phase $0.54$ and $0.75$ in the plotted discrete audit). ELF-M and ELF-L reach them earlier still. A shard-level robustness audit (\figref{fig:basin-robustness-ci}) confirms that these basin-entry shifts are larger than variation across four independent 128-sample sampling shards: the Margin-8 entry phase moves from $0.863\pm0.008$ for ELF-B SDE32 to $0.675\pm0.010$ for ELF-M SDE64 and $0.587\pm0.017$ for ELF-L SDE64 under that audit's aggregation. A follow-up condition-matched interpolation check (\tabref{tab:p5-tau-scaling}) gives the same qualitative trend within both SDE32 and SDE64. Small absolute differences between the plotted landmarks, the robustness audit, and the condition-matched table reflect temporal resolution, sample set, and whether interpolation is performed before or after shard aggregation; we use them as convergent trend estimates rather than interchangeable calibrated constants. These error bars measure sampling/shard-level robustness under fixed checkpoints, not training-seed uncertainty or a controlled scaling law. The delta-margin sign flip remains visible as a mid-trajectory positive peak followed by a negative late correlation, but its shape is scale-dependent. We therefore use the sign flip as evidence of a revision-to-commit transition, not as a universal scalar confidence law that generalizes across architectures without recalibration.

We report two complementary correlation estimates. Step-wise correlations use per-sample Spearman statistics within a single denoising step; they have high temporal resolution and identify where the sign flip happens, but are noisy at 128 samples. The shard-level audit first aggregates each 128-sample shard over coarse trajectory regions and then reports the standard error of the mean (SEM) across shards; it loses some temporal resolution but establishes sampling robustness. The same convention applies to crossing phases: figure annotations are useful landmarks, whereas \tabref{tab:p5-tau-scaling} is the condition-matched interpolation table. Apparent numerical differences between the estimates should therefore be read as a resolution--variance trade-off, not as conflicting measurements.

\figref{fig:sampler-basin} shows that the qualitative trajectory pattern is not specific to SDE noise. ODE32 has a different entropy trajectory, but it still exhibits margin growth and a late delta-margin sign change. This separates basin entry from a purely stochastic-sampler artifact. The sampler changes how the trajectory explores the interface; it does not remove the need to enter the decoder basin.

The external checks in \figref{fig:external-basin} test the same concept outside ELF under a stricter hierarchy of evidence. LangFlow exposes a weaker latent-step margin trajectory: its native margins become positive but do not reach the high ELF lower-tail margins, and sparse final-token probes remain negative in the lower tail on both OWT and the One Billion Word Benchmark (LM1B)~\citep{lm1b2013}. \figref{fig:external-interface-overlay} places these external signals beside ELF's final decoder-basin margins and Cola-DLM target-recovery diagnostics to show that the state objects differ even when the interface questions are shared. A direct ELF-to-LangFlow cross-decoder test would not isolate basin transfer, because the released LangFlow checkpoints use GPT-2 or BERT tokenizers with learned $768$-dimensional token embeddings rather than ELF's T5 decoder interface; without a separately trained adapter, such a test would confound decoder geometry with vocabulary and coordinate-system mismatch. A direct LangFlow delta-margin audit (\figref{fig:langflow-rho}) makes the analogy more precise: on LM1B and OWT, the mid-trajectory Spearman peak between self-conditioning delta and native margin is positive ($0.18$ and $0.32$), while the late average becomes negative ($-0.35$ and $-0.53$). CoBit exposes a bit-level analogue: structured proxy codes and real OWT code ids show nearly identical sigma-sweep behavior (\figref{fig:bitstream-real-code}); for example, at $\sigma=0.5$ the bit recovery is $0.822$ for the structured proxy and $0.828$ for real OWT codes. These observations support the interface language but do not independently prove ELF's mechanism. \figref{fig:external-interface-overview} makes this hierarchy explicit: ELF carries the main mechanism evidence, while LangFlow, CoBit, and Cola-DLM are boundary checks that ask whether denoisability, decoder compatibility, and trajectory reliability remain meaningful when the state object and decoder change.

A new native-interface control makes the LangFlow boundary concrete (\secref{app:confirmatory-readout}). On 512 fresh LM1B trajectories, frozen lookup in LangFlow's own learned embedding table agrees with $89.69\%$ of final native decisions; keeping only its contextual backbone term agrees with $55.86\%$. The native logits explicitly add a linear embedding skip to the backbone output, so high late lookup agreement has an architectural explanation and is not independent evidence for ELF's T5 mechanism. A separately fitted affine map reaches $62.62\%$ across three fits at the initial 131,072-token budget; a post-fit coverage audit localizes most errors to sparsely observed token types. A subsequent, explicitly result-motivated extension to 1,048,576 training positions reaches $74.15\%$ on the same test with the same update budget. The enlarged readout still falls below frozen lookup. These controls expose the probe's budget dependence without isolating vocabulary coverage from optimization effects or supporting a cross-model ranking. The native-term audit holds trajectories and self-conditioning fixed; it does not remove either term during generation.

\begin{figure}[!htbp]
  \centering
  \includegraphics[width=\linewidth]{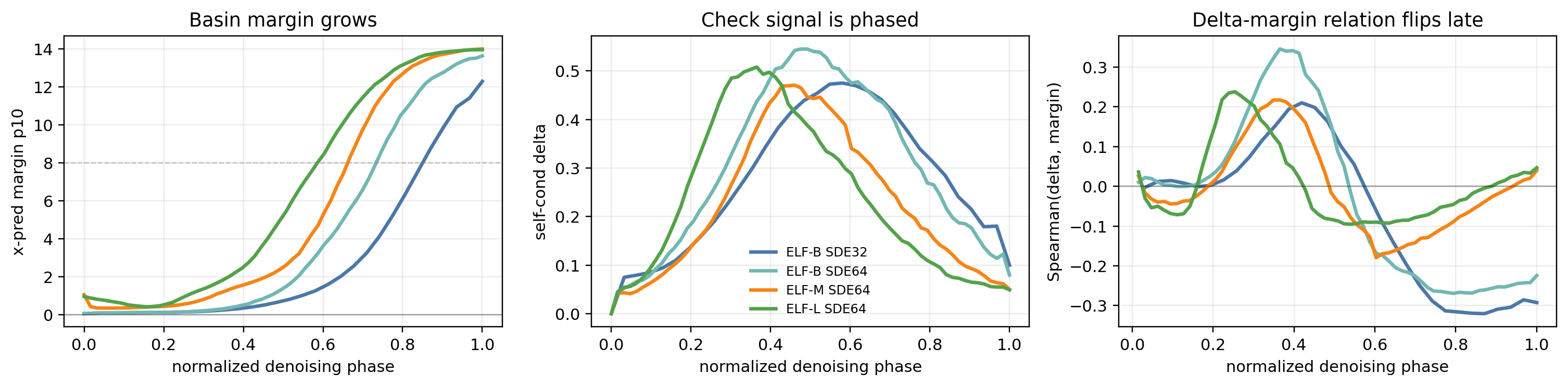}
  \caption{Basin-navigation signals across ELF scale. Left: 10th-percentile native decoder margin over normalized denoising phase. Middle: self-conditioning disagreement over the same phase axis. Right: step-wise Spearman correlation between self-conditioning disagreement and decoder margin. Margin growth is shared by ELF-B/M/L, while the strongest self-conditioning--margin relation shifts with scale.}
  \label{fig:crossscale-basin}
\end{figure}

\begin{figure}[!htbp]
  \centering
  \includegraphics[width=\linewidth]{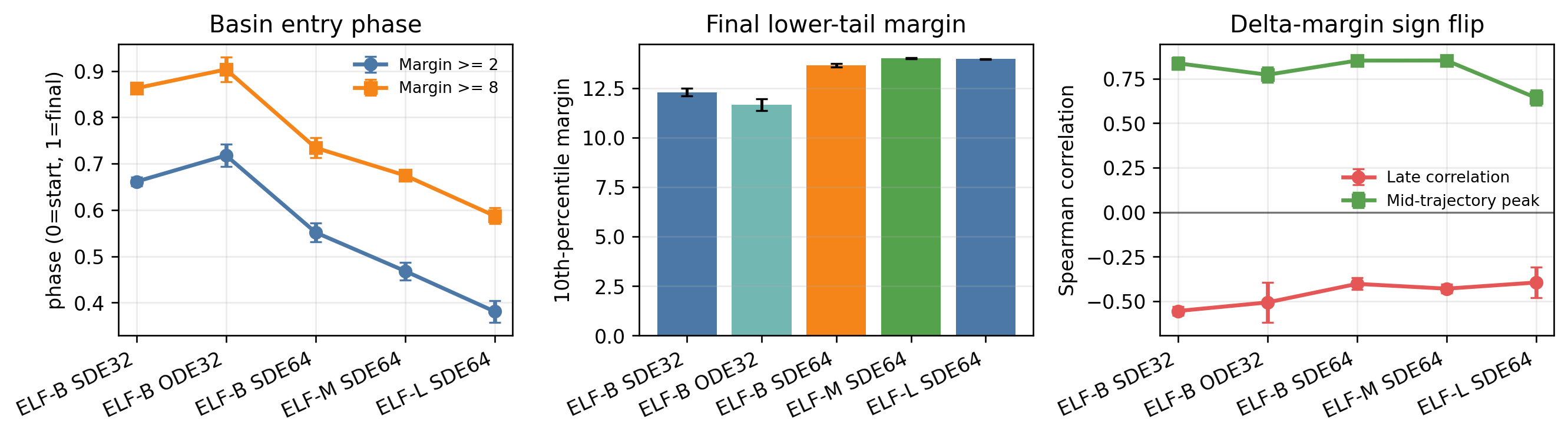}
  \caption{Shard-level robustness of the basin-navigation audit. Left: basin-entry phase under two margin thresholds. Middle: final lower-tail decoder margin. Right: mid-to-late delta-margin sign-flip strength. Each condition contains four independent 128-sample GPU shards; error bars show SEM across sampling shards under fixed public checkpoints, not training-seed uncertainty. Basin entry shifts earlier with model scale, while the self-conditioning--margin relation remains scale-dependent rather than a universal scalar law.}
  \label{fig:basin-robustness-ci}
\end{figure}

\begin{figure}[!htbp]
  \centering
  \includegraphics[width=\linewidth]{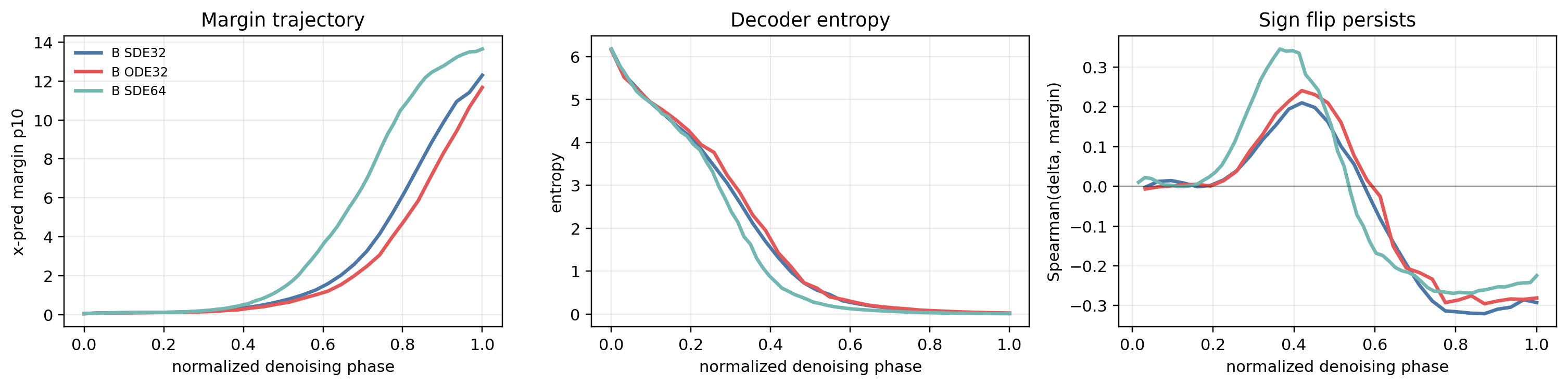}
  \caption{ODE and SDE sampler comparison. Left: native decoder margin grows under both samplers. Middle: decoder entropy follows a different schedule-dependent trajectory. Right: the delta--margin correlation still changes sign late. Basin entry is therefore not only an artifact of SDE noise injection.}
  \label{fig:sampler-basin}
\end{figure}

\begin{figure}[!htbp]
  \centering
  \includegraphics[width=\linewidth]{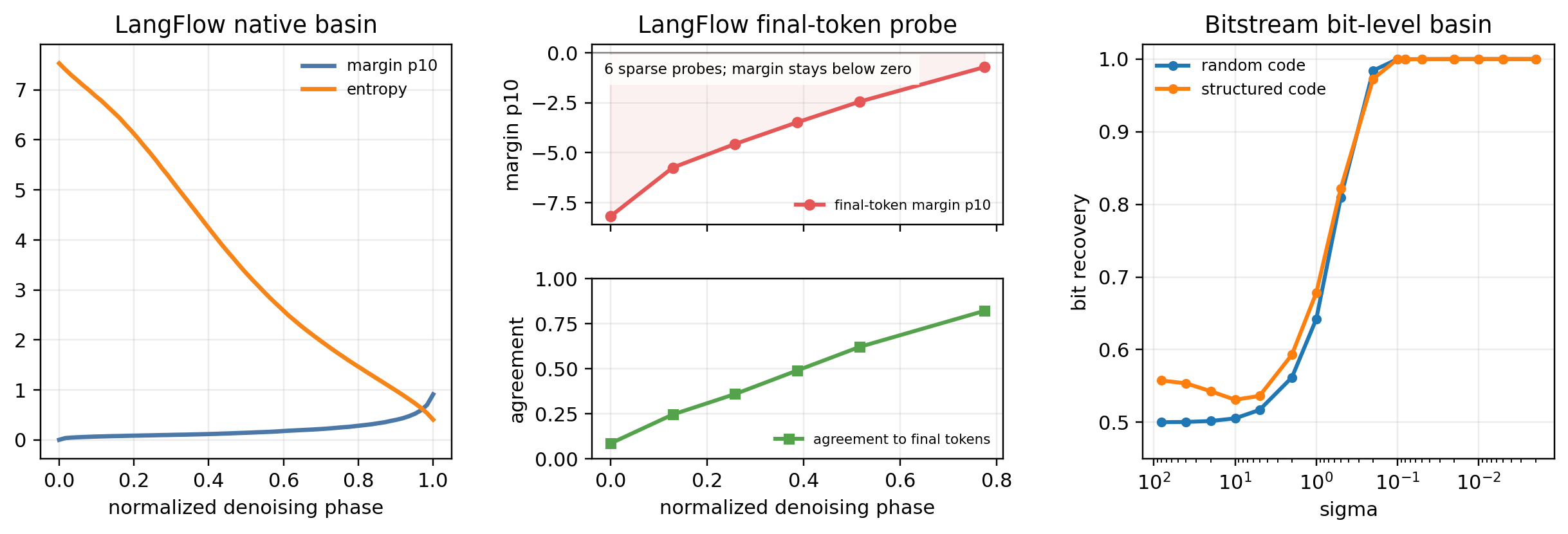}
  \caption{External boundary checks. Left: LangFlow native step margins provide a weaker latent-step analogue of ELF's decoder-margin audit. Middle: sparse LangFlow final-token probes show rising agreement but a lower-tail final-token margin that remains below zero. Right: CoBit produces a bit-level basin analogue through recovery and margin under increasing code noise. Axes use each model's native interface units and should be compared qualitatively, not as a shared numerical scale.}
  \label{fig:external-basin}
\end{figure}

\begin{figure}[!htbp]
  \centering
  \includegraphics[width=\linewidth]{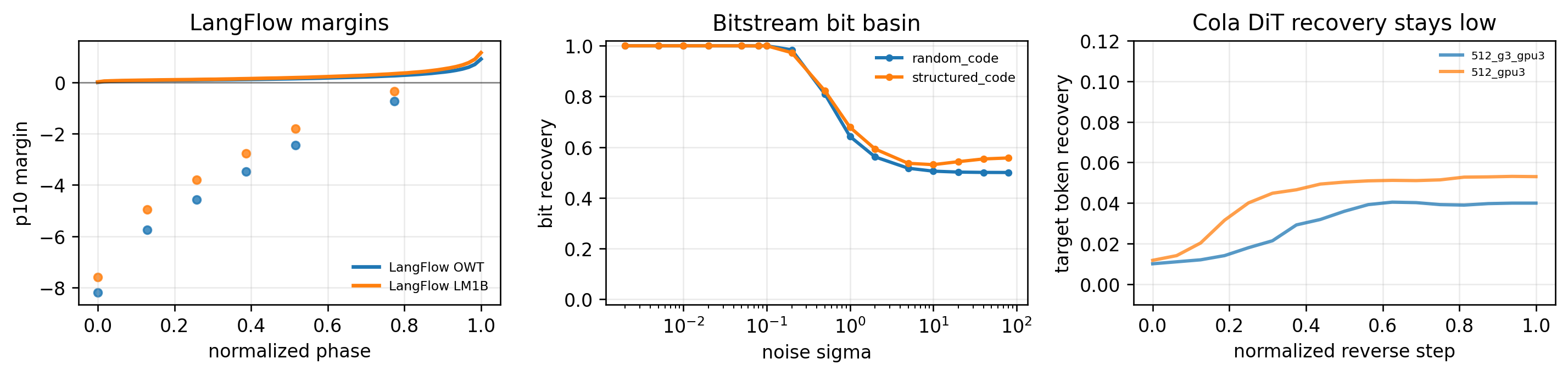}
  \caption{External interface overlay. Left: LangFlow native step margins rise, while sparse final-token commitment probes remain stricter. Middle: CoBit shows a bit-level basin analogue rather than a full text-generation reproduction. Right: Cola-DLM DiT target-token recovery remains low, matching the VAE boundary interpretation. The panel compares interface diagnostics, not benchmark rank.}
  \label{fig:external-interface-overlay}
\end{figure}

\begin{figure}[!htbp]
  \centering
  \includegraphics[width=\linewidth]{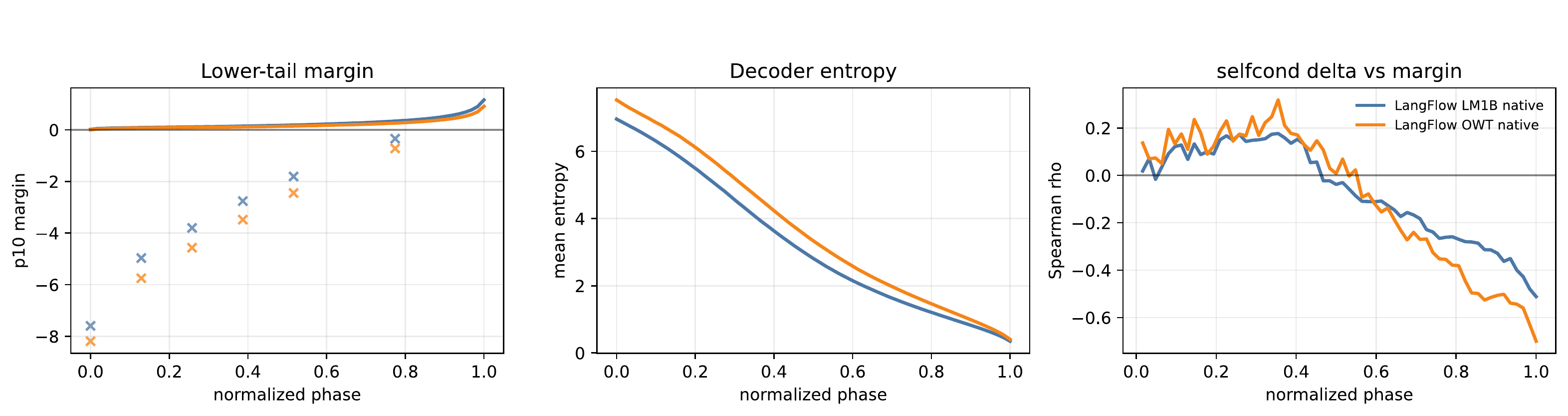}
  \caption{LangFlow delta-margin audit. Left: native lower-tail margins rise along the trajectory; the sparse cross markers show final-token margin probes and remain lower because they ask whether the intermediate logits already support the eventual final tokens. Middle: decoder entropy falls as the latent state becomes more committed. Right: the self-conditioning delta/margin relation is positive in the middle and negative late. The qualitative revision-to-commit pattern is ELF-like, but LangFlow margins remain smaller and final-token commitment is stricter.}
  \label{fig:langflow-rho}
\end{figure}

\begin{figure}[!htbp]
  \centering
  \includegraphics[width=\linewidth]{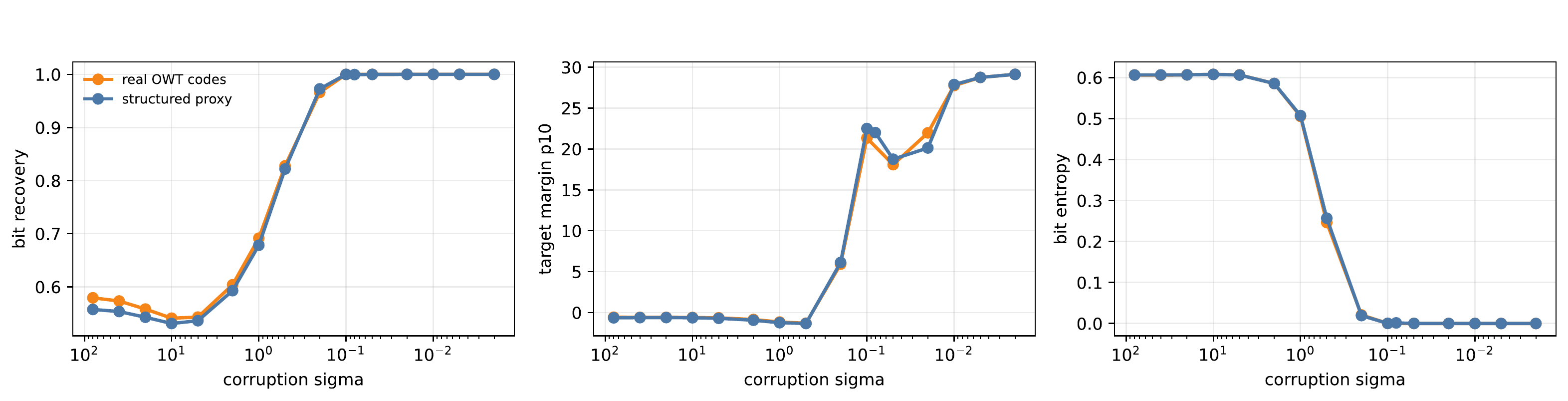}
  \caption{CoBit real-code validation. A 512-sample real OWT code sweep matches the earlier structured-code proxy closely, so the bit-level basin analogue is not an artifact of synthetic code ids.}
  \label{fig:bitstream-real-code}
\end{figure}

\begin{figure}[!htbp]
  \centering
  \includegraphics[width=\linewidth]{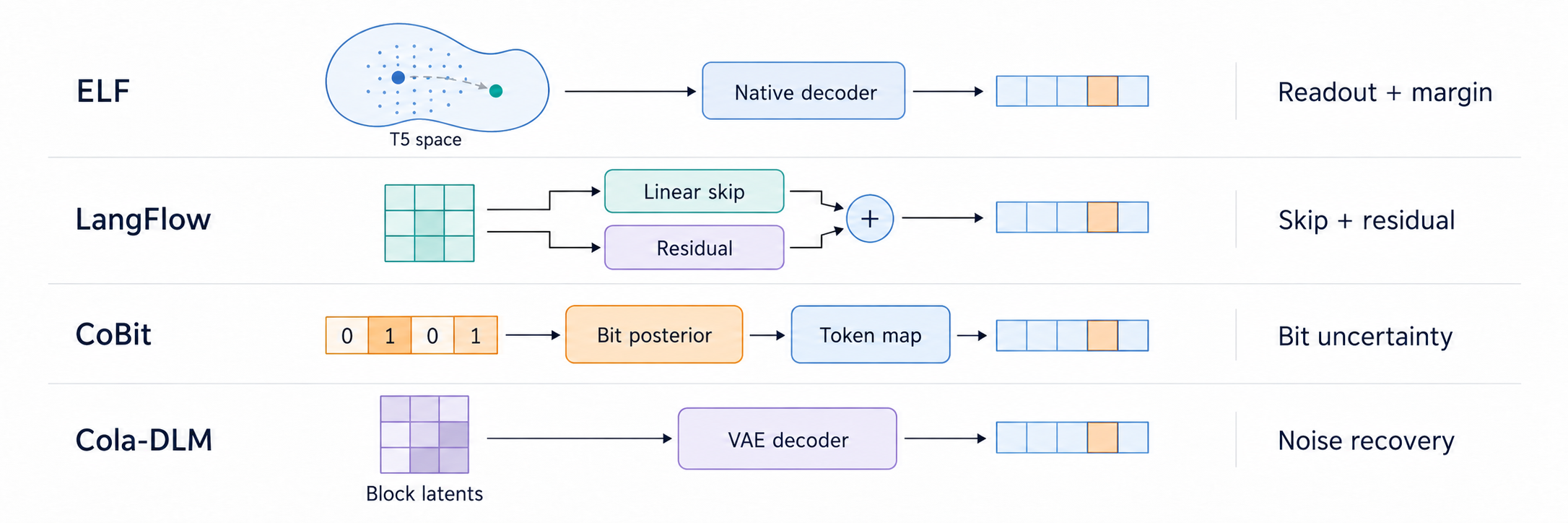}
  \caption{Conceptual interfaces beyond ELF. ELF decodes states in its T5-derived representation space with its own native decoder; the T5 encoder is not that decoder. LangFlow combines a linear skip with a residual path, CoBit maps bit posteriors to tokens, and Cola-DLM uses a VAE decoder. These different state objects motivate interface-specific checks. The diagram contains no measured curves or cross-system quality ranking; the external checks are boundary evidence rather than matched replications of ELF.}
  \label{fig:external-interface-overview}
\end{figure}

\subsection{Layer-Wise Decoder Basins}
\label{sec:layer-wise-basins}

The preceding experiments test trajectories. A complementary question is whether decoder compatibility follows from hidden-state content alone or depends on the exact output interface. We sweep T5-small hidden states through the native ELF decoder after ELF's fixed latent mean/std transformation. BERT, RoBERTa, and GPT-2 use their own native language-modeling heads. The latter tests shifted next-token prediction, whereas BERT and RoBERTa test unmasked same-position reconstruction; their scores are not a cross-model quality ranking.

The historical sweep in \figref{fig:layer-basin} contains a strong normalization confound. T5 hidden indices 0--5 precede terminal RMSNorm, whereas index 6 and \texttt{last\_hidden\_state} are the same normalized output, not two distinct depths. That final output has clean recovery $0.99996$, target-margin p10 $15.42$, and recovery $0.99993$ at $t=0.7$. The low raw margins at layers 4--5 therefore do not establish that lexical information disappears and subsequently returns. GPT-2 likewise exposes a normalized final output after unnormalized intermediate states; its large raw-margin drop is not directly an increase in decision error.

Post-hoc paired controls resolve much of this ambiguity (\secref{app:layer-interface-controls}). On 128 additional disjoint texts, applying the trained terminal RMSNorm to the \emph{same} T5 layer-5 state changes clean recovery from $78.82\%$ to an observed $100.00\%$; per-token RMS scaling without learned coordinate gains reaches $99.97\%$, while the tested single global-RMS matching rule reaches only $64.90\%$. For GPT-2 layer 10, a positive global rescaling preserves every argmax while changing mean per-sequence margin p10 from $-126.63$ to $-22.85$; terminal LayerNorm changes recovery from $22.54\%$ to $25.94\%$. The same normalization can hurt earlier GPT-2 layers, so it is not a universal repair. These interventions support a narrower conclusion: depth, normalization, and readout jointly determine measured readability, and raw margin magnitude cannot separate their effects. They do not measure a new generator or establish improved text quality.

\begin{figure}[!htbp]
  \centering
  \includegraphics[width=\linewidth]{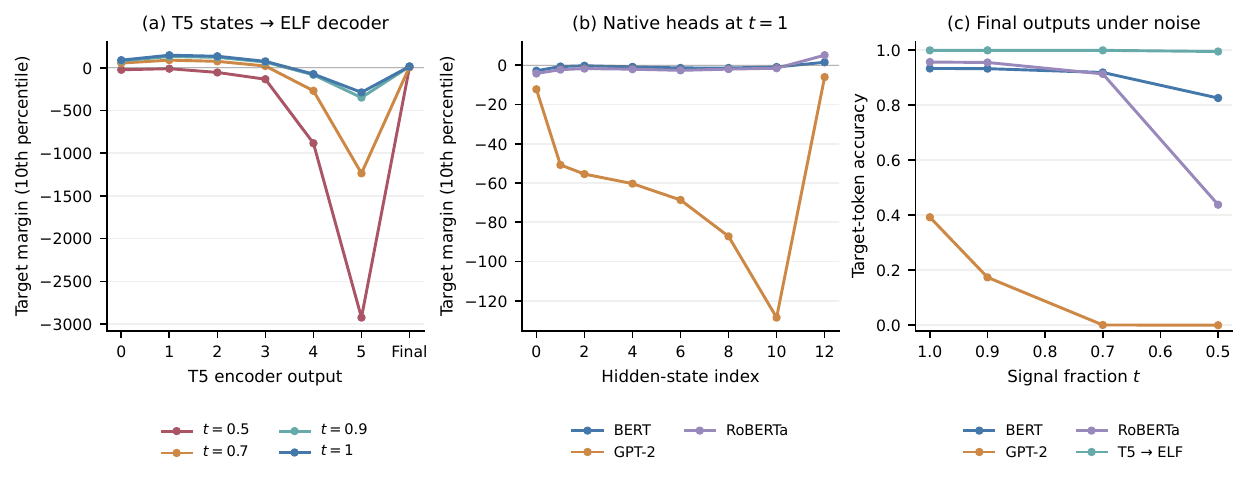}
  \caption{Historical layer/interface sweep, with corrected panel descriptions. Left: T5 target-margin p10 at four signal fractions $t$; intermediate states precede terminal RMSNorm, while Final is the normalized encoder output (duplicate final-state aliases are displayed once). Middle: native-head target-margin p10 at $t=1$; final GPT-2 index 12 already includes LayerNorm. Right: final-output target-token accuracy across corruption levels. GPT-2 uses shifted next-token targets; BERT and RoBERTa use unmasked reconstruction. Historical batch-pooled margin summaries are retained, and scores are not a cross-model quality ranking. Paired normalization controls in \figref{fig:layer-normalization-controls} separate part of the apparent depth effect from the readout interface.}
  \label{fig:layer-basin}
\end{figure}

\subsection{A Minimal Schedule Validates the Signal}
\label{sec:minimal-schedule}

The trajectory-reliability signal suggests a minimal schedule: use stronger self-conditioning early, default self-conditioning in the middle, and weaker self-conditioning late. \tabref{tab:multiseed} shows three-seed results. Front-loading is consistently slightly better than fixed SC=5 in terms of PPL, while preserving similar entropy. It is not an optimized sampler; it is a small intervention validating the diagnostic.

\figref{fig:frontier} places this schedule in the PPL--entropy plane, and \figref{fig:signal-ablation} shows the more important time-region control: moving the same guidance strength to the wrong part of the trajectory degrades quality. On the same seed and 512 samples, front-loading gives PPL $23.45$, while reversing the schedule (strong SC late, weak SC early) gives PPL $26.48$. This is the most controlled intervention because it keeps the same guidance values but swaps when they are applied. In contrast, simple online delta-gating is not yet a reliable sampler: true, delayed, randomized, and clipped gates all remain within a narrow band. We therefore make a conservative claim: ELF exposes trajectory information that can validate a time-region schedule, but converting this into a generally superior adaptive policy remains future work.

We also tested even more literal one-line interventions suggested by the diagnostics. We averaged final states with late clean predictions, applied token-level margin-aware temperature scaling, and used a dispersive logit scaling control. Across four low-memory runs (front-loaded 256 samples, a second front-loaded seed with 128 samples, and fixed SC=3/5 with 128 samples each), late clean-prediction averaging gives only small PPL changes: last-two-clean averaging improves by $0.03$--$0.13$ PPL, while a non-adjacent clean-prediction average improves by $0.18$--$0.36$ PPL. This suggests a weak smoothing effect near the final predicted-clean states rather than a strong adaptive policy. Margin-aware temperature is more decisive as a negative result: it consistently worsens PPL by $0.20$--$0.23$, and a shuffled-margin control is similarly bad. Averaging the final state with initial noise hurts strongly ($+0.68$--$0.98$ PPL). These controls strengthen the diagnostic framing: internal signals and decoder margins are meaningful measurements, but local one-line rules are not yet robust methods that can be trusted without further validation.

\begin{table}[!htbp]
\centering
\caption{Three-seed 1024-sample confirmation of the front-loaded self-conditioning schedule. The table reports corpus-level GPT-2-Large PPL and generated-text entropy across three sampling seeds for ELF-B and ELF-M.}
\label{tab:multiseed}
\begin{tabular}{llcc}
\toprule
Model & Setting & PPL mean $\pm$ std & Entropy mean $\pm$ std \\
\midrule
ELF-B & fixed SC=3 & $24.01 \pm 0.24$ & $5.158 \pm 0.006$ \\
ELF-B & fixed SC=5 & $22.99 \pm 0.23$ & $5.135 \pm 0.008$ \\
ELF-B & front-loaded & $\mathbf{22.82 \pm 0.23}$ & $5.135 \pm 0.008$ \\
\midrule
ELF-M & fixed SC=3 & $26.35 \pm 0.31$ & $5.244 \pm 0.008$ \\
ELF-M & fixed SC=5 & $23.85 \pm 0.41$ & $5.202 \pm 0.011$ \\
ELF-M & front-loaded & $\mathbf{23.77 \pm 0.40}$ & $5.202 \pm 0.011$ \\
\bottomrule
\end{tabular}
\end{table}

\begin{figure}[!htbp]
  \centering
  \includegraphics[width=.86\linewidth]{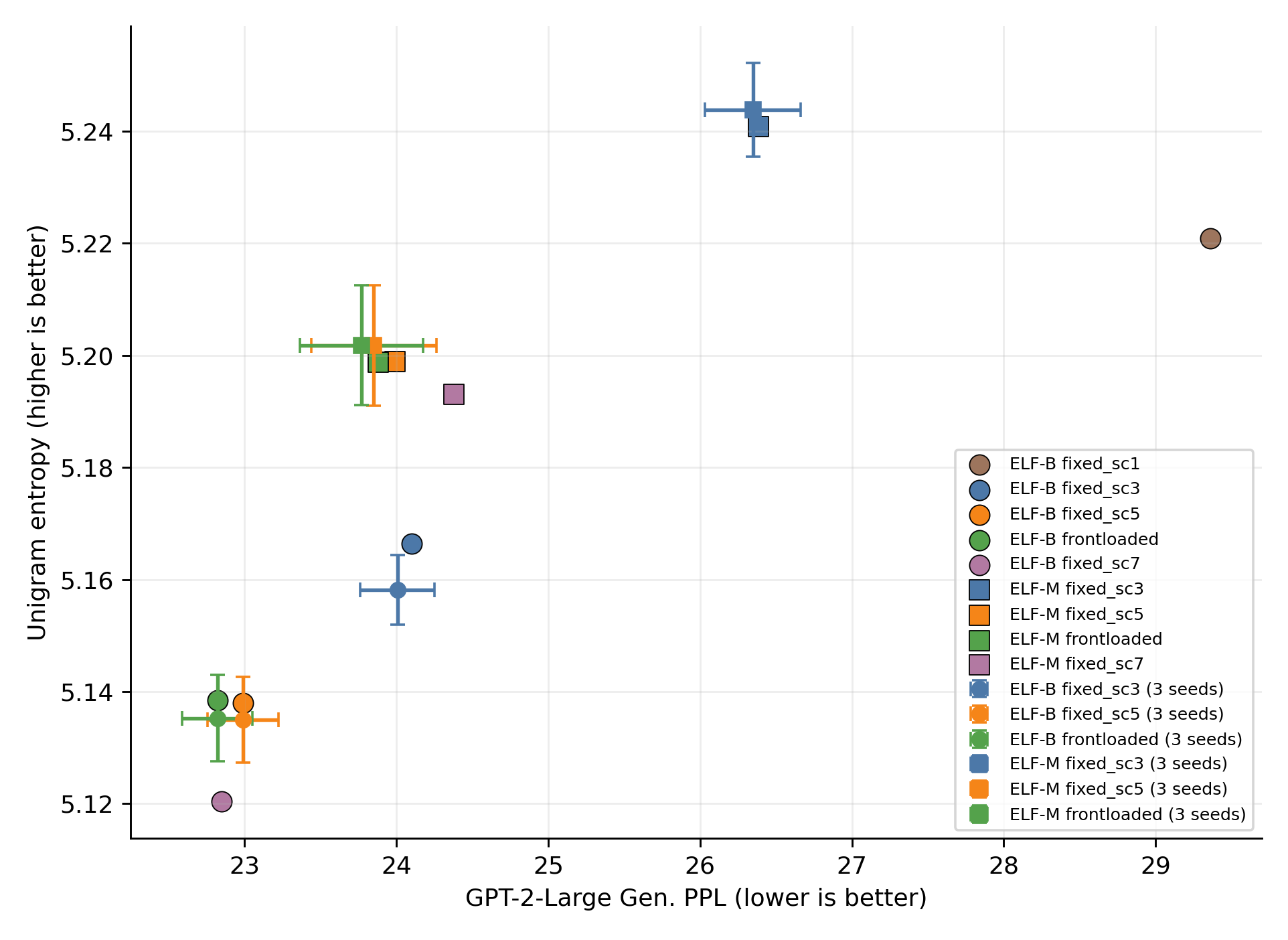}
  \caption{PPL--entropy frontier for fixed and front-loaded self-conditioning schedules. Each point is a sampling configuration evaluated with GPT-2-Large PPL and sample entropy. Front-loading moves slightly toward lower PPL without the large entropy collapse caused by stronger fixed guidance.}
  \label{fig:frontier}
\end{figure}

\begin{figure}[!htbp]
  \centering
  \includegraphics[width=.95\linewidth]{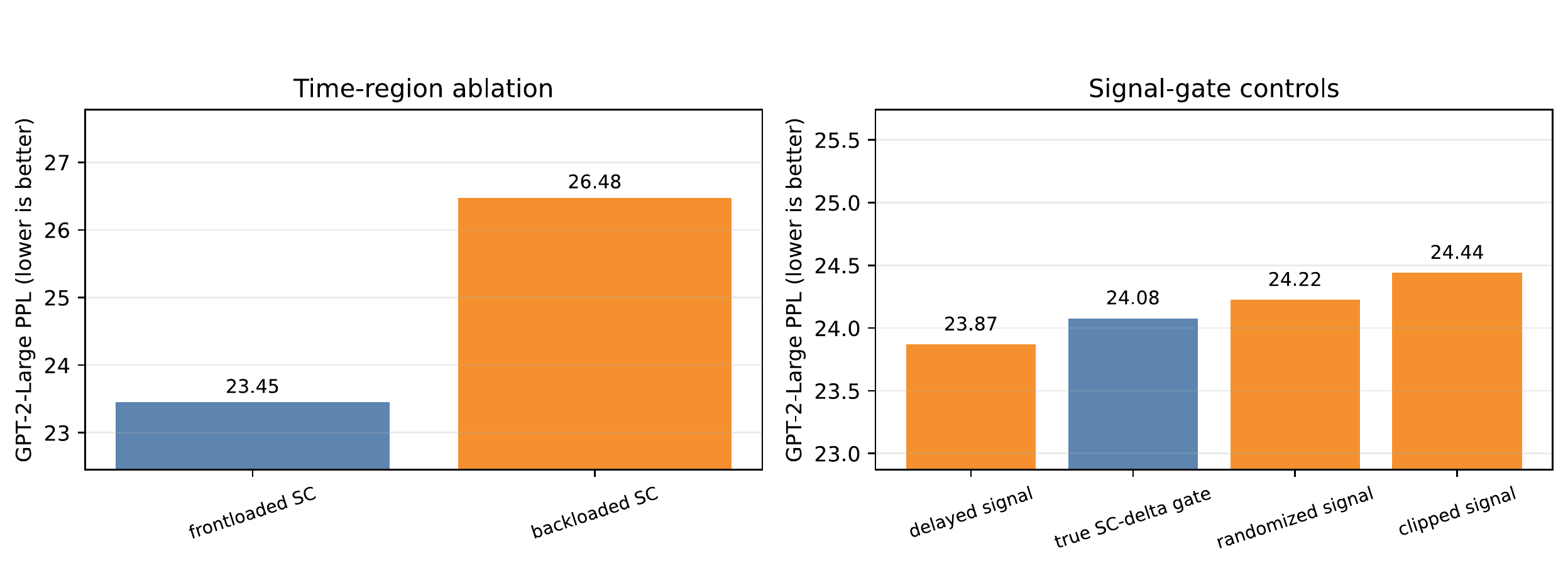}
  \caption{Time-region and signal-gate controls. Left: front-loading self-conditioning improves over the reversed schedule, confirming that the discovered time region is not arbitrary. Right: online delta-gating variants remain close to each other and do not form a mature adaptive sampler. The schedule result validates the phase signal, while the gate controls keep the method claim limited.}
  \label{fig:signal-ablation}
\end{figure}

\subsection{Three Minimal Basin Probes}
\label{sec:three-minimal-probes}

\leadin{Three probes, one interface audit.} The basin view naturally suggests three minimal probes. They measure different properties of the same generated-state interface rather than define three separate decoding products. Unless otherwise specified, \emph{basin} below refers to the native decoder margin basin defined in \secref{sec:diagnostic}; ZSBD and MDP measure token alignment and affine recoverability on generated states, not the geometry or linearity of the basin itself. The word \emph{minimal} refers to the added modeling machinery: BGEE is a rule over an existing margin, ZSBD uses a frozen lookup, and MDP uses one linear readout. Their sample counts and monitoring costs differ, and those costs are reported separately. BGEE is a \emph{basin-timing probe}: if the trajectory has already entered a high-margin decoder basin, stop. ZSBD is a \emph{token-alignment probe}: once inside the basin, ask whether final states are already close to frozen T5 token embeddings. MDP is a \emph{linear-recoverability probe}: ask how much of the native decoder interface can be recovered by a single readout trained on generated final states. \figref{fig:three-basin-probes} summarizes this structure and makes the relationship among the three probes visually explicit.

\begin{figure}[!htbp]
  \centering
  \includegraphics[width=\linewidth]{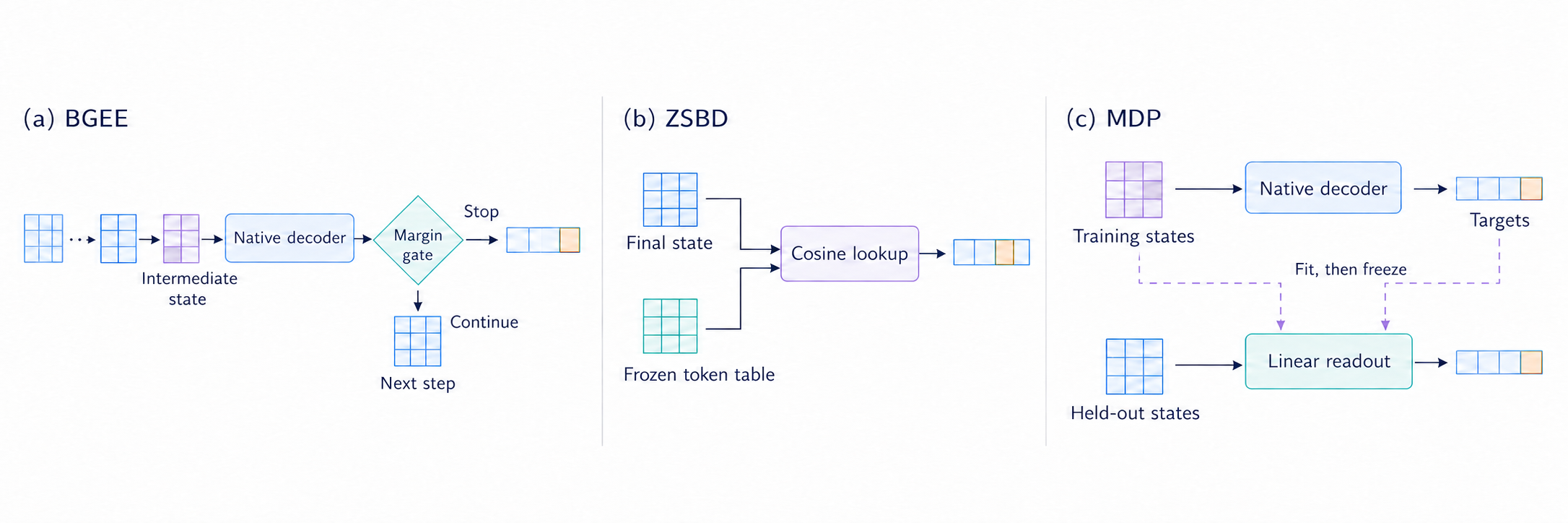}
  \caption{Three minimal probes, shown schematically. BGEE monitors an intermediate state and applies a held-out gate to decide whether to stop; its native-margin monitor has an explicit computational cost. ZSBD compares final states with the frozen, labeled T5 token table by cosine lookup, without training a new readout. MDP fits a linear readout to native decisions on generated training sequences and evaluates on disjoint held-out sequences. The token strips illustrate the task, not exact per-sample agreement.}
  \label{fig:three-basin-probes}
\end{figure}

\parhead{BGEE: a basin-timing probe}
The previous schedule experiment validates that the signal matters, but it still uses a fixed number of denoising steps. A natural diagnostic question is whether the measured basin can support a minimal action: if the trajectory has already entered a high-margin decoder basin, stop. We implement \emph{Basin-Guided Early Exit} (BGEE), a training-free rule that monitors the 10th-percentile native decoder margin of the predicted clean latent. Once the margin exceeds a threshold for a short run of consecutive steps, we decode immediately. No policy is learned, no new model is trained, and the rule uses the same decoder-margin quantity used by the diagnostic.

\figref{fig:bgee} shows the speed-quality curve on ELF-B SDE64 with 512 samples. The conservative Margin-12 threshold exits at an average number of function evaluations (NFE) of $53.1$, saving $17.1\%$ of denoising steps while essentially matching full decoding (geometric PPL $18.45$ vs.\ $18.38$, mean PPL $19.24$ vs.\ $19.21$, and token agreement $0.969$ to the full decode). These geometric PPL values come from the single early-exit audit and should not be compared directly with the multi-seed arithmetic PPL table above. A more aggressive Margin-8 threshold exits at average NFE $45.4$, saving $29.1\%$ with moderate degradation (geometric PPL $19.28$, agreement $0.925$). Fixed exits can also save compute and sometimes lie on a similar frontier. BGEE contributes a decoder-facing timing criterion: it exits when the measured interface has been reached rather than at a hand-picked step.

The cross-scale confirmation in \figref{fig:bgee-crossscale} makes the same point more strongly. The conservative Margin-12 gate saves $23.0\%$ NFEs on ELF-M and $27.3\%$ on ELF-L while preserving high agreement to the full decode ($0.965$ and $0.972$). The Margin-8 gate saves $32.2\%$ and $41.2\%$ NFEs with the expected moderate degradation. We do not interpret the absolute PPL values across model sizes as scaling claims, since these public checkpoints and evaluation settings are not a controlled training sweep. The robust trend is narrower: the larger checkpoints enter the measured decoder basin earlier, and BGEE converts that earlier entry into a larger compute saving.

BGEE is not presented as a final sampler. It is a simple timing probe that follows from the mechanism. To check that the result is not purely hand-picked on the main diagnostic audit, we also perform a deterministic half-split validation: margin gates are selected on one calibration half of generated samples using an agreement-and-PPL constraint and then evaluated on the held-out half. The selected held-out gates save $16.6\%$, $23.4\%$, and $27.6\%$ NFEs on ELF-B/M/L SDE64 respectively, with token agreement above $0.96$ and small geometric-PPL changes relative to full decoding. This monotonic increase with checkpoint size is an empirical trend rather than a fitted law: the three checkpoints were not produced as a controlled scaling suite. It nevertheless gives a concrete hypothesis for future larger DLMs: earlier basin entry should either keep improving with denoiser scale, saturate at a representation-limited floor, or break under a different training recipe. We also test a random-time control that scans the same eligible denoising steps in a shuffled temporal order before applying the same margin rule. At matched compute on a 512-sample control run, ordered Margin-12 saves $15.9\%$ NFEs with geometric-PPL change $+0.13$ and agreement $0.967$, while random-time Margin-5 saves $15.5\%$ with geometric-PPL change $+0.44$ and agreement $0.958$. The gate therefore uses both the margin threshold and the ordered denoising trajectory.

The NFE numbers are algorithmic, not an optimized systems claim. The literal monitor used here decodes each predicted-clean state to compute a lower-tail margin, which adds overhead. A small wall-clock smoke test confirms the caveat: on four 32-sample GPU shards, unmonitored 64-step sampling takes $25.3$ seconds on average, per-step monitoring takes $45.5$ seconds, and a conservative batch-level early-exit implementation takes $44.0$ seconds. We also audited the simplest cheap proxy, the self-conditioning delta. It is informative but phase-dependent: it is positively related to margin early, becomes a strong negative proxy late, and is less stable across scales. It is therefore not a drop-in replacement for native-margin monitoring. An efficient implementation would need a learned cheap proxy, sparse monitoring, or dynamic batching. In this paper BGEE is diagnostic: decoder-basin entry becomes an actionable stopping criterion once a cheap monitor exists.

\begin{figure}[!htbp]
  \centering
  \includegraphics[width=\linewidth]{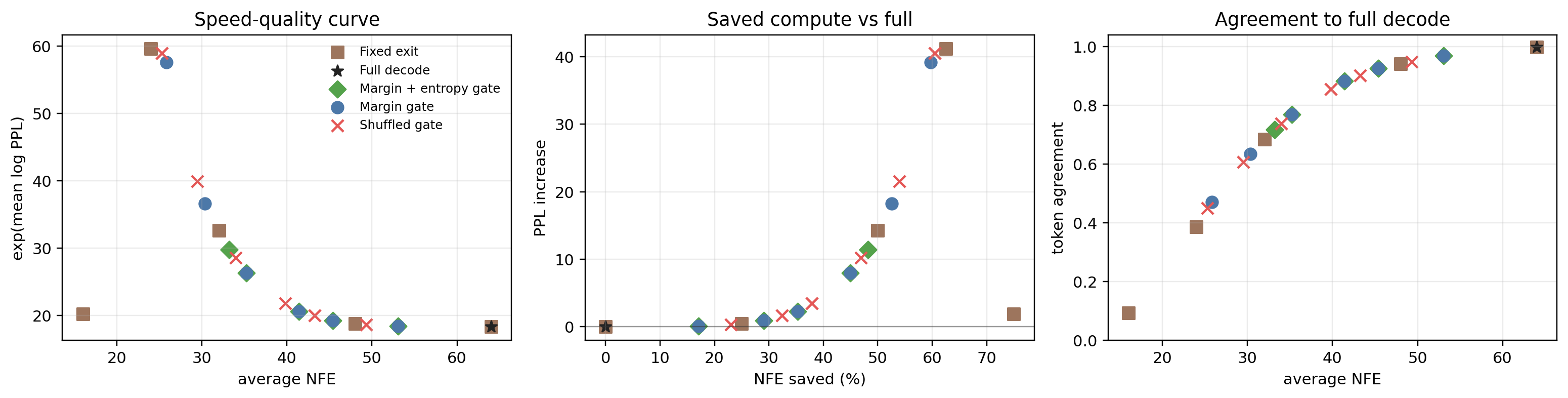}
  \caption{Basin-Guided Early Exit (BGEE) as a basin-timing probe. Left: speed-quality curve for fixed exits, margin gates, and shuffled controls. Middle: compute saved relative to the full 64-step decode. Right: token agreement with the full decode. Margin gates are not claimed to dominate every fixed exit; they show that measured native-decoder basin entry can act as an interpretable stopping signal.}
  \label{fig:bgee}
\end{figure}

\begin{figure}[!htbp]
  \centering
  \includegraphics[width=.92\linewidth]{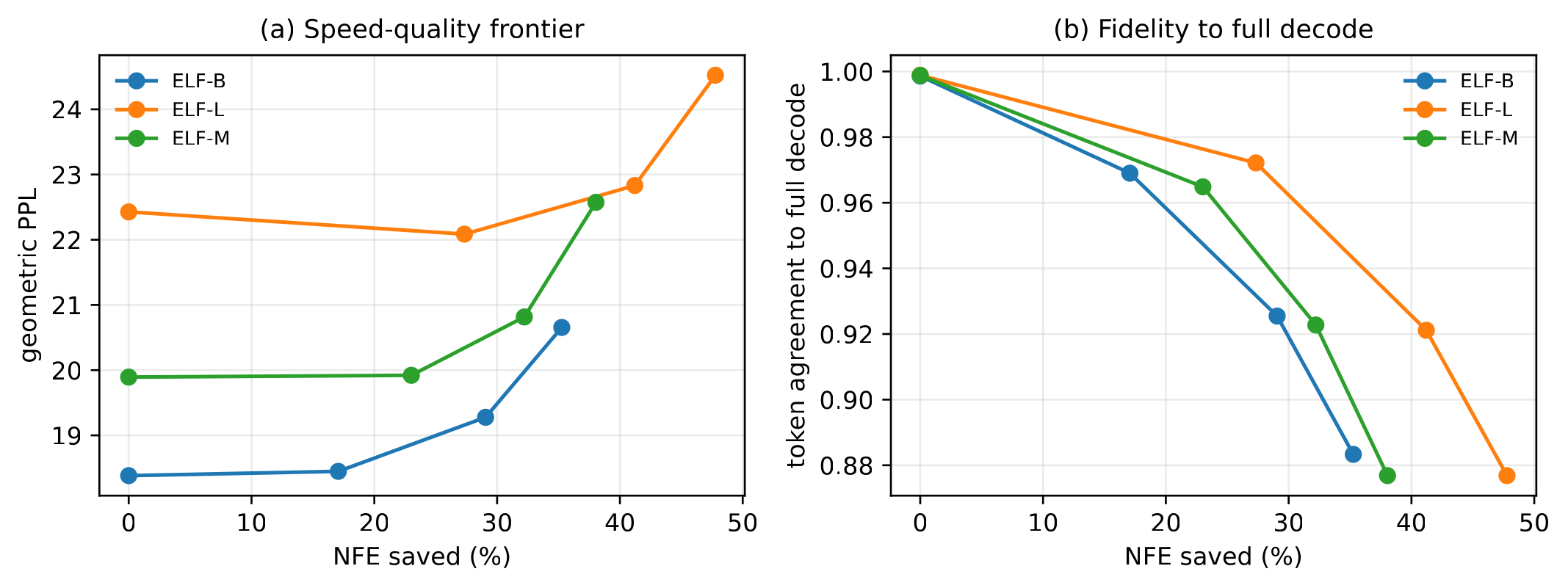}
  \caption{Cross-scale BGEE confirmation. Left: geometric PPL versus denoising compute saved for margin-gated exits across ELF-B/M/L. Right: token agreement with the full 64-step decode under the same exits. Larger checkpoints tolerate earlier exits at comparable fidelity, matching the diagnostic finding that they enter the final decoder basin earlier.}
  \label{fig:bgee-crossscale}
\end{figure}

\parhead{MDP: a linear-recoverability probe} We present MDP before ZSBD because it first estimates how linearly recoverable the native interface can become when a small matching set is allowed; ZSBD then asks the stricter no-new-readout question. Once ELF has navigated into the decoder basin, how much of the native decoder is still essential? We define \emph{Minimal Decoder Protocol} (MDP). We freeze the ELF denoiser, sample generated final states, decode them with the native decoder to obtain target tokens, and train a single token-wise linear readout from late latent states to those native tokens. The historical summaries select a late-state/noise variant on the same evaluation split used to report its PPL, so those best-of-sweep values are descriptive and can be optimistic. We retain them and separately report a fixed $z_T$ readout with training noise $0.02$, without per-budget variant selection (\secref{app:readout-selection}). Historical matching-pool size includes training and evaluation sequences. Those splits change with pool size, and some smaller-pool evaluation shards become training shards at a larger budget; these are within-budget held-out comparisons, not a learning curve on one untouched common test set. This is not meant to replace the official decoder in a deployed model. It is a controlled interface test: failure shows that the selected readout and fitting budget are insufficient, while success shows that most native decisions are linearly recoverable on the sampled final-state manifold.

\figref{fig:mdp} shows that both statements are partly true. With only 1024 generated latents, the best linear readout reaches about $81.6\%$ token agreement to the native decoder and remains far worse in PPL ($36.1$ vs.\ $21.7$ on the held-out split). With 4096 generated latents and 1.2M token-position training examples, the same protocol reaches $93.9$--$94.1\%$ token agreement and reduces geometric PPL to $28.0$--$28.4$, compared with $23.9$ for the native decoder on the same split. The best 4k readout also closely matches the native decoder's diversity: sample unigram entropy is $5.06$, distinct-2 is $0.714$, and 4-gram repetition is $0.0676$. The larger-pool comparisons retain further gains: with 8192 generated latents, agreement rises to $97.0\%$ and PPL falls to $25.5$; with 16384 generated latents, agreement reaches $97.6\%$ and PPL falls to $24.8$; with 32768 generated latents, agreement reaches $97.9\%$ and PPL stays around $24.8$, compared with native PPL $23.6$ on the same held-out split. The reported 16k--32k gain is smaller than the 4k--8k gain, but changing splits and variant selection prevent a controlled saturation claim. The fixed-variant comparison retains the bulk-readout conclusion and the remaining native-agreement and PPL gaps.

\leadin{New fits on common tests.} To remove the changing-test-set ambiguity, we freeze nested training pools of 1,024, 4,096, and 16,384 sequences, one 512-sequence validation set, and one 1,024-sequence historical test before fitting. Three fitting seeds use the same 4,096-update budget and validation-only checkpoint selection. All nine readouts are also evaluated on 512 newly generated test sequences. Mean agreement rises from $96.43\%$ to $97.28\%$ on the historical test and from $96.32\%$ to $97.17\%$ on the fresh test; the latter paired gain is $0.85$ percentage points (95\% sequence-bootstrap interval $[0.80,0.90]$). Fit-seed variation is reported separately from these conditional test-sequence intervals. This supports a fitting-pool effect at fixed optimization work, without treating the previously used historical pool as globally untouched or equating the new schedule with the original sweep. The fresh-test linear-readout PPL remains $26.27$--$26.33$ versus $24.71$ for native decoding under a common native boundary. The structured error tail persists (\secref{app:confirmatory-readout}).

\begin{figure}[!htbp]
\centering
\includegraphics[width=\linewidth]{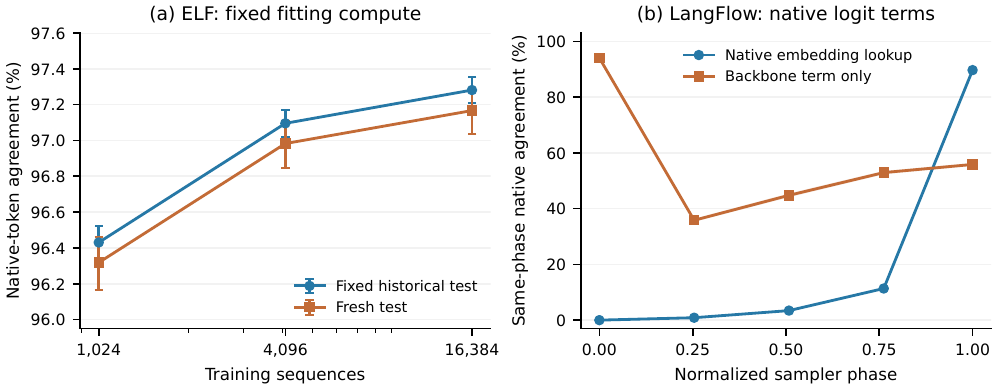}
\caption{New readout controls with the protocol fixed before fitting. (a) ELF mean agreement across three fits per nested training budget, evaluated on the same historical and fresh test sequences. Error bars are conditional 95\% sequence-bootstrap intervals; fitting-seed SD is separate in \tabref{tab:confirmatory-elf}. (b) LangFlow's native embedding lookup and isolated backbone output on the complete model's stored trajectories. Agreement is with the full decoder at the same phase, not with eventual final tokens. Shading gives sequence-bootstrap intervals. Different tokenizers, sequence lengths, and training-token budgets preclude cross-panel score ranking.}
\label{fig:confirmatory-readout}
\end{figure}

The cross-scale confirmation in \figref{fig:mdp-crossscale} shows the same qualitative behavior on ELF-M and ELF-L. With only 1024 generated latents, a linear readout recovers most native tokens ($84.6\%$ on ELF-M and $86.7\%$ on ELF-L), but the PPL gap remains large. Increasing the matching set to 4096 generated latents produces the same jump observed on ELF-B: ELF-M reaches $93.7\%$ agreement with PPL $33.1$ versus native PPL $26.2$, and ELF-L reaches $94.5\%$ agreement with PPL $38.1$ versus native PPL $30.3$. The margin lower tail increases in tandem, with p10 margin moving from roughly $0.41$ at 1k to roughly $1.4$ at 4k on both larger checkpoints. We therefore interpret MDP as an interface-complexity probe rather than as a recommendation to remove the native decoder.

The result sharpens our decoder claim. A direct unembedding baseline is overly stringent, and small learned decoders can be misleading when evaluated only by PPL. MDP gives a cleaner intermediate: a single readout trained to imitate the native interface can recover most decoder-basin decisions from generated latents, and larger fitted pools recover more native decisions in the saved comparisons. The historical comparison changes held-out sequences; the new fixed-test repetition above supplies the corresponding controlled comparison at one fixed optimization budget. Even at 32k, however, the native decoder retains lower measured conditional PPL. The residual-tail analysis below locates the disagreements: the remaining errors concentrate in low-margin, rare, numeric, subword, and token-embedding-hard cases. Most token decisions are therefore linearly recoverable on these final states; the tested linear readouts and fitting budgets leave a structured residual tail. Whether its composition explains text-score differences is tested separately in \secref{sec:prospective-readout}.

\begin{figure}[!htbp]
  \centering
  \includegraphics[width=.96\linewidth]{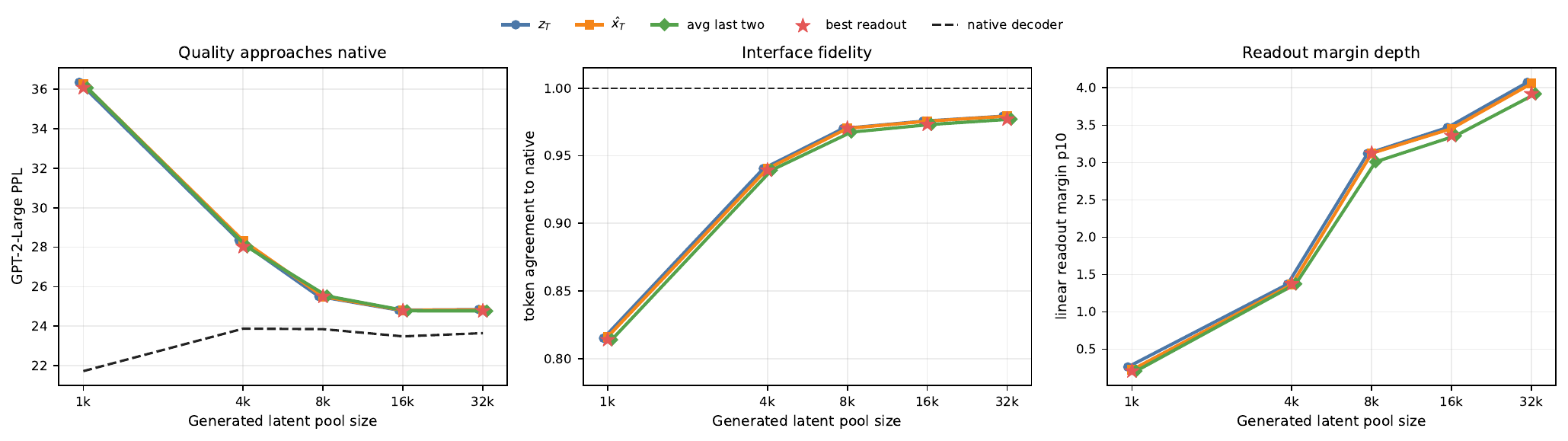}
  \caption{Minimal Decoder Protocol (MDP). Left: GPT-2-Large PPL for 1k--32k generated-sequence matching pools, with the dashed curve showing native decoding on each corresponding evaluation split. The historical sweep selects variants on that split; fixed-variant results and the changing training/evaluation counts appear in \secref{app:readout-selection}. Middle: token agreement between the linear readout and the native decoder. Right: 10th-percentile margin of the learned linear readout, measuring its logit confidence depth rather than a robustness radius. A single linear readout recovers most of the native interface, while native agreement and conditional PPL remain distinct measurements. Tail composition and repair effects are tested in \secref{sec:prospective-readout}.}
  \label{fig:mdp}
\end{figure}

\begin{figure}[!htbp]
  \centering
  \includegraphics[width=.92\linewidth]{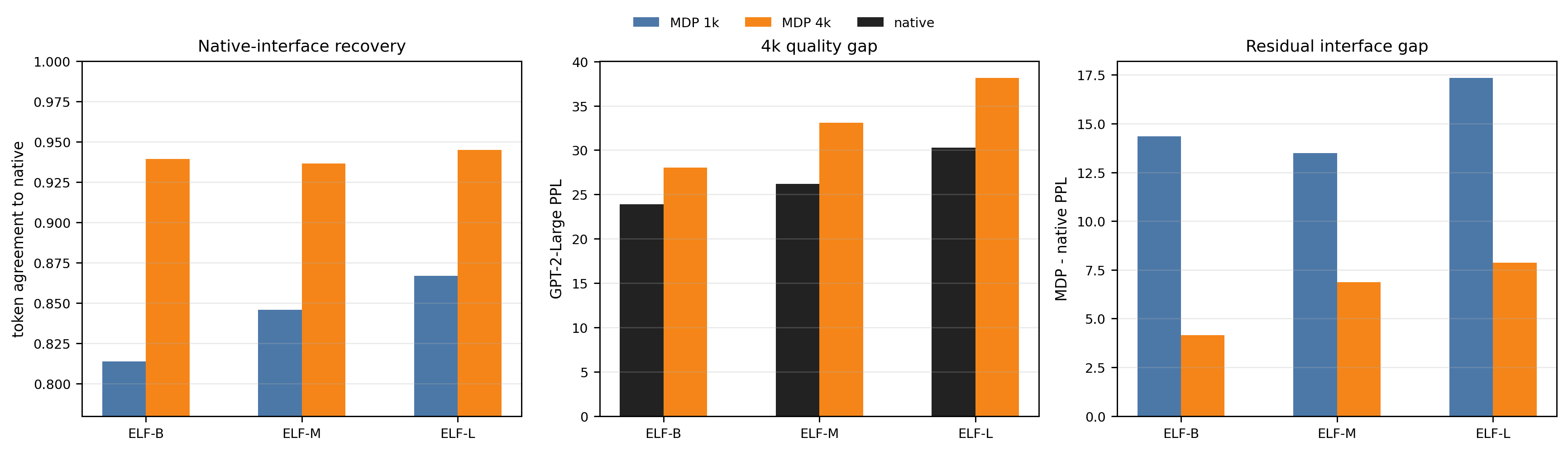}
  \caption{MDP cross-scale confirmation. Left: native-token agreement for 1k and 4k matching sets across ELF-B/M/L. Middle: GPT-2-Large PPL for the corresponding linear readouts where evaluated. Right: residual gap to native decoding. The 4k readout recovers roughly $94\%$ of native tokens across checkpoints, making MDP a probe of decoder-interface complexity rather than a replacement decoder.}
  \label{fig:mdp-crossscale}
\end{figure}

\parhead{ZSBD: a token-alignment probe} MDP still trains a readout. We next ask a stricter minimal-readout question: after ELF has entered the decoder basin, can the state be decoded with \emph{no newly trained readout at all}? In \emph{Zero-Shot Basin Decoding} (ZSBD), we normalize each final latent vector and each frozen T5 token embedding, then choose the nearest token by cosine similarity. This is not the native contextual decoder; it is a frozen Voronoi lookup in the original token-embedding table. It is therefore a strong interface-specific prior, not a prior-free decoder: success means that ELF has transported final states toward its own tied T5 token-embedding geometry.

The result in \figref{fig:zsbd-rbn} is compelling, though it does not replace native decoding. On 4096 ELF-B generated latents, nearest-neighbor lookup from either $z_T$ or the final clean prediction recovers $93.45\%$ of the native decoder's tokens. Evaluated on 1024 samples with GPT-2-Large, ZSBD obtains PPL $32.56$. This is worse than the native decoder and worse than the best MDP readout, but it is far from failure: ELF final states have moved close enough to token-embedding Voronoi cells that a frozen lookup recovers most native decisions. A 16384-sample sanity check gives the same conclusion ($93.34\%$ agreement for $z_T$ and $\hat{x}_T$). The cross-scale check is stable and slightly stronger at larger checkpoints: using the same protocol on 4096 generated latents, ELF-B/M/L obtain $93.45\%$, $95.22\%$, and $95.57\%$ agreement, respectively (\figref{fig:zsbd-crossscale}). ZSBD therefore sharpens the basin-navigation picture. The denoiser is not merely producing arbitrary continuous vectors that require a complex decoder to interpret; it is transporting states toward a token-aligned region. The native decoder still matters for the remaining tokens and for calibration, but the decoder basin has a simple geometric shadow.

We then ask whether this token-embedding alignment is a final-state accident or a trajectory phenomenon. \figref{fig:intermediate-zsbd} decodes intermediate predicted-clean states by the same frozen nearest-neighbor lookup. The agreement rises with basin entry. On 512 ELF-B SDE32 samples, ZSBD-to-native agreement is essentially zero at step 0, reaches $0.61$ around phase $0.52$, $0.87$ around phase $0.77$, and ends at $0.93$. On 512 ELF-M SDE64 samples, it reaches $0.79$ around phase $0.51$ and $0.96$ at the final step. On 256 ELF-L SDE64 samples, it reaches $0.75$ already around phase $0.38$ and ends at $0.96$. The ELF-B ODE32 control follows the same qualitative curve and ends at $0.93$. The mean nearest-neighbor cosine increases in parallel with the native decoder margin. This directly links ZSBD to basin navigation: token-embedding Voronoi alignment emerges during the same process that makes the native decoder readable, and the transition moves earlier under same-interface denoiser scale. The overview in \figref{fig:basin-probe-overview} places this trajectory-level ZSBD curve beside basin-entry timing and held-out BGEE savings, showing that the three probes are different projections of the same basin state, with BGEE measuring timing, ZSBD measuring geometric alignment, and MDP measuring linear recoverability of native labels.

\begin{figure}[!htbp]
  \centering
  \includegraphics[width=\linewidth]{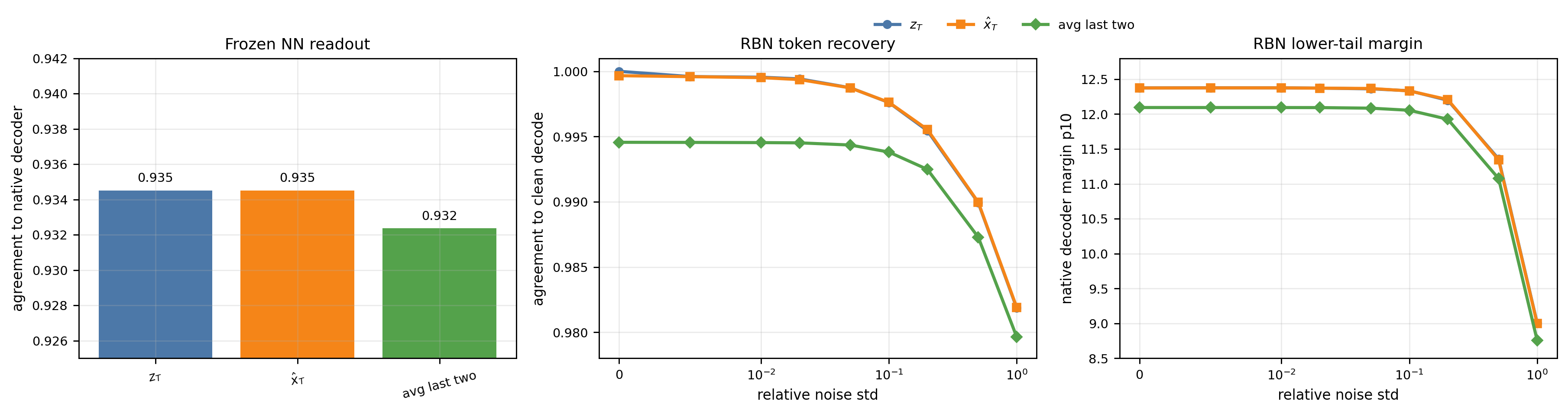}
  \caption{Zero-shot and reverse-basin tests. Left: ZSBD decodes final ELF states by nearest-neighbor lookup in frozen T5 token embeddings and recovers over $93\%$ of the native decoder's token decisions with no newly trained readout, while using the tied T5 token table. Middle/right: RBN corrupts final states and decodes with the native decoder; token recovery and margins degrade slowly under isotropic noise, showing token retention under the tested isotropic perturbations, without implying an attractive force or a certified radius.}
  \label{fig:zsbd-rbn}
\end{figure}

\begin{figure}[!htbp]
  \centering
  \includegraphics[width=.92\linewidth]{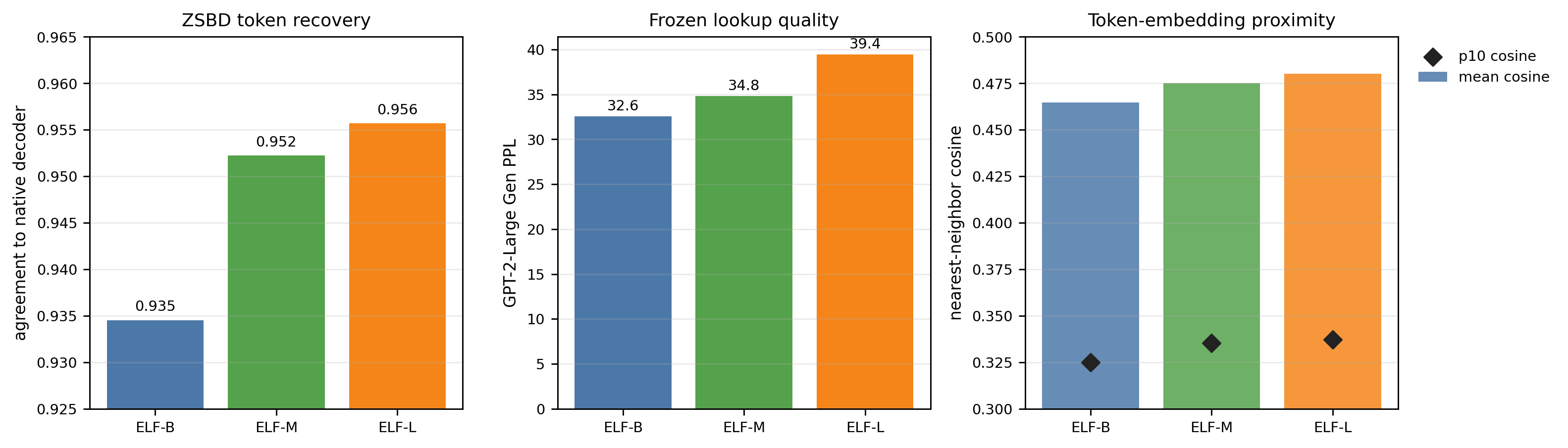}
  \caption{ZSBD cross-scale confirmation. Left: frozen nearest-neighbor agreement with the native decoder, using the T5 token-embedding table and no newly trained readout. Middle: GPT-2-Large PPL of the ZSBD outputs where evaluated. Right: stability of the lookup as sample count increases. Agreement improves slightly with model scale, but PPL remains worse than native decoding; ZSBD is evidence for token-aligned basin entry, not for removing the decoder.}
  \label{fig:zsbd-crossscale}
\end{figure}

\begin{figure}[!htbp]
  \centering
  \includegraphics[width=\linewidth]{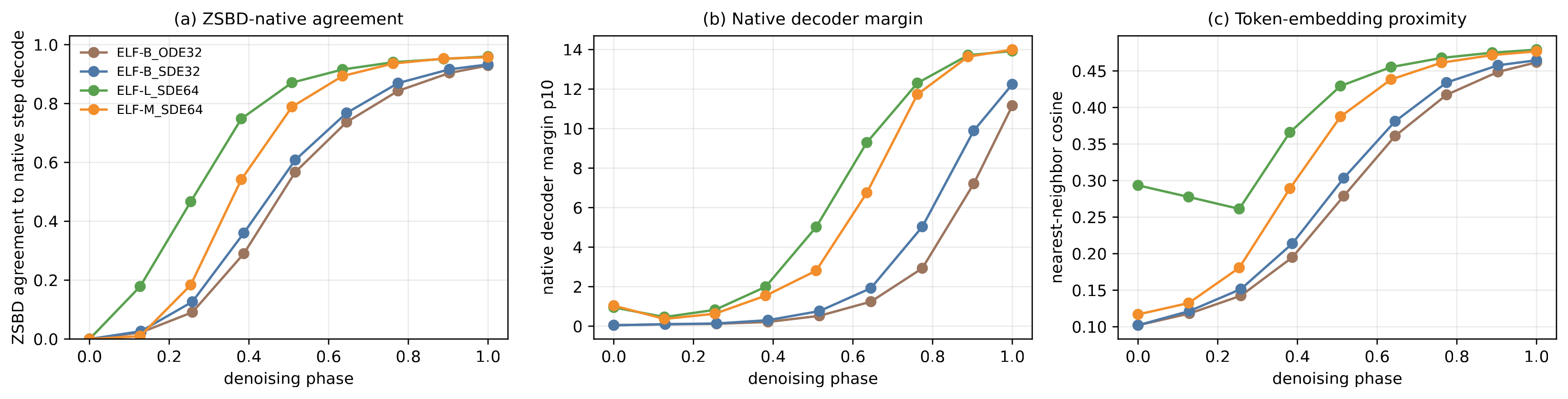}
  \caption{Intermediate-step ZSBD at larger scale. Left: ZSBD-native agreement over normalized denoising phase. Middle: native lower-tail decoder margin over the same phase axis. Right: nearest-neighbor cosine to frozen T5 token embeddings. Frozen token lookup becomes reliable as the trajectory enters the decoder basin, and the transition appears earlier for larger checkpoints and under both SDE and ODE sampling.}
  \label{fig:intermediate-zsbd}
\end{figure}

\begin{figure}[!htbp]
  \centering
  \includegraphics[width=\linewidth]{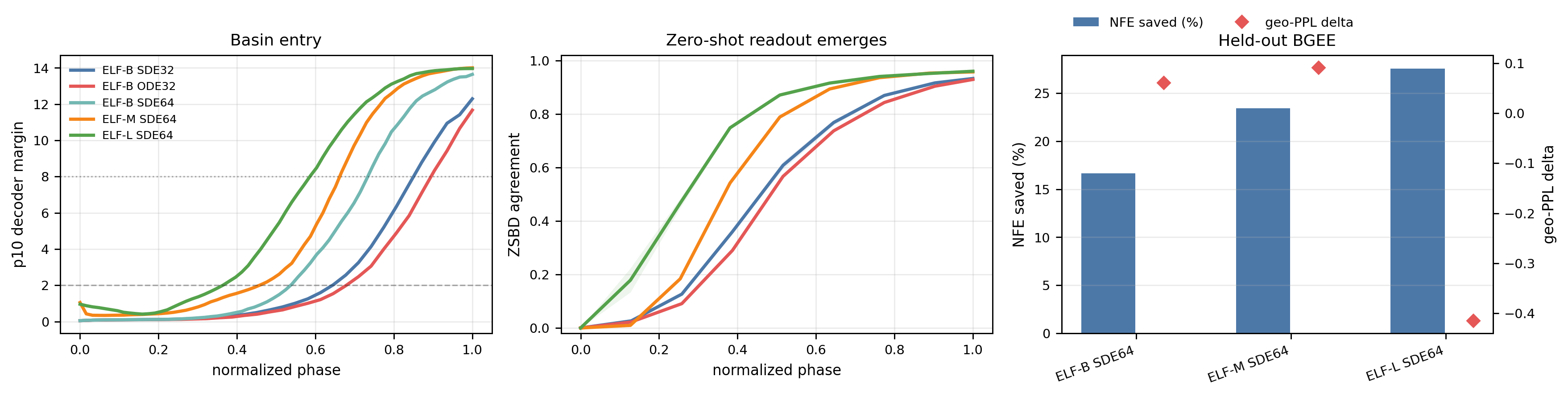}
  \caption{Basin probes in one view. Left: ELF checkpoints enter the decoder basin at different normalized phases. Middle: frozen ZSBD readout emerges along the same trajectory, with shard-level SEM where independent shards are available. Right: a held-out BGEE gate turns earlier basin entry into algorithmic NFE savings under a margin monitor.}
  \label{fig:basin-probe-overview}
\end{figure}

\parhead{Basin geometry after entry}
We also run the reverse experiment. Starting from generated final states, RBN adds isotropic Gaussian noise and decodes with the native ELF decoder. If final-state editing under a Basin-Constrained Decoding (BCD) sentiment direction failed only because the decoder basin were extremely tight, small random perturbations should immediately destroy token recovery. The opposite happens. For ELF-B $z_T$, relative noise standard deviation $0.5$ still gives $98.99\%$ agreement to the clean native tokens, and even noise $1.0$ leaves $98.19\%$ agreement with 10th-percentile margin around $9.0$. The effect is stronger at larger scale: at noise standard deviation $1.0$, ELF-M and ELF-L retain $99.21\%$ and $99.43\%$ agreement, respectively, with 10th-percentile margins above $10$ (\figref{fig:rbn-crossscale}). Thus the decoder basin is not locally fragile under isotropic noise.

\leadin{Direction-dependent stability.} \figref{fig:directional-rbn} compares matched-norm directions on ELF-B final states. Isotropic noise, one fixed random direction, and the BCD sentiment direction all retain about $98.2\%$ token agreement at scale $1.0$, with p10 margins near $8$--$9$. In contrast, the first two PCA directions of the generated latent cloud cross the basin boundary: agreement falls to $64.0\%$ and $44.1\%$, and p10 margins collapse to $0.46$ and $1.14$. The cross-scale check in \figref{fig:directional-rbn-crossscale} confirms the qualitative pattern while showing that the fragile principal direction is checkpoint-dependent. ELF-M is especially fragile along its first PCA direction (agreement $0.015$ at scale $1.0$), whereas ELF-L is more fragile along its second PCA direction (agreement $0.489$); isotropic and sentiment directions remain above $0.99$ on the same checkpoints. This raises a natural question: is the extreme ELF-M PC1 a linguistic attribute direction, or a decoder-interface direction? The follow-up audit in \figref{fig:elfm-pc1-axis-audit} answers the question in favor of the second interpretation: ELF-M's fragile PC1 is not a clean sentiment, tense, or formality axis, but aligns with the frozen T5 token table's punctuation and token-boundary anisotropy. This explains why decoder-basin sentiment editing is a boundary result rather than a contradiction. Late ELF states are robust under many perturbations, but high-variance data-manifold directions can leave the decoder basin. The final state is therefore excellent for stopping and readout probes, but poor as a generic semantic editing canvas.

\figref{fig:inside-decoder-basin-summary} summarizes the mechanism that emerges from these probes. Entry is measurable: margins and frozen token lookup rise during denoising. Readout is simple inside the decoder basin: a single linear readout recovers nearly all native token decisions on the generated manifold. But the basin is not a license for arbitrary latent editing. It retains token decisions under the tested random and simple sentiment perturbations but is fragile along certain high-variance directions. Retention means the perturbed state receives the same token label; the decoder does not physically pull it back toward its original latent position. Native-decision disagreements are concentrated in a structured residual tail, including rare and out-of-domain tokens. Their concentration alone does not establish which corrections improve an external text score.

\begin{figure}[!htbp]
  \centering
  \includegraphics[width=\linewidth]{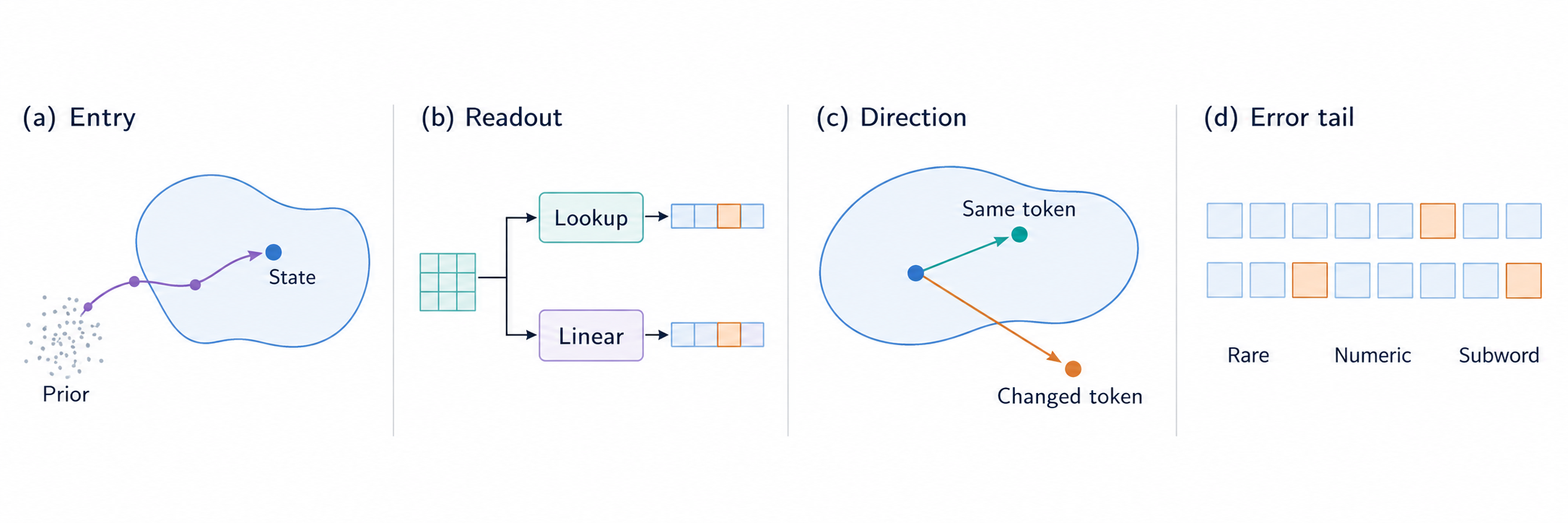}
  \caption{Conceptual synthesis after entry. (a) A trajectory approaches a decoder-readable region. (b) Frozen lookup and a linear readout probe the simplicity of native token realization on generated states. (c) Perturbations can retain or change a token depending on direction; this illustrates decision-boundary sensitivity, not an attractive force or certified radius. (d) Difficult positions include rare, numeric, and subword tokens; their highlighted proportions are illustrative, not measured frequencies. Quantitative readout scaling, anisotropy, and residual-tail results remain in \figref{fig:mdp}, \figref{fig:directional-rbn}, and \figref{fig:fragile-tokens}. }
  \label{fig:inside-decoder-basin-summary}
\end{figure}

\leadin{Residual tail.} The ZSBD/MDP gap is concentrated in a fragile lower tail rather than spread uniformly. \figref{fig:fragile-tokens} uses a separate refitted 32k-pool residual audit with a 4096-sequence evaluation protocol and $4.19$M valid token positions. Its readout fit, evaluation scope, and token-counting protocol differ from the fixed-variant 8192-sequence agreement check in \secref{app:readout-selection}. The single linear readout agrees with the native decoder on $97.89\%$ of tokens, but its errors are not random. The native 10th-percentile margin over all tokens is $12.49$, whereas the 10th-percentile margin among MDP-mismatched tokens is only $0.37$. Very-low native-margin tokens have a $9.69\%$ MDP error rate, compared with $0.45\%$ or less once the native margin leaves the bottom quintile. ZSBD-hard tokens are also MDP-hard: when frozen token-embedding lookup is wrong, the MDP error rate is $23.36\%$, versus only $0.60\%$ when ZSBD is correct. Frequency and type give a complementary view. Rare tokens have an $8.06\%$ MDP error rate, numbers have $8.69\%$, and subwords have $5.21\%$, while punctuation and common word-start tokens are much easier. Position bins are nearly flat. The 64-sample cross-scale tail supplement in \tabref{tab:tail-scale-supplement} refines this picture: rare and numeric ZSBD errors fall sharply from ELF-B to ELF-M/L, while a very-high-frequency ZSBD floor persists. Most native token decisions are recoverable by a simple map on these final states; the residual disagreements concentrate in low-margin, token-embedding-hard, and lexical-boundary cases. This diagnoses fidelity, not the effect of each category on PPL. The equal-error intervention in \secref{sec:prospective-readout} finds that native-margin priority can perform worse than random repair despite this enrichment.

\begin{figure}[!htbp]
  \centering
  \includegraphics[width=.82\linewidth]{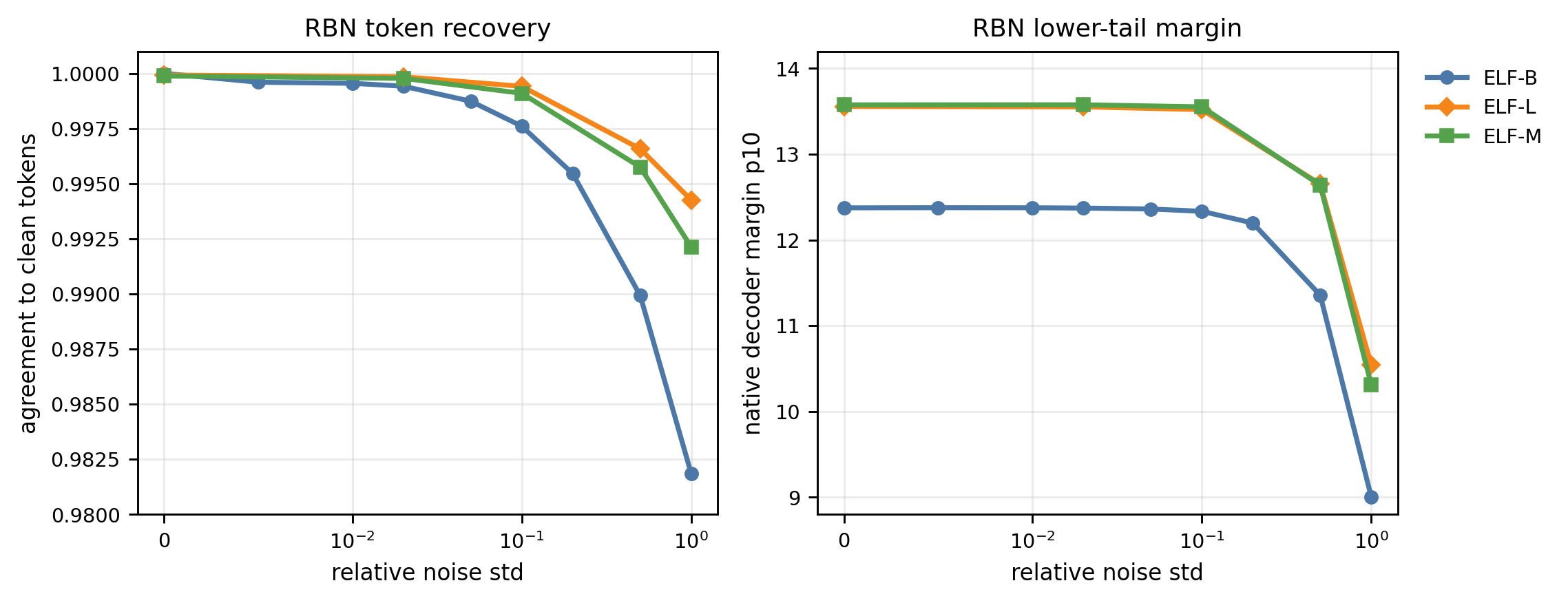}
  \caption{RBN cross-scale confirmation. Left: token agreement after adding isotropic noise to final ELF states. Right: corresponding native decoder margin under the same perturbations. Across ELF-B/M/L, token identity remains stable under large random perturbations, showing tolerance along the sampled isotropic directions; this does not exclude nearby boundaries along other directions.}
  \label{fig:rbn-crossscale}
\end{figure}

\begin{figure}[!htbp]
  \centering
  \includegraphics[width=\linewidth]{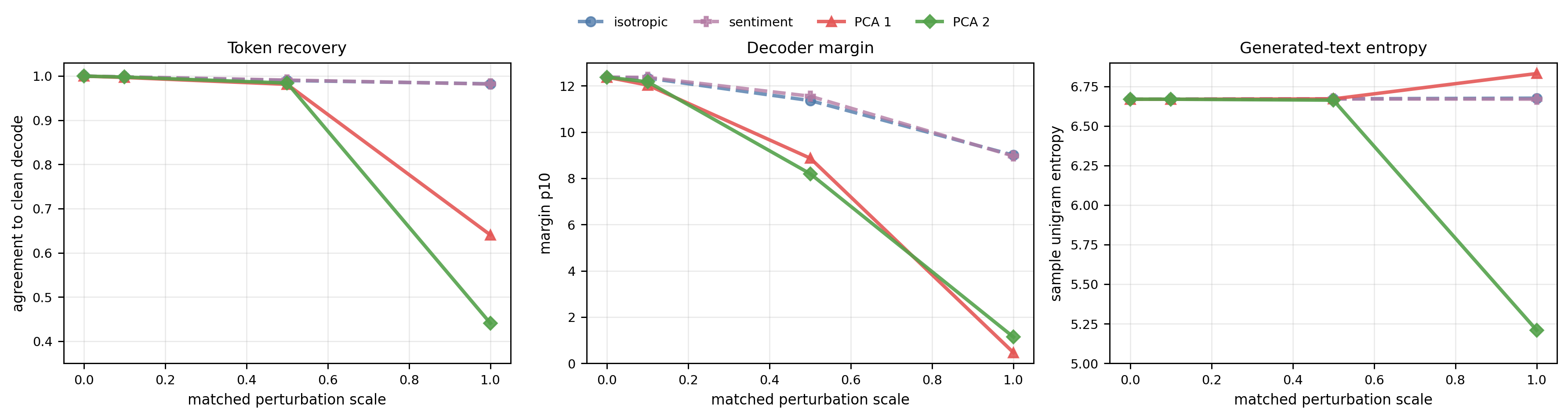}
  \caption{Directional reverse-basin navigation. Left: token agreement under isotropic and matched-norm directional perturbations. Middle: native margin under the same directions. Right: decoder entropy or quality proxy for the perturbed decodes. The final ELF basin is robust to isotropic and random directions but much narrower along dominant PCA directions, revealing anisotropy rather than a uniformly wide or tight basin.}
  \label{fig:directional-rbn}
\end{figure}

\begin{figure}[!htbp]
  \centering
  \includegraphics[width=\linewidth]{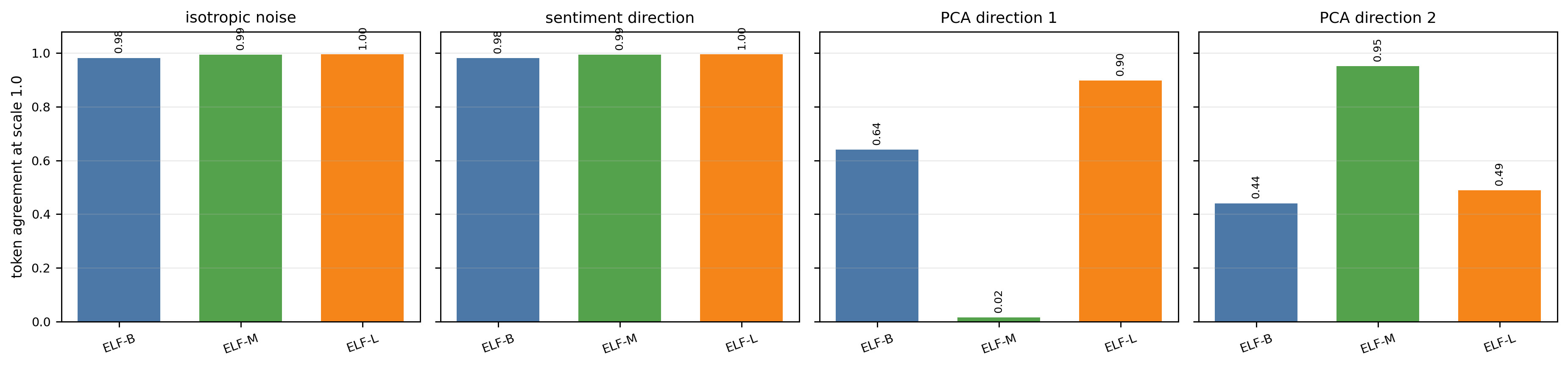}
  \caption{Cross-scale directional RBN. From left to right, the panels compare isotropic noise, a sentiment direction, the first PCA direction, and the second PCA direction at matched perturbation scale across ELF-B/M/L. Isotropic and sentiment directions remain mostly inside the decoder basin, while at least one high-variance PCA direction sharply reduces native-token agreement. The fragile direction changes with checkpoint scale, but anisotropy persists.}
  \label{fig:directional-rbn-crossscale}
\end{figure}

\begin{figure}[!htbp]
  \centering
  \includegraphics[width=\linewidth]{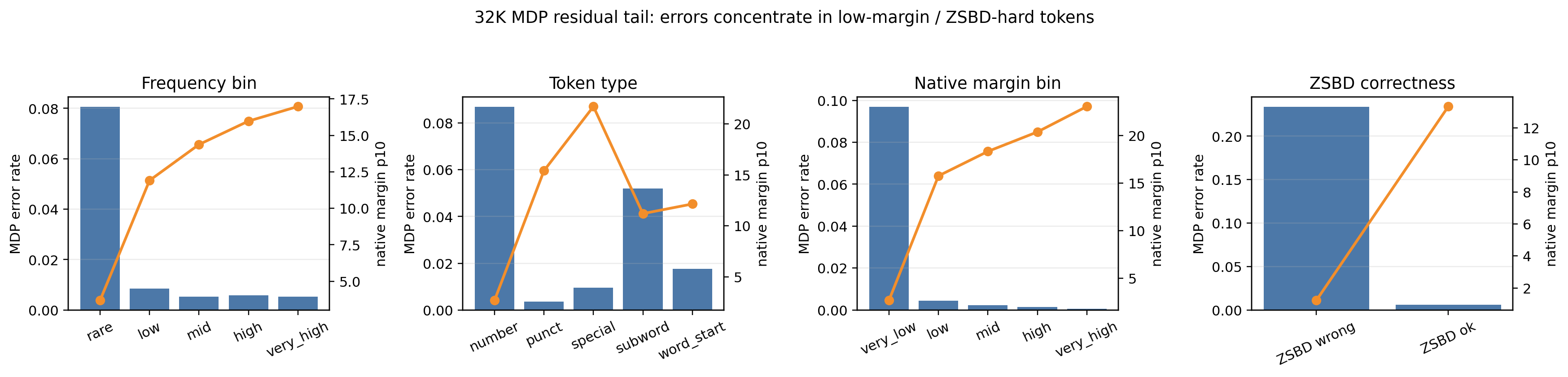}
  \caption{32k MDP residual-tail analysis. Left: error rate by native margin bucket. Middle-left: overlap between MDP errors and ZSBD-hard tokens. Middle-right: token-frequency and token-type buckets. Right: position bucket effects. A single linear readout recovers $97.89\%$ of native tokens on 4096 held-out generated samples, but the remaining errors concentrate in low-margin, ZSBD-hard, rare, numeric, and subword cases.}
  \label{fig:fragile-tokens}
\end{figure}

\leadin{Same-token decoder-score gap.} A paired audit on 1024 generated ELF-B samples compares decoder scores on generated states and clean re-encodings of the same tokens. This controls token identity while testing whether their lower-tail confidence differs; it does not measure distance to a clean-state manifold. For each native-decoded token sequence, we feed the exact token ids through the frozen T5 encoder, apply ELF's latent normalization, and compare native-decoder target margins for this clean interface state, the last predicted clean state $\hat{x}_T$, the average of the last two clean predictions, and the final latent $z_T$. \figref{fig:clean-generated-basin-gap} shows a lower-tail gap, not a global confidence collapse. Clean T5 states recover the same target tokens at $99.996\%$ and have much deeper target-margin lower tails: their p01, p05, and p10 margins exceed $z_T$ by $7.98$, $4.53$, and $1.56$ margin units, respectively. The median gap is negative ($-0.85$), so generated states can be more confident on easy positions while remaining shallower in the fragile tail. Thus the two state sources have different lower-tail decoder scores despite nearly identical token targets. This is consistent with a residual calibration gap, but neither establishes manifold membership nor proves that re-encoding would repair generation.

\begin{figure}[!htbp]
  \centering
  \includegraphics[width=.92\linewidth]{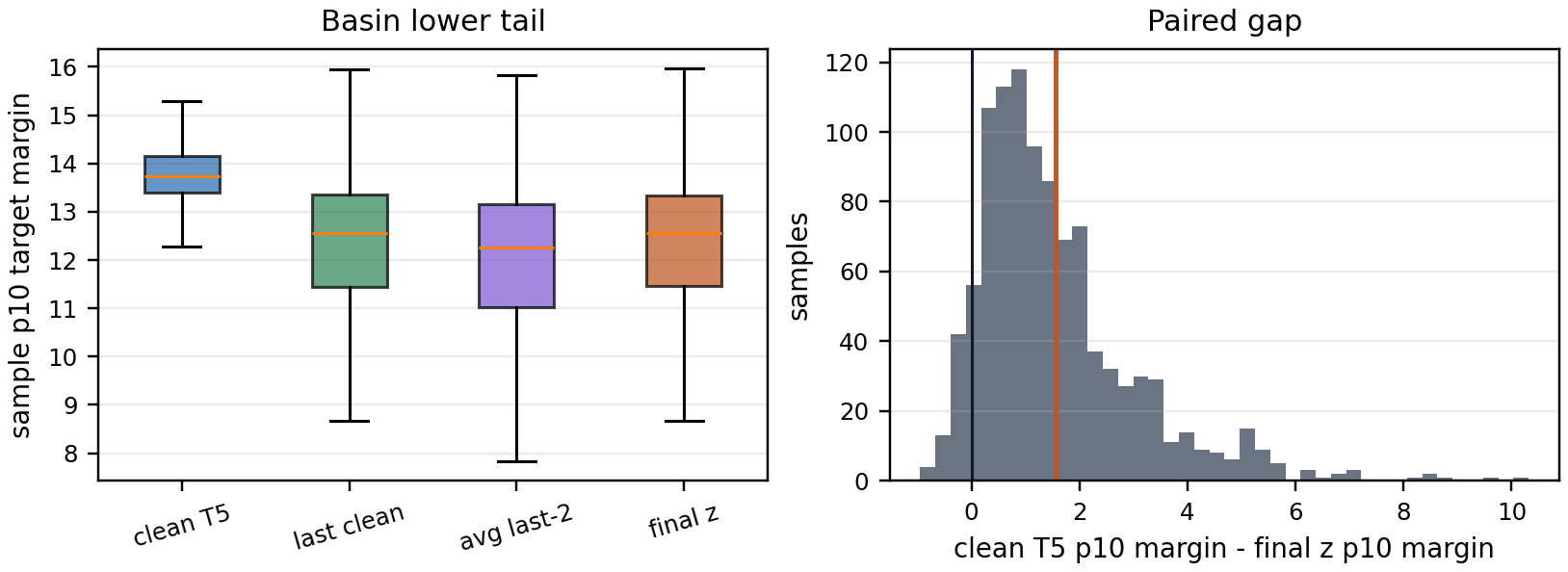}
  \caption{Paired clean-vs-generated decoder-basin depth. For the same generated token sequences, clean T5 interface states have a tighter and deeper lower tail than generated final states. Left: sample-level p10 target margin for clean T5, last predicted clean state, average of the last two clean predictions, and final latent. Right: paired p10 gap between clean T5 and $z_T$; the orange line marks the mean gap ($1.56$) and the black line marks zero. The effect is concentrated in the lower tail rather than the median, locating a residual confidence difference in the tail. This correlation alone does not identify the cause of the PPL gap.}
  \label{fig:clean-generated-basin-gap}
\end{figure}

\leadin{Tail-gating check.} A final check asks whether this residual tail is predictable enough to route selectively, rather than replacing the native decoder everywhere. In a separate readout fit using the 32k matching pool, we form hybrid outputs: use the linear readout for bulk tokens, but route suspected tail positions to the native token decision. The oracle agreement control routes only the MDP errors; practical gates use only readout-derived signals such as low linear-readout confidence, MDP--ZSBD disagreement, and predicted rare/numeric/subword token types. \figref{fig:rtgr} shows the result. MDP alone has PPL $24.82$ and $98.01\%$ agreement on the 4096-sample held-out split. Routing the lowest-confidence $5\%$ of token positions covers $94.6\%$ of MDP errors and improves PPL to $23.85$; routing the lowest-confidence $10\%$ covers $98.8\%$ of errors and reaches PPL $23.82$, close to the oracle-routed output ($23.80$) and native decoding ($23.65$). The broader practical gate routes $33.9\%$ of positions and reaches PPL $23.81$, but this is less parsimonious than the confidence gate. The routing fractions use confidence quantiles on this evaluation set, without an independently calibrated deployment threshold. Agreement is measured only at native-valid positions; EOS and later output positions can differ, so even oracle routing with $100\%$ measured agreement need not reproduce the native text length or PPL. Low affine-readout confidence locates many disagreements in this historical fit; differing stopping boundaries prevent isolating its conditional score benefit. Selecting unknown disagreements by affine confidence is distinct from ranking correction consequences by native margin among already known disagreements (\secref{app:deep-margin-policies}). The latter requires oracle information.

\begin{figure}[!htbp]
  \centering
  \includegraphics[width=\linewidth]{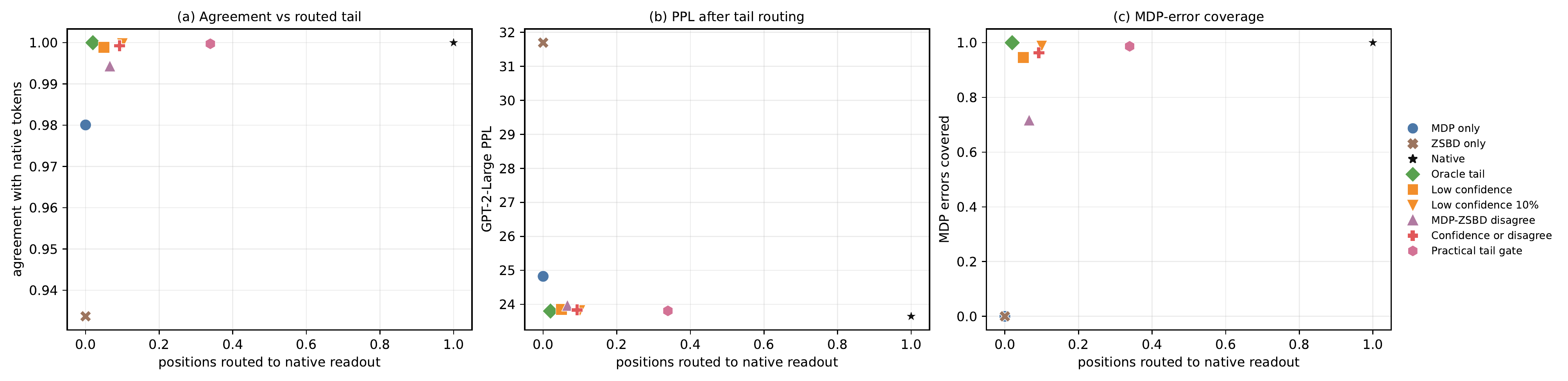}
  \caption{Tail-gated readout as a tail-calibration check. Left: native-token agreement versus the fraction of token positions routed from the linear readout to the native decision. Middle: GPT-2-Large PPL under the same routing policies. Right: fraction of MDP errors covered by each gate. Low MDP-confidence gates route only $5$--$10\%$ of token positions yet cover most MDP errors and nearly close the PPL gap, showing concentrated native-decision disagreements in this historical fit. These evaluation-set gates do not establish transferable thresholds or end-to-end savings.}
  \label{fig:rtgr}
\end{figure}

This closes the constructive probe part of \secref{sec:experiments}. The remaining subsections are calibration and boundary checks: they test whether the basin story survives multi-metric evaluation, a learned-latent system, training-side perturbations, and the simple explanations ruled out at the end of the section.

\subsection{PPL, MAUVE, and Decoder Calibration}
\label{sec:decoder-calibration}

\leadin{Metric cross-checks.} \figref{fig:boundary} adds objective checks that prevent overclaiming. First, MAUVE does not perfectly track GPT-2 PPL, but the 2048-sample MAUVE estimate supports the qualitative schedule story: front-loading improves MAUVE for ELF-B ($0.956$ vs.\ $0.950$ for fixed SC=3) and ELF-M ($0.961$ vs.\ $0.955$ for fixed SC=5 and $0.945$ for fixed SC=3). Second, entropy-matched decoder calibration shows that low PPL can be a decoder artifact. A learned linear decoder produces low-PPL argmax text, but only moderate token agreement with the official decoder; when sampled at a temperature chosen to approach the official decoder's entropy, PPL worsens sharply. Third, the Cola-DLM VAE noise sweep shows a narrow decoder-compatible basin: clean recovery is almost perfect, but token recovery collapses rapidly around mid-noise corruption levels, exposing a sharp boundary in decoder compatibility.

The decoder results refine the direct-unembedding baseline. A direct linear map may be overly stringent, so we train small learned decoders from final ELF latents to the official decoder's argmax tokens. This is a more targeted baseline because the target interface is the official decoder itself. With a 1024-latent matching set, the two-layer multi-layer perceptron (MLP) decoder obtains lower GPT-2-Large PPL than the official decoder ($16.29$ vs.\ $18.25$), but token agreement to the official decoder is only $0.737$. After entropy-matched sampling, its PPL rises to $46.34$ and agreement drops to $0.631$. \figref{fig:ppl-entropy-scatter} shows the per-sample operating regions behind this effect: schedules and decoders occupy clouds in the PPL--entropy plane rather than a single scalar ranking. This does not contradict the later 32k MDP result. MDP uses many generated final states to imitate the native argmax interface and is judged primarily by agreement; the small-decoder calibration test shows that low external PPL alone can validate the wrong operating region. The result is more subtle than direct failure: small decoders can reduce PPL, but they do so by changing the operating region. This supports the claim that the official decoder is not an incidental readout. It defines the basin in which the denoiser's final states are meaningful.

We also run a model-level entropy calibration against GPT-2-small (\figref{fig:entropy-calibrated-gpt2}). This is a conservative check prompted by the known PPL failure mode in diffusion LMs. We decode the same ELF final logits under temperatures $T\in\{0.4,0.6,0.8,1.0,1.2,1.5\}$ and generate GPT-2-small samples under a temperature sweep. ELF's final decoder is already extremely sharp: for front-loaded ELF-B, increasing decoder temperature from $0.4$ to $1.5$ changes GPT-2-tokenized unigram entropy only from $6.891$ to $6.895$, while token agreement to argmax remains above $0.991$. At the closest GPT-2-small entropy point ($T=0.8$, entropy $6.858$), GPT-2-small has lower GPT-2-Large PPL ($12.23$ vs.\ $24.90$ for front-loaded ELF-B) and higher distinct-2 ($0.488$ vs.\ $0.324$). At lower GPT-2 temperatures, PPL becomes extremely small but repetition rises sharply, reproducing a PPL artifact in a simple AR baseline. This result prevents a misleading claim: our front-loaded schedule is not an entropy-matched replacement for AR generation. It is evidence that ELF trajectories expose reliability signals, while PPL/entropy frontiers must be reported explicitly.

\begin{figure}[!htbp]
  \centering
  \includegraphics[width=\linewidth]{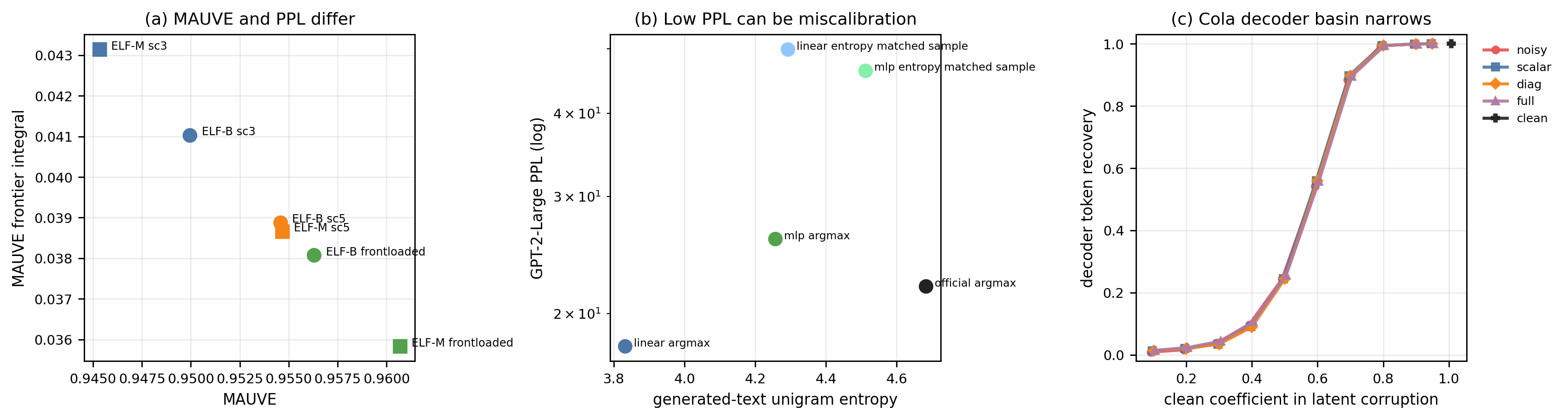}
  \caption{Multi-metric and boundary evidence. (a) MAUVE and PPL can rank schedules differently. (b) A small decoder can obtain low PPL by miscalibration; entropy matching exposes the issue. (c) Cola-DLM latents have a narrow decoder-compatible basin under noise.}
  \label{fig:boundary}
\end{figure}

\begin{figure}[!htbp]
  \centering
  \includegraphics[width=\linewidth]{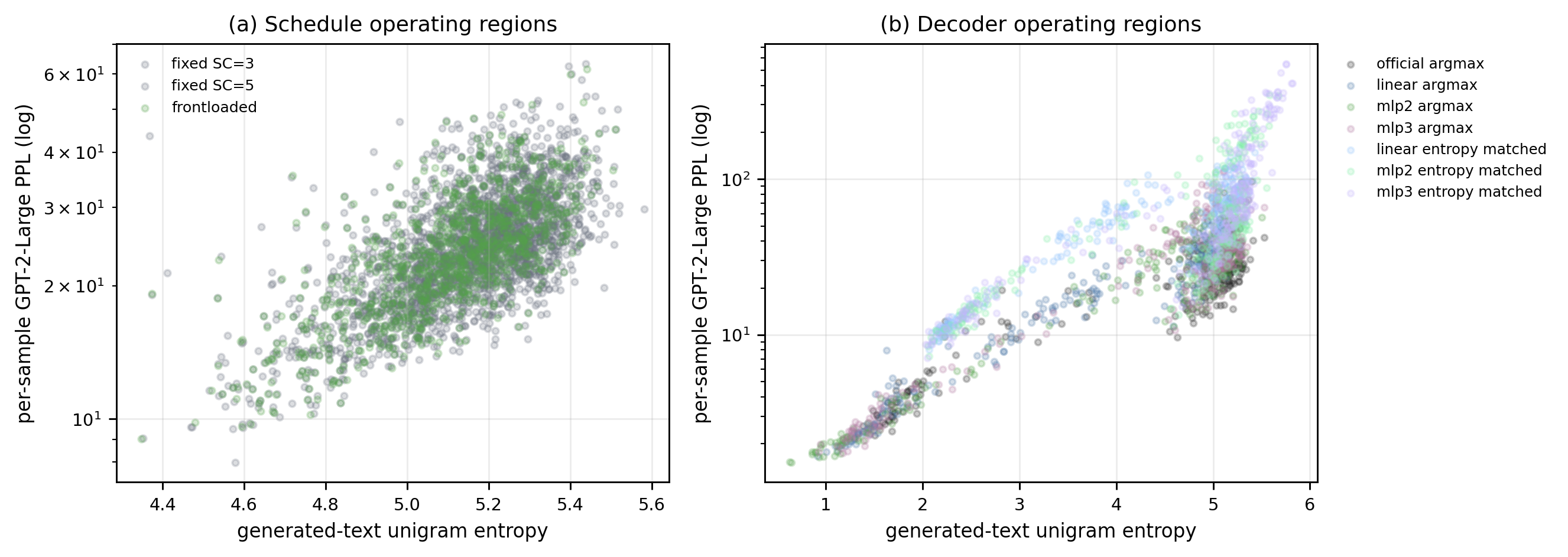}
  \caption{Per-sample operating regions. Left: schedule choices move samples along a PPL-entropy cloud rather than a single scalar. Right: learned decoders occupy different entropy/PPL regions from the official decoder, so low PPL is not sufficient evidence of decoder fidelity.}
  \label{fig:ppl-entropy-scatter}
\end{figure}

\begin{figure}[!htbp]
  \centering
  \includegraphics[width=.82\linewidth]{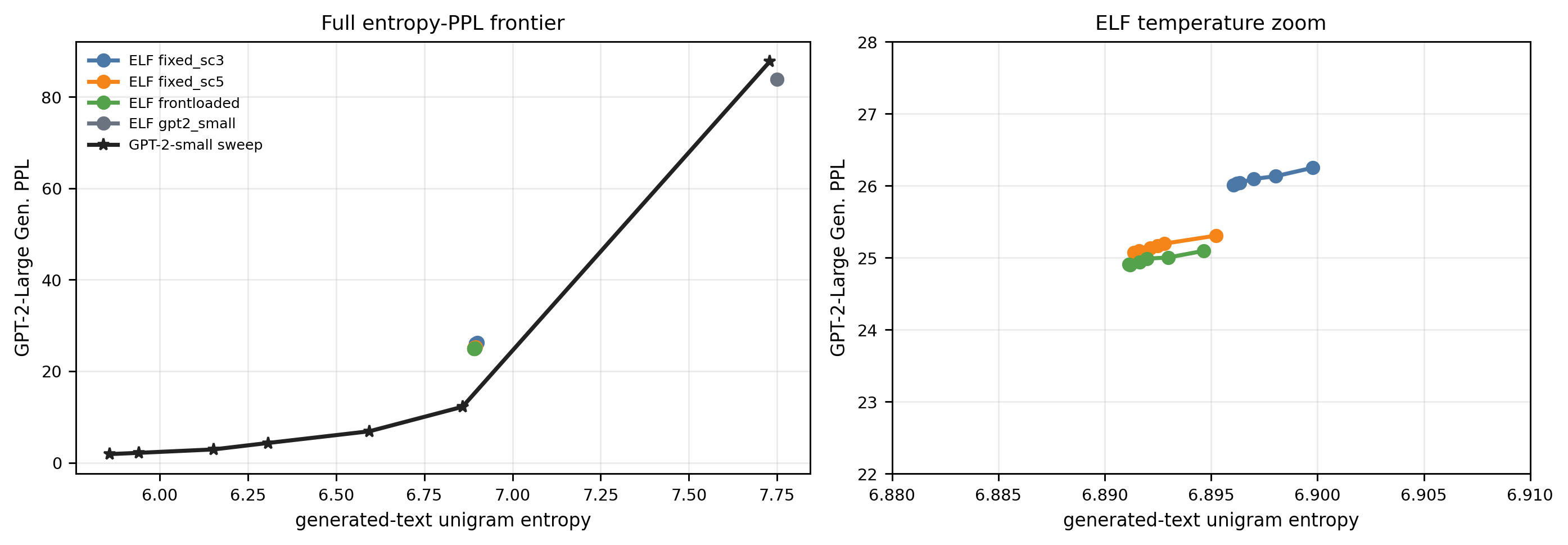}
  \caption{Entropy-calibrated ELF vs.\ GPT-2-small. Left: full temperature sweep in the PPL--entropy plane. Right: zoom near the closest entropy-matched operating points. ELF final decoding is nearly argmax-like across temperatures; the comparison is not an ELF win claim, but a demonstration of the frontier hidden by PPL-only reporting.}
  \label{fig:entropy-calibrated-gpt2}
\end{figure}

\figref{fig:margin-basin} gives the aggregate margin sweep illustrating the signed-reference-margin identity used in \thmref{thm:decoder-margin}. Clean final latents have token recovery $0.989$ and a positive-margin fraction $0.989$. With only a mild corruption, token recovery drops to $0.870$ at $t=0.95$ and $0.767$ at $t=0.90$. At $t=0.50$, recovery is $0.532$, and the 10th percentile margin is strongly negative. A first-order boundary calculation on sampled token positions gives a median distance proxy of $3.00$ but a 10th percentile of only $0.063$, suggesting heterogeneous local sensitivity. These are linearized estimates rather than exact nonlinear boundary distances; the independent flip audit in \secref{sec:calibration} directly tests predictive utility.

\begin{figure}[!htbp]
  \centering
  \includegraphics[width=.70\linewidth]{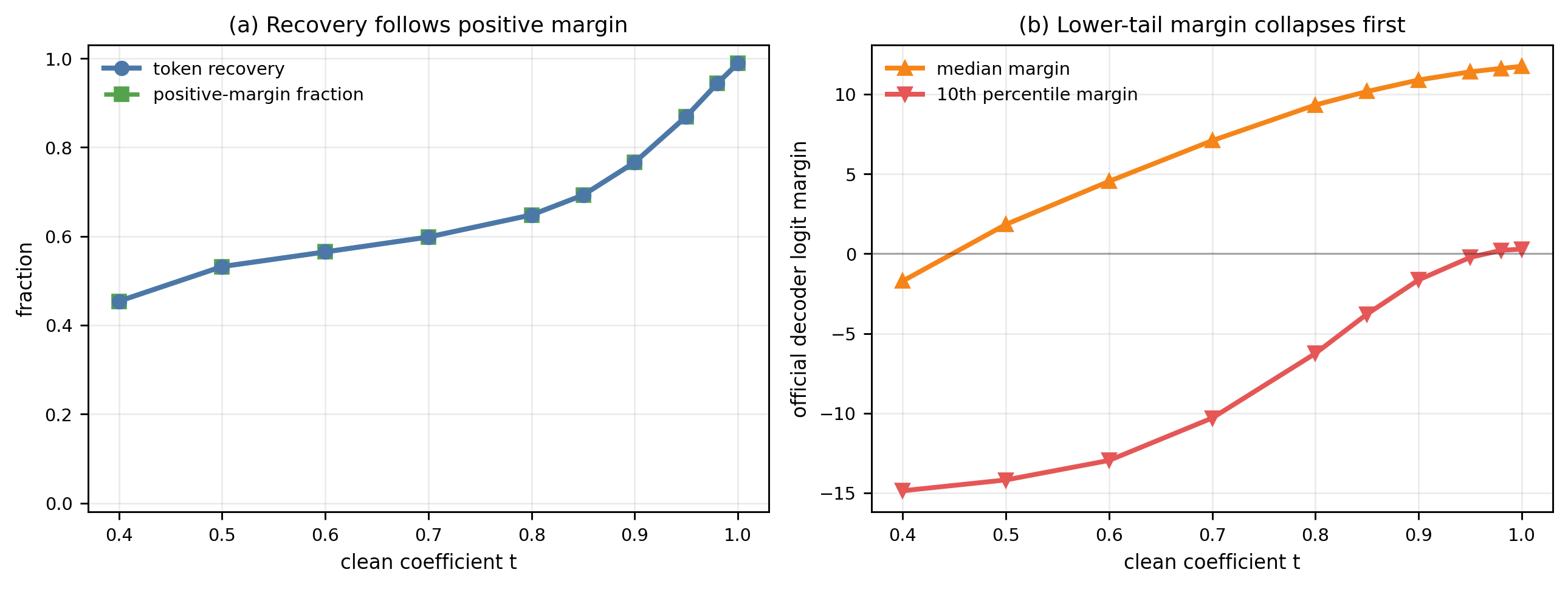}
  \caption{Decoder-margin basin test. Left: token recovery and positive-margin fraction as final ELF latents are corrupted away from the clean endpoint. Right: median and 10th-percentile native logit margins under the same corruption. Recovery and signed reference-token margin collapse together; this is a diagnostic correspondence, not an empirical estimate of the theorem's certified radius.}
  \label{fig:margin-basin}
\end{figure}

The per-token view is even more diagnostic. \figref{fig:margin-scatter} bins $663{,}552$ token positions by clean margin and latent error norm. Low-margin tokens collapse quickly: recovery falls from $0.94$ for errors below $0.96$ to below $0.10$ for errors above $11.9$. High-margin tokens remain stable under much larger errors: recovery is still $0.981$ in the largest error bin. This is consistent with the variables in \thmref{thm:decoder-margin}, but does not measure its uniform Lipschitz bound. Latent error alone does not determine recovery: clean margin provides additional information within this fixed decoder.

\begin{figure}[!htbp]
  \centering
  \includegraphics[width=.78\linewidth]{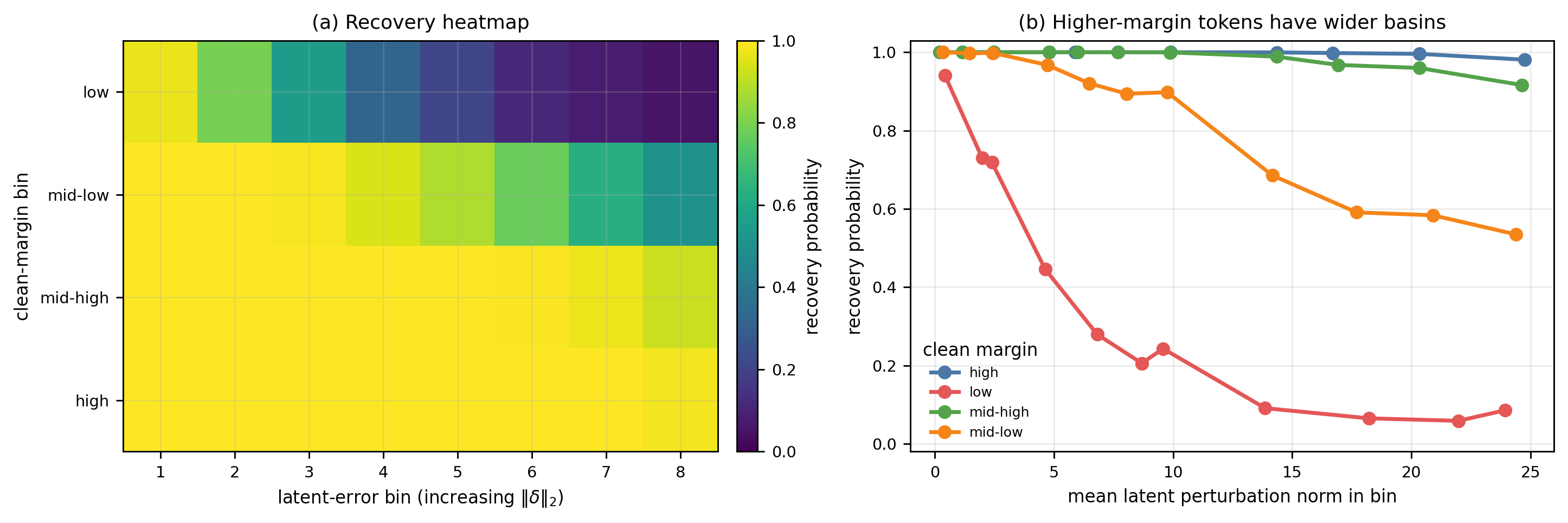}
  \caption{Per-token decoder-margin evidence. Left: token recovery as a function of latent error and clean decoder margin. Right: empirical recovery curve grouped by margin/error regime. For the same latent error, high clean margins preserve tokens while low margins collapse, motivating a joint margin/sensitivity diagnostic without validating a certified Lipschitz radius.}
  \label{fig:margin-scatter}
\end{figure}

\subsection{Cola-DLM Boundary Case}
\label{sec:cola-boundary}

\leadin{Boundary case.} Cola-DLM provides an architectural contrast rather than a parallel benchmark. It learns a VAE latent space and trains a DiT prior over it, whereas ELF uses a frozen T5 encoder with a shared denoiser and decoder. Despite these differences, both systems face the same diagnostic question: can a continuous state be safely handed to a decoder after a noisy trajectory? In Cola, clean VAE latents reconstruct with very high fidelity---a clean reconstruction metric alone would suggest the interface is safe. The noise sweep reveals otherwise: recovery remains high near $t=0.9$ but collapses rapidly between $t=0.7$ and $t=0.5$. A balanced four-GPU DiT trajectory sweep sharpens the interpretation. Across 512 samples per classifier-free guidance (CFG)~\citep{classifierfree2022} setting, decoder entropy drops from about $5.7$ nats to $0.49/0.32/0.15/0.19$ for CFG $0/1/3/7$, and CFG $3$--$7$ reaches a stable entropy-$\leq1$ basin around steps $12$--$13$. But teacher-forced target-token recovery under our diagnostic remains only $0.8$--$4.4\%$, so this is prior transport into a decoder-confident basin, not recovery of our held-out teacher target. This teacher-forced diagnostic does not evaluate Cola-DLM's full free-generation benchmark quality. A sampled continuation is not required to reconstruct a particular held-out target, so low target overlap here is not by itself a generation failure or a failed diffusion model; the controlled corruption-reconstruction sweep supplies the decoder-robustness evidence. The boundary case reinforces the broader point: latent diffusion language models should report decoder-basin robustness under realistic noise levels, not only clean reconstruction fidelity. Matched scalar, diagonal, and full linear MSE denoisers do not remove the boundary: at $t=0.6$, token recovery rises only from $0.541$ to at most $0.561$, and at $t=0.5$ it remains at most $0.256$ (\tabref{tab:cola-linear-denoise}). Linear conditional-mean recovery can reduce latent error without learning semantic prior transport or the decoder's nonlinear boundary geometry.

\subsection{Training-Side Validation: Corruption Recovery and Margin Depth}
\label{sec:causal-validation}

\leadin{Training-side prediction.} The basin view makes a training-side prediction: decoder training should not only classify clean latents, but widen the decoder basin around plausible noisy latents. The released ELF training recipe~\citep{elf2026}, implemented in PyTorch, already contains such a mechanism. In OWT training, $20\%$ of examples take a decoder branch, and decoder inputs are sampled by mixing clean T5 latents with scaled Gaussian noise. We test this directly with 5k/10k ablations on four RTX 3090 GPUs, targeted 50k controls, and two independently initialized matched pairs trained for five full OWT passes on four A800 GPUs.

The result is interventional in the limited but useful sense that we vary one training ingredient at a time. With no decoder-input noise, clean recovery is high ($0.951$), but retention collapses under mid-noise: only $47.0\%$ of clean recovery remains at $t=0.6$. With the official-like noise scale of $5$, clean recovery is still high ($0.995$), but retention rises to $94.4\%$ at $t=0.6$ and $81.7\%$ at $t=0.5$. Too much noise is not monotone: scale $10$ weakens the interface. Varying the decoder-branch frequency separates transport from basin widening. The tested pure-denoising run fails native decoding; decoder-only training achieves the strongest tested native/OOD corruption recovery in this short-run contrast but learns no flow objective and collapses under end-to-end generation. The five-epoch matched pairs preserve this separation at every saved checkpoint, so it is not an artifact of the original short runs. In this training setup, the mixed objective couples the two competencies (\secref{sec:compatibility-transport}; full controls in \secref{app:transport-controls}).

We also test a tempting repair suggested by \thmref{thm:decoder-margin}: add an explicit hinge-style margin penalty to the decoder CE branch. In this limited-scale exploration, the penalty preserves clean recovery and is optimized during training, but after both 5k and 10k steps it still does not beat the noisy-CE baseline on mid-noise retention or clean margin tail. The result does not rule out all margin objectives; it shows that, in these short runs with the official implementation, directly maximizing a scalar logit gap is not a drop-in replacement for training the decoder on the neighborhood the denoiser must actually visit. This completes the training-side check: decoder-input noise widens the basin in practice, while the tested margin penalty does not, reinforcing the distinction between training a decoder on the local neighborhood of expected noisy latents and directly maximizing per-sample scalar margins.

\subsection{Summary of Ruled-Out Explanations}
\label{sec:ruled-out-explanations}

Within the tested settings, the experiments reject eight simple explanations. First, continuous diffusion is not explained by Gaussianity alone, because covariance controls can be denoisable without being recoverable. Second, it is not explained by token identity alone, because word-shuffled embeddings preserve token identities but fail the order-sensitive diagnostics. Third, the fitted token-wise readouts do not exactly reproduce native decisions or scores on the tested states; fixed-focal context interventions separately show that the full native mapping can depend on other positions. Fourth, sampler gains should not be summarized by PPL alone, because MAUVE, entropy, and repetition expose different frontier movements. Fifth, clean latent reconstruction is insufficient, because Cola's decoder recovery degrades under the tested latent corruptions and linear MSE denoising repairs it only marginally. Sixth, a wide decoder basin is not sufficient transport: decoder-only training improves native and OOD corruption recovery yet cannot map the Gaussian prior into diverse text. Seventh, prescribed noise-coefficient preservation is not sufficient for quality: CPS restores the coefficient ratio but changes the checkpoint's effective transport path and worsens PPL. Eighth, the early low-rank proposal is not simply the intrinsic dimension of the frozen T5 token table: the rank-origin control finds the full token table, generated token-embedding sequences, and clean contextual re-encodes all far above the step-0 predicted-clean rank under the same entropy-rank statistic.

These controls are small by design. The relevant baseline for linear denoising is a covariance-matched control, not a competing DLM. The relevant baseline for the decoder is a small learned decoder trained to imitate the official decoder, not direct unembedding alone. The relevant baseline for trajectory reliability is a shuffled or time-reversed signal, not a stronger policy. Each control preserves a superficial statistical property while removing the linguistic property the diagnostic is meant to test. Small negative or boundary results carry weight because they remove plausible but incorrect explanations, and each control is tailored to a specific confound: covariance preserves second-order statistics while removing linguistic content, and shuffled controls preserve bag-of-words structure while breaking contextual order.

\takeawaybox{Finding}{Recovery and generation need separate tests}{%
Strong recovery of supplied corrupted states can coexist with collapsed prior generation. The training and sampler controls therefore require decoder compatibility and transport quality to be evaluated separately.%
}{takeaway:recovery}

\FloatBarrier

\section{Discussion}
\label{sec:discussion}

The central contribution is a measured separation: a simple readout can preserve most native decisions while missing decisions with unequal output consequences. Matched repairs establish that separation on ELF and AURORA; the fresh selection test supplies a limited use under available test-time information. The other families identify which parts of that reasoning require a new test when the interface changes.

\subsection{Takeaways Across Model Families}
\label{sec:cross-family-takeaways}

The synthesis has two evidence levels. ELF and AURORA replicate the central readout-and-repair comparison. LangFlow, CALM, Cola-DLM, CoBit, and LLaCO change the state object, available information, or downstream computation; their boundary tests challenge specific extrapolations. Encoder swaps and ELF scale variants are further controls, not additional independent replications. \tabref{tab:cross-family-evidence} makes these roles explicit. The takeaways below are empirical lessons from the declared settings, not universal thresholds or a complete causal theory.

\parhead{1. Readability must be tested on the states that matter}
The $97.17\%$ ELF and $95.74\%$ AURORA agreements concern generated endpoint states under native boundaries. LangFlow's fitted agreement changes from $62.62\%$ to $74.15\%$ when the available fitting pool increases under fixed optimization compute; its native vocabulary shortcut supplies a distinct baseline. In CALM, fixing the pretrained unembedding improves posterior-state fidelity, yet changing to prior-state fitting raises agreement only from $43.21\%$ to $43.88\%$ in the matched temperature-1, teacher-prefix test. The fits are still improving. These results make fitting data, budget, and state distribution part of the readout claim; they do not establish an intrinsic linear-decoding limit. Cola-DLM's corruption sweep and ELF's decoder-only training further show why clean reconstruction must be followed by tests of relevant perturbed and prior-generated states.

\parhead{2. Count the misses, then measure their consequences}
Equal-count repairs in ELF and AURORA produce opposing pooled effects for rare-first and low-native-margin-first selection. XSum extends the distinction to reference similarity, with unresolved source-consistency benefit. Score accounting finds gains on unchanged evaluator tokens despite increased unigram cost; it does not isolate rarity from confidence and source overlap. The useful distinction is between where disagreements occur and what restoring them changes. A method that targets frequent failure cases should therefore be tested against matched correction budgets before those cases are treated as the most valuable repair targets.

\parhead{3. Preserve the information consumed downstream}
For terminal ELF and AURORA readouts, native token agreement measures decision fidelity; context interventions and output scores test what that summary misses. For LLaCO's adapter, the next computation consumes a soft embedding. Its geometric and learned modes can agree while the posterior remains diffuse, or while both modes select the mask token. Top-1 agreement then says little about the error from replacing the soft input by one codeword. CoBit's bit recovery likewise concerns a different unit from lexical agreement. The local margin bound supplies the corresponding caution under perturbation: stability depends on sensitivity as well as logit depth, and the measured first-order proxy is not a certificate. The statistic to preserve must be justified by the operation being replaced.

\parhead{4. A useful signal must earn its full computational cost}
The task-fitted selector improves held-out BERTScore over the same-call-budget random expectation, but the terminal native pass is a small part of the tested ELF and AURORA pipelines. Selection overhead prevents claiming net savings. LLaCO's exact adapter cache reduces component latency by trading an additional 4\,GiB of memory, while its tested sparse approximation falls back away from clean codewords. Neither result measures full-model acceleration. The practical implication is to evaluate quality and all executed work together, including selection, fallback, stopping, and memory, before moving a successful diagnostic into a generation algorithm.

These findings suggest a development order: fit a minimal substitute on the intended states, compare equally sized interventions to identify consequential omissions, and test any resulting policy using available inputs and total system cost. The state-distribution and soft-interface checks explain why this order must be repeated after the representation or decoder changes. They motivate future methods without claiming that one score or one substitute works for every family.

% Preserve the long-form technical discussion for the arXiv appendix while
% keeping the main reading path centered on readout misses and their consequences.
\long\gdef\extendeddiscussion{%
\parhead{A unified view: denoising as basin navigation}
The closest conceptual analogy is a ``propose, compare, and enter basin'' view of generation. In ELF, early denoising proposes diverse candidate clean states with low margins. In the middle trajectory, self-conditioning disagreement is largest, and candidate predictions become readable but are not yet in the high-margin decoder basin. A phase-resolved causal sweep (Appendix, \figref{fig:sc-phase-causal}) adds intervention evidence for the timing of SC utility: zeroing self-conditioning in equal-width 3-step windows causes the largest quality degradation around steps 16--20 (PPL increase $+7.0$ at center 16, $+6.4$ at center 20), while early and late zeroing is less harmful. Late denoising enters the high-margin decoder basin, where final-token agreement approaches one. PD compresses this transport and moves entry earlier; decoder-only training does not instantiate it; CPS changes it even when the prescribed noise coefficients are preserved. ``Compare'' remains an empirical description of trajectory behavior, not a literal claim about the denoiser's internal algorithm.

The pre-entry region has a different evidentiary status from the locked basin. Two small tensor-capture repeats of the first 10 ELF-B SDE steps (\figref{fig:preentry-bootstrapping}) show directed bootstrapping rather than immediate basin entry: consecutive predicted-clean updates have mean cosine $0.751$ and $0.780$, the effective rank of $\hat{x}_t$ rises from about $15$--$18$ to $152$--$160$ while the noisy state $z_t$ remains near rank $474$, and agreement with final tokens rises from $0.8\%$ to $16$--$21\%$. At the same time, the 10th-percentile native margin remains low ($0.049$ to $0.18$--$0.22$); p10-margin $0.5$ crossings are rare and schedule-sensitive, and no repeat reaches p10 margin $1$, $2$, or $8$ within these early steps. A matched 256-sample zero-SC audit further narrows the causal interpretation: setting the SC input to zero during steps 0--5 does not suppress early rank growth or delay above-chance agreement. At step 5, the predicted-clean effective rank is higher under zero-SC than baseline ($111.2$ vs.\ $103.4$), and at step 10 final-token agreement is also slightly higher ($20.8\%$ vs.\ $20.5\%$). The 5\% agreement crossing is not delayed (mean first step $5.79$ vs.\ $6.11$). A rank-origin control rules out the simplest token-table explanation: under the same entropy-rank definition, the full T5 token table has rank about $452$, generated label-token embedding sequences rank about $254$, and clean contextual re-encodes rank about $339$, all far above the step-0 predicted-clean rank. Thus early denoising is already structured, but the first rank expansion is not causally explained by the first few SC feedback updates or by a low intrinsic rank of the T5 token table. An open theoretical question in this trajectory analysis is the origin of the early proposal: the paper measures the onset of order and rules out the first few SC feedback updates as a necessary source under this intervention, while a full account of how that prior creates the first macroscopic semantic proposal from near-Gaussian noise remains future work. The operational model in \secref{sec:interface-phase-diagram} treats early rank growth as an additional chaotic-phase observable; a predictive theory would need to explain why $\hat{x}_t$ concentrates onto $O(15)$--$O(20)$ effective dimensions at $t\approx0$ and how that low-rank denoiser seed expands into a full token proposal before SC has its largest measured quality impact in the competition phase.

The interface phase diagram sharpens this view without changing the scope of the paper. A static basin says which final states are decoder-readable; the phase diagram asks when alignment begins, how wide the competition region is, and which tokens fail to lock early. The contribution is the measurement program: turn decoder-basin entry into a time-resolved diagnostic with crossing phases, transition widths, token-wise tails, sampler/scale checks, and readout probes. The follow-up transition grid refines the claim: earlier entry is the stable scale/compute signal, whereas transition width itself is not monotone across the tested scales and samplers. We use the phase language empirically, not as a claim of a thermodynamic phase transition.

\parhead{From diagnostic to minimal probes}
BGEE, ZSBD, and MDP follow a representation-first diagnostic style: before training a new generator, ask what the representation already makes easy. In our setting, the native decoder margin measures confidence in the current decision; generated final latents test whether a token lookup or minimal readout can reproduce that decision. Neither measurement alone establishes future stability or useful text quality. BGEE uses the first signal as an empirical early-stop trigger. ZSBD tests the no-newly-trained-readout limit by decoding through frozen T5 token-embedding nearest neighbors. MDP then quantifies how much decoder complexity remains after basin entry. These probes are intentionally small, and their role is measurement. They expose concrete control variables that future samplers, decoders, or training objectives can optimize, from margin-thresholded early stopping to token-level routing and lightweight margin proxies. \tabref{tab:basin-report} condenses the three probes into their measured basin property, key result, and interpretation.

\begin{table}[!htbp]
\centering
\caption{Basin report for the three minimal probes. The probes measure timing, token alignment, and affine recoverability within the same operational interface audit.}
\label{tab:basin-report}
\small
\begin{tabular}{>{\raggedright\arraybackslash}p{0.12\linewidth}>{\raggedright\arraybackslash}p{0.20\linewidth}>{\raggedright\arraybackslash}p{0.32\linewidth}>{\raggedright\arraybackslash}p{0.25\linewidth}}
\toprule
Probe & Exploits & Key result & Interpretation \\
\midrule
BGEE & Basin-entry timing &
Held-out Margin-12 gate saves $16.6/23.4/27.6\%$ NFEs on B/M/L &
Stopping probe, not a tuned sampler \\
ZSBD & Token-embedding alignment &
Frozen lookup recovers $93$--$96\%$ native tokens with no newly trained readout &
Readout probe; residual native-decision gap \\
MDP & Token-wise affine recoverability &
4k linear readout recovers about $94\%$ agreement; 32k reaches $97.9\%$ &
Interface-complexity probe, not replacement \\
\bottomrule
\end{tabular}
\end{table}

The companion stress tests in \tabref{tab:basin-stress-tests} explain why the probes are credible but bounded: they test basin sharing, intervention direction, token-wise entry, ZSBD geometry, and residual-tail predictability without turning the paper into a new sampler or decoder proposal.

\begin{table}[!t]
\centering
\caption{Additional fixed-checkpoint stress tests around the three minimal probes. Each row reports the measured effect and the corresponding interpretation.}
\label{tab:basin-stress-tests}
\footnotesize
\begin{tabular}{>{\raggedright\arraybackslash}p{0.16\linewidth}>{\raggedright\arraybackslash}p{0.38\linewidth}>{\raggedright\arraybackslash}p{0.36\linewidth}}
\toprule
Stress test & Key result & Interpretation \\
\midrule
Cross-decoder transfer &
B/M/L final latents decoded by the other two ELF decoders retain high source-native agreement: B$\to$M/L gives $0.988/0.987$, M$\to$B/L gives $0.996/0.996$, and L$\to$B/M gives $0.995/0.996$. &
Evidence for shared ELF-family basins under the same T5-small interface, not for universal T5 geometry. \\
Causal basin intervention &
At step 16 with $\alpha=0.5$ on 512 samples, projection gives $\Delta p10=+0.085$ and $\Delta$PPL $=-0.49$; anti-basin gives $-0.186$ and $+0.49$; random-matched is near zero ($+0.009$, $-0.08$). &
The basin direction behaves like a controllable variable, but this pilot is not a tuned guidance policy. \\
Same-start multi-path SDE &
For 256 fixed initial states with four independent SDE streams each, final token agreement across paths is only $29.9\%$ on average, while the mean final 10th-percentile margin remains $12.37$. &
SDE sampling enters high-margin decoder-basin regions, but the selected token labels are path-dependent; the claim is not a unique attractor. \\
Long-form topic boundary &
On 1000 unconditional ELF-B samples, adjacent-sentence cosine is OWT-like ($0.413/0.396$ for SDE64/SDE32 vs.\ $0.386$), but first-to-last cosine is far lower ($0.064/0.067$ vs.\ $0.292$). ODE32 improves first-to-last cosine to $0.110$, still far below OWT. SDE64 improves geometric PPL over SDE32 ($20.3$ vs.\ $25.1$) without improving this drift. &
Decoder-basin entry supports local token readability and semantic families, but it does not certify document-level discourse planning. \\
Mid-SC causal intervention &
On 256 matched-noise samples with equal-width 3-step windows across seven phases (centers 4/8/12/16/20/24/28), zeroing self-conditioning causes a competition-region quality peak: PPL change $+7.0$ at center 16, $+6.4$ at center 20, vs.\ $-0.1$ at center 4 and $+1.7$ at center 28. Freezing is a weaker mid-hump (max $+1.8$). Shuffled SC lowers PPL early but collapses entropy; late shuffling is catastrophic ($+63.9$ at center 28). &
Self-conditioning has a measured causal effect on quality in the competition region under the tested zeroing intervention; shuffled SC carries distinct cross-sample mismatch effects that should not be conflated with zero/frozen interventions. \\
Token-wise entry &
On 512 samples and $523{,}578$ tokens, persistent native-token entry reaches $99.96\%$; persistent entry with margin at least $8$ reaches $94.7\%$, but numeric and rare tokens are much weaker ($79.8\%$ and $81.4\%$). &
Heterogeneous token entry limits sample-level BGEE; the tested post-hoc hard-freezing rule degrades quality, while gentler blends preserve the measured PPL and entropy under some settings. \\
ZSBD geometry &
Frozen-embedding cosine gives $93.4\%$ agreement; unnormalized dot-product lookup gives $81.7\%$; Euclidean lookup gives $3.4\%$; cosine after whitening gives $4.3\%$; permuted labels are nearly zero. &
ZSBD relies on labeled angular token geometry, not vector norm, Euclidean proximity, or frequency alone. \\
Residual tail targeting &
A ZSBD-wrong gate covers $6.63\%$ of positions but contains $73.4\%$ of 32k MDP errors; the rare-token quintile contains $76.3\%$. &
The native decoder remains important for the hard tail; targeted repair should focus on few positions rather than replace the decoder wholesale. \\
Sampler extension checks &
ZSBD confidence routing, a 4.7k-parameter margin proxy, and equal-budget SC-update rejection expose useful signals. A condition-limited late bypass also works: on ELF-B SDE32, a step-28 full-decoder-head readout trained on 2048 trajectories reaches $95.3\%$ native-token agreement, and routing the lowest-10\% ZSBD cosine-gap positions to the native decoder reaches $97.8\%$ in the 1024-sample routing audit. Soft token-wise commitment preserves PPL in high-threshold variants but changes $16$--$18\%$ of tokens; hard freezing and one-step projection from noise fail. &
These are compiler and future-sampler signals, not production samplers; the step-28 bypass is validated only for the tested condition and readout. \\
\bottomrule
\end{tabular}
\end{table}

\parhead{Stress tests around the probes}
The three probes above are minimal, so we add stress tests that ask whether the basin is shared, intervenable, token-wise, and geometrically identifiable (\tabref{tab:basin-stress-tests}). Cross-decoder transfer is the cleanest generality check available inside the public ELF family: final latents generated by one checkpoint can be decoded by another checkpoint's decoder with very high agreement to the source checkpoint's native tokens. This argues against the strongest artifact explanation, namely that the denoiser merely navigates to a private decoder quirk. The supported claim is narrower: the basin is largely shared across ELF-B/M/L decoders that use the same T5-small latent and token interface. We do not extrapolate this result to independently trained decoders with different architectures or vocabularies.

The paired intervention test gives the most direct intervention evidence. At a mid-trajectory step, we replace the predicted clean latent with a decoder-basin projection, an anti-basin perturbation, or a random direction matched in scale, then continue denoising with the same noise stream. On a 512-sample four-shard expansion with $\alpha=0.5$, projection improves the final margin tail and lowers GPT-2-Large PPL on a 256-text audit ($\Delta p10=+0.085$, $\Delta\mathrm{PPL}=-0.49$), while the anti-basin perturbation worsens both ($\Delta p10=-0.186$, $\Delta\mathrm{PPL}=+0.49$). The matched random control remains close to neutral ($\Delta p10=+0.009$, $\Delta\mathrm{PPL}=-0.08$). Projection is not a competitive sampler--token agreement to the unedited baseline is only about $0.72$--but it shows that basin direction is not merely a retrospective correlation.

Two additional trajectory stress tests bound the phase-diagram interpretation. First, we fix each initial noise state and run four independent SDE noise streams. On 256 starts, the final paths agree on only $29.9\%$ of valid tokens on average (10th percentile $21.4\%$), even though each path ends in a confident decoder-basin region with mean 10th-percentile margin $12.37$. Thus stochastic ELF sampling should be described as path-selected entry into high-margin decoder-basin regions, not convergence to one unique attractor that is determined solely by the initial noise state. Re-exporting all four paths gives a semantic counterpart: the median cross-path document cosine is $0.222$, far above random cross-start pairs ($0.023$), but still far from a unique shared document. This path-selection result raises a document-level question: do high-margin decoder-basin states also certify coherent discourse? They do not under the tested unconditional settings. On 1000 ELF-B samples, adjacent-sentence similarity is close to OWT (median cosine $0.413$ for SDE64 and $0.396$ for SDE32, versus $0.386$ for OWT), but first-to-last sentence similarity is much lower ($0.064$ and $0.067$ versus $0.292$). The available ODE32 samples improve first-to-last cosine to $0.110$, but remain much closer to SDE than to OWT; deterministic sampling therefore changes the boundary modestly without closing it. A gamma and self-conditioning sweep improves some local transition metrics, but no tested setting restores OWT-like first-to-last topic stability. Per-sample GPT-2-Large PPL confirms that this is a distinct blind spot: SDE64 improves geometric PPL over SDE32 ($20.3$ versus $25.1$) while preserving nearly the same first-to-last drift, and ZSBD has worse geometric PPL ($32.6$) but slightly better first-to-last stability ($0.088$). Within generated groups, first-to-last sentence cosine has only weak correlation with log PPL (Spearman $\rho\approx -0.16$ to $-0.22$). More denoising steps therefore refine local token likelihood within a selected decoder-basin family, but they do not choose a more coherent document-level family under these settings. Likewise, ZSBD's geometrically smoother nearest-neighbor mapping yields marginally better topic continuity than the native decoder (first-to-last cosine $0.088$ vs.\ $0.064$) at the cost of calibration (geometric PPL $32.6$ vs.\ $20.3$), suggesting that sharper per-token decision boundaries may amplify semantic drift introduced by SDE noise during the competition phase. High-margin basin entry supports local token readability and path-selected semantic families; it does not certify document-level discourse planning.

Second, we corrupt self-conditioning in matched-noise trajectories. A broad mid-window intervention first shows that this signal is not merely correlational: on 256 samples, zeroing the SC state in steps 14--22 worsens PPL from $24.47$ to $42.63$ and lowers the 10th-percentile margin from $12.14$ to $9.90$; freezing it worsens PPL to $32.69$ and lowers the margin to $10.54$. Shuffling SC across the batch lowers geometric PPL to $21.18$, but sample unigram entropy collapses to $3.84$ and repeated 4-gram fraction rises to $24.2\%$, reinforcing the need for multi-metric evaluation. The phase-resolved sweep in \figref{fig:sc-phase-causal} localizes the causal effect: zeroing equal-width 3-step windows is nearly harmless at center 4, peaks around centers 16--20, and weakens again by center 28. Thus the SC-disagreement peak is not only statistically consistent with revision; under zero/frozen interventions it marks the window with the largest measured SC contribution to quality. Batch-shuffled SC measures a different failure mode: late locked states are highly sample-specific, so cross-sample SC mismatch becomes catastrophic rather than a clean test of compare-phase necessity.

Finally, token-wise entry and geometry ablations explain why the probes work and where they should fail. Most positions enter the decoder basin before the final step, but they do so at different phases; a 512-sample audit shows that persistent native-token entry reaches $99.96\%$ of token positions, while persistent entry with margin at least $8$ reaches $94.7\%$ and leaves a sequence-level tail. The p90 sample still has $453$ positions that enter the margin-$8$ basin late or never, and the longest late span has p90 length $17$. This makes a token-wise ``freeze easy, continue hard'' policy more promising than adding another sample-level gate. ZSBD's strength is also specific: cosine nearest-neighbor lookup in the labeled T5 token table works, while Euclidean distance, whitening, or label permutation destroys it. This does not mean that the frozen token table itself carries word order. The denoising trajectory supplies contextual, ordered latent states; the token table supplies labeled angular anchors for reading out those states. The remaining MDP errors concentrate in ZSBD-wrong, rare, numeric, and subword positions, so the natural repair target is the tail rather than a larger bulk readout.

Several additional fixed-checkpoint analyses, summarized in \figref{fig:existing-result-mining}, refine this boundary without adding another probe. Clean T5 states are deeper than generated ELF final states in the lower margin tail, same-start SDE paths select different high-margin basins, mid-trajectory self-conditioning is intervention-sensitive, and cheap confidence or proxy signals identify many tail cases. Together with the token-wise and residual-tail audits above, these signals suggest a shared fragile-token phenotype: late basin entry, low native margins, ZSBD disagreement, and MDP residuals repeatedly concentrate on rare, numeric, subword, and boundary-near cases. The full-population 32k audit places $91.7\%$ of MDP errors in the bottom native-margin quintile and $73.4\%$ among ZSBD-wrong positions. A separate inspection CSV retains all errors but uniformly subsamples matches; its unweighted conditional probabilities and quantiles describe an error-enriched sample, not population risks. We use the full-population summaries for these overlap claims. This is not a claim that every diagnostic selects the same token positions, and the document-level topic audit is not token-aligned; rather, the recurring tail structure identifies where future interfaces should allocate calibration or verification. The negative counterparts are equally useful: hard token-wise freezing and one-step projection show that simple readout after basin entry does not imply simple transport into the basin.

We also tested an additional boundary case: using a sentiment direction in the final latent state with a margin gate. The first version was essentially negative: decoder-basin editing shifted target success only marginally while random-direction controls remained close. Directional RBN refined the interpretation. Under the tested amplitudes, the sentiment direction mostly preserves token identity. We therefore ran an adaptive line-search version that pushes the same direction until the native margin lower tail reaches a safety threshold. Inside the high-quality decoder-basin region, most token decisions remain unchanged: with safety margin $2$, target success rises from $68.5\%$ to $72.9\%$, token agreement remains $95.9\%$, and PPL stays near the baseline. At the boundary, the attribute can be forced: safety margin $0$ reaches $99.7\%$ target success, but token agreement falls to $61.6\%$ and PPL jumps to $51.0$.

\parhead{How to interpret the probes}
BGEE, ZSBD, and MDP measure different properties in an operational interface audit rather than define separate optimized methods. BGEE measures \emph{when} a trajectory enters the basin, ZSBD measures whether final states are \emph{geometrically aligned} with token embeddings, and MDP measures how \emph{linearly recoverable} the native interface is on generated states. None is presented as a replacement for a tuned sampler or the native decoder. This is also why we keep BCD as a boundary result: the same basin that makes readout simple can make post-hoc semantic editing hard.

\parhead{Interface-driven hybrid inference}
The diagnostics suggest a candidate systems direction. BGEE detects an empirical basin-entry threshold, but the tested native-margin monitor is too expensive to yield wall-clock savings. A possible hybrid would combine a cheap phase monitor with affine readout and selective native decoding. On final generated ELF-B states, the historical 32k linear readout reaches $97.9\%$ native-token agreement. In a separate tail-gating fit, routing the lowest affine-confidence $5$--$10\%$ of positions covers most MDP errors and nearly closes the PPL gap (\figref{fig:rtgr}). These observations do not yet establish a deployed sampler: selective execution of the contextual decoder, stopping behavior, transferable gates, and measured end-to-end savings remain to be validated for this token-routing design. The high-native-margin repair study ranks known errors with oracle information; it does not validate low-affine-confidence selection of unknown errors. The separate whole-output audit in \secref{sec:native-free-routing} tests selection without native test outputs; its call-budget result does not validate this earlier token-routing design or establish wall-clock savings.

\parhead{A method blueprint: preserve the target, learn the path}
One candidate training-time design is an \emph{asymmetric} interface: retain a high-capacity clean target while restricting the noisy input pathway. This is not a necessary architectural prescription. Jointly designed compressed targets, decoders, and structured priors offer another route, as BLD demonstrates~\citep{feng2026branching}; our post-hoc PCA damage to a fixed ELF interface does not test that co-designed alternative. A candidate training direction is to combine the flow loss with late-phase decoder consistency and exposure to trajectory-aligned perturbations. The present evidence does not establish that prioritizing low native margins is beneficial: the equal-error output repair gives the opposite score ordering. Any training-time tail weighting therefore needs its own matched evaluation. This separates four jobs that a single bottleneck or pointwise loss otherwise conflates: preserve semantic target capacity, make high-noise proposals learnable, transport states into the target basin, and calibrate the residual decoder tail. AURORA-LM's independently proposed full-width target and restricted noisy pathway~\citep{aurora2026} is compatible with the first two principles, while DeltaFlow's noise-adaptive memory and temporal-state consistency~\citep{deltaflow2026} provide a complementary transport-side instance. Our noisy-decoder, margin-loss, stale-state, and one-step-projection controls motivate the latter two, but we present this as a falsifiable design blueprint rather than as an evaluated new architecture. A fair test must compare matched training seeds and compute, measure both synthetic basin width and generated-state basin depth, and require gains in PPL and distributional metrics rather than decoder margin alone.

We then tried a stronger repair: use the native decoder Jacobian itself. Instead of moving along an encoder-space Yelp Polarity direction, we backpropagate through the ELF decoder to increase logits of positive sentiment tokens, and we apply the same margin line search. This direction is more efficient, but it exposes the same basin trade-off. At safety margin $4$, target success improves from $65.6\%$ to $76.6\%$ while PPL rises mildly from $24.9$ to $26.8$ and only $2.6\%$ of tokens change. At safety margin $2$, target success reaches $99.2\%$, but PPL rises to $41.7$ and $10.8\%$ of tokens change. At the boundary, the classifier is fully controlled but text quality collapses. A token-selective variant that perturbs only the top $6.7\%$ sentiment-evidence positions leaves almost all token decisions unchanged: target success stays at $65.6\%$ and token agreement remains above $99.4\%$. Thus decoder-basin editing is not impossible, but clean control requires moving a nontrivial portion of the sequence through the decoder Jacobian. This is a useful boundary for representation-first readout claims. ELF final states are not a generic editable semantic canvas. Reliable control likely needs trajectory-level or training-time conditioning, whereas the decoder basin is best used for stopping and interface probing.

We test this with a trajectory-level repair. Instead of editing only the final state, we apply the decoder-gradient direction at one denoising step, replace the predicted clean latent, and continue the trajectory with the same SDE noise stream. This controlled repair asks whether earlier intervention is less constrained by the decoder basin. On 128 ELF-B samples, intervening at step 16 with a loose margin threshold raises target success from $64.1\%$ to $69.5\%$ while keeping PPL essentially unchanged ($24.6$ to $24.4$), but it changes $23.7\%$ of tokens. Later or more margin-conservative interventions change fewer token decisions. This supports the boundary interpretation: trajectory-level guidance can move the attribute frontier slightly without immediate PPL collapse, but clean controllable generation will likely require training-time conditioning or a guidance objective designed for control, not post-hoc decoder-basin editing.

The diagnostic suggests a practical pretraining question: can we design encoders whose representations are explicitly diffusion-ready under all axes, not only good for masked prediction or next-token prediction? T5 contextual embeddings appear strong because they combine denoisability, contextual order, and decoder compatibility in the ELF interface, and their span-corruption pretraining may be unusually well matched to local corruption and recovery. \tabref{tab:encoder-readiness-decomposition} refines this summary into four operational factors: contextuality, lexical geometry, scale mismatch, and decoder compatibility. This is a feature of the studied interface, not a universal theorem about all continuous language states. The protocol also reveals why blindly swapping encoders is risky. BERT-like, RoBERTa-like, GPT-like, token-level, byte-level, and compressed-tokenizer representations may each fail on different axes. A future compressed-tokenizer study for DLMs should therefore measure not only compression ratio and LM loss, but also decoder margin and recoverability under Gaussian corruption.

The released ELF checkpoints audited here use a frozen pretrained T5 encoder; this is a scope limitation of our study, not a requirement of the ELF framework. This makes it important to distinguish representation modeling from fully self-contained language modeling. These checkpoints demonstrate that flow transport can work well once a strong contextual interface is given; our audits do not test learning a continuous language space from scratch at frontier scale. The diagnostic protocol turns this dependency into a measurable object. It does not assume the state space is frozen: it can be applied to frozen T5 embeddings, learned VAE latents, compressed tokenizers, or future co-evolved semantic encoders. Our evidence suggests that ELF works in this setting because T5 contextual embeddings and the native decoder happen to form a compatible interface. Cola-DLM asks a more upstream question: can the representation itself be learned and co-adapted with the prior? Our Cola boundary results show why reconstruction alone is not enough for that path. A learned latent should be evaluated by denoisability, semantic recovery, order sensitivity, decoder-basin width, and prior compatibility, not by clean VAE recovery alone. From this perspective, token budget accounting is incomplete unless the representation interface is also characterized.

\parhead{What remains unclosed}
The evidence closes an \emph{operational} loop, not an origin theorem for T5. We can identify the observable ingredients that make the ELF interface usable: contextual states are denoisable without losing order, the tied token table gives angular anchors for ZSBD, the native decoder has a broad enough margin neighborhood after noisy-CE exposure, and the denoiser learns to transport Gaussian states into that neighborhood. This account organizes the observed relationships among readability, margin, and trajectory signals in the audited checkpoints. It does not derive those properties from the T5 pretraining objective, the decoder weights, or the ELF training dynamics. In particular, the paper does not prove why span-corruption pretraining, tied embeddings, layer normalization, vocabulary geometry, and decoder-input noise combine to make the T5-small interface unusually diffusion-ready.

A small zero-training pilot rejects a simple static token-table explanation. We computed each T5-small token embedding's nearest-neighbor cosine width over the full $32{,}128$-token vocabulary and estimated token frequency from the first $50{,}000$ tokenized OWT examples ($47.9$M tokens). High-frequency tokens do \emph{not} occupy wider static Voronoi cells: full-vocabulary frequency and token-embedding width have Spearman $\rho=-0.225$. On the 1,911 unique token labels in the token-wise ELF audit, token-aggregated embedding width is only weakly related to final decoder margin ($\rho=0.209$), whereas frequency is more strongly related to final margin ($\rho=0.418$). Thus the interface origin is unlikely to be static token-table Voronoi geometry alone; it likely involves the pretrained language prior, contextual encoder geometry, decoder calibration, and noisy decoder-neighborhood training together. Closing that origin-level question requires controlled interventions that this fixed-checkpoint study cannot provide: retraining or adapter-training with one ingredient removed at a time, dimension-matched encoder swaps, explicit tied-vs-untied decoder comparisons, and weight-level analyses of local decoder Jacobians and high-noise denoiser proposals. The present contribution is therefore to make the interface-origin problem measurable and falsifiable, not to claim a first-principles derivation of why T5 has this property.

On the decoder side, ELF's shared decoder is sometimes described as a simplification, but our results suggest a more nuanced view. The final decoder step is a nonlinear interface trained under a CE objective; it is not merely the last flow step and not merely a matrix multiplication. Sharing weights may be useful because the denoiser and decoder learn compatible hidden states, but the benefit should be evaluated against fair separate decoders with sufficiently broad hyperparameter ranges. Our small-decoder results show that even when a decoder is trained directly on final ELF latents, PPL and token agreement can diverge. This motivates future work on decoder-aware flow objectives, margin regularization, or explicit basin widening.

\parhead{Training-time basin widening}
The diagnostic also suggests which training-side factors should matter. Decoder-input noise should help because it trains the decoder on a neighborhood rather than only on the clean manifold; MSE auxiliary losses can prevent the latent from becoming easy to classify but hard to denoise; adaptive timestep sampling should focus compute where decoder compatibility changes fastest. The released ELF training recipe~\citep{elf2026} already follows this principle: OWT training uses a decoder branch on 20\% of examples and constructs its decoder input by mixing clean T5 latents with Gaussian noise at scale $5.0$. This is consistent with the basin-widening effect predicted by the margin view, so we treat it as a substantive implementation detail rather than incidental. LDLM's joint-training recipe uses decoder-input noise and decoder-stabilization losses~\citep{ldlm2026}. That design motivates an in-training compatibility audit when the representation co-evolves with the prior; it does not independently validate our fixed-checkpoint mechanism. Draft-conditioned refinement~\citep{zhang2026latent} is another indication that decoder readability matters, while the evidence in this trajectory analysis consists of fixed-checkpoint measurements of when ELF basin entry happens, how it changes with trajectory phase, and which operations become reliable after entry.

We also ran a small 4-GPU training ablation with the official PyTorch ELF-B code to test this prediction directly. Holding the architecture and training length fixed, a clean-only decoder branch reaches high clean token recovery ($0.951$) but retains only $47.0\%$ of that recovery at a mid-noise state ($t=0.6$) and $42.1\%$ at $t=0.5$. Adding decoder-input noise widens the basin substantially. A mild scale of $1$ reaches nearly perfect clean recovery ($0.999$) and retains $76.4\%$ at $t=0.6$, while the official-like scale of $5$ reaches $0.995$ clean recovery and retains $94.4\%$ at $t=0.6$ and $81.7\%$ at $t=0.5$. Very large noise is not monotonic: scale $10$ weakens clean recovery ($0.917$) and has lower retention than scale $5$. Thus the official-style decoder-noise branch is not merely a robustness trick. It trains a decoder-compatible neighborhood around clean T5 latents, exactly the basin-widening effect predicted by the interface view.

Varying the frequency of the decoder branch sharpens the same conclusion. With the noise scale fixed at $5$, a pure denoising run learns the L2 branch but has essentially zero clean token recovery, showing that transport training alone does not create a token interface. A tiny $5\%$ decoder branch already creates a usable decoder (clean recovery $0.962$), but its 10th-percentile margin is much smaller than the official-like $20\%$ branch ($1.87$ vs.\ $6.11$). At the other extreme, a decoder-only run has the strongest synthetic noisy-latent basin (retention $0.981$ at $t=0.6$), but no denoising loss at all. Its 50k extension collapses to essentially one sequence rather than acquiring transport; a staged switch back to the mixed objective restores diversity but not quality. The matched five-epoch replacement then preserves the separation across both training seeds and every saved checkpoint: decoder-only always has the deeper handed-latent corruption basin, while the mixed objective always has the deeper generated final-state basin (\secref{app:transport-controls}). This separates two competencies that are often conflated: the decoder branch can widen the basin, but the flow branch is needed to transport Gaussian states into it. The useful training recipe is therefore not ``more decoder'' or ``more denoising'' in isolation, but a balance between transport and basin widening.

We also tested the most literal repair suggested by \thmref{thm:decoder-margin}: add an explicit top-1-vs-runner-up margin penalty to the decoder branch while keeping the decoder-input noise scale at $5$ and the decoder branch at $20\%$ of examples. At 5k steps, clean token recovery remains unchanged ($0.9946$ vs.\ $0.9948$), but retention under corrupted latents drops from $0.944$ to about $0.912$ at $t=0.6$ and from $0.817$ to about $0.787$ at $t=0.5$; the clean 10th-percentile margin also decreases. We then extended the official-like baseline and the strongest margin run to 10k steps. The margin objective is clearly optimized---its training margin loss decreases from about $0.85$ to $0.56$ and clean recovery reaches $0.998$---but it still does not surpass the noisy-CE baseline: retention at $t=0.6$ is $0.973$ vs.\ $0.979$, retention at $t=0.5$ is $0.880$ vs.\ $0.894$, and the 10th-percentile margin is $7.76$ vs.\ $8.40$. The negative result is useful but limited: in these short runs with the official implementation, decoder robustness did not improve under the tested scalar margin penalty. Noisy CE trains a local calibrated basin around states the denoiser is likely to visit, while the blunt margin penalty can optimize its own objective without producing the strongest recovery under the tested corruptions. A stronger basin-aware objective should therefore be neighborhood-aware rather than only pointwise: it should regularize local decoder-Jacobian or Lipschitz sensitivity in directions the denoising trajectory actually visits, or contrast clean and corrupted states inside a local manifold neighborhood, instead of only increasing a visible scalar logit gap on clean or synthetic decoder inputs.

The same basin constraint applies to few-step and one-step language generation. Methods such as FMLM and CDLM reduce sampling depth directly~\citep{fmlm2026,cdlm2026}. Drifting Models~\citep{drifting2026} raise an even sharper question: can the iterative process be moved from inference into training? Our results suggest a language-specific constraint. A direct map from noise to text latents is useful only if it lands inside the decoder-compatible basin, not merely near the clean latent under MSE. A natural follow-up experiment is a small drift head trained on ELF trajectories and evaluated by the same margin, entropy, and recovery diagnostics used here.

PPL remains useful because it is sensitive, well understood, and cheap, but it should not be the sole metric for latent or embedded DLMs. We treat it as a complementary distributional proxy: matched entropy makes comparisons more informative, but neither a PPL increase nor a decrease alone determines linguistic utility. A sampler can improve PPL by reducing entropy, increasing repetition, or moving closer to high-frequency reference text. MAUVE, JS divergence, repetition, and entropy provide complementary views, and decoder agreement is needed when the sampler changes the final-state distribution. The diagnostic framing therefore changes the usual sampler question from ``which curve has the lowest PPL'' to ``which part of the interface frontier has moved, and what was traded away?''

Tri-mode systems such as Nemotron-Labs-Diffusion~\citep{fu2026nemotronlabsdiffusion} make continuous-diffusion diagnostics more relevant. If diffusion is one inference mode inside a larger autoregressive model, the system may inherit the AR backbone's fluency while still exposing diffusion-specific interface risks. Trajectory reliability signals could become a routing criterion: use diffusion when the trajectory is stable, fall back to AR or self-speculative verification when decoder margins are small. Conversely, if tri-mode diffusion lacks such internal reliability signals, that would mark an important boundary beyond which the ELF-style mechanism does not cleanly generalize.

ELF's design naturally favors fast parallel generation over sequential deliberation. A fixed-step sampler can produce an entire sequence at roughly fixed denoising depth, which is attractive for latency-bound generation. But this is not the same capability as chain-of-thought reasoning or long-horizon planning. Autoregressive models condition on their own intermediate tokens and can externalize computation through a written chain~\citep{wei2022cot}; embedded diffusion compresses this process into a continuous trajectory that is later discretized. Hybrid systems such as block diffusion~\citep{blockdiffusion2025} and tri-mode AR/diffusion/self-speculation models~\citep{fu2026nemotronlabsdiffusion} explore this design space from the architecture side. Hybrid verification is one design option, but our unconditional readout experiments do not compare reasoning architectures or establish that autoregressive verification is necessary. Trained continuous systems may instead improve the representation, transport, decoder, or task objective; those possibilities require task-level evaluation.

Discrete DLMs avoid the Gaussian-noise puzzle by keeping corruption in token space, usually through masking or remasking~\citep{d3pm2021,llada2025}. They also raise an interface question. They cross a continuous-to-discrete boundary at every denoising step through logits, softmax probabilities, and sampling or argmax decisions. Our axes have discrete analogues: denoisability becomes mask-recovery accuracy; recoverability becomes whether intermediate distributions preserve semantic and syntactic constraints; decoder compatibility becomes whether intermediate logits remain inside stable decision regions; trajectory reliability becomes whether confidence, entropy, or agreement predicts final correctness before the last step. Testing these analogues is outside this paper, but the broader point is the same: our measurements suggest questions about compatibility and stability that could be tested in those models; they do not establish that discrete and continuous systems have the same failure mechanism.

CCDD~\citep{ccdd2025} frames continuous diffusion as expressive but practically difficult to train and decode. Our experiments identify possible interface failures in the tested systems. A continuous state space may be expressive, but if its perturbed states are not recoverable, if order information is compressed away, or if the decoder basin is narrow, the expressive latent can still fail at the token interface. This motivates testing hybrid continuous-discrete systems: the continuous component can carry semantic organization, while the discrete component can anchor token realization and verification.

Several questions remain open. First, the diagnostic should be scaled: can readiness scores predict training efficiency for larger jointly trained systems? Second, decoder basins need to be engineered rather than merely measured: how can training widen the basin without inducing low-entropy collapse? Third, future systems should test whether the representation and decoder can co-evolve a readable basin instead of inheriting one from a frozen encoder. Fourth, the probes here should be converted into practical systems only after cheaper monitors, token-wise policies, or dynamic batching remove their diagnostic overhead. Fifth, cross-architecture validation should extend beyond the completed AURORA replication to more independently trained continuous, latent, and hybrid DLMs. Sixth, controlled, prompted, and long-context generation beyond the paired short-summary task should be treated as basin-distortion tests: a prompt or control signal may move the target basin, sharpen some token boundaries, or trade local readability for global fidelity. Cross-modal extensions such as speech- or image-conditioned ELF variants provide a natural version of this test because an external encoder and projection layer may anisotropically distort the text decoder basin. Closed-loop adaptive samplers are another natural direction: instead of using BGEE as a diagnostic early-exit rule, a sampler could adjust step size, stochasticity, guidance, or verification based on the same margin, entropy, and self-conditioning signals. Finally, the local margin analysis should be extended into a theory of nonlinear, anisotropic decoder basins and into neighborhood-aware representation pretraining objectives. The short margin-loss ablation here argues against a blunt pointwise logit-gap penalty, but it does not rule out decoder-margin or local-sensitivity regularizers applied to the trajectory neighborhoods that future continuous DLMs actually visit. Directional RBN already shows that the most fragile principal direction is checkpoint-dependent---ELF-M is most fragile along its first principal component, whereas ELF-L is more fragile along its second. The ELF-M follow-up audit points to punctuation and token-boundary anisotropy rather than a named semantic axis, but a full explanation requires local decoder-Jacobian or logit-level sensitivity analysis rather than token-neighbor inspection alone. Our online delta-gating controls, decoder-noise ablations, and margin-loss negative result define useful starting points, but they also argue for the same discipline as the rest of the paper: establish the diagnostic evidence before optimizing the method.

\parhead{Higher-level extensions}
Several extensions fall outside the present claim. A fuller interface phase-transition theory would need to predict rank growth, transition width, token-tail persistence, and scale shifts from the denoiser and decoder dynamics, not only name the observed phases; random-matrix, information-theoretic, or energy-landscape tools may be useful for that future step. The fixed-checkpoint study supplies observable intervention constraints; deriving the pre-entry rank dynamics requires controlled retraining or architecture/noise-schedule interventions. The pre-entry effective-rank divergence---$\hat{x}_t$ rank rising from $\approx15$--$18$ to $\approx160$ while $z_t$ stays near 474---provides a concrete target for such a theory. It is not explained by the frozen token table alone: the same entropy-rank calculation gives about $452$ for the full T5 token table, $254$ for generated label-token embedding sequences, $339$ for clean contextual re-encodes, and $474$ for Gaussian sequences. The low-rank start is therefore better read as a property of the denoiser-prior map, for example a low-rank high-noise proposal, rather than as the intrinsic dimensionality of the target token manifold. The numerical proximity between $R_0$ and $\log_2 |V|$ is at most a vocabulary-resolution hypothesis, not evidence, because entropy rank and information capacity have different units; testing it would require changing vocabularies or estimating token mutual information at step 0. A tiny fixed-checkpoint spectral pilot sharpens this boundary. Estimating $P_D$ from local pairwise decoder-margin Jacobians gives a weak but phase-consistent feature-space alignment signal: using a competition-phase basis, centered trajectory deltas rise from about $0.035$--$0.036$ early to about $0.061$ after entry. The step-0 predicted-clean state has entropy rank $18.0$, while randomized finite-difference response matrices have mean entropy rank about $407$ and their concatenated feature matrix has entropy rank about $433$. These observations distinguish the low-rank base output from the broad feature geometry of its perturbation responses. They neither estimate the full contextual Jacobian spectrum nor by themselves rule out a low-rank vectorized Jacobian, as the counterexample in \secref{sec:decoder-basin-theory} shows. Output bias, singular-spectrum reweighting, and nonlinear proposal effects remain candidate explanations. Other plausible explanations, such as a training-data covariance floor or an architecture-conditioned low-rank dictionary, remain open.

A descriptive fit makes the constraint sharper. From steps 1--10, the measured rank growth is approximately linear at $11.8$ ranks per step under the normal sampler and $11.4$ ranks per step when SC is zeroed for steps 0--5, with $R^2=0.998$ and $0.994$, respectively. A square-root fit over the same early region is not the primary summary; the observed entropy rank grows nearly linearly on this short interval. This fit does not distinguish new-subspace accumulation from reweighting of existing singular directions or identify an underlying dynamical law. The growth is correlated with early agreement and margin, but it is not itself basin entry: at step 10, agreement is only about $20.5\%$ and $p_{10}$ margin is about $0.21$. The operational theory in \secref{sec:decoder-basin-theory} therefore gives a set of testable hypotheses---low-rank sequence proposals, changing singular spectra, decoder-sensitive alignment, competition, and margin locking---but it does not yet derive the rank slope, any putative critical aligned rank, or the tail-token offsets from weights or training dynamics.

An information-theoretic \emph{interface capacity} could measure how much token information a decoder preserves under a constrained perturbation distribution, turning noise-sweep retention curves into a global analogue of the local margin bound. As a first constraint rather than a theorem, we converted existing retention curves into a $q$-ary symmetric-error proxy. Final-state RBN curves remain near the token-table ceiling, with proxy AUCs of about $14.75$, $14.87$, and $14.90$ bits for ELF-B/M/L final states, while decoder-noise training curves occupy a lower and more diagnostic range: the 10k noisy-CE baseline reaches $7.30$ proxy bits under the corrupted-latent sweep, compared with $5.08$ for clean-only decoder training. This proxy gives the right robustness ordering, but it is not a calibrated estimate of $C(E,D;\sigma^2)$ because it replaces the true high-dimensional decoder channel by a symmetric-error abstraction.

An interface-readiness score (IRS) could combine denoisability, order sensitivity, clean margin depth, and lightweight readout agreement before training a DLM. A small exploratory aggregation gives an interface-specific ordering---T5 contextual states score $77.3$, BERT $54.0$, GPT-2 $35.8$, and a fresh random-T5/random-decoder control drops to $27.1$---but it omits the trajectory-reliability axis and is therefore a prototype, not a released benchmark. Its raw-margin component depends on logit scale and the component weights are not prospectively validated; the aggregate is not a calibrated cross-architecture readiness ranking. The expanded T5-family controls sharpen the interpretation. T5-small is the only contextual interface in the panel clearly above $70$ IRS under the ELF-compatible audit, whereas T5-base, T5-large, mT5-small, Switch-base-8, ByT5-small, and T5Gemma variants score in the $35.7$--$48.6$ range, with margin-depth scores collapsed to zero under the tested native or ELF-compatible decoders. \tabref{tab:encoder-readiness-decomposition} decomposes this pattern into four operational factors: contextuality, lexical geometry, scale mismatch, and decoder compatibility. This is not evidence that larger encoders are worse in general, and it is not an origin theorem for why span-denoising pretraining created T5-small's favorable basin. It shows that diffusion readiness is a property of an interface pair $(E,D)$ and its normalization, not of an encoder family name or parameter count. This also resolves an apparent tension with the denoiser-scale trend: ELF-B/M/L improve basin-entry timing while sharing the same T5-small interface, whereas replacing the interface changes the target geometry itself. Recent encoder-scaling work retrains ELF for the tested T5 and T5Gemma interfaces and finds benefits from stronger pretrained embeddings and distillation~\citep{zhang2026diffusibility}. That setting is complementary: our frozen-pair mismatch checks do not predict the quality of a flow and decoder retrained for the new encoder.

An interface compiler could use the measured phase diagram to generate sparse margin monitors, confidence routers, SC-update candidate selectors, or SDE-gamma schedules automatically. ZSBD confidence routing, the lightweight margin proxy, equal-budget SC rejection, and the step-28 full-decoder-head bypass provide initial evidence for this direction. The boundary is equally important: hard token-wise early commitment (TEC) and one-step projection fail, while soft TEC preserves PPL only by allowing many token decisions to change and by using an expensive full-monitor prototype. Few-step consistency distillation is another natural extension, but our pilot makes it a future problem rather than a current contribution: a per-step light denoiser reaches $66.5\%$ native-token agreement at step 28, whereas joint transport-and-CE variants reach only $19.2\%$ and $17.1\%$ with misleading low-PPL collapse modes. A cross-architecture basin taxonomy could compare ELF, LDLM, RePlaid, DiHAL-style hidden-state replacement, TEncDM, LangFlow, and discrete DLMs under normalized readiness axes. These directions require additional checkpoints, adapters, human or task evaluation, and training runs. We treat them as implications of the diagnostic protocol rather than completed contributions; each requires validation against the same readiness axes. \tabref{tab:future-roadmap} summarizes the split between engineering follow-ups, training-side problems, and theory-scale open questions.

\begin{table}[!htbp]
\centering
\footnotesize
\caption{Future-work roadmap implied by the completed diagnostics. Tier 1 items have strong fixed-checkpoint signals and are primarily engineering follow-ups. Tier 2 items need new training recipes. Tier 3 items require theory or larger-scale validation before they should be treated as mature methods.}
\label{tab:future-roadmap}
\setlength{\tabcolsep}{3pt}
\renewcommand{\arraystretch}{0.92}
\begin{tabular}{>{\raggedright\arraybackslash}p{0.09\linewidth}>{\raggedright\arraybackslash}p{0.22\linewidth}>{\raggedright\arraybackslash}p{0.29\linewidth}>{\raggedright\arraybackslash}p{0.25\linewidth}}
\toprule
Tier & Direction & Current evidence & Next bottleneck \\
\midrule
1 & Late decoder bypass and routing &
At step 28, a full-decoder-head readout trained on 2048 trajectories reaches $95.3\%$ native-token agreement; routing the lowest-$10\%$ ZSBD cosine-gap positions reaches $97.8\%$ in an oracle-tail audit. &
Turn the signal into a deployable compiler rule: cheap confidence, dynamic batching, and non-oracle tail handling. \\
1 & Soft token-wise commitment &
On 1024 matched samples, high-threshold soft TEC preserves PPL ($23.55$--$23.66$ vs.\ $23.70$ baseline) while estimating $39$--$56\%$ potential token-step savings. &
Find the commit-rate/quality Pareto frontier with a cheap monitor; the current full-monitor prototype is slower and changes many tokens. \\
1 & Pretraining-readiness audit &
The expanded IRS panel and \tabref{tab:encoder-readiness-decomposition} rank the ELF-matched T5-small contextual interface highest ($77.3$), while larger or different T5-family interfaces fall to $35.7$--$48.6$ under the tested pairings. &
Package the audit and add trajectory reliability; report it as an interface-pair diagnostic rather than an encoder leaderboard. \\
2 & Neighborhood-aware basin training &
Decoder-input noise widens noisy-latent retention, whereas scalar margin loss fails in short 5k/10k runs. &
Design local Jacobian or contrastive neighborhood losses without inducing low-entropy collapse. \\
2 & Few-step consistency distillation &
A per-step light denoiser reaches $66.5\%$ agreement at step 28; joint transport-and-CE variants collapse to $19.2\%$ and $17.1\%$. &
Separate transport consistency from decoding calibration; a multi-step objective is needed before a 2--4 step sampler is credible. \\
2 & Training-time token commitment &
Our post-hoc hard-freezing variant fails; some soft TEC settings preserve measured PPL and entropy. A stale-token surrogate gives a small positive signal: training a step-28 full-decoder adapter with $25\%$ step-24 token replacement improves robustness to step-20 stale mixes by $1.0$--$2.2$ pp without hurting clean agreement. &
Train the denoiser with random token freezing or soft commitments and test whether it learns a globally consistent transport policy; the adapter surrogate does not test transport or PPL. \\
3 & Interface phase-transition theory &
The pre-entry audit measures rank expansion from $\approx15$--$18$ to $\approx160$; zero-SC keeps the step-1--10 rank slope similar ($11.4$ vs.\ $11.8$ ranks/step), a rank-origin control rejects the token-table intrinsic-dimension explanation, a candidate alignment hypothesis distinguishes rank expansion from margin entry without asserting hysteresis, and a tiny $P_D$/finite-difference pilot finds weak phase alignment and broad response-feature geometry without identifying full Jacobian rank. &
Derive the denoiser proposal dynamics, any putative critical aligned rank, and the tail-token offsets that connect the low-rank seed to later margin-based basin entry. \\
3 & Information-theoretic interface capacity &
The margin bound and noise-sweep retention curves expose local and empirical basin robustness; a $q$-ary proxy ranks RBN final states above decoder-noise training variants and separates clean-only from noisy-CE training. &
Define and estimate $C(E,D;\sigma^2)$ for high-dimensional token decoders without relying on the symmetric-error proxy used here. \\
\bottomrule
\end{tabular}
\end{table}

\parhead{Falsifiable predictions}
The diagnostic view makes several testable predictions for future systems. First, frozen-pretrained-interface DLMs should show a monotone token-alignment transition along the denoising trajectory: ZSBD-like agreement should rise with margin, and stronger denoisers or longer sampling should move high-margin crossings earlier until a representation-imposed floor is reached. This is a within-trajectory prediction, not a claim that transition width must shrink monotonically with scale. Second, learned or self-organized latent spaces trained without decoder-neighborhood exposure should show narrower noisy basins and later entry than comparable systems with decoder-input noise or trajectory-neighborhood training. Third, late-entry and residual-tail tokens should remain enriched for rare, numeric, subword, and token-boundary cases unless the decoder is explicitly calibrated for those classes; the tail may be heterogeneous, with scale-sensitive rare/numeric errors coexisting with an interface-floor component where frozen-embedding lookup and the native decoder boundary disagree. Fourth, the strongest SC-causal account of pre-entry is falsified in this fixed checkpoint: disabling SC during steps 0--5 does not reduce effective-rank growth or delay first above-chance agreement, so this early structure does not require the first SC feedback loop; the ablation does not uniquely identify its internal origin or certify semantic content. This does not contradict the phase-resolved causal sweep, which shows that SC has its largest measured quality impact later in the competition phase. Fifth, the cross-scale Margin-8 entry phases reported in \secref{sec:interface-phase-diagram} predict a monotone decrease with denoiser capacity under the same frozen T5-small interface and matched SDE sampling. The follow-up check in \tabref{tab:p5-tau-scaling} makes this prediction sharper: SDE32 and SDE64 give consistent three-point trends, $\tau_M(8)\propto N^{-0.138}$ and $N^{-0.115}$ over ELF-B/M/L, respectively, whereas the single-shard ODE32 estimate is much flatter (about $N^{-0.03}$) and has lower final margins. The coarse grid in \tabref{tab:phase90-transition-grid} further shows that this prediction should be stated as earlier entry, not monotone transition-sharpness scaling. Future systems with larger same-interface denoisers should therefore either continue the SDE entry trend, saturate at a representation-imposed floor, or break under a different sampler/training recipe. We do not present these three public checkpoints as a controlled scaling law, and we do not make an encoder-scale claim across incompatible latent dimensions; the falsifiable claim is the direction and sampler-conditioned reproducibility of earlier basin entry under a shared interface.

The status of these predictions is intentionally uneven. The monotone token-alignment transition and residual-tail enrichment are supported inside the ELF audits, while the new tail supplement shows that enrichment is not a single mechanism: rare and numeric ZSBD errors shrink with scale, but a high-frequency ZSBD/native-boundary floor remains. The learned-latent prediction is only partially supported by Cola-DLM and still needs more systems. The SC prediction splits into two causal regimes: early zero-SC rules out those feedback updates as necessary for the observed pre-entry rank expansion, while phase-resolved interventions support SC utility in the competition region. The rank-origin control further falsifies the simplest explanation that the $O(15)$--$O(20)$ start is just the intrinsic dimension of the T5 token table; what remains is a denoiser-dynamics problem. The scale prediction is now restricted to same-interface denoiser scale and to entry timing; extending it to new encoders or to transition-sharpness scaling would require a normalized interface metric rather than comparing raw hidden dimensions.

\parhead{What the mechanism changes}
The basin view is useful only because it separates competencies that aggregate metrics conflate. Representation probes determine whether a readable target exists; trajectory probes determine whether the learned prior reaches it; decoder probes determine what remains after entry. This ordering changes how failures should be repaired. Poor corruption retention motivates decoder-neighborhood exposure, whereas strong corruption retention coupled to poor generation motivates testing transport training and sampler--checkpoint alignment. This is a diagnostic design heuristic, not an optimal repair theorem. The decoder-only, CPS, and PD controls instantiate these three cases directly.

\parhead{Design implication}
The immediate output is a two-stage audit, not a new generation leaderboard. Before training, test denoisability, semantics, order, and decoder robustness under realistic corruption. After training, test entry time, lower-tail margin, entropy, and quality under sampler or curriculum changes. BGEE, frozen lookup, and linear readout show that these signals can support future early-exit or selective-decoder systems, but the present monitored prototypes are diagnostics rather than wall-clock improvements. \tabref{tab:checklist} states the minimum reporting scaffold.

\parhead{What remains open}
The experiments measure how successful ELF trajectories become locally readable, not why T5 pretraining creates the target geometry or how the denoiser forms its first semantic proposal near the Gaussian endpoint. They also do not turn local basin entry into document planning: the long-form audit shows OWT-like adjacent-sentence similarity but severe first-to-last topic drift. These are boundaries of the title claim, not missing evidence for it. Extended stress tests, failed compilation shortcuts, hypotheses for learned interfaces, and falsifiable scale/tail predictions are retained in \secref{app:extended-discussion} rather than interrupting the main argument.

\parhead{What the completed grid adds}
The interior grid supplies a tradeoff rather than an optimal decoder probability: $p=0.5$ has the lowest noncollapsed endpoint PPL in both seeds, yet more bigram repetition than $p=0.2$ and higher MAUVE in only one. The epoch-2 swaps retain most decisions on noncollapsed states but do not restore normal diversity to decoder-only states. The pooled gradient-direction calibration favors sensitivity-normalized scores over both raw depth and an amplitude-aware baseline; fixed-scale and random-direction controls expose exceptions, and uncertainty remains substantial under the small held-out sequence sets. These findings sharpen the target/transport distinction without proving globally independent network modules, a universal mixture optimum, or certified decoder radii.

\parhead{A practical checklist for future DLMs}
The results suggest a reporting checklist, summarized in \tabref{tab:checklist}. We intend this as a standard auditing scaffold for future continuous or latent DLMs: before a new state space, decoder, sampler, or training objective is compared by a single quality scalar, it should first be checked for denoisability, semantic recoverability, order sensitivity, decoder compatibility, and trajectory reliability.
This checklist is modest, but it would have caught most of the misleading shortcuts exposed in our experiments.
}

\subsection{Using the Takeaways}
\label{sec:using-takeaways}

For a new interface, begin with its smallest faithful substitute and a validation set drawn from the intended state distribution. A weak probe should first trigger fitting and coverage controls; high agreement should trigger matched consequence tests. If those tests support a useful selection signal, evaluate it with only deployable information, independent output stopping, and total executed cost. These are successive empirical checks, not a theorem that guarantees a successful replacement.

The same order helps locate a failed intervention. Poor recovery under relevant corruption motivates testing decoder exposure to those states. Good supplied-state recovery with poor prior generation motivates testing transport and sampler--checkpoint compatibility. Stable top-1 decisions with a diffuse soft input motivate preserving the distribution used downstream. These hypotheses are motivated by the measured controls and require a new matched test after retraining or changing the interface. \tabref{tab:checklist} retains the broader diagnostic checklist; \secref{app:extended-discussion} records the detailed training tradeoffs and open mechanisms.

\begin{table}[!htbp]
\centering
\footnotesize
\setlength{\tabcolsep}{3pt}
\renewcommand{\arraystretch}{0.92}
\caption{A minimal diagnostic checklist for future continuous or latent DLM papers. The checklist translates the audited failure modes into suggested measurements, adapted to each model's native interface; it is not a set of necessary conditions for useful generation.}
\label{tab:checklist}
\begin{tabular}{>{\raggedright\arraybackslash}p{0.34\linewidth}>{\raggedright\arraybackslash}p{0.58\linewidth}}
\toprule
Before claiming \ldots & Report \ldots \\
\midrule
a new state space is diffusion-ready &
linear denoising NMSE, semantic recovery, order sensitivity, native-decoder margin under corruption, and at least one trajectory reliability signal. \\
a decoder is a harmless readout &
native-decoder agreement, margin recovery under latent corruption, and entropy calibration. \\
a sampler is better &
PPL together with entropy, repetition/distinctness, MAUVE or another distributional metric. \\
an internal signal is useful &
shuffled-signal, delayed-signal, and time-region controls. \\
a latent autoencoder is enough &
clean reconstruction and decoder-basin robustness under the prior's noise range. \\
a control or prompt method works &
control success together with native-margin retention, token fidelity, entropy, repetition, and prompt-shift or domain-shift basin diagnostics. \\
a training objective widens the basin &
clean recovery, noisy-basin retention, margin and entropy calibration, and transport competence under states the prior actually visits. \\
\bottomrule
\end{tabular}
\end{table}

\FloatBarrier

\section{Conclusion}
\label{sec:conclusion}

Simple readouts preserve most native lexical decisions on the tested generated ELF and AURORA states. The scientific question is what the remaining disagreements change. Equal-count repairs show that disagreement count alone does not determine evaluator consequences; conditional-task tests and score accounting extend this result without making native decisions semantic truth or identifying rarity as an independent cause.

Predicting those consequences is a separate problem. Frozen forecasts fail to transfer uniformly, while a task-fitted selector improves BERTScore over random selection on fresh articles at a fixed native-call budget. All-native decoding remains higher-scoring, source-consistency benefit from rare-first repair remains unresolved, and neither the selection test nor the component optimizations establishes end-to-end acceleration.

The model families make the scope precise. LangFlow and CALM require fitting and state-distribution controls before a low readout score is interpreted. LLaCO requires distinguishing directional identity from the soft input consumed downstream. Cola-DLM and ELF training controls separate recovery of supplied states from generation through the prior. The reusable contribution is a sequence of tests: specify what is preserved, measure what the omissions change, and verify whether acting on that information improves the complete system. This supplies a basis for simpler future interfaces while keeping the open mechanism and efficiency questions explicit.

\section{Limitations}
\label{sec:limitations}

The added CALM experiment changes the state unit, vocabulary, corpus, and probe budget together; it does not isolate a causal effect of compression. Its fitting follow-ups are exploratory: the first uses 512 additional documents, and the fitting-population control reuses those documents with newly generated prior states after model selection. Conditional CALM-M predictions use temperature 1 and teacher prefixes rather than the recommended lower-temperature free-running setting. Matching the fitting population improves agreement only modestly under the tested budget, but all six fits still improve at the last checkpoint; convergence and an intrinsic linear-readout limit are not established. The LLaCO experiment audits the released geometry and directional adapter under forward corruption, not the full 8B backbone or task performance. Its exact-cache timing requires 4\,GiB additional memory, and its omitted-mass bound controls adapter input error rather than generated-token stability. Neither extension establishes a general generation algorithm or end-to-end speedup.

The context-only extension uses previously analyzed endpoints and synthetic donor or Gaussian displacements. It tests dependence of a fixed decoder mapping, not task performance or the mechanism of trajectory-level commitment failure. Large donor substitutions can be off distribution; their effects do not quantify the benefit of context on natural outputs. None of the 384 reused endpoints contains EOS, so the validity-mask follow-up does not cover EOS boundaries or shorter contexts. The blockwise margin condition is a standard sufficient bound with unmeasured segment-wise sensitivities, not a computed certificate.

The prospective extension transfers three fixed ELF-B readouts to 1,536 new trajectories from three released checkpoint conditions and two generation seeds. The calibration set was already inspected; only the new test generations are prospective. ELF-M and the distilled student change multiple training and architecture variables jointly and do not constitute independent denoiser training-seed repetitions. The quality predictors use native labels, native margins, and native evaluator scores; they are oracle diagnostics, not cheap deployable gates. The frozen margin threshold is not scale invariant across decoders. Failure of this specified ridge predictor would not rule out other features or nonlinear predictors, and success would not identify a causal origin for the tail. Equal-error repair uses native decisions and holds the native output boundary fixed. Its paired effects concern an external language-model score, not factual accuracy, human preference, or a learned decoder's end-to-end utility. Different T5 outputs can induce different evaluator tokenizations and truncation coverage. Confidence intervals retain paired sequence dependence across shared generation seeds and remain conditional on the selected checkpoints and fits.

The unconditional AURORA replication adds one independently released architecture and one released OWT denoiser checkpoint. Three affine fitting seeds and two generation seeds do not constitute repeated denoiser training. The 512-token prefix horizon is shorter than the released 1,024-token default, and the fitting pool and tokenizer differ from ELF. The primary count-based rarity rule and the prespecified relative-frequency sensitivity address distinct cutoffs; neither isolates architecture from fitting coverage. AURORA\textquotesingle s GPT-2-based interface and the primary GPT-2-Large evaluator belong to related model families; OPT-125M tests evaluator sensitivity rather than human preference. Each evaluator uses its own common scorable set without truncating any arm. Native EOS separators and latent trajectories are fixed across repair arms. The study tests a conditional score diagnostic at this interface, not native-length benchmark performance, factual correctness, or a deployable replacement decoder.

The paired XSum extension covers one short conditional task and one released conditional checkpoint per architecture. It transfers readout weights without task-specific fitting, but jointly changes task and denoiser training. Its two generation seeds do not replicate denoiser training. XSum references may themselves omit or distort source information; BERTScore and SummaC-ZS are imperfect automatic proxies, and their agreement or disagreement cannot establish human preference or factual accuracy. Native-oracle repair also requires knowledge of the errors it corrects. Main arms hold native stopping fixed, with a separate autonomous-readout control; this design diagnoses edit consequences and supplies no inexpensive deployable correction policy. The local freeze, mechanical boundary amendment, evaluator compatibility records, and all retained outcomes accompany the evidence.

The separate native-decoding selection test reuses the observed XSum cohort for task-specific risk fitting; only its 256 new articles and two generation seeds are prospective. Its primary endpoint is pooled BERTScore at one 50\% batch budget, not factuality, an online threshold, or out-of-task generalization. Intervals hold fitted weights and realized selections fixed. The stopping/canvas control is weak because its features are partly constant, so indirect prediction of native length remains possible. Both endpoints are computed for counterfactual evaluation; equal call counts do not establish equal compute, and shared-GPU stage measurements do not demonstrate end-to-end savings. Existing lexical probes omit EOS training, so the test does not compare optimized independently stopping decoders (\secref{app:native-free-routing}).

The later residual accounting, prediction decomposition, and confidence policies are exploratory follow-ups on previously examined cohorts. Equal native-token correction counts do not ensure equal evaluator-token changes; strict edit-dose subsets are output-dependent and cannot replace the original analysis. A fixed unigram reference tests one frequency-based explanation, not every evaluator bias. Adjacent-margin blocking retains confidence and source-overlap imbalance. High-margin repair still requires native labels and is tested on XSum with NLL and BERTScore, not factuality or human judgments. Its comparison with rare-first was added after observing the high-margin/random result and remains descriptive. The prediction decomposition is exact for squared error, not for MAE, and does not identify which simultaneous student changes cause the bias. The context adapter tests only rank eight and up to four preceding states. Its second normalization round was designed after observing first-round transfer failure, and neither round tests human quality or acceleration.

The study analyzes released checkpoints and targeted diagnostic training controls rather than training a new frontier DLM to convergence. This choice limits claims about ultimate generation quality. ELF-M and ELF-L are used for trend confirmation rather than a training-seed scaling law, because local PyTorch checkpoints and sampling settings may differ from Tensor Processing Unit (TPU)-scale evaluation. The same-interface scale audit narrows the claim rather than expanding it: SDE32 and SDE64 show earlier Margin-8 entry with larger ELF denoisers under the shared T5-small interface, but this is not an encoder-scale law. Raw cross-decoder probing between ELF and Cola is not a meaningful native test because their latent dimensions, tokenizers, scaling conventions, and decoder targets differ; learning a bridge would test adapter capacity rather than native decoder compatibility. The cross-decoder transfer results are restricted to ELF-B/M/L decoders that share the same T5-small interface and public checkpoint family, while the original two-seed boundary control and the expanded four-dose training grid are restricted to ELF-B. Neither supports a universal claim about T5 geometry or a cross-scale training-seed law.

The expanded grid has only two training seeds and a fixed five-epoch budget. Resumed shuffles and execution differences are detailed in \secref{app:v3-protocol}. Its equal-expected-flow comparison does not hold decoder exposure, learning-rate stage, or optimizer history fixed. Calibration has only eight held-out sequences per retrained checkpoint. Amplitude-aware and fixed-relative-scale controls are reported separately: four of nine pooled gradient-direction improvement intervals include zero, two fixed-scale gradient cells reverse the point-estimate ordering, and random directions show no uniform improvement. Joint phase-and-scale conditioning retains the gradient point-estimate ordering, but four of nine conditional improvement intervals also include zero. Observed flips below the first-order proxy radius explicitly exclude a certificate interpretation (\tabref{tab:calibration-threshold}). These exploratory comparisons do not establish calibrated flip probabilities or a global bound. The gradient perturbation arm is deliberately aligned to the same first-order boundary geometry used by the predictors.

The transport controls also have explicit boundaries. The three-seed sampler and teacher/student audits measure sampling-seed variation under fixed released checkpoints. The new five-epoch experiment adds two independently initialized, configuration-matched ELF-B pairs, replicating the $p=0.2$ versus $p=1.0$ separation within this two-seed study but not establishing a scaling law or an optimal decoder/flow mixture. CPS is a post-hoc sampler substitution without retraining, so it tests checkpoint--sampler compatibility rather than the best attainable CPS-trained language model. The PD comparison audits the released teacher and student distributions; low paired token agreement is not evidence about the unpublished training trajectory. The earlier 50k and staged runs remain mechanism controls rather than a matched recipe leaderboard. In particular, the staged run resumes model weights only and changes seed/batch configuration, so its diversity recovery and poor PPL support a boundary but not a clean causal estimate of curriculum order.

Positive logit scaling also limits comparisons of raw Margin-8 crossing times across checkpoints: these are operational threshold timings, not normalized robustness radii. The accompanying token-agreement and perturbation evidence supports their diagnostic interpretation; raw-margin timing alone cannot establish a scale-invariant robustness law.

The primary object here is also a fixed, already trained interface. The extended IRS checks make this boundary sharper: T5-base, T5-large, Switch, mT5, ByT5, and T5Gemma checkpoints are valuable controls, but swapping them into the audit changes dimensionality, tokenization, normalization, or decoder pairing. Their lower scores under the tested pairings therefore indicate interface mismatch, not a general defect of those models. In fully joint or self-organized latent systems, the representation and decoder basin may move during training; for those systems, the same axes should be used as an in-training audit of when order and decoder compatibility emerge, not as a post-hoc certificate inherited from ELF. The interface phase-transition model has the same fixed-checkpoint boundary. The fixed-checkpoint study supplies observable intervention constraints; deriving the pre-entry rank dynamics requires controlled retraining or architecture/noise-schedule interventions. The rank-origin and zero-SC controls constrain the pre-entry mechanism as far as this checkpoint allows, but deriving the rank slope or tail-token offsets from weights, or establishing whether a critical aligned rank exists, remains beyond this manuscript. The shard-level SEM bars use four independent sampling shards under fixed public checkpoints; they are appropriate for robustness visualization but not for formal training-seed hypothesis tests. MAUVE protocols differ by experiment: the primary distributional audit uses 2,048 samples; BGR uses 1,024 and a GPT-2-small featurizer; the new dose-grid endpoints use 1,024 and GPT-2-Large. We compare MAUVE only within its declared protocol. The closure OWT reference subset is not independently held out from training, so these scores do not establish unseen-data generalization.

There are also methodological limitations. Linear probes can underestimate nonlinear linguistic information, although this bias is acceptable here because the goal is to test minimal recoverability. GPT-2-Large PPL is an imperfect distributional proxy and can reward low-entropy text. We therefore use it only for within-family comparison and pair it with entropy, repetition, distinctness, MAUVE, reference-distribution distance, native-token agreement, and decoder margins. The study does not introduce a new general-purpose generator or claim a human-preference improvement over ELF; the XSum extension adds automatic task evaluation, while human evaluation, broader task coverage, and deployable interventions remain necessary before turning the probes into general generation methods. The basin definition is operational rather than a full global theory: it is a decoder-margin super-level set measured along generated trajectories, not a claim that the high-dimensional set is convex, spherical, or architecture-independent. Directional RBN shows that the basin is strongly anisotropic, so future theory should use local directional sensitivities or anisotropic norms rather than only isotropic Gaussian bounds. The fixed-phase analysis establishes incremental information after controlling several same-step observables, but omitted nonlinear confounders can remain. The paired basin intervention is a fixed-checkpoint test at one model size and one main intervention phase; it validates the sign of the control variable but does not replace a trained guidance objective.

Token-wise entry, ZSBD geometry ablation, and residual-tail targeting explain where the current probes work, and the 512-sample sequence-level audit identifies numeric, rare, ZSBD-wrong, and late-margin tokens as stress cases. AG News and Yelp Polarity corruption sweeps now provide a lightweight domain-shift check, but they are not a full OOD suite over code, long contexts, numeric exact match, prompt shift, or free-generation domain adaptation. The cross-scale tail supplement also has a limited role: its B/M/L pilots preserve exact native-token definitions, but the 64-sample M/L rows are tail-typing evidence rather than a formal residual-tail scaling law. Broad controlled-generation performance and long-form coherence remain beyond the short XSum extension and the local diagnostic: decoder margins and readout agreement measure local interface readability, and our topic and BCD audits show that local sentence coherence, control success, low PPL, and high decoder margin can trade off rather than improve together. A prompt or control signal should therefore be treated as a possible geometric distortion of the target basin, and the unconditional audits provide a reference baseline for the paired XSum test and broader controlled-generation studies. MDP is also a generated-manifold interface probe: its token-position training examples are drawn from generated sequences and are therefore correlated within samples, so the linear readout should not be interpreted as an independent and identically distributed (IID) token classifier or a universal decoder. The new fixed-test study repeats readout fitting, not denoiser training, on one historical split and one fresh generation seed. Its fixed update budget does not establish convergence or asymptotic sample efficiency. LangFlow changes tokenizer, sequence length, and the native linear-skip architecture; its approximately token-matched extension was chosen after inspecting the initial test results. These comparisons do not isolate architecture effects or supply an independent confirmation of the post-fit coverage explanation. The external checks are too few and too heterogeneous to support a rank correlation between basin width and PPL across architectures. The omission of full MDLM, Duo, flow-map, and LangFlow benchmark reproduction is deliberate: reproducing the system-level frontier would answer a different question from auditing representation--transport--decoder coupling. The GPT-2-small comparison is an unconditional sampling comparison rather than a full AR benchmark suite. The Cola-DLM analysis focuses on VAE and trajectory boundary behavior rather than full benchmark reproduction. Finally, our front-loaded schedule and BGEE provide evidence that trajectory signals are usable, but they are not competitive decoding algorithms. BGEE in particular reports NFE savings under an expensive diagnostic monitor; the direct monitored implementation is slower in wall-clock time than unmonitored sampling. This is why the hybrid-inference proposal in the Discussion is framed as future architecture design rather than as a claimed speedup in this manuscript. A local artifact manifest maps the core claims to scripts, inputs, checkpoints, and derived tables. A lightweight evidence supplement accompanies this revision with the new numerical controls and recomputation code; large historical arrays, training data, and checkpoints remain outside that supplement. The appendix records the experiment-specific protocols.

The additional sampler-extension checks are also bounded. The SC-update rejection experiment uses an equal $K=3$ candidate budget, but it is a candidate selector rather than an optimized wall-clock sampler; its PPL gains trade off against entropy, distinctness, and repetition. The lightweight margin proxy is a cheap pre-filter for native-margin monitoring, not a formal certificate. Basin-Gated Retransport improves a paired 32-step trajectory in all four held-out seeds, but its frozen calibration threshold triggers on $27\%$--$65\%$ of samples and it does not beat a compute-matched longer single trajectory: mean PPL is $21.84$ versus $20.44$, and pooled MAUVE is $0.941$ versus $0.955$. It is therefore evidence that low final margin identifies weaker draws, not evidence that candidate resampling solves transport. Hard Token-wise Early Commitment and one-step projection are negative controls under the tested implementations, not impossibility results for all token-wise commitment or few-step projection methods. Clock Diffusion trains position-dependent noise schedules for semi-autoregressive progress and confidence-based commitment~\citep{schiff2026clock}; our post-hoc freezing intervention does not test that trained setting. Together, these checks show that the interface phase diagram cannot be compiled into a sampler by local thresholding alone.

The fixed-NFE distill--diversify--select control has a similarly explicit boundary. Four independent generation seeds show that the released progressive-distillation student makes multiple 8-step proposals viable, and both OPT-125M and Qwen2.5-0.5B selectors reduce held-out GPT-2-Large PPL. However, selector computation is not included in denoiser NFE, and the strongest 8-by-4 selection lowers pooled MAUVE and diversity while increasing repetition. This is a falsifiable direction for diversity-aware selection, not a claimed free quality improvement or a complete sampler.

\FloatBarrier

\section*{Acknowledgments}
We thank the ELF authors for releasing the checkpoints and code studied here, and refer readers to the ELF release for the compute resources used to train those checkpoints. We thank the Cola-DLM team at ByteDance, the LangFlow team at UIUC, the AURORA authors, and the CoBit authors for releasing their models and code under open-source licenses. We also thank the creators of OpenWebText, Yelp Polarity, AG News, and the T5, mT5, ByT5, Switch Transformer, T5Gemma, BERT, RoBERTa, and GPT-2 pretrained models for making their datasets and checkpoints publicly available. Our diagnostics and targeted training controls ran on four-GPU workstations and a four-A800 evaluation server, with later continuation and closure audits on RTX 5090 GPUs; we did not retrain the released ELF models at their original scale.

This work was supported by Development and Reform Commission of Shenzhen Municipality (F-2024-Z99-503929), Shenzhen Medical Research Fund (C2401007), National Natural Science Foundation of China (82220108017, 82141128), The Capital Health Research and Development of Special (2024-1-2052), Science \& Technology Project of Beijing Municipal Science \& Technology Commission (Z201100005520045), and Sanming Project of Medicine in Shenzhen (SZSM202311018); the authors gratefully acknowledge this support.

\bibliographystyle{plainnat}
\bibliography{references}

\appendix
\renewcommand{\thefigure}{S\arabic{figure}}
\setcounter{figure}{0}
\renewcommand{\thetable}{S\arabic{table}}
\setcounter{table}{0}

\section{Extended Related Work}
\label{app:extended-related-work}

\parhead{Continuous diffusion language models}
Recent surveys on diffusion language models and parallel text generation position this field as part of a broader shift toward iterative and non-autoregressive generation~\citep{dlmsurvey2025,paralleltextsurvey2025}. Diffusion language models began largely in discrete or masked-token spaces~\citep{d3pm2021}, where the state can be interpreted as a partially corrupted token sequence. Early continuous and simplex-based language diffusion systems, including Diffusion-LM, DiffuSeq, SSD-LM, and latent language diffusion~\citep{diffusionlm2022,diffuseq2022,ssdlm2022,latentdiffusionlanguage2022}, showed that iterative denoising can be adapted to text generation and control, while later masked diffusion work such as MDLM and scaled masked DLMs~\citep{mdlm2024,scalingmdm2024} improved the discrete-token recipe. Continuous DLMs move the corruption process into an embedding or latent space instead. Because this area is moving quickly, several 2025--2026 references below are recent arXiv preprints; we use them to describe methods and scope, without treating their reported results as replications of our experiments. ELF~\citep{elf2026} performs flow matching in a frozen T5~\citep{t52020} contextual embedding space and discretizes only at the final step~\citep{flowmatching2022}. This design is intentionally close to the image diffusion recipe~\citep{ddpm2020}: stay continuous during denoising, then decode once. It also makes ELF an especially clean interface test: the representation space is supplied by a pretrained encoder rather than bootstrapped end to end, so the central question becomes why this fixed representation--decoder pair is diffusion-ready. ELF-S2T extends the same ELF backbone to speech recognition and translation by conditioning on a frozen Whisper encoder through a linear projector, extending the interface to conditional speech tasks~\citep{elfs2t2026}. Cola-DLM~\citep{cola2026} makes a different choice: it learns a compact Text VAE and trains a block-causal diffusion Transformer (DiT)~\citep{dit2023} prior over VAE latents. TextLDM~\citep{textldm2026} similarly imports latent diffusion ingredients into text, including alignment and latent modeling components. CoBit~\citep{bitstream2026} explores another route by removing the $\mathcal{O}(V)$ vocabulary bottleneck and using entropy-gated continuous bitstreams. Token Time Continuous Diffusion (TTCD) instead assigns token-specific continuous times so positions can progress from Gaussian noise at different rates~\citep{ttcd2026}. Tooling work such as dLLM~\citep{dllm2026} reflects the same trend toward reproducible DLM training and evaluation. These systems agree that non-autoregressive denoising can be useful for language, but they leave open a more basic question: what property of the state space makes it useful, and when does the decoder become the bottleneck?

\parhead{Conditional generation and the scope of readability}
Recent conditional models sharpen the boundary between decoding and solving a task. CEDR combines teacher representations, asynchronous denoising, and task-specific training for reasoning~\citep{cheng2026cedr}. These task-trained reasoning objectives differ from our OWT interface audits and fixed-readout transfer on XSum. Noise Your Prompt improves tested combinatorial and modest-data summarization tasks through training-time corruption of conditioning tokens, while reporting no gain on open-ended dialogue~\citep{jung2026noiseprompt}. These results motivate separate measurement of prompt information, transport, decoded-response fidelity, and task success. Neither our unconditional audit nor token stability alone establishes reasoning ability; our contribution is the controlled interface diagnosis and its failure boundaries.

\parhead{AURORA-LM and asymmetric interface design}
AURORA-LM~\citep{aurora2026} constructs a high-capacity decoder-facing latent, freezes that interface, and restricts only the noisy-input pathway while retaining a full-width clean-latent prediction target. Its high-noise reweighting and self-trajectory consistency further address the difficulty of reaching that target. This architecture separates the capacity of the decoder-facing target from that of the noisy-input pathway. Our study instead measures target compatibility and transport under the released ELF interface; it does not reproduce AURORA-LM training or establish the cause of its reported gains.

\parhead{Transport-side ELF extensions}
DeltaFlow~\citep{deltaflow2026} replaces ELF's dense bidirectional-attention denoiser with noise-adaptive bidirectional gated delta networks and adds scheduled temporal-state consistency across nearby noise levels. It reports lower 32-step PPL at comparable unigram entropy together with a denoiser-throughput gain at long sequence length. It changes the transport architecture and training regularization. Those system results address a different question from our fixed-interface diagnostic interventions.

\parhead{ELF and Cola-DLM as complementary contemporaneous systems}
ELF and Cola-DLM appeared in the same rapidly developing period but instantiate different generative objects. ELF diffuses a length-preserving sequence of frozen contextual states and applies its native shared-weight token readout at the endpoint. Cola-DLM first learns a compressed Text VAE interface, then uses a block-causal DiT to transport a global latent prior, and finally realizes local text through a conditional decoder. Consequently, exact recovery of a held-out token block is a natural ELF-style endpoint diagnostic but is not the generative objective of a free-running Cola prior. Our Cola experiments respect this distinction: the VAE corruption sweep measures robustness of the learned representation--decoder pair, while the DiT trajectory audit separates decoder confidence from teacher-target alignment. Cola is therefore a learned-latent boundary case, not an inferior ELF baseline and not an independent replication of ELF's native-margin phase diagram. This comparison is important precisely because it shows which interface questions survive after the state semantics change: clean reconstruction, prior transport, decoder confidence, and target alignment must be reported as different quantities.

Recent work also tries to remove the frozen-interface limitation rather than only analyze it. Latent Diffusion Language Model (LDLM)~\citep{ldlm2026} jointly trains a latent encoder, diffusion prior, and decoder, and reports that decoder-input noise, MSE decoder loss, warmup, and adaptive timestep sampling are all important for robust latent diffusion. ConvergeFlow~\citep{convergeflow2026} instead constrains its predictor to the convex hull of token embeddings and proves asymptotic convergence in probability to valid token embeddings under its stated regularity and time-grid conditions. Its convergence result concerns valid terminal embeddings under a co-designed predictor and transport process; it does not equate local decoder readability with the quality of an arbitrary learned prior. DiHAL~\citep{dihal2026} asks where diffusion should enter a pretrained transformer, using geometry-based layer scores and hidden-state replacement to avoid direct token-level recovery. RePlaid~\citep{replaid2026} shows that likelihood-trained continuous diffusion can scale competitively with discrete diffusion when the architecture and training protocol are modernized, while DiLaDiff~\citep{diladiff2026} adds a distilled latent-augmented prior to accelerate masked diffusion language modeling. Flow Map Language Models (FMLM) study distillation toward one-step generation~\citep{fmlm2026}; a subsequent fixed-point-flow formulation interprets self-conditioning as an iterative solver and distills both its flow and fixed-point dimensions~\citep{fixedpointflows2026}. Categorical Flow Maps and their scaled variants~\citep{cfm2026,scalingcfm2026} show that Gaussian-to-categorical flow maps can reach few-step text generation at much larger scale. Consistent Diffusion Language Models (CDLM)~\citep{cdlm2026} adapt consistency ideas to discrete diffusion through multi-path training. Coevolutionary Continuous-Discrete Diffusion (CCDD)~\citep{ccdd2025} argues that continuous diffusion has strong expressive power but that practical trainability and token decoding remain obstacles. These methods make different choices about latent states, prior dynamics, and token realization; their results are not a controlled comparison of those choices. \tabref{tab:related-interface} summarizes this literature through the interface lens used in this paper: each family exposes a different part of the state--decoder compatibility problem, and ELF is the primary fixed-checkpoint testbed used in our time-resolved basin audit.

Two updates make the transport distinction directly testable. ELF v2 adds progressive-distillation results based on the method of \citet{salimans2022progressive}, and the official release provides a pretrained teacher and final-round student evaluated over 1/2/4/8/16/32 steps~\citep{elf2026}; we audit their trajectories rather than treating few-step PPL as a black-box speed result. The v2 recipe is especially informative for our decomposition: each student step matches multiple substeps of a fixed 64-step teacher, the denoising MSE branch is replaced by a displacement-derived distillation target, and the original cross-entropy decoding branch is retained. It is therefore a transport-targeted curriculum that preserves the model's two-branch representation/decoder design, although both branches' shared weights continue to update. Coefficients-Preserving Sampling (CPS) derives a stochastic flow sampler that preserves the intended signal/noise coefficients~\citep{coeffpreserving2025}; we use it as a sampler counterfactual, not as an assumed correction for language. A contemporaneous masked-DLM study similarly reframes decoding as continuous $x$-prediction with token-wise asynchronous progress~\citep{maskedxflow2026}. Its state space is different, but its motivation is consistent with our token-wise finding: positions accumulate confidence at different rates, so premature hard commitment can destroy useful partial progress.

ELF already establishes the system-level Gen. PPL--entropy frontier against discrete and continuous DLM baselines such as MDLM, Duo, flow-map systems, and LangFlow~\citep{elf2026,mdlm2024,duo2025,fmlm2026,cfm2026,scalingcfm2026,langflow2026}. We therefore do not reproduce that frontier or re-benchmark MDLM, Duo, flow-map baselines, or LangFlow here. When LangFlow, CoBit, and Cola-DLM appear later, they are audited boundary systems, not ELF-frontier reproductions. The reason is diagnostic rather than logistical: these systems expose different native state objects and decoder or proxy interfaces. Masked or uniform-token diffusion crosses token boundaries directly, while flow-map, continuous-latent, bitstream, and VAE systems define different state variables and readouts. We therefore apply the same diagnostic axes where an interface or interface proxy can be measured, using external models only as boundary checks rather than benchmark competitors.

\begin{table}[!htbp]
\centering
\footnotesize
\caption{Related systems viewed through the representation--decoder interface. The table is not a benchmark comparison; it separates which part of the interface problem each line primarily exposes and how it is used in this diagnostic study.}
\label{tab:related-interface}
\begin{tabular}{>{\raggedright\arraybackslash}p{0.24\linewidth}>{\raggedright\arraybackslash}p{0.27\linewidth}>{\raggedright\arraybackslash}p{0.39\linewidth}}
\toprule
Work family & Interface question & Role in this paper \\
\midrule
Discrete or masked DLMs~\citep{d3pm2021,mdlm2024,scalingmdm2024,duo2025} &
Token or mask states avoid Gaussian hidden-state corruption, but logits still cross a token boundary. &
Background for trajectory reliability and PPL-only evaluation pitfalls. \\
Early continuous or latent DLMs~\citep{diffusionlm2022,diffuseq2022,ssdlm2022,latentdiffusionlanguage2022} &
Embeddings, simplex states, or latents can be denoised, but often need per-step token losses, rounding, or a separate decoder. &
Motivates separating denoisability from decoder compatibility. \\
ELF~\citep{elf2026} &
Frozen T5 contextual states plus a shared native decoder define a clean fixed interface. &
Main testbed for time-resolved decoder-basin entry, scale/sampler audits, and minimal probes. \\
Learned or co-evolved latent interfaces~\citep{cola2026,textldm2026,ldlm2026,replaid2026,aurora2026} &
Clean reconstruction or latent alignment may not imply robustness under the prior's noisy states. &
Context for testing decoder robustness under corruption alongside clean reconstruction. \\
Alternative native interfaces~\citep{langflow2026,bitstream2026} &
Latent-step margins and bit-level recovery use different native units from ELF's T5 decoder. &
External checks showing that the same diagnostic questions remain meaningful without claiming benchmark rank. \\
Decoder-bottleneck and representation-first views~\citep{codar2026,zhang2026latent,tokenopt2025} &
Final projection, decoder readability, or tokenizer geometry can dominate what downstream generation can recover. &
Closest conceptual context; this paper adds fixed-checkpoint, time-resolved basin diagnostics and BGEE/ZSBD/MDP probes. \\
\bottomrule
\end{tabular}
\end{table}

\parhead{Latent interfaces and decoder bottlenecks}
In image latent diffusion, the decoder maps continuous latents to pixels, and small pixel-level errors often remain visually tolerable~\citep{ldm2022}. Text is different. The decoder crosses a discrete boundary, so local changes in latent space can abruptly change token identity, syntactic role, or lexical frequency. CoDAR~\citep{codar2026} studies this issue from an architectural angle, identifying token rounding and final projection as bottlenecks and replacing the final projection with a contextual autoregressive (AR) decoder. Our work studies the same broad interface problem from a diagnostic angle. We do not introduce a stronger decoder; we ask how to measure whether a representation and native decoder are already compatible, how PPL can be misleading, and whether the denoising trajectory exposes reliability before decoding. Our decoder calibration and Cola-DLM boundary experiments are motivated by this difference. They ask whether a low reconstruction error, low PPL, or clean latent recovery actually indicates that the prior is handing states to the decoder in the right region.

A contemporaneous technical report by \citet{zhang2026latent} studies a related decoder-readability failure mode from another angle. Its setting is draft-conditioned latent refinement with a frozen BERT~\citep{bert2019} encoder on ROCStories~\citep{rocstories2016}: a draft latent is refined by a DraftPrior, FlowNet, and MetricNet, and the report observes that high latent cosine similarity need not yield reliable token distributions. Its draft-conditioned setting differs from the unconditional generation studied here. It does not audit released ELF checkpoints, full denoising from Gaussian noise, native-margin basin entry over time, cross-scale/sampler phase diagrams, or the residual-tail structure after entry. Our contribution is diagnostic and time-resolved: we define decoder-basin entry through native margins, measure when ELF trajectories enter that basin, and test what becomes simple after entry through BGEE, ZSBD, and MDP.

Several recent results in vision suggest that generative modeling depends strongly on representation geometry. REPA aligns diffusion Transformer representations with pretrained encoders to make training easier~\citep{repa2025}. Dispersive regularization improves image generation by shaping hidden representations~\citep{dispersive2025}. Noise-conditioning ablations and masked-autoencoding-style simplification studies challenge assumptions that once seemed necessary~\citep{noisecond2025}. Drifting Models~\citep{drifting2026} go further by moving iterative distribution evolution from inference to training, naturally yielding one-step generation. These studies share a methodological pattern: simplify the system, identify which representation property matters, and then strip away machinery only when the evidence supports it. Our work follows that representation-first pattern for language diffusion, but with an additional decoder-margin issue created by discrete tokenization.

\parhead{Contextual representation geometry}
The geometry of contextual language representations is itself highly structured. Prior analyses show that contextual embeddings are anisotropic and occupy narrow regions of representation space rather than behaving like isotropic semantic vectors~\citep{ethayarajh2019contextual}. This background motivates, but does not settle, our controls. A covariance spectrum can make denoising easy under MSE, while whitening or covariance matching can preserve second-order structure without preserving linguistic recoverability. Our diagnostic therefore treats covariance and anisotropy as possible confounders to control, rather than as explanations by themselves.

Recent theory also gives a complementary reason to care about latent prediction itself. \citet{korchinski2026latents} show in a hierarchical grammar model that predicting learned latents can recover hidden compositional structure with far fewer samples than token-level prediction. Our experiments are not a sample-complexity proof, but they add a practical constraint to this latent-first view: a latent can be statistically useful and still fail if the denoising trajectory does not enter the native decoder's basin.

TokenOpt~\citep{tokenopt2025} provides a useful analogy from vision: a highly compressed 1D image tokenizer can support generation and editing through simple test-time token manipulations, without training a generative model. The lesson is not that language should copy the same pipeline, but that a good tokenizer or representation can make a crude downstream procedure powerful. Our BGEE, ZSBD, and MDP experiments follow the same diagnostic-to-minimal-probe philosophy in a different domain. They are three probes of one object, the decoder basin: BGEE measures when the trajectory reaches it, ZSBD asks whether final states align with token-embedding Voronoi cells, and MDP asks whether the native token decisions are locally linearly recoverable on the generated-state manifold. The resulting procedures are intentionally modest. Their role is not to define standalone decoding methods, but to expose what the representation-decoder interface has already made easy for a cheap downstream readout.

Representation alignment work shows that pretrained autoregressive models can be adapted to diffusion-style generation without full retraining~\citep{align2026}. This supports the broader idea that existing language representations contain useful structure for non-autoregressive generation. However, reusing a representation is not the same as validating an interface. A representation can be aligned with a source model while still being hard for a target decoder to recover from noisy states. Our diagnostic protocol is therefore complementary: it can be used before training to decide whether a proposed encoder, tokenizer, or VAE latent is a plausible state space for diffusion.

Classifier-free guidance~\citep{classifierfree2022}, entropy-aware reward guidance~\citep{entrgi2026}, and progress-aware policies such as Prophet, SchED, Fast-dLLM, Just on Time, DAWN, DMax, AHD, Dystruct, locally coherent parallel decoding, and DCDM~\citep{prophet2025,sched2025,fastdllm2025,jot2026,dawn2026,dmax2026,ahd2026,dystruct2026,lcpd2026,dcdm2026} all exploit the fact that diffusion trajectories contain information before the final output. Recent discrete or block-wise DLM accelerators make the same pressure explicit: Streaming-dLLM skips converged tokens, Roll Out and Roll Back uses verified denoising orders as efficiency teachers, and BlockBatch uses multi-scale consensus to reduce denoising NFEs~\citep{streamingdllm2026,rolloutrollback2026,blockbatch2026}. Trajectory-Aware Commit Gating distinguishes instantaneous confidence from cross-step commitment readiness, while C$^4$ coordinates sequence-level early exit with local candidate-aware commitment~\citep{tacg2026,candidatecommit2026}. FReDA trains recursive refinement on model-generated drafts without a prescribed forward corruption process, detaching preceding refinement passes during training~\citep{freda2026}. K-Forcing studies parallel generation through a different training formulation~\citep{kforcing2026}. These works often optimize inference behavior directly through confidence thresholds, token freezing, dependency-aware unmasking, history-stable decoding, dynamic structure, semantic chunking, or cache-aware parallel decoding. Our trajectory-reliability experiments are narrower: we ask whether internal signals from the existing checkpoint predict final sample quality and whether a small time-region intervention validates the signal. This distinction matters because an optimized sampler can improve metrics without explaining why a representation-decoder interface works.

\parhead{The PPL illusion in diffusion language models}
Generated-text PPL (Gen. PPL) has known failure modes in diffusion language modeling. In discrete masked diffusion, \citet{zheng2025masked} show that masked diffusion training and sampling can be interpreted as time-agnostic masked modeling, and identify a numerical issue in Gumbel-based categorical sampling: finite-precision truncation effectively lowers sampling temperature, reducing diversity and making PPL-only comparisons unfair. In continuous DLMs, \citet{repetitionattractor2026} identify a dominant self-conditioning direction associated with repetition and propose a training-free correction. That intervention targets a specific repetition mechanism; our PCA, decoder-dose, and readout controls test distinct ways that PPL can diverge from diversity or native-token fidelity. A continuous or latent DLM can obtain low PPL by moving into a low-entropy decoder region, repeating high-frequency text, or changing calibration at the final token boundary. The common structure is that a single scalar can indicate fluency while hiding diversity, fidelity, or interface failure. This is why our evaluation pairs PPL with entropy, repetition, MAUVE score (for measuring the gap between neural and human text distributions)~\citep{mauve2023}, Jensen-Shannon (JS) divergence~\citep{lin1991js}, decoder agreement, and margin recovery.

Recent masked-DLM samplers approach diversity directly: Orthogonal Diverse Diffusion repels parallel candidates in feature space, and two-temperature sampling softens confidence-based remasking to improve cost-normalized exploration~\citep{odd2026,twotemperatures2026}. These methods motivate a useful distinction for our setting. Diversity can be created inside a sampling path, or fixed denoising compute can be divided among independently distilled short paths and selected afterward. We test the latter only as an appendix boundary: its low-PPL gains do not automatically imply a better generated distribution.

\parhead{AR-integrated diffusion and self-speculation}
Nemotron-Labs-Diffusion~\citep{fu2026nemotronlabsdiffusion} unifies autoregressive, diffusion, and self-speculation decoding in a single model. Such systems may blur the line between AR and diffusion language modeling, but they do not remove the need for diagnostics. If diffusion becomes an operating mode inside a larger model, one still needs to know when the diffusion state is reliable, when a decoder or verifier is calibrated, and which objective frontier is being optimized. Our results suggest a useful future audit for tri-mode models: test whether their diffusion branch exposes similar trajectory-reliability signals and whether their speculative decoder has a broad or narrow decision basin.

\parhead{Mechanistic diagnostic studies}
Recent work on transformer internals, such as the finding that adjacent layers systematically correct each other's predictions~\citep{tlcm2025}, demonstrates the value of diagnostic studies: identify a surprising phenomenon, rule out trivial explanations, localize where and when it appears, and synthesize the evidence into a compact mechanism. Closest to our trajectory analysis, \citet{temporal2026} measures when token commitment, linguistic probes, confidence, entropy, and re-masking sensitivity emerge in LLaDA~\citep{llada2025} trajectories. Our study asks a complementary interface question in continuous ELF-style latents: when does the trajectory enter a decoder-readable basin, and what does the native decoder still contribute after entry? What the above work leaves open is not which sampler or architecture is best, but a simpler diagnostic question: does the denoising trajectory enter a decoder-readable basin, and how would we know? Our protocol and experiments make this interface question measurable.

\parhead{Concurrent interface and transport studies}
Probe guidance uses frozen internal states to steer language flows and links useful guidance to low-entropy training checkpoints~\citep{dilip2026guidelanguageflow}. Untied self-conditioning studies train--inference mismatch in few-step flow sampling~\citep{li2026untiedselfconditioning}, while Flow Duality makes source geometry an explicit part of categorical flow behavior~\citep{haxholli2026flowduality}. These are complementary transport and source-design questions. Our contribution is the measured separation between handed-state target compatibility, local stability predictors, and the quality of states actually generated by a fixed native interface.

\parhead{Boundaries of fixed-interface controls}
Recent encoder scaling and distillation studies retrain matched generators around their representations~\citep{zhang2026diffusibility}; they address a different question from inserting an alternative encoder into an unchanged decoder. Likewise, learned block compression~\citep{feng2026branching} changes the sequence representation and trains a compatible interface, whereas our PCA control removes feature dimensions after training. Work on the design space of continuous DLMs studies self-conditioning as a response to train--inference mismatch~\citep{davis2026designspace}; disabling it at evaluation in an existing checkpoint does not test that training explanation. StepTune optimizes timestep placement on calibration data at a fixed sampling budget~\citep{zheng2026steptune}. The failure of our particular frozen sampler shortcuts therefore does not rule out better schedules. Finally, AudioCache reuses attention and feed-forward computation across diffusion steps for audio generation~\citep{shi2025audiocache}. Such computation reuse is complementary to reducing decoder-monitor calls or stopping generation; its audio results do not establish a text-quality guarantee. These distinctions constrain what our negative controls support without importing performance claims across different training settings.

\parhead{What recent generative improvements change}
SERF supplies a cleaner input on the same interpolation path to an EMA teacher and adds representation or token-distribution distillation to the training objective~\citep{kouzelis2026serf}. It tests whether trajectory supervision improves a trained model; our post hoc lexical heads leave the trajectory unchanged. Rectified Language Flow introduces singular logits, fixed Hellinger--SVD embeddings, self-conditioning guidance, and regularized reflow~\citep{su2026relf}. Its endpoint-convergence analysis concerns a co-designed token-embedding predictor under stated assumptions. It is distinct from a local stability bound for an existing contextual decoder, and neither result identifies the quality of an arbitrary learned trajectory. These methods reinforce the need to state which components are held fixed in a negative result.

\parhead{Which tail is being measured}
GRAFT compares native generations with a human word-frequency reference, then commits selected rare words before regenerating~\citep{yang2026graft}. Its object is a distributional lexical deficit. Our object is a conditional discrepancy: on the same latent, which native decisions does an affine substitute fail to reproduce, and what follows from restoring equally many of them? A native decision can be a useful fidelity target without matching the reference text distribution or being factually correct. Consequently, our rare-first repair does not establish a remedy for lexical underproduction, and a distribution-level intervention is not a matched test of our terminal-repair policy.

\parhead{Trajectory information and learned commitment}
dSpark learns commitment schedules using their effect on the uncertainty of unresolved positions and distills the resulting behavior~\citep{myrzakhan2026dspark}. TDLM instead trains temporal attention over an unresolved position's hidden-state history within a frozen masked denoiser~\citep{sun2026tdlm}. Both change what later decoding steps can use; our unsuccessful spatial-context residual does not test either mechanism. A separate study relates within-run convergence to uncertainty across repeated endpoints in masked DLMs~\citep{amador2026endpoint}. Cross-run concentration, correctness, native-token fidelity, and excess evaluator NLL are distinct targets. A positive predictor for one does not validate a predictor for another, and endpoint-referenced descriptions must be distinguished from features available at an actual decision time.

\FloatBarrier

\section{Extended Discussion and Falsifiable Predictions}
\label{app:extended-discussion}
\extendeddiscussion
\FloatBarrier

\section{Reproducibility Details}
\label{app:reproducibility}

\parhead{Code organization}
The artifact is organized by scientific role rather than by the order in which experiments were run. Representation scripts produce the five-axis readiness controls; trajectory scripts record native margins, agreement, entropy, self-conditioning, and effective-rank summaries; probe scripts implement BGEE, ZSBD, and MDP; and transport-control scripts run the official/CPS/ODE sampler audit, released teacher/student audit, decoder-objective training controls, OOD corruption sweeps, Cola linear-denoising boundary, BGR compute-matched check, coordinate-geometry intervention, and fixed-NFE PD allocation. Final aggregate CSV/JSON files used by \tabref{tab:transport-controls}, \figref{fig:matched-fulltrain}, \tabref{tab:matched-fulltrain-endpoints}, \tabref{tab:pd-full-sweep}, \tabref{tab:pd-official-reproduction}, \tabref{tab:decoder-transport-boundary}, \tabref{tab:bgr}, \tabref{tab:pd-diversify-select}, and \tabref{tab:q2-coordinate-mechanism} are retained in the local experiment archive. The manuscript compilation archive contains the figures, tables, and bibliography, not the underlying datasets or checkpoints. The matched-training artifact additionally preserves both audited configuration pairs, completion markers, and final per-domain summaries. This organization mirrors the claim hierarchy: representation defines the target, trajectory analyses test entry, and transport controls test sufficiency.

\parhead{Terminology}
We use \emph{state space} for the continuous vectors being denoised, \emph{native decoder} for the decoder trained with that state space, and \emph{interface} for the pair. We reserve \emph{diffusion-ready} for the full interface property, not merely for smooth latent geometry. This terminology matters because several negative controls are smooth but not linguistically meaningful, or decodable when clean but not robust under corruption.

\parhead{Evidence hierarchy}
The historical supporting studies distinguish three evidence levels. \emph{Interface measurements} use representation controls, PCA controlled degradation, and fixed-checkpoint trajectory audits with shard-level robustness. \emph{Minimal probes} such as BGEE, ZSBD, and MDP test whether these measurements expose actionable variables; they are not presented as tuned production methods. \emph{External architecture} checks on Cola-DLM, LangFlow, and CoBit are boundary diagnostics rather than full benchmark reproductions. The later frozen prediction, equal-count repair, AURORA replication, and conditional-task studies have their own protocols and evidence scopes. Throughout the appendix, sample counts, shard counts, and checkpoint/training-seed assumptions are documented rather than inferred.

\parhead{Tool disclosure}
Generative AI tools assisted with diagnostic ideas, implementation and debugging of audit scripts, organization and interpretation of exploratory results, candidate-literature discovery, figure preparation, and manuscript editing. The conceptual diagrams in \figref{fig:overview}, \figref{fig:external-interface-overview}, \figref{fig:three-basin-probes}, and \figref{fig:inside-decoder-basin-summary} were generated with an image-generation tool and checked against the stated protocols. Quantitative result plots were rendered from recorded measurements using Python. AI outputs were not treated as scientific evidence or bibliographic authority. The revision checks cited source records and saved experimental outputs; responsibility for the final claims, code, figures, and submitted text remains with the authors.

\parhead{Local-data policy and artifact manifest}
Datasets, checkpoint files, generated samples, logs, and aggregate results remain inside the working environments used for this study; experiment tracking and model uploads were disabled. A machine-readable local manifest records source paths, file sizes, modification times, and SHA-256 hashes for files below the configured size threshold. It maps each core table to its scripts and derived CSV/JSON inputs without copying or uploading data. Public packaging is deferred until the experimental program is frozen, but the current local artifact is sufficient to audit every number introduced by the final transport controls.

\parhead{Checkpoint loading}
We use the official PyTorch~\citep{pytorch2019} ELF code and public Hugging Face Transformers checkpoints~\citep{transformers2020}. The ELF-B, ELF-M, and ELF-L checkpoints are loaded with strict key matching against the official model definitions. The JAX~\citep{jax2018} checkpoints are kept for reference, but all hidden-state and trajectory diagnostics in this manuscript use the PyTorch implementation because it exposes intermediate states more conveniently on local GPUs. We do not compare PyTorch and JAX numerics; the claims concern the released checkpoints and the representation--decoder interface, not framework-specific implementation differences.

\parhead{Final sampler and ELF v2 distillation audit}
The final P0/P1 audit pins the released artifacts as follows:
\begin{itemize}[leftmargin=1.4em,itemsep=0.1em,topsep=0.1em]
\item official PyTorch/distillation source commit: \path{346f00bca2920f950a1153c5f37ead861e313049};
\item ELF-B teacher \path{checkpoint_95085}, SHA-256:\\[-0.2em]
\path{52c21bfb7f920e785e0e780645c411c94e4541af4cee8f85b6f555f826d4c283};
\item final-round ELF-B-PD student \path{checkpoint_19017}, SHA-256:\\[-0.2em]
\path{c593da8527f0a207568b7e81c7f39011f5a6f262df909db61b4d9c7055e7f8c2}.
\end{itemize}
The recipe follows the \href{https://github.com/lillian039/ELF/tree/distillation}{official ELF distillation branch}. Sampler and teacher/student trajectory summaries use seeds 42--44 with 256 samples per seed; paired sampler arms share schedules, initial Gaussian states, and compatible innovation streams. The independent official student reproduction uses 1000 samples at each of 1/2/4/8/16/32 steps. Evaluation batch size is 64 for the teacher/student runs after an end-to-end memory smoke test and 16 for every sampler arm; batch size changes throughput only and is held fixed within each comparison.

\parhead{Trajectory extraction}
The official generation loop is scan-like and does not expose intermediate signals by default. Our audit wrapper runs the same SDE/ODE updates explicitly and records selected state summaries. We record per-sample final decoder entropy, self-conditioning delta, agreement with zero-self-conditioned predictions, and effective-rank proxies. To keep memory bounded, full hidden states and centered kernel alignment (CKA)~\citep{cka2019} is computed only at selected time steps or on subsampled rows.

\parhead{Basin-navigation audit}
For the mechanism audit, we run ELF-B with 32-step SDE and SC=3 on 512 samples split evenly across four RTX 3090 GPUs. At every step we decode the predicted clean latent $\hat{x}_t$ and record top1-top2 margins, decoder entropy, self-conditioning disagreement, and token-level delta-margin correlations. A second 64-sample audit stores selected-step logits to measure margins to the final decoded tokens and agreement with the final sequence. This two-pass design keeps the full mechanism audit cheap while still giving a direct decoder-basin crossing test.

\parhead{Fixed-phase incremental analysis}
The selected-step audit records one row per sample at steps 0, 4, 8, 12, 16, 20, 24, 28, and 31. For every nonsaturated step with at least 32 complete samples, we fit two deterministic five-fold sample-level regressions for eventual native-token agreement. The baseline uses current decoder entropy, mean self-conditioning change, and relative step drift; the augmented model adds current margin p10. Standardization is fitted only on each training fold, and we report the change in out-of-fold $R^2$. We also rank-transform the variables, residualize both outcome and margin against the baseline controls, and report their residual Spearman correlation. A complementary 128-sample fixed-sampler analysis predicts per-sample GPT-2-Large log PPL from late and maximum self-conditioning change, then separately adds final margin p10 or mean margin. These are incremental-association analyses: fixed phase removes the direct time trend and out-of-fold evaluation limits in-sample overfitting, but neither procedure identifies a causal effect. Intervention effects are measured separately in matched trajectory and training comparisons.

\parhead{Cross-scale and sampler audits}
The same audit is repeated for ELF-B SDE64, ELF-M SDE64, ELF-L SDE64, and ELF-B ODE32 with 512 samples per run. The plots use normalized denoising phase rather than raw step index so that 32-step and 64-step trajectories can be compared directly. For external checks, LangFlow is audited with the same step-margin interface where available, and CoBit is evaluated through a sigma sweep over structured proxy codes and a repaired-tokenizer subset of 512 real OWT code sequences.

\parhead{Basin-Guided Early Exit}
BGEE uses the native decoder's 10th-percentile predicted-clean margin as a stopping signal. We evaluate fixed exits and margin-threshold exits on the same 512 ELF-B SDE64 samples and on 256-sample ELF-M/L SDE64 confirmations. Quality is reported with geometric PPL, arithmetic mean PPL, sample entropy, distinct-2, repetition, and token agreement to the full 64-step decode. Main plots show simple interpretable margin levels. As protocol-rigor checks, we split generated samples in half, select a gate on the validation half subject to token agreement and geometric-PPL constraints, and report held-out test metrics. We also run a 512-sample random-time control that shuffles eligible denoising steps before applying the same margin or margin-plus-entropy rule; this separates temporal ordering from the threshold itself. The wall-clock timing smoke test uses four 32-sample GPU shards and three repeats per shard. It compares unmonitored 64-step sampling, full per-step margin monitoring, and conservative batch-level Margin-12 early exit; it is not an optimized dynamic-batching implementation.

\parhead{Basin-Gated Retransport}
BGR uses four independent generation seeds, each with 128 calibration samples and 256 held-out samples. The calibration median of native decoder-margin p10 is frozen before the held-out first trajectory is generated. For each held-out sample below that threshold, we generate one fresh independent 32-step SDE trajectory. A frozen OPT-125M language model selects between the two completed texts; it is distinct from the GPT-2-Large evaluator and receives no evaluator feedback. The compute-matched arm is a single SDE trajectory with step count $\operatorname{round}(32(1+r))$, where $r$ is that seed's observed held-out trigger rate. This gives 50, 41, 49, and 53 steps for seeds 20260811--20260814. Generation uses batch size 64, selector batch size 16, and PPL batch size 8 after an end-to-end A800 memory smoke test; these batching choices are fixed within each seed. Paired PPL intervals use 5,000 sample-level bootstrap resamples. Pooled MAUVE uses the 1,024 texts from all four seeds in each arm, the same cached 1,024-example OWT reference subset and GPT-2-small featurizer as the primary MAUVE audit, 64 buckets, one k-means redo, 100 maximum iterations, and clustering seed 25. The scripts write candidate texts, fixed-arm texts, per-sample scores, seed summaries, and the pooled aggregate without uploading data or model artifacts.

\parhead{Fixed-NFE distill--diversify--select control}
We compare one 32-step teacher path, one 32-step progressive-distillation (PD) student path, and PD allocations $16\!\times\!2$, $8\!\times\!4$, $4\!\times\!8$, and $2\!\times\!16$. Every arm spends exactly 32 denoiser evaluations per returned text; independent candidates use independent Gaussian starts and sampler randomness. Formal evaluation uses four independent generation seeds and 256 returned texts per arm and seed. OPT-125M is frozen before evaluation and selects candidates by corpus language-model loss; GPT-2-Large is evaluator-only. A separate Qwen2.5-0.5B selector tests whether the conclusion depends on a GPT-family selector that is especially aligned with the evaluator. Pooled MAUVE uses 1,024 texts per policy against the same cached OWT reference features, with the same GPT-2-small featurizer and 64 buckets. Denoiser NFE is matched, but total FLOPs are not: candidate decoding and selector evaluation are reported as auxiliary work rather than hidden inside the NFE accounting.

\parhead{Minimal Decoder Protocol}
MDP uses generated ELF-B final latents rather than ground-truth T5 latents. For each sample, we store the final state $z_T$, the last predicted clean latent $\hat{x}_T$, the average of the last two clean predictions, and the native decoder's argmax tokens. We then train a single token-wise linear readout on a train split and evaluate on held-out generated latents. The 1k experiment uses 1024 generated samples with a 768/256 train/eval split; the 4k experiment uses 4096 generated samples with a 3072/1024 split and caps training at 1.2M valid token positions; the 8k, 16k, and 32k fitting-budget comparisons use 6144/2048, 12288/4096, and 24576/8192 splits, respectively. The 16k and 32k runs are loaded one readout at a time to avoid duplicating CPU memory; the 32k run caps training at 9.6M valid token positions and evaluates GPT-2-Large PPL on 4096 held-out texts. We report agreement to the native decoder, GPT-2-Large PPL, sample entropy, distinct-2, repetition, and margin. Small Gaussian noise augmentation is applied only to the readout input during training, not to the ELF denoiser or native decoder. \figref{fig:mdp-full} gives the fuller 4k readout frontier behind the compact MDP panels in the main text.

The paired clean-manifold audit in \figref{fig:clean-generated-basin-gap} reuses the 1024-sample generated MDP shards but does not train a readout or resample trajectories. For each native-decoded token sequence, we feed the same token ids to the frozen T5 encoder with the stored valid-token mask, normalize by ELF's latent standard deviation, and decode this clean interface state with the same native ELF decoder used for $z_T$, $\hat{x}_T$, and the average of the last two clean predictions. The reported margins are target margins against the native argmax tokens from the generated sample, so the comparison isolates a same-token clean-vs-generated basin-depth gap rather than a text-distribution or tokenizer mismatch.

\begin{figure}[!htbp]
  \centering
  \includegraphics[width=\linewidth]{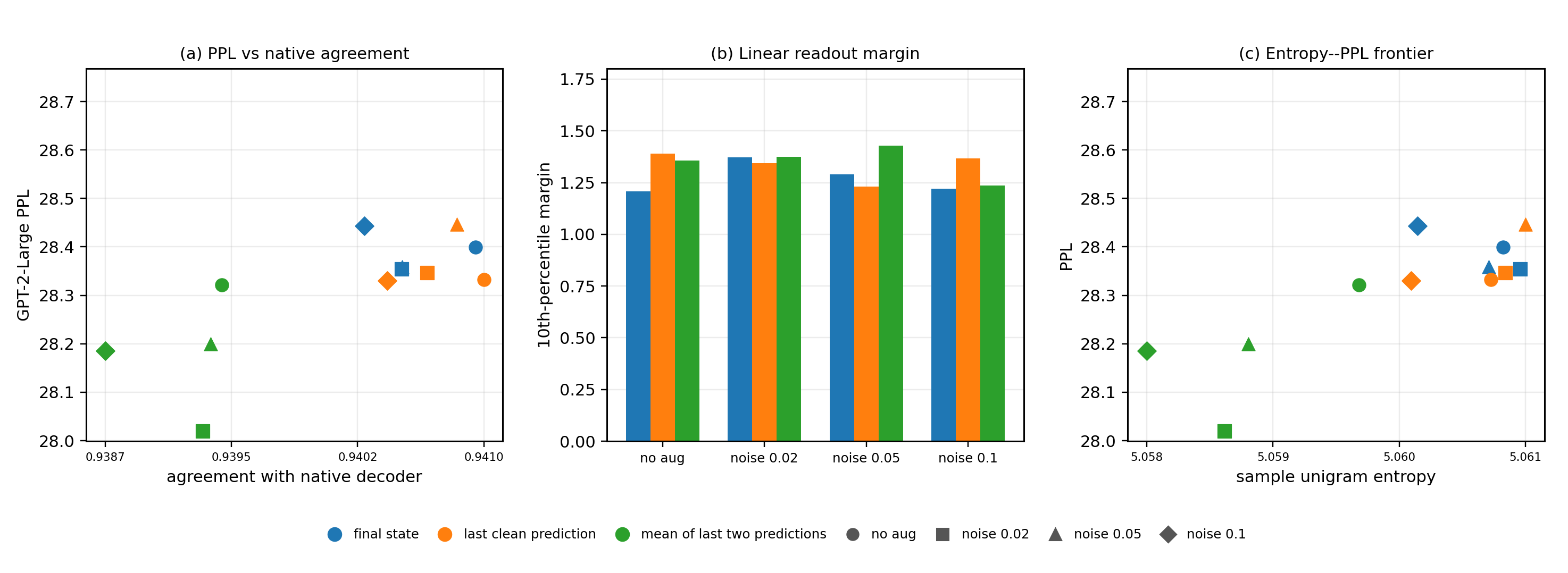}
  \caption{Full MDP readout frontier on 4096 generated latents. Left: GPT-2-Large PPL versus native-token agreement for different generated-latent readouts and noise augmentations. Middle: 10th-percentile native margin for the same readout family. Right: PPL versus generated-text unigram entropy. The best linear readouts approach the native decoder's operating region while retaining a remaining PPL gap.}
  \label{fig:mdp-full}
\end{figure}

\parhead{Zero-shot and reverse-basin tests}
ZSBD reuses the same generated ELF-B final latents as MDP. We decode $z_T$, $\hat{x}_T$, and the average of the last two clean predictions by cosine nearest-neighbor lookup against the frozen T5 token-embedding table, with no learned readout parameters. The main figure uses 4096 samples; 8192- and 16384-sample checks are reported as stability controls. We evaluate GPT-2-Large PPL on 1024 decoded samples. RBN uses 1024 generated final latents. For each late-state variant, we add isotropic Gaussian noise with relative standard deviations $\{0,0.005,0.01,0.02,0.05,0.1,0.2,0.5,1.0\}$, decode with the native ELF decoder, and report token agreement to the clean native decode, 10th-percentile margin, decoder entropy, and PPL at selected noise levels. Both experiments are post-hoc analyses of generated states; no ELF weights are changed.

\section{Transport--Compatibility Controls}
\label{app:transport-controls}

This section retains the complete controls behind \secref{sec:compatibility-transport}. They are deliberately separated from the primary reading path because several outcomes are negative or protocol-specific; their role is to identify what decoder-basin measurements do and do not establish.

\parhead{Sampler counterfactual}
The official stochastic update uses an intermediate contraction $\alpha=1-\gamma\,\Delta t$, moves the time coefficient to $t'=\alpha t$, and mixes the state with fresh noise. Under $z_t=t x+(1-t)\epsilon$, the resulting noise standard deviation is $\sqrt{\alpha^2(1-t)^2+(1-\alpha)^2}$ rather than exactly $1-\alpha t$. CPS instead chooses the stochastic coefficients to preserve the prescribed signal/noise coefficients~\citep{coeffpreserving2025}. We implement both in the same ELF-B audit code, use identical model, schedule, step count, CFG/SC settings, sample count, and seeds, and include ODE as a deterministic counterfactual. CPS achieves a noise-coefficient standard-deviation ratio of exactly $1$ at all audited steps. Its PPL penalty relative to the official SDE is positive in every paired seed: $+18.17$, $+18.00$, and $+21.36$ for seeds 42--44; ODE penalties are $+20.96$, $+21.49$, and $+21.91$. The $\pentry$ change is small (CPS minus official ranges from $-0.0059$ to $+0.0023$). This separates three properties: noise-coefficient preservation, token commitment, and high-quality transport are not interchangeable. It also avoids calling either sampler universally correct: the fixed checkpoint may be adapted to the released sampling path, and CPS was derived for a different use case.

\parhead{Full progressive-distillation sweep}
ELF v2 follows progressive distillation~\citep{salimans2022progressive}: over an interval $[t,r]$, a fixed teacher advances $\mathbf{z}_t$ to $\mathbf{z}_r$ with $K$ solver substeps and converts that displacement into the clean-state target
\begin{equation}
  \widetilde{\mathbf{x}}
  = \mathbf{z}_t + \frac{1-t}{r-t}\bigl(\mathbf{z}_r-\mathbf{z}_t\bigr),
  \qquad
  \mathcal{L}_{\mathrm{PD}}
  = \mathbb{E}\left[\left\|\mathbf{x}_{\theta}(\mathbf{z}_t,t)-\widetilde{\mathbf{x}}\right\|_2^2\right].
  \label{eq:pd-audit-target}
\end{equation}
This loss replaces the original denoising MSE branch, while the shared-weight cross-entropy decoding branch is retained. The release uses one fixed 64-step teacher across five rounds: students target $16,8,4,2,$ and $1$ steps, respectively matching $4,8,16,32,$ and $64$ teacher substeps per student interval. The 16-step student is initialized from the teacher and each later round from the preceding student. Teacher and student share the logit-normal time schedule and training-time self-conditioning CFG condition. At inference, the released protocol uses noise reinjection $\gamma=1.5$ and self-conditioning scale $2.5$ for 1--8 steps, and $\gamma=2.0$ with self-conditioning scale $2.0$ for 16--32 steps. Our audit keeps these settings and evaluates the final-round released student against the released teacher with paired random seeds. Thus the intervention primarily changes the transport target and curriculum, but is not a denoiser-only ablation: the shared parameters and retained decoding branch are also optimized during distillation.

\tabref{tab:pd-full-sweep} reports every audited step count. The 1-step teacher has deceptively low PPL because its entropy is only $2.00$; this is a collapse regime, not a teacher advantage. From four steps onward, the PD student has both competitive entropy and much better PPL, with substantially earlier persistent-token commitment. Direct teacher--student token agreement remains low ($0.01$--$0.12$), and sequence exact match is zero. We do not interpret this as failed distillation: the two stochastic checkpoints need not select identical token labels from a common seed. Distributional quality and each model's native margin trajectory are the relevant quantities.

\begin{table}[!htbp]
\centering
\small
\caption{Full teacher/student audit. Mean $\pm$ standard deviation over three seeds and 256 samples per seed. ``Entry'' is mean persistent-token entry $\pentry$, not the margin crossing $\tau_M$.}
\label{tab:pd-full-sweep}
\begin{tabular}{rrrrrrr}
\toprule
Steps & Teacher PPL & Student PPL & Teacher entropy & Student entropy & Teacher $\pentry$ & Student $\pentry$ \\
\midrule
1  & $3.00\pm0.03$   & $133.20\pm1.63$ & $1.998\pm0.005$ & $5.250\pm0.007$ & $1.000\pm0.000$ & $0.948\pm0.001$ \\
2  & $30.25\pm13.29$ & $70.94\pm5.52$  & $4.509\pm0.565$ & $5.248\pm0.029$ & $0.967\pm0.008$ & $0.862\pm0.026$ \\
4  & $61.35\pm20.07$ & $32.64\pm2.50$  & $5.173\pm0.226$ & $5.128\pm0.014$ & $0.962\pm0.003$ & $0.775\pm0.022$ \\
8  & $89.90\pm9.99$  & $21.87\pm0.18$  & $5.350\pm0.029$ & $5.031\pm0.011$ & $0.957\pm0.006$ & $0.668\pm0.029$ \\
16 & $36.01\pm1.30$  & $21.82\pm0.55$  & $5.198\pm0.013$ & $5.051\pm0.017$ & $0.922\pm0.019$ & $0.569\pm0.036$ \\
32 & $23.05\pm1.73$  & $20.98\pm0.21$  & $5.113\pm0.010$ & $5.031\pm0.011$ & $0.858\pm0.017$ & $0.504\pm0.053$ \\
\bottomrule
\end{tabular}
\end{table}

\begin{table}[!htbp]
\centering
\small
\caption{Independent 1000-sample reproduction of the official PD student. We report the two-step discrepancy rather than selecting a favorable seed.}
\label{tab:pd-official-reproduction}
\begin{tabular}{rrrrr}
\toprule
Steps & Released PPL & Measured PPL & Released entropy & Measured entropy \\
\midrule
1  & 136.10 & 133.55 & 5.26 & 5.260 \\
2  & 68.25  & 73.36  & 5.24 & 5.266 \\
4  & 34.33  & 33.23  & 5.16 & 5.151 \\
8  & 23.18  & 22.88  & 5.07 & 5.067 \\
16 & 22.12  & 22.16  & 5.06 & 5.063 \\
32 & 21.32  & 21.14  & 5.04 & 5.033 \\
\bottomrule
\end{tabular}
\end{table}

\parhead{Decoder-only, mixed, and staged training}
The 5k interface sweep is a uniform retrospective evaluation of pre-existing short runs, not a preregistered recipe search. At $t=0.6$, increasing the decoder-branch probability from $0.2$ to $1.0$ widens both in-domain and OOD corruption basins (\tabref{tab:decoder-transport-boundary}). This result is exactly why corruption recovery cannot be interpreted as generation competence. The $p=1$ objective never receives a flow example; its excellent recovery is conditional on being handed a corrupted clean latent from the evaluation construction. From Gaussian initialization, the 5k checkpoint generates no nonempty sequences. At 50k, it produces a single repeated sequence family: PPL $1.02$, entropy numerically zero, and unique-text fraction $1/256$. The low PPL is a textbook collapse artifact.

\begin{table}[!htbp]
\centering
\small
\caption{Corruption-basin recovery at $t=0.6$ after 5k steps. Each domain uses 512 sequences; native OWT, AG News, and Yelp contain different token counts. The decoder-only condition has stronger measured corruption retention but no learned prior transport.}
\label{tab:decoder-transport-boundary}
\begin{tabular}{lrrrr}
\toprule
Objective & Clean recovery & OWT recovery & AG News recovery & Yelp recovery \\
\midrule
Official-like mixed ($p=0.2$) & 0.995 & 0.939 & 0.726 & 0.899 \\
Decoder-only ($p=1.0$) & 0.998 & 0.979 & 0.832 & 0.948 \\
\bottomrule
\end{tabular}
\end{table}

The official-like 50k run remains diverse (PPL $43.57$, entropy $4.45$, unique fraction $1.0$), but is not intended to reproduce the fully trained released checkpoint. A staged run resumes model weights from the 5k decoder-only checkpoint, then trains to 50k with $p=0.2$. It restores diversity (entropy $5.16$, unique fraction $1.0$) and preserves a broad corruption basin: recovery at $t=0.6$ is $0.966$ on native OWT, $0.902$ on AG News, and $0.965$ on Yelp. Nevertheless, generation remains poor (PPL $115.41$, entry $0.745$, final $M_{10}=0.515$). This is supporting, not matched causal, evidence: the staged run uses model-only resume, a different training seed/batch configuration, and no paired optimization path. Its safe conclusion is that adding a transport objective restores diversity but does not retroactively turn a decoder-specialized solution into a good generator under this short curriculum.

\parhead{Matched five-epoch training}
The original boundary audit extends the short controls to five epochs. Two independent seeds each initialize an ELF-B pair from scratch with common nominal OWT data-order settings for five passes. These settings alone do not verify identical realized data order through resumed execution. The pairs share optimizer (Muon), learning rate $2.5\times10^{-4}$, warmup, effective batch $48$, decoder noise scale $5$, denoiser schedule, self-conditioning probability, and evaluation randomness; an automated configuration audit permits differences only in decoder-branch probability ($0.2$ versus $1.0$) and output metadata. Every epoch checkpoint receives the same 512-sequence native corruption audit and 256-sample SDE32 end-to-end audit, and the final checkpoints additionally receive 512-sequence AG News and Yelp Polarity corruption audits. Decoder-only has higher corrupted-clean $M_{10}$ in all $10/10$ paired epoch checkpoints (paired advantages $1.62$--$2.35$), but lower generated final-state $M_{10}$ in all $10/10$ (paired deficits $2.14$--$9.44$). This is the matched comparison summarized in \figref{fig:matched-fulltrain}; final endpoints appear in \tabref{tab:matched-fulltrain-endpoints}.

\begin{table}[!htbp]
\centering
\footnotesize
\setlength{\tabcolsep}{4pt}
\caption{Final endpoints of the two matched five-epoch training pairs. Corruption $M_{10}$ uses 512 sequences per domain, seed, and condition at $t=0.6$; generated-state statistics use 256 paired-randomness SDE32 samples. ``Empty'' means that none of the 256 decoded strings contains text, so PPL and entropy are undefined; its nominal $0.004$ unique fraction is the single shared empty string divided by 256.}
\label{tab:matched-fulltrain-endpoints}
\begin{tabular}{lrrrr}
\toprule
& \multicolumn{2}{c}{Mixed ($p=0.2$)} & \multicolumn{2}{c}{Decoder-only ($p=1.0$)} \\
Metric & Seed 42 & Seed 43 & Seed 42 & Seed 43 \\
\midrule
Native corruption $M_{10}$ & 6.81 & 6.75 & 9.01 & 8.83 \\
AG News corruption $M_{10}$ & $-2.89$ & $-0.94$ & 4.54 & 3.72 \\
Yelp corruption $M_{10}$ & 3.35 & 3.19 & 7.34 & 6.80 \\
Generated final-state $M_{10}$ & 9.42 & 9.15 & 0.63 & 0.26 \\
Generated PPL & 34.85 & 36.84 & 2.88 & -- (empty) \\
Generated entropy & 5.05 & 5.18 & 1.16 & -- (empty) \\
Unique-text fraction & 1.000 & 1.000 & 0.531 & 0.004 (empty) \\
\bottomrule
\end{tabular}
\end{table}

The historical table uses signed ground-truth-token corruption margins and the original sampler protocol. The separate closure table in \secref{app:dose-grid} reports newly evaluated endpoints under a globally fixed time grid; differing numbers across the two tables are not replicate estimates of the same sampling protocol.

The matched result sharpens rather than changes the short-run conclusion. Decoder-only training learns a robust conditional map around corrupted clean states, including under domain shift, but its denoising path does not successfully carry the Gaussian prior into those high-margin regions under the audited training and sampling budgets. In one seed its apparently excellent PPL is produced by a low-entropy repeated regime; in the other, text generation disappears entirely. The mixed objective has smaller handed-latent margin depth but a much stronger generated-state margin and noncollapsed output. Corruption compatibility is therefore a property of the target interface, while useful generation requires transport into that interface.

\parhead{Why linear denoising does not create a Cola-DLM basin}
We fit scalar, diagonal, and full linear MSE denoisers on matched Cola-DLM VAE corruptions. They reduce latent MSE, but token recovery changes only modestly (\tabref{tab:cola-linear-denoise}). At $t=0.6$, the best recovery rises from $0.541$ to $0.561$; at $t=0.5$, it rises only to $0.256$. The full linear map can even improve latent NMSE while underperforming the scalar/diagonal map in token recovery. This is consistent with a geometric mismatch: linear least squares estimates the best linear predictor, which equals the conditional mean only under additional assumptions such as joint Gaussianity, whereas token recovery depends on nonlinear, anisotropic decoder boundaries. These fitted maps reduce latent error without learning an iterative transport path or changing the frozen native decoder. Cola therefore remains a boundary case: the experiment shows limited recoverability under simple denoising, not an ELF-style basin transition.

\begin{table}[!htbp]
\centering
\small
\caption{Cola-DLM VAE token recovery under matched corruption and linear MSE denoising (512 sequences). Improvements are real but too small to turn the narrow mid-noise interface into an ELF-like transport result.}
\label{tab:cola-linear-denoise}
\begin{tabular}{rrrrr}
\toprule
$t$ & No denoiser & Scalar & Diagonal & Full linear \\
\midrule
0.5 & 0.246 & 0.246 & 0.246 & 0.256 \\
0.6 & 0.541 & 0.561 & 0.560 & 0.557 \\
0.7 & 0.885 & 0.898 & 0.899 & 0.898 \\
\bottomrule
\end{tabular}
\end{table}

\FloatBarrier

\section{Additional Basin Stress Tests}
\label{app:basin-stress-tests}

\parhead{Cross-decoder basin transfer}
We test whether generated decoder-basin states are private to one decoder or shared within the ELF checkpoint family. We generate 128 32-step SDE samples from each of ELF-B, ELF-M, and ELF-L, take the final latents from the source checkpoint, and decode them with ELF-B, ELF-M, and ELF-L decoders. Agreement is measured against the source checkpoint's own native decoded tokens for the same source latents. This protocol avoids a ground-truth text comparison; it asks whether a different ELF decoder assigns the same token labels to states that the source denoiser has navigated into. B final latents decoded by M/L agree with B-native tokens at $0.9875/0.9870$, M final latents decoded by B/L agree with M-native tokens at $0.9960/0.9957$, and L final latents decoded by B/M agree with L-native tokens at $0.9952/0.9956$. \figref{fig:cross-decoder-transfer} shows the resulting transfer matrix. The result indicates that ELF-family decoders share a large portion of the decoder-basin region under the common T5-small interface. It does not address independent training seeds or arbitrary T5-space decoders.

\begin{figure}[!htbp]
  \centering
  \includegraphics[width=.78\linewidth]{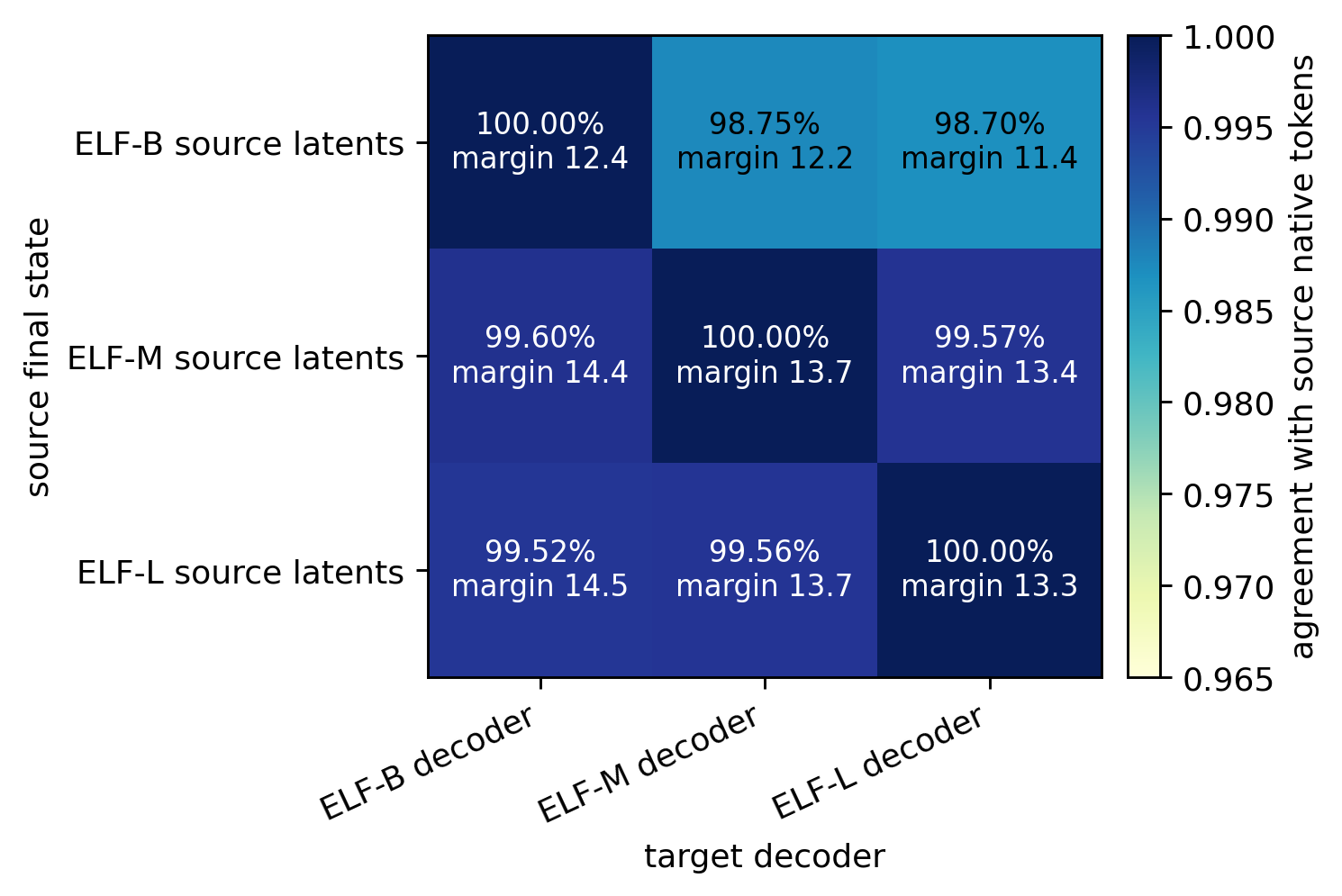}
  \caption{Cross-decoder basin transfer inside the ELF family. Rows indicate the source denoiser that generated the final latents; columns indicate the target decoder used for readout. Each cell reports agreement with the source model's native decoded tokens and the corresponding source-label margin. High off-diagonal agreement indicates that generated decoder-basin states are largely shared across these checkpoint decoders under the same T5-small interface.}
  \label{fig:cross-decoder-transfer}
\end{figure}

\parhead{Paired basin intervention}
To test whether the basin signal is only correlational, we intervene on the predicted clean latent at step 16 of a 32-step ELF-B SDE trajectory and then continue denoising with the same initial noise and SDE noise stream. The projection intervention moves each position toward its nearest frozen T5 token embedding direction; the anti-basin intervention moves away from that direction; the random-matched control uses a random direction with matched norm. We expand the original pilot to 512 samples across four GPU shards. With $\alpha=0.5$, projection increases the final margin tail by $\Delta p10=+0.085\pm0.045$ and lowers GPT-2-Large PPL by $0.49$ on a 256-text audit, while anti-basin perturbation decreases the margin tail by $0.186\pm0.064$ and raises PPL by $0.49$. A random-matched perturbation is near neutral ($\Delta p10=+0.009\pm0.024$, $\Delta$PPL $=-0.08$). \figref{fig:causal-basin-intervention} visualizes the paired changes in margin, entropy, and edit size. A smaller $\alpha=0.1$ 512-sample stability check preserves high token agreement but is too weak to cleanly separate projection from matched random noise. This is an interventional mechanism check, not a tuned guidance sampler: at $\alpha=0.5$, token agreement to the unedited baseline is only about $0.72$--$0.78$.

\begin{figure}[!htbp]
  \centering
  \includegraphics[width=.86\linewidth]{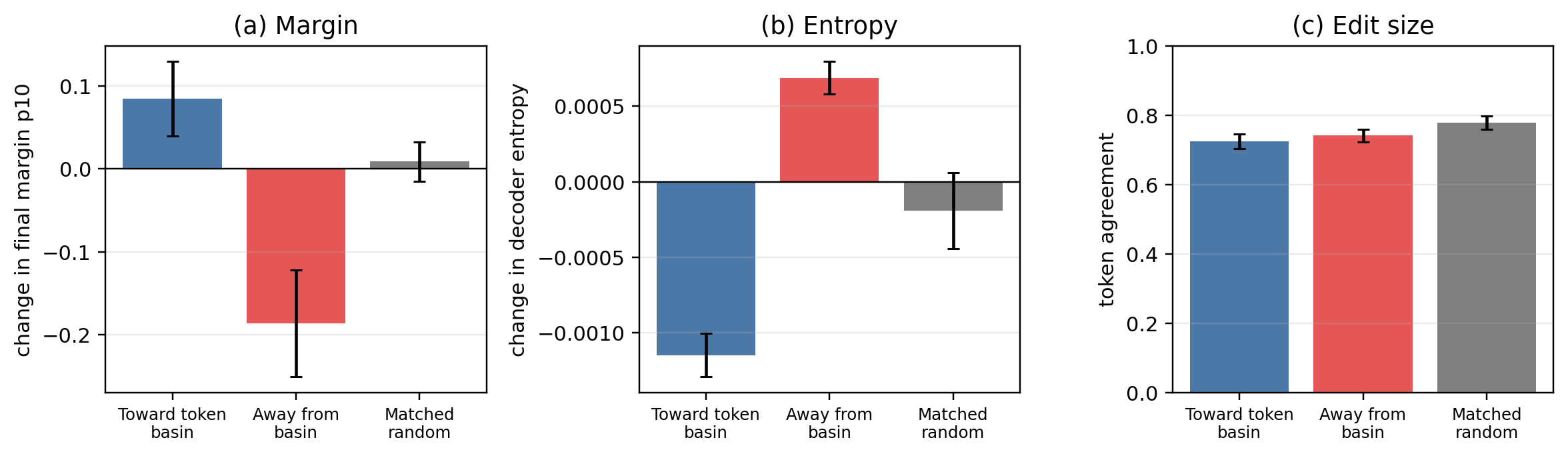}
  \caption{Paired basin intervention at 512 samples. Left: change in final 10th-percentile margin after a step-16 intervention. Middle: change in decoder entropy. Right: token agreement to the unedited baseline, which measures edit size. The same trajectory is continued after projecting toward token-embedding basins, pushing away from them, or applying a matched random perturbation; error bars are shard-level SEM.}
  \label{fig:causal-basin-intervention}
\end{figure}

\FloatBarrier

\parhead{Phase-resolved self-conditioning causal sweep}
\label{app:sc-phase-causal}
The mid-SC intervention in the main text tests a single broad window. To resolve where self-conditioning is causally most important, we run a seven-phase sweep on 256 ELF-B SDE32 samples that share the same initial states and SDE noise streams (seed 20260605). Each intervention corrupts self-conditioning within an equal-width 3-step window centered at steps 4, 8, 12, 16, 20, 24, and 28. We apply three corruption types: zero (set the self-conditioning input to zero), frozen (hold the self-conditioning state at the first window step), and shuffled (replace each sample's SC state with a batch-shifted SC state from another sample). All trajectories use the fixed SC=3 sampler. The baseline (no intervention) gives PPL $24.23$ and 10th-percentile margin $12.16$.

\figref{fig:sc-phase-causal} shows the phase-resolved result. Zeroing SC causes a competition-region peak: PPL change is $-0.07$ at center 4, $+1.76$ at center 8, $+4.89$ at center 12, $+7.02$ at center 16, $+6.44$ at center 20, $+3.87$ at center 24, and $+1.70$ at center 28. The peak coincides with the mid-trajectory SC-disagreement maximum observed in the trajectory audit, supporting an intervention-specific contribution to quality in the competition phase. Frozen SC produces a weaker mid-hump (max $+1.76$ at center 12 and $+1.70$ at center 16, falling to $+0.21$ at center 28), consistent with a less disruptive intervention that preserves the directional self-conditioning signal while preventing its iterative update. Shuffled SC follows a distinct pattern and should not be conflated with the zero/frozen results: early shuffling at centers 4--8 lowers PPL by $2.82$--$3.22$ while reducing sample entropy, and by center 12 PPL already rises by $+1.78$; late shuffling (centers 24--28) is catastrophic, with PPL changes of $+21.93$ and $+63.90$, and margin-p10 drops of $-5.19$ and $-10.30$. The late-shuffle result reflects cross-sample mismatch in the locked phase, where each trajectory's SC state has become sample-specific.

These results support a three-part interpretation. First, the phase-resolved zero-SC sweep provides fixed-checkpoint intervention evidence for SC utility: self-conditioning is most important in the same competition window identified by the trajectory audit's SC-disagreement peak, and far less important in early chaos or late locked phases. Second, the graded degradation (not a sharp cliff) suggests that the competition region is a distributed computation rather than a single binary decision step. Third, shuffled SC's distinct effects---early diversity collapse and late catastrophic failure---warn against treating all SC corruptions as measuring the same mechanism.

\begin{figure}[!htbp]
  \centering
  \includegraphics[width=\linewidth]{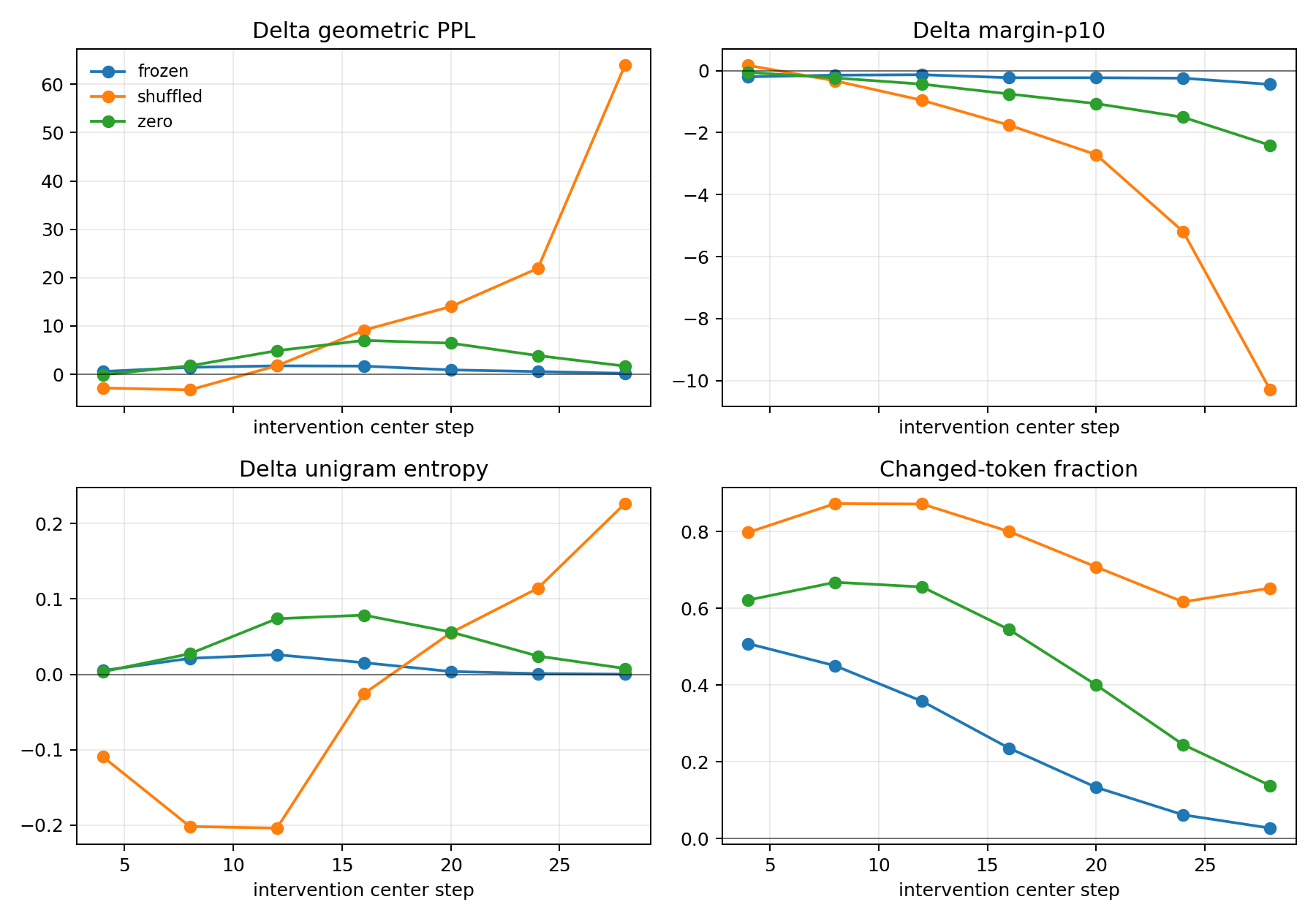}
  \caption{Phase-resolved self-conditioning causal sweep on 256 ELF-B SDE32 samples with the same initial states and SDE noise streams. Top-left: PPL change relative to the matched baseline for zero, frozen, and shuffled self-conditioning in equal-width 3-step windows at seven centers. Zero-SC degradation peaks at centers 16--20, aligning with the competition region where self-conditioning disagreement is largest. Top-right: final 10th-percentile decoder-margin change; shuffled SC becomes increasingly destructive in the locked phase, indicating cross-sample state mismatch rather than the same mechanism as zero/frozen SC. Bottom-left: unigram-entropy change, showing that early shuffled-SC PPL improvements coincide with lower diversity. Bottom-right: changed-token fraction relative to the matched baseline.}
  \label{fig:sc-phase-causal}
\end{figure}

\parhead{Same-interface scale-trend supplement}
The main text uses scale shifts as a falsifiable trend, not a scaling law. To avoid mixing model size with sampler depth, we add a condition-matched check of the Margin-8 crossing phase $\tau_M(8)$. For each shard, we linearly interpolate the first phase at which the predicted-clean lower-tail native margin crosses $8$, then average across fixed-checkpoint sampling shards. \tabref{tab:p5-tau-scaling} shows that SDE32 and SDE64 both produce monotone B/M/L trends with similar log-log slopes. ODE32 is flatter and has lower final margin depth at larger scale, but its M/L rows are single-shard estimates and are included only as indicative sampler contrast. The supported conclusion is therefore sampler-conditioned earlier basin entry with scale, not a universal exponent.

\begin{table}[!htbp]
\centering
  \caption{Condition-matched scale supplement for the Margin-8 crossing phase. $\tau_M(8)$ is linearly interpolated per shard from the 10th-percentile native decoder margin of $\hat{x}_t$ and is the canonical table used for same-sampler entry-timing trends. SEM is across fixed-checkpoint sampling shards; one-shard ODE32 entries report no SEM, so the ODE32 slope is indicative only. The SDE trends fall in the rough $N^{-0.12}$--$N^{-0.14}$ range under matched sampler settings, but these public checkpoints are not a controlled scaling suite.}
\label{tab:p5-tau-scaling}
\small
\begin{tabular}{llrrrr}
\toprule
Sampler & Model & Params (M) & Shards & Samples & $\tau_M(8)$ \\
\midrule
SDE32 & ELF-B & 105 & 4 & 512 & $0.849 \pm 0.007$ \\
SDE32 & ELF-M & 342 & 2 & 256 & $0.736 \pm 0.031$ \\
SDE32 & ELF-L & 652 & 2 & 256 & $0.657 \pm 0.036$ \\
SDE64 & ELF-B & 105 & 4 & 512 & $0.726 \pm 0.024$ \\
SDE64 & ELF-M & 342 & 4 & 512 & $0.667 \pm 0.012$ \\
SDE64 & ELF-L & 652 & 4 & 512 & $0.582 \pm 0.017$ \\
ODE32 & ELF-B & 105 & 4 & 512 & $0.889 \pm 0.024$ \\
ODE32 & ELF-M & 342 & 1 & 128 & $0.852$ \\
ODE32 & ELF-L & 652 & 1 & 128 & $0.851$ \\
\midrule
\multicolumn{5}{l}{Log-log slope over B/M/L, SDE32} & $-0.138$ \\
\multicolumn{5}{l}{Log-log slope over B/M/L, SDE64} & $-0.115$ \\
\multicolumn{5}{l}{Log-log slope over B/M/L, ODE32} & $\approx -0.03$ \\
\bottomrule
\end{tabular}
\end{table}

To separate entry timing from transition sharpness, we add a coarse missing-cell check for B-SDE64 and M/L-SDE32. The added runs use 64 samples and nine recorded phases, so they are interpretive supplements rather than replacements for the condition-matched $\tau_M(8)$ table above or the larger summaries used in \figref{fig:interface-phase-diagram}. \tabref{tab:phase90-transition-grid} supports the same qualitative conclusion as the shard-level margin audit: SDE compute and denoiser scale move entry earlier, but transition sharpness itself does not form a clean monotone scaling law. In SDE32, M and L enter earlier than the large B-SDE32 cell; in SDE64, B is earlier than B-SDE32 but later and lower-agreement than M/L. However, M has a narrower ZSBD transition width than L in both SDE32 and SDE64. ODE32 transition-width cells are incomplete at larger scale and are omitted from this grid except for the ELF-B sampler contrast. We therefore treat entry timing as the robust same-interface scale signal, while leaving sharpness scaling as an open phase-shape question.

\begin{table}[!htbp]
\centering
\caption{Coarse transition-grid supplement. ``Large'' rows reuse the main larger-sample intermediate trajectory summaries; ``coarse64'' rows are newly added 64-sample, nine-phase audits for missing sampler/scale cells. The $\tau_M(8)^\dagger$ column is an interpolated phase-shape summary under each row's source, not the per-shard native-margin interpolation used in \tabref{tab:p5-tau-scaling}; it can differ slightly because the sample set, phase grid, and aggregation order differ. ODE32 cells are incomplete at coarse scale and are omitted except for the ELF-B sampler contrast. The table supports earlier entry with scale or more SDE steps, but not monotone transition-sharpness scaling.}
\label{tab:phase90-transition-grid}
\small
\begin{tabular}{llrrr}
\toprule
Condition & Source & ZSBD 10--90 width & $\tau_M(8)^\dagger$ & Final ZSBD agree. \\
\midrule
ELF-B ODE32 & large & $0.631$ & $0.923$ & $0.929$ \\
ELF-B SDE32 & large & $0.635$ & $0.853$ & $0.933$ \\
ELF-M SDE32 & coarse64 & $0.538$ & $0.742$ & $0.952$ \\
ELF-L SDE32 & coarse64 & $0.602$ & $0.685$ & $0.957$ \\
ELF-B SDE64 & coarse64 & $0.529$ & $0.721$ & $0.945$ \\
ELF-M SDE64 & large & $0.461$ & $0.667$ & $0.957$ \\
ELF-L SDE64 & large & $0.519$ & $0.596$ & $0.960$ \\
\bottomrule
\end{tabular}
\end{table}

\parhead{Pre-entry bootstrapping audit}
The main phase diagram summarizes pre-entry with decoder-facing scalars. To examine the earliest dynamics more directly, we instrument the first 10 steps of ELF-B SDE32 and store both the noisy state $z_t$ and the predicted clean state $\hat{x}_t$ in two 32-sample repeats. This tensor-capture audit uses the same public checkpoint and sampling settings as the main ELF-B trajectory audit, but it is intentionally smaller because it stores full intermediate states and decodes every early predicted-clean state. We measure consecutive update-direction cosine for $\hat{x}_t-\hat{x}_{t-1}$, per-sample effective rank, native-decoder margin, entropy, and agreement with the final native tokens.

The result is shown in \figref{fig:preentry-bootstrapping}. The early updates are not a random walk: consecutive predicted-clean update directions have mean cosine $0.751$ and $0.780$ across the two repeats. The predicted-clean state gains rank quickly, from about $15$--$18$ at step 0 to $152$--$160$ at step 10, while $z_t$ remains near full rank around $474$. Decoder entropy falls from about $6.20$ to $2.58$--$2.89$ nats, and agreement with final tokens rises from $0.8\%$ to $16.4$--$21.1\%$. However, the lower-tail native margin remains shallow: p10 margin rises only from $0.049$ to $0.18$--$0.22$. Crossings of p10 margin $0.5$ are rare and schedule-sensitive ($2/32$ samples in one repeat and $1/32$ in the other), and neither repeat crosses p10 margin $1$, $2$, or $8$. This supports a narrow interpretation. Pre-entry contains directed bootstrapping of candidate clean states, but it is still outside the robust decoder basin used by BGEE, ZSBD, and MDP.

\begin{figure}[!htbp]
  \centering
  \includegraphics[width=\linewidth]{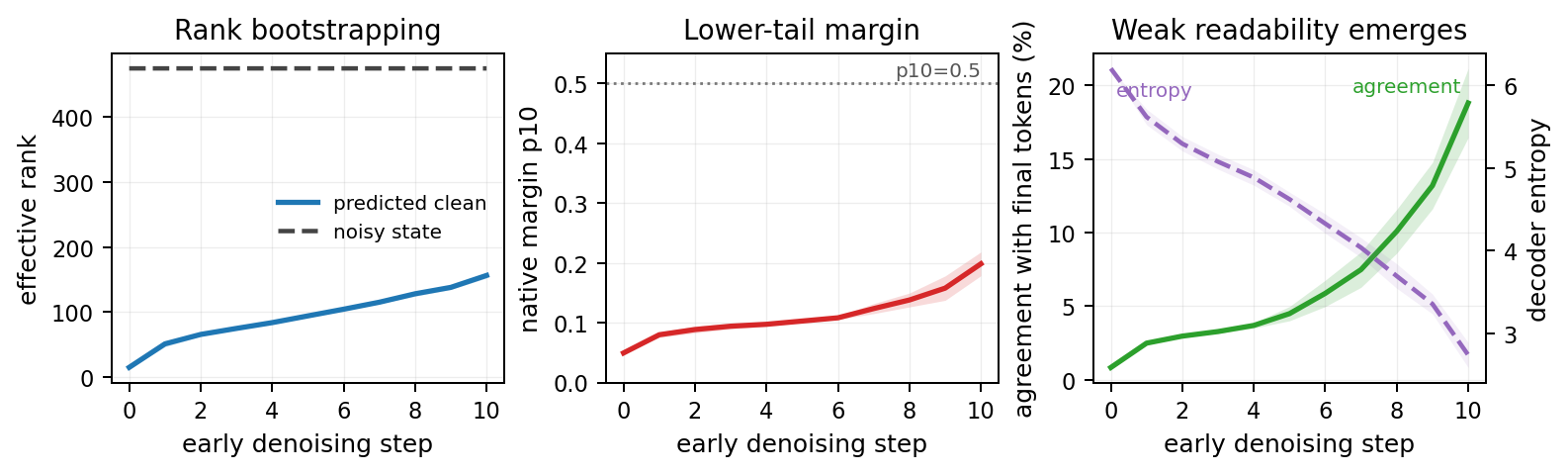}
  \caption{Pre-entry bootstrapping audit on two 32-sample ELF-B SDE32 repeats. Curves show the seed mean and shaded bands show the seed range. Left: the effective rank of predicted clean states grows rapidly while the noisy state remains high rank. Middle: the lower-tail native margin remains shallow during the first 10 steps. Right: agreement with final native tokens rises above chance before high-margin basin entry. The audit supports directed early bootstrapping, but not early locked-basin geometry.}
  \label{fig:preentry-bootstrapping}
\end{figure}

We then test whether the first few SC updates cause this rank expansion. In a matched-noise 256-sample audit, baseline and intervention trajectories share the same initial states, time grid, and SDE noise streams; the intervention sets the SC input to zero during steps 0--5 and then resumes the ordinary sampler. \figref{fig:preentry-zero-sc} shows that the strongest early-SC causal account is not supported. Step 0 is identical by construction. During steps 1--5, zero-SC does not reduce the effective rank of $\hat{x}_t$; it slightly increases it, reaching $111.2$ vs.\ $103.4$ at step 5. Agreement with final native tokens is not delayed either: at step 10 it is $20.8\%$ vs.\ $20.5\%$, and the mean first crossing of 5\% agreement is $5.79$ steps vs.\ $6.11$ for baseline. Final margin changes are small (mean final margin-p10 $12.34$ vs.\ $12.43$). Thus the early-SC ablation does not remove rank expansion in this checkpoint, while SC's larger measured quality impact emerges later in the competition window measured by \figref{fig:sc-phase-causal}.

\begin{figure}[!htbp]
  \centering
  \includegraphics[width=\linewidth]{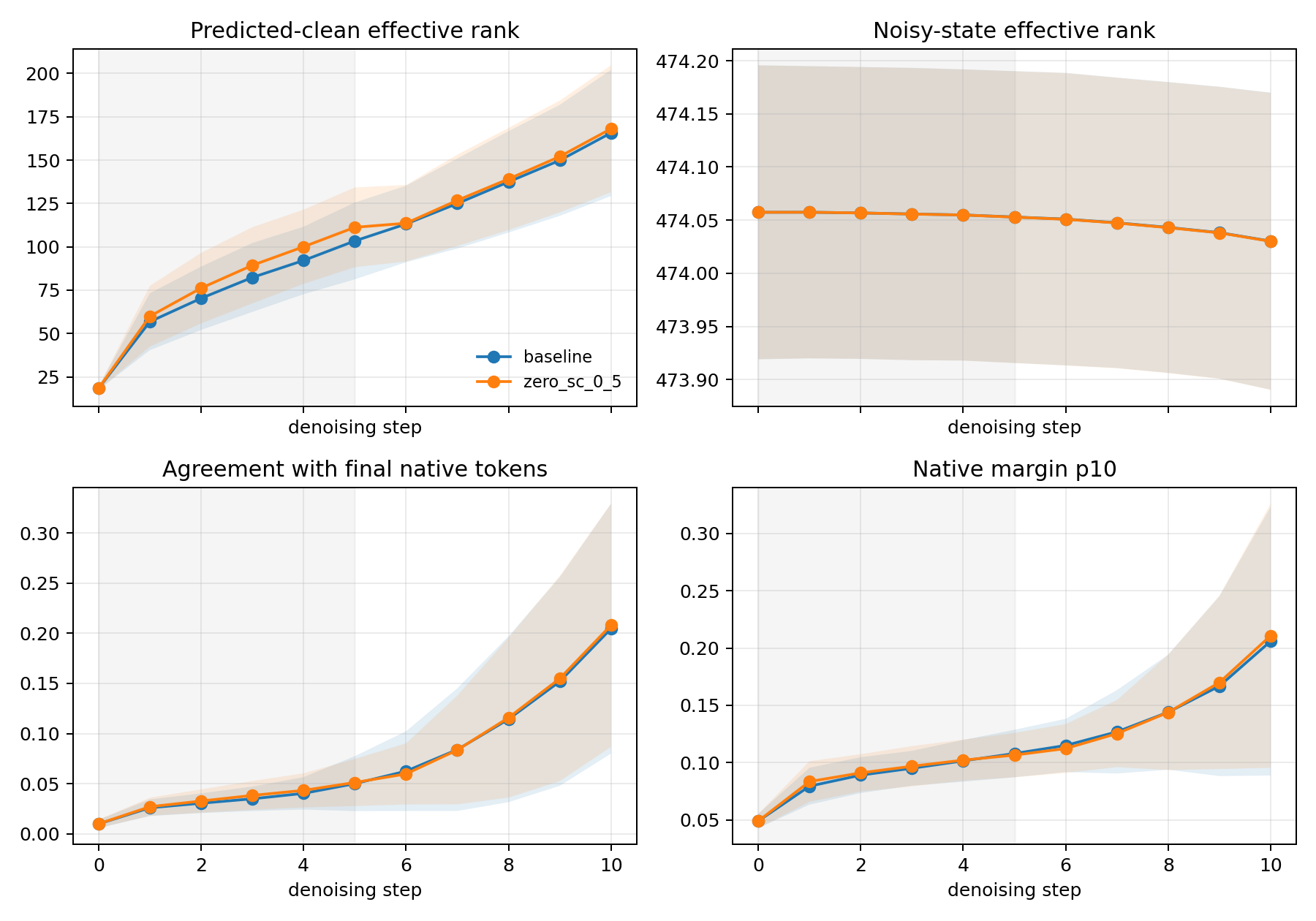}
  \caption{Pre-entry zero-self-conditioning causal audit on 256 matched ELF-B SDE32 trajectories. The intervention zeros SC during steps 0--5, shaded in gray, while preserving the same initial states and SDE noise streams as baseline. Top-left: predicted-clean effective rank grows at least as fast under early zero-SC as under baseline. Top-right: noisy-state rank remains near full rank and is unaffected. Bottom-left: agreement with final native tokens is not delayed. Bottom-right: the lower-tail native margin remains shallow in both cases. The result separates chaotic-phase rank expansion from the competition-phase effect of removing SC shown in \figref{fig:sc-phase-causal}.}
  \label{fig:preentry-zero-sc}
\end{figure}

\parhead{Rank-origin control}
The low step-0 rank could have had a trivial explanation: perhaps the frozen T5 token geometry itself has only $O(15)$ effective dimensions under our entropy-rank statistic. A lightweight control rejects this explanation. Using the same rank definition, the full T5 token-embedding table has entropy rank about $452$; generated label-token embedding sequences have rank about $254$; clean T5 contextual re-encodes of the same generated text have rank about $339$; and Gaussian sequences with the same shape have rank about $474$. These values are all far above the step-0 predicted-clean rank. The control does not derive the rank-expansion dynamics, but it rules out the most direct token-table intrinsic-dimension account and is consistent with the linearized model in \secref{sec:decoder-basin-theory}: the low-rank seed is a property of the high-noise denoiser-prior map.

\parhead{Token-wise basin entry}
Sample-level early exit hides the fact that token positions enter the basin at different times. We therefore record all 32 predicted-clean states on 512 ELF-B samples and compute, for each valid final-token position, the first and persistent phase at which native-token match, ZSBD match, or a margin threshold is reached. Across $523{,}578$ valid positions, persistent native-token entry reaches $99.96\%$ with mean phase $0.533$ and shard-mean p90 phase $0.794$. Persistent ZSBD entry reaches $93.17\%$, and persistent entry with margin at least $8$ reaches $94.69\%$ with mean phase $0.580$. \figref{fig:tokenwise-basin-entry} shows the average entry timing, while \figref{fig:tokenwise-sequence-stress} shows why sequence-level hard tails remain. Together they explain both the value and the ceiling of BGEE: many easy positions are ready early, but a late hard tail remains, dominated by numeric tokens, rare subwords, and margin-borderline cases.

\begin{figure}[!htbp]
  \centering
  \includegraphics[width=.86\linewidth]{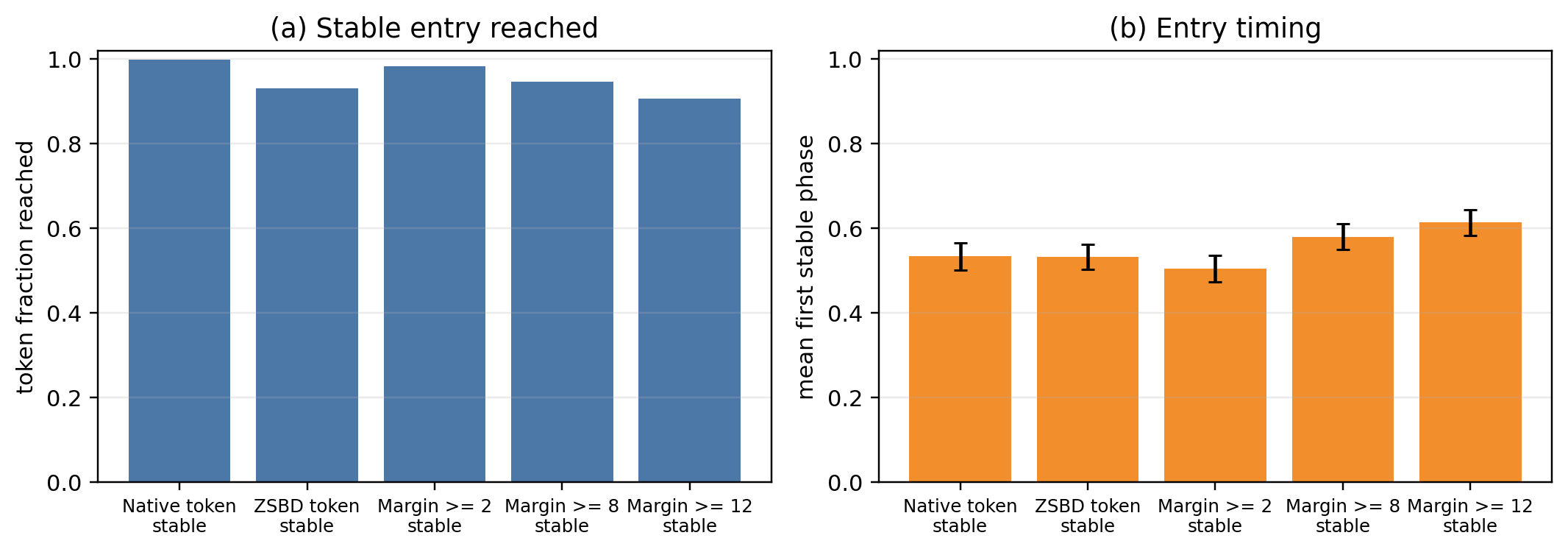}
  \caption{Token-wise basin entry on 512 ELF-B 32-step SDE trajectories. Left: fraction of valid token positions that ever satisfy each persistent-entry criterion. Right: mean first persistent phase for the same criteria, with shard-level SEM. Entry is staggered across positions: many positions stabilize before the final step, but margin-strict and ZSBD criteria retain a late hard tail.}
  \label{fig:tokenwise-basin-entry}
\end{figure}

\begin{figure}[!htbp]
  \centering
  \includegraphics[width=.86\linewidth]{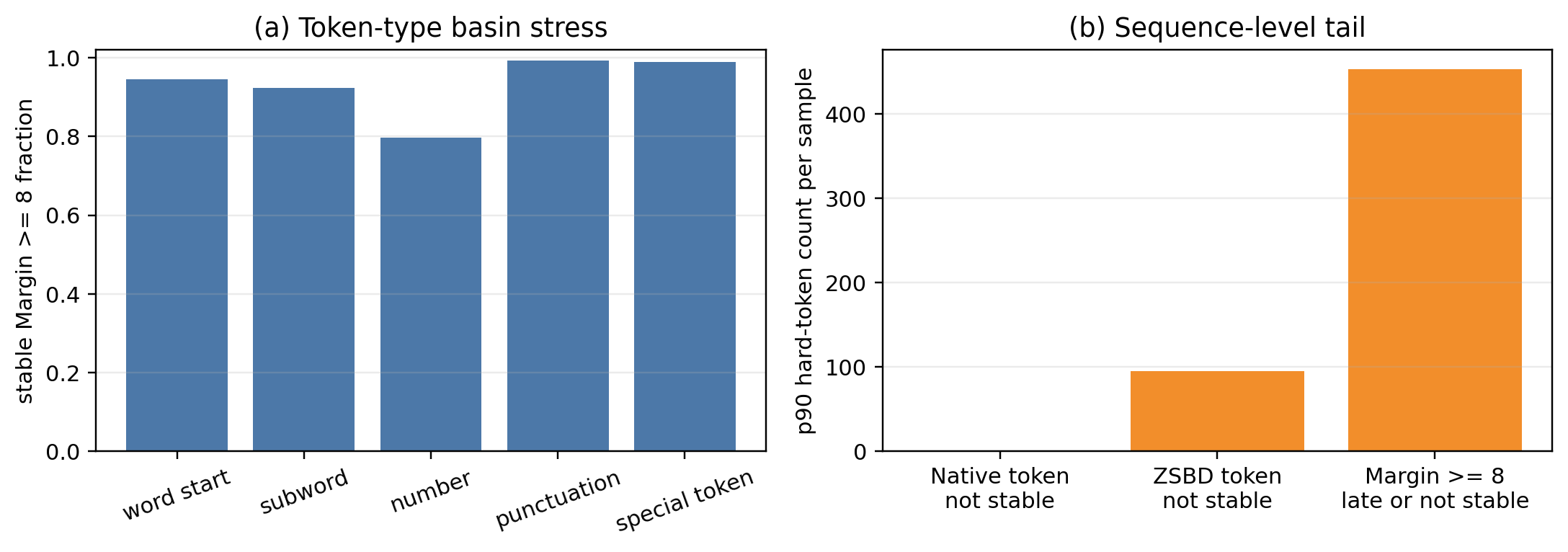}
  \caption{Token-wise basin stress at 512 samples. Left: persistent entry fraction with margin at least $8$, grouped by token type. Right: p90 hard-token count per sample under native-token, ZSBD-token, and margin-at-least-$8$ criteria. Token-average agreement hides a sequence-level tail; numeric/subword-heavy samples keep many positions late or unstable.}
  \label{fig:tokenwise-sequence-stress}
\end{figure}

\FloatBarrier

\parhead{ELF-M PC1 axis audit}
Directional RBN exposes one especially sharp anomaly: a matched perturbation along ELF-M's first principal component leaves only $1.5\%$ native-token agreement, while isotropic, random, and sentiment directions all stay near $99.4\%$. To interpret this axis without turning it into a new method, we compare the same PC1 projection against token type, sample-level surface statistics, and the frozen T5 token-embedding projection. \figref{fig:elfm-pc1-axis-audit} summarizes the result and gives the missing readout for this anomaly. The left panel establishes the specificity of the failure: the same perturbation budget is harmless in random and sentiment directions but catastrophic along PC1. The middle panel shows what dominates the axis: punctuation has much larger absolute projection than numbers, word starts, or ordinary subwords. The right panel explains why this is an interface effect rather than a semantic edit direction: generated-latent PC1 projection is strongly aligned with the T5 token table ($r=0.82$ at the token-label level). At the sample level, the projection is also moderately associated with punctuation and quote density. Thus the ELF-M anomaly is best read as a decoder-sensitive token-boundary axis. Normal trajectories can remain on its valid generated manifold, but a global perturbation along that axis crosses many punctuation and boundary decisions at once. The conclusion is narrow but useful: anisotropy here is not just ``some principal component is fragile,'' but a concrete mismatch between high-variance generated-state directions and the decoder's token-boundary geometry.

\begin{figure}[!htbp]
  \centering
  \includegraphics[width=\linewidth]{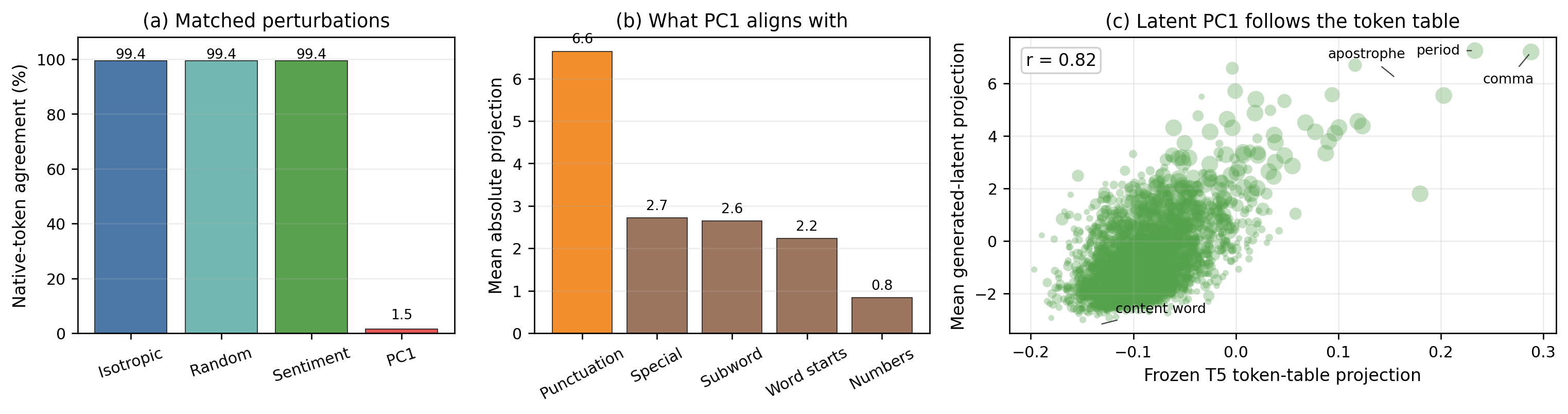}
  \caption{ELF-M PC1 axis audit. Left: matched perturbations at scale $1.0$ show that isotropic, random, and sentiment directions remain inside the decoder basin, while the first principal component nearly destroys native-token agreement. Middle: the same PC1 projection is dominated by punctuation tokens rather than by numbers, word starts, or ordinary subwords. Right: mean generated-latent projection by token label is strongly aligned with the frozen T5 token-table projection ($r=0.82$), with punctuation tokens at the high-projection edge. This identifies the fragile direction as a token-boundary anisotropy rather than a named semantic direction.}
  \label{fig:elfm-pc1-axis-audit}
\end{figure}

\FloatBarrier

\parhead{Outlier-induced quantization directions}
The PC1 control above shows that equal-norm perturbations need not be equally safe. A separate 2-bit stress test makes this point without choosing a semantic or principal-component direction. We apply ternary uniform quantization to clean T5 states using a max-absolute scale per token-position vector. This apparently mild operation leaves only $58.1$--$69.8\%$ of native decisions unchanged across ELF-B/M/L and held-out WikiText, even though isotropic noise at similar scale preserves almost all native token decisions. We therefore separate range from direction using four fixed-state interventions: per-coordinate calibration, randomized orthonormal Hadamard rotation before quantization and its inverse after quantization, a norm-preserving reorientation of the original quantization residual, and matched-RMS Gaussian noise. We also clip the largest $1\%$ of each vector's range before the same ternary quantizer. The rotation controls are motivated by the outlier-redistribution principle used in QuaRot and SpinQuant~\citep{quarot2024,spinquant2024}, but here they probe a frozen representation--decoder interface rather than quantizing model weights or claiming low-bit inference.

\tabref{tab:q2-coordinate-mechanism} shows a sharp separation. Random-basis quantization restores $98.6$--$98.9\%$ agreement; reorienting the original residual while preserving every vector's error norm restores $99.7$--$99.9\%$; and matched-RMS Gaussian noise gives the same range. Clipping only $1\%$ of the per-vector dynamic range restores $99.1$--$99.6\%$. Per-coordinate calibration is also strong in both the native and randomized bases, so the evidence does not support a decoder that is generically tied to the original T5 coordinate axes. Instead, max-range ternary quantization produces a structured error direction that crosses decoder boundaries, with a few coordinate extremes controlling that direction. This is a second, intervention-based instance of the margin-bound warning: scalar latent RMSE is insufficient unless error orientation and local decoder sensitivity are also controlled. The result is deliberately not presented as a compressed tokenizer. At 512 latent dimensions, the symbol payload remains much larger than a token id, and calibration scales/metadata are excluded from the table.

\begin{table}[!htbp]
\centering
\caption{Mechanism audit for the 2-bit anomaly on 256 clean sequences per setting. Values are agreement with each checkpoint's native clean-state decisions. B/M/L use the same OpenWebText sequence sample; B-Wiki is held-out WikiText. RMSE is reported for B-OWT and is measured against the same clean latent. Residual reorientation and Gaussian controls preserve the original per-vector q2 RMS exactly. Calibration metadata and bitrate are not counted; this is a geometry audit, not a codec comparison.}
\label{tab:q2-coordinate-mechanism}
\small
\begin{tabular}{lrrrrr}
\toprule
Perturbation & B-OWT & M-OWT & L-OWT & B-Wiki & B-OWT RMSE \\
\midrule
Per-vector q2 & $0.623$ & $0.616$ & $0.698$ & $0.581$ & $0.753$ \\
Per-coordinate q2 & $0.996$ & $0.988$ & $0.994$ & $0.996$ & $0.674$ \\
Random-basis per-vector q2 & $0.986$ & $0.986$ & $0.989$ & $0.986$ & $0.715$ \\
Norm-preserving residual reorientation & $0.999$ & $0.997$ & $0.998$ & $0.999$ & $0.753$ \\
Matched-RMS Gaussian & $0.999$ & $0.997$ & $0.998$ & $0.999$ & $0.753$ \\
Per-vector q2, $99\%$ range & $0.996$ & $0.991$ & $0.994$ & $0.995$ & $0.623$ \\
\bottomrule
\end{tabular}
\end{table}

\FloatBarrier

\parhead{ZSBD geometry ablation}
ZSBD uses cosine nearest neighbors in the frozen T5 token-embedding table. To test what part of this geometry matters, we rerun the lookup on 1024 ELF-B final latents with unnormalized dot-product lookup, Euclidean nearest-neighbor lookup, cosine lookup after whitening, small frequency-biased cosine variants, and a permuted-label embedding table. \figref{fig:zsbd-geometry-ablation} summarizes the ablation. Cosine reaches $93.41\%$ agreement with the native decoder. Unnormalized dot-product lookup drops to $81.67\%$, Euclidean lookup to $3.44\%$, cosine after whitening to $4.29\%$, and permuted labels to near zero. A small frequency bias reaches about $94.03\%$, only slightly above pure cosine. The result identifies the useful structure as labeled angular token geometry, not vector norm, Euclidean proximity, or frequency alone. This also resolves an apparent tension with the order-sensitivity controls: bare token embeddings are order-blind, but ZSBD is applied to contextual final latents whose directions have already been shaped by the denoising trajectory. Cosine lookup reads out that contextual angular alignment; it does not make the token table itself an order-sensitive representation.

\begin{figure}[!htbp]
  \centering
  \includegraphics[width=.78\linewidth]{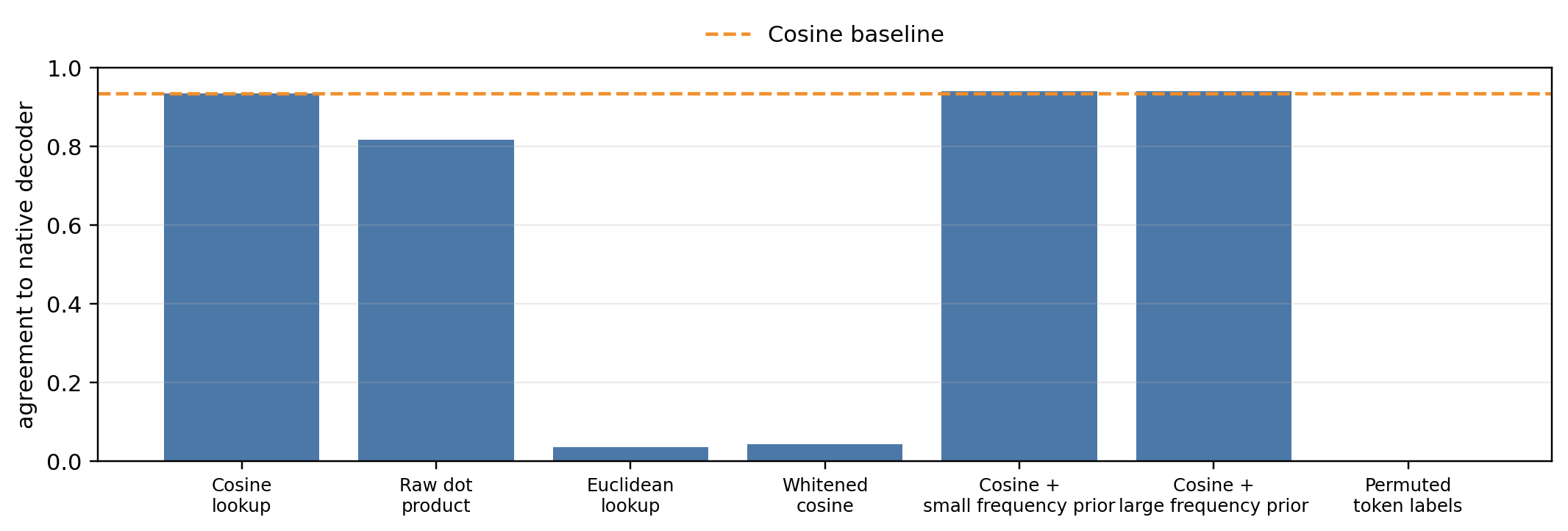}
  \caption{Geometry ablation for Zero-Shot Basin Decoding. Bars compare nearest-neighbor lookup rules on the same final ELF-B states; the dashed line marks the pure cosine baseline. Cosine lookup in the labeled T5 token-embedding table is strong, while Euclidean distance, whitening, or label permutation destroys agreement. Frequency gives a small calibration boost but does not explain the effect.}
  \label{fig:zsbd-geometry-ablation}
\end{figure}

\parhead{Residual-tail repair target analysis}
We reuse the 32k MDP residuals and ask which cheap gates capture the remaining disagreement with the native decoder. The MDP error rate is $2.11\%$. A ZSBD-wrong gate covers only $6.63\%$ of token positions but contains $73.4\%$ of MDP errors. The rare-token quintile covers $20\%$ of positions and contains $76.3\%$ of errors, while a native low-margin quintile is an oracle-style diagnostic that contains $91.7\%$ of errors. Numerics and subwords are smaller but sharper stress cases: numeric tokens cover only $0.65\%$ of positions but have an $8.69\%$ error rate, and subwords cover $13.0\%$ with a $5.21\%$ error rate. The tail-gated readout check in the main text uses the same 24576/4096 32k MDP split and evaluates GPT-2-Large PPL on the held-out 4096 texts. Its practical confidence gates use the learned linear readout's top-1-vs-top-2 margin; the ZSBD-disagreement gate compares the linear readout to frozen T5 nearest-neighbor lookup; and the broad tail gate adds predicted rare, numeric, or subword positions. These gates are evaluated as diagnostics of tail predictability, not as a deployed partial-native decoder.

We then ask whether the tail disappears with same-interface denoiser scale. A matched 512-sample exact ELF-B rerun preserves the 32k story at smaller scale: MDP agreement is $94.81\%$, ZSBD agreement is $93.16\%$, and the MDP-mismatch margin p10 is only $0.82$. For cross-scale direction, we use 64-sample exact native-token pilots for B/M/L; these pilots intentionally under-train the linear readout, so their MDP agreement is not used as a scaling law. Their training-free ZSBD and native-margin statistics are still informative because they use the same frozen token-embedding lookup and native decoder definition across scales. \tabref{tab:tail-scale-supplement} shows a heterogeneous tail. Rare, numeric, and subword ZSBD errors fall sharply from B to M/L, consistent with scale reducing some navigation difficulty. But the very-high-frequency bucket remains near a $13$--$15\%$ ZSBD error floor, indicating a different component: embedding-nearest-neighbor ZSBD and the native decoder boundary are not equivalent even for common tokens. Thus the residual tail is not only a rare-token problem. It mixes scale-sensitive lexical stress cases with a persistent difference between the tested lookup rule and the native decoder.

\begin{table}[!htbp]
\centering
\caption{Cross-scale residual-tail supplement from exact native-token pilots. The pilots use 64 samples for B/M/L and are used for ZSBD/native-margin tail typing, not as MDP scaling laws. Error columns report ZSBD error rates within each group.}
\label{tab:tail-scale-supplement}
\small
\begin{tabular}{lrrrrrr}
\toprule
Model & ZSBD agree. & Rare err. & Number err. & Subword err. & Very-high-freq err. & Native p10 \\
\midrule
ELF-B & $0.933$ & $0.120$ & $0.078$ & $0.110$ & $0.149$ & $12.35$ \\
ELF-M & $0.953$ & $0.042$ & $0.026$ & $0.079$ & $0.134$ & $13.59$ \\
ELF-L & $0.957$ & $0.026$ & $0.021$ & $0.067$ & $0.133$ & $13.75$ \\
\bottomrule
\end{tabular}
\end{table}

\parhead{Sampler-extension boundary checks}
The same diagnostics suggest several training-free or small-training sampler extensions, but the fixed-checkpoint results separate promising monitors from unsafe shortcuts. Pure ZSBD has its own confidence signal: routing only the lowest $10\%$ of ZSBD cosine-gap positions to the native decoder raises agreement from $93.37\%$ to $96.77\%$, and routing the lowest $20\%$ reaches $99.73\%$. A lightweight margin proxy is also promising. A 4.7k-parameter MLP using only online non-decoder features predicts the sequence 10th-percentile margin with test Spearman $0.985$ and MAE $0.53$; a conservative margin-$8$ proxy gate has $0.5\%$ false-safe rate and $75.7\%$ recall. These results support a cheap pre-filter followed by occasional native-margin verification. Equal-budget SC-update rejection gives a second training-free signal. We set the threshold on a held-out fixed-SC audit, using the normal sampler's free update delta rather than the more expensive zero-vs-SC disagreement. With $K=3$ candidates on 128 new samples, the random equal-budget baseline has PPL $23.45$, threshold rejection reaches $22.08$, and best-score selection reaches $20.96$; however, entropy and distinctness fall and repetition rises, so this is a PPL-diversity frontier, not a solved sampler.

We next test a direct use of the margin diagnostic that does not retrain ELF: \emph{Basin-Gated Retransport} (BGR). After one complete 32-step trajectory, BGR triggers a fresh independent 32-step trajectory only when the first trajectory's native margin p10 falls below a threshold frozen on 128 independent calibration generations. A frozen OPT-125M selector chooses between completed candidates; GPT-2-Large is used only after selection. The gate budget is fixed at the calibration median, and the primary control is one longer trajectory run at the observed mean BGR NFE. Four independent 256-sample held-out seeds are reported in \tabref{tab:bgr}. BGR improves over its paired first trajectory in all four seeds, and every seed's paired 95\% bootstrap interval for the PPL gain is positive. The mean improvement is $24.01\rightarrow21.84$ PPL at $48.03$ NFE. The fairer compute control is decisive, however: a single $48.25$-NFE trajectory obtains PPL $20.44$ and wins in two of four seeds, while pooled MAUVE on 1024 generations is $0.955$, versus $0.941$ for BGR and $0.918$ for the 32-step arm. BGR retains higher distinct-2 and lower repetition than the longer fixed arm, but the pooled distributional metric does not reverse the conclusion. Moreover, the frozen gate triggers on $27.0\%$--$65.2\%$ of held-out samples across seeds, showing that an absolute margin threshold calibrated on 128 independent generations is not stable enough for deployment. BGR therefore remains an informative boundary result: final margin identifies weak draws, but selective resampling and reranking do not outperform spending the same denoising compute on transport itself.

\begin{table}[!htbp]
\centering
\footnotesize
\setlength{\tabcolsep}{3pt}
\renewcommand{\arraystretch}{0.94}
\caption{Basin-Gated Retransport (BGR) under strict independent calibration and compute matching. Each row contains 256 held-out ELF-B generations. ``Fixed'' is one trajectory at the rounded observed BGR NFE. Mean $\pm$ SD is across four generation seeds; the gate and selector never access GPT-2-Large.}
\label{tab:bgr}
\resizebox{\linewidth}{!}{%
\begin{tabular}{lrrrrrrr}
\toprule
Seed & Trigger & BGR NFE & PPL-32 & PPL-BGR & PPL-fixed & Dist-2 BGR/fixed & Rep-4 BGR/fixed \\
\midrule
20260811 & $56.3\%$ & $50.0$ & $21.78$ & $19.34$ & $20.09$ & $.674/.686$ & $.098/.096$ \\
20260812 & $27.0\%$ & $40.6$ & $23.98$ & $23.00$ & $24.16$ & $.713/.716$ & $.069/.074$ \\
20260813 & $52.0\%$ & $48.6$ & $26.65$ & $23.74$ & $19.09$ & $.718/.661$ & $.063/.118$ \\
20260814 & $65.2\%$ & $52.9$ & $23.62$ & $21.27$ & $18.40$ & $.704/.647$ & $.076/.135$ \\
\midrule
Mean $\pm$ SD & $50.1\!\pm\!16.4\%$ & $48.0\!\pm\!5.2$ & $24.01\!\pm\!2.01$ & $21.84\!\pm\!1.96$ & $20.44\!\pm\!2.58$ & $.702/.677$ & $.077/.106$ \\
\bottomrule
\end{tabular}}
\end{table}

The released PD student enables a different compute allocation: use its viable short trajectories to draw multiple independent proposals. \tabref{tab:pd-diversify-select} reports the matched 32-denoiser-NFE frontier. Without selection, 8-step and 16-step student proposals already have higher pooled MAUVE than the teacher32 baseline ($0.936$ and $0.933$ versus $0.926$), although their mean PPL is slightly worse. A frozen OPT selector reverses PPL strongly: $16\!\times\!2$ reaches $18.77\pm0.41$, and $8\!\times\!4$ reaches $15.91\pm0.28$, versus teacher32 at $21.46\pm1.27$. The result survives a cross-family selector: Qwen selection gives $18.87\pm0.39$ and $16.01\pm0.29$, respectively, when GPT-2-Large remains evaluator-only. The most aggressive selection is not a free improvement. For $8\!\times\!4$, OPT selection lowers pooled MAUVE from $0.936$ to $0.894$, lowers distinct-2, and raises four-gram repetition. The conservative $16\!\times\!2$ allocation is nearly MAUVE-neutral ($0.928$ after selection), but still pays uncounted selector work. The mechanism conclusion is sharper than a sampler claim: transport distillation creates viable independent proposals, while scalar likelihood selection becomes the next bottleneck and pushes toward a narrower mode. A new method would need a diversity- or task-aware selector and an external conditional or pass@$k$ validation; these experiments do not establish one.

\begin{table}[!htbp]
\centering
\small
\caption{Fixed 32-denoiser-NFE allocation using the released PD student. PPL is GPT-2-Large mean $\pm$ SD across four independent 256-sample seeds. MAUVE pools 1,024 texts. OPT and Qwen are frozen selectors; GPT-2-Large never selects. Diversity statistics are reported for the representative seed 42 and show the direction of the selection trade-off. Selector compute is additional to the matched denoiser NFE.}
\label{tab:pd-diversify-select}
\begin{tabular}{llrrrr}
\toprule
Allocation & Policy & PPL & MAUVE & Distinct-2 & Rep-4 \\
\midrule
Teacher $32\!\times\!1$ & First & $21.46\!\pm\!1.27$ & $.926$ & $.693$ & $.076$ \\
Student $32\!\times\!1$ & First & $20.91\!\pm\!.38$ & $.922$ & $.662$ & $.094$ \\
Student $16\!\times\!2$ & First & $21.96\!\pm\!.50$ & $.933$ & $.675$ & $.085$ \\
Student $16\!\times\!2$ & OPT select & $18.77\!\pm\!.41$ & $.928$ & $.648$ & $.103$ \\
Student $16\!\times\!2$ & Qwen select & $18.87\!\pm\!.39$ & -- & -- & -- \\
Student $8\!\times\!4$ & First & $22.31\!\pm\!.36$ & $.936$ & $.668$ & $.094$ \\
Student $8\!\times\!4$ & OPT select & $15.91\!\pm\!.28$ & $.894$ & $.589$ & $.146$ \\
Student $8\!\times\!4$ & Qwen select & $16.01\!\pm\!.29$ & -- & -- & -- \\
\bottomrule
\end{tabular}
\end{table}

A softer token-wise commitment variant tests whether gentler updates mitigate the observed hard-freezing failure. Instead of holding a committed position fixed, we blend the model update with the previous latent after a position crosses a native-margin threshold, using $\alpha\in\{0.1,0.3,0.5\}$ and thresholds $\{4,8,12\}$ on 1024 matched ELF-B SDE32 samples. The variants commit $93.8$--$99.5\%$ of valid tokens and estimate $39.4$--$55.9\%$ potential token-step savings. High-threshold variants preserve the measured PPL and text entropy: the margin-12 settings obtain PPL $23.55$--$23.66$ versus the matched baseline $23.70$, with essentially unchanged text entropy. They do not preserve native identity, however: token agreement to the matched baseline is only $81.9$--$84.1\%$ for those same margin-12 variants, and the current full-monitor implementation is about $3.07\times$ slower than baseline. Thus soft TEC is a useful boundary result. The tested gentle blends preserve two aggregate text statistics while changing many native decisions; they do not establish semantic fidelity or acceleration. A production rule would still need a cheap monitor, sparse execution, and an explicit fidelity target. This is why TEC remains a future sampler-design problem rather than an additional claim of this paper.

As a small training-time TEC surrogate, we ask whether the late interface readout can be made robust to stale token positions without retraining the denoising transport. Using the cached Phase 73 trajectories, we train two step-28 full-decoder adapters on 512 samples: one on clean step-28 states, and one with $25\%$ of valid token positions randomly replaced by their step-24 states during training. This stale-state augmentation leaves clean agreement unchanged ($94.85\%$ to $94.92\%$). Under evaluation with $25\%$ and $50\%$ step-24 replacement, agreement improves only slightly ($+0.24$ and $+0.29$ pp), but under harder step-20 replacement it improves by $+1.03$ and $+2.16$ pp. This is a positive feasibility signal for exposing the interface to committed or stale positions during training. It is not a solution to TEC: the denoiser was not retrained, no PPL or wall-clock gain is claimed, and the result only shows that a local adapter can learn some stale-state tolerance.

Two natural shortcuts fail and are therefore kept as negative controls. Token-wise Early Commitment freezes positions once their native margin exceeds a threshold. At 128 samples, margin-$8$ TEC commits $91.9\%$ of tokens and estimates $45.4\%$ token-step savings, but PPL worsens from $23.83$ to $38.67$, the 10th-percentile margin falls from $12.33$ to $2.56$, and agreement to the matched baseline falls to $28.9\%$; margin-$12$ is less aggressive but still poor. The tested post-hoc freezing rule therefore fails to preserve the baseline output and its PPL. This comparison does not identify whether stale context, update error, or another change mediates the failure. The context-only intervention in \secref{app:context-boundary} tests a narrower mechanism: a fixed focal latent need not imply a fixed native token. One-step projection from noise also fails. A token-level ridge map from initial noise to final clean states, trained on generated trajectories and evaluated on a disjoint held-out sequence split, agrees with the full ELF sampler on only $3.42\%$ of tokens. The oracle token-embedding projection reaches $97.83\%$ agreement, showing that the token embeddings can represent the final labels; the tested one-step regression does not learn a successful noise-to-endpoint map. This does not rule out trained one-step generators or distillation.

The same conclusion appears in late-interface distillation. A lightweight per-step denoiser trained to predict the final state from selected trajectory states improves sharply with phase and reaches $66.5\%$ native-token agreement at step 28, far above the one-step-from-noise baseline but still far from a usable sampler. Under the tested training budget, adding decoded cross-entropy to form a joint transport-and-decode objective does not solve this problem: the best joint variant reaches only $19.2\%$ agreement, and a Jacobian-regularized variant reaches $17.1\%$, even though both can produce deceptively low PPL in earlier phases. This failure is informative. Late states are learnable enough to support bypass and routing, but learning transport and token calibration jointly in a few-step student is a different problem from reading out a generated basin.

\begin{figure}[!htbp]
  \centering
  \includegraphics[width=\linewidth]{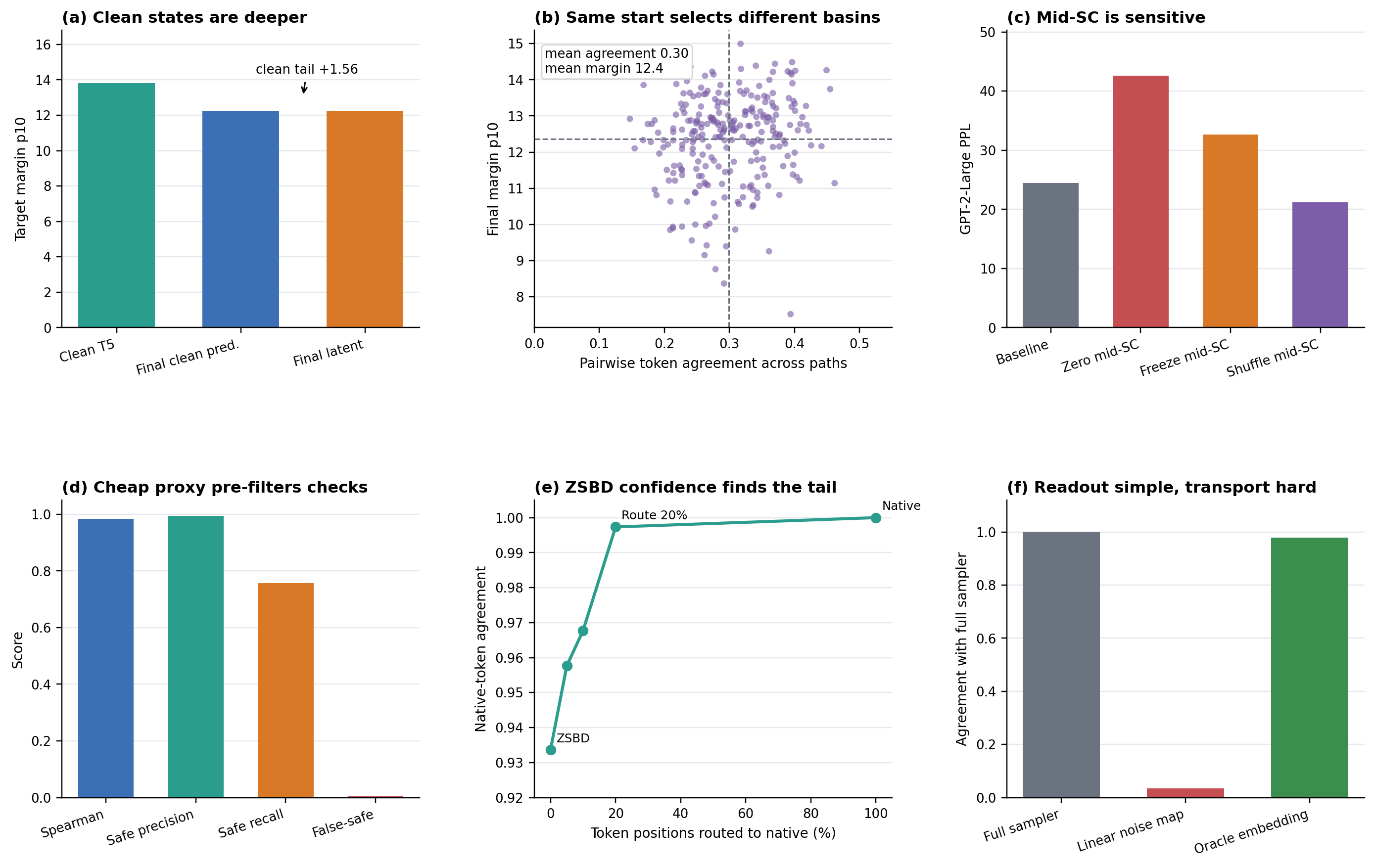}
  \caption{Additional insights mined from completed experiments. (a) Clean T5 interface states are deeper than generated final states in the lower margin tail, so ELF reaches a readable generated decoder-basin region with lower native-margin depth than the same-token T5 re-encoding; this score gap does not measure manifold distance. (b) Same-start SDE paths end in high-margin decoder-basin regions but choose different token labels, showing path-selected rather than unique attraction. (c) Mid-trajectory self-conditioning interventions change PPL and margins; shuffling can lower PPL while collapsing diversity. (d) A small online proxy can pre-filter native-margin checks with high precision. (e) ZSBD cosine-gap confidence identifies much of the token tail without a learned readout. (f) One-step projection fails even though an oracle token-embedding projection reaches $97.83\%$ agreement, separating simple readout from nontrivial transport.}
\label{fig:existing-result-mining}
\end{figure}

\parhead{Final-basin latent editing}
As a boundary test for latent editing, we estimate a sentiment direction from Yelp Polarity embeddings and add it to ELF final latents under a native-margin acceptance gate. This experiment tests whether the decoder basin remains a linearly editable semantic canvas. It does not under conservative margins: target success changes only modestly because most tokens remain inside the decoder basin. We then run an adaptive line search along the same sentiment direction. This converts the failure into a sharper boundary test. \figref{fig:bcd-final} shows that a safety margin of $2$ changes only $4.1\%$ of tokens and raises target success from $68.5\%$ to $72.9\%$ with little PPL change; forcing the margin down to zero raises target success to $99.7\%$ but changes $38.4\%$ of tokens and raises PPL to $51.0$. Finally, we test two stronger final-state repairs on 128 samples across four GPU shards: a native decoder-gradient direction that directly increases positive sentiment-token logits, and a token-selective variant that applies the same direction only to the top sentiment-evidence positions. \figref{fig:bcd-decoder-gradient} shows that the all-position decoder-gradient direction is more effective but still margin-limited; at safety margin $4$ it changes $2.6\%$ of tokens and raises target success to $76.6\%$ with PPL $26.8$, while at safety margin $2$ it reaches $99.2\%$ target success but raises PPL to $41.7$. Token-selective variants have little effect on native decisions, changing only about $0.6\%$ of tokens and leaving target success unchanged. \figref{fig:bcd-trajectory-gradient} shows that a trajectory-level repair is more benign but still not a clean editor: step-16 decoder-gradient intervention raises target success from $64.1\%$ to $69.5\%$ with similar PPL ($24.6$ vs.\ $24.4$), but requires changing $23.7\%$ of tokens. BCD is therefore a boundary result rather than a probe contribution.

\begin{figure}[!htbp]
  \centering
  \includegraphics[width=\linewidth]{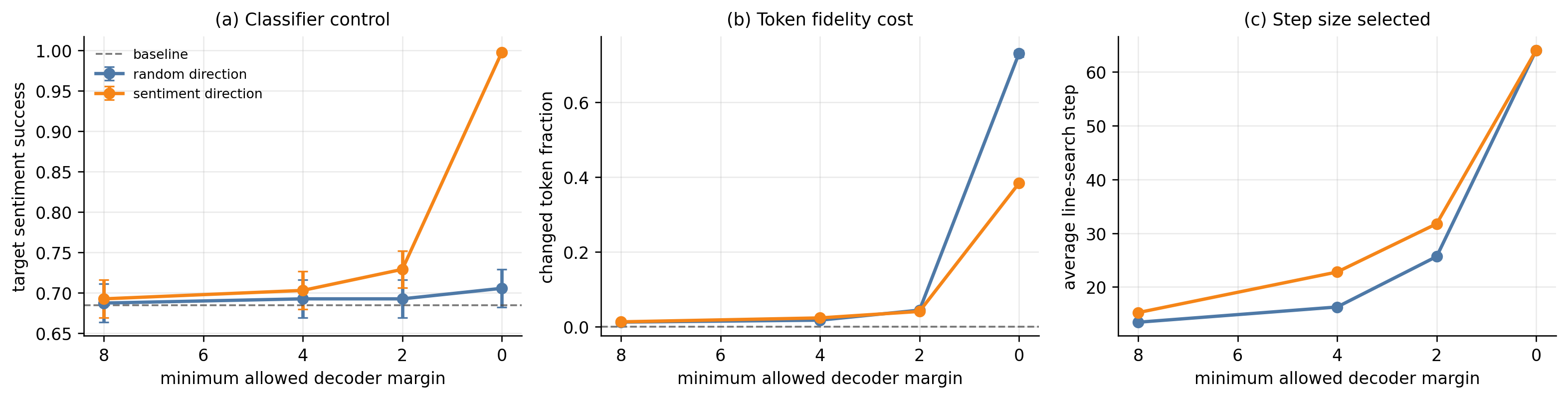}
  \caption{Decoder-basin BCD sentiment editing. Left: target sentiment success under adaptive line search. Middle: changed-token fraction relative to the unedited decode. Right: average line-search step selected by the margin constraint. Safe-margin edits leave most token decisions unchanged, whereas boundary-level edits can force target sentiment only by taking large steps that sacrifice token fidelity.}
  \label{fig:bcd-final}
\end{figure}

\begin{figure}[!htbp]
  \centering
  \includegraphics[width=\linewidth]{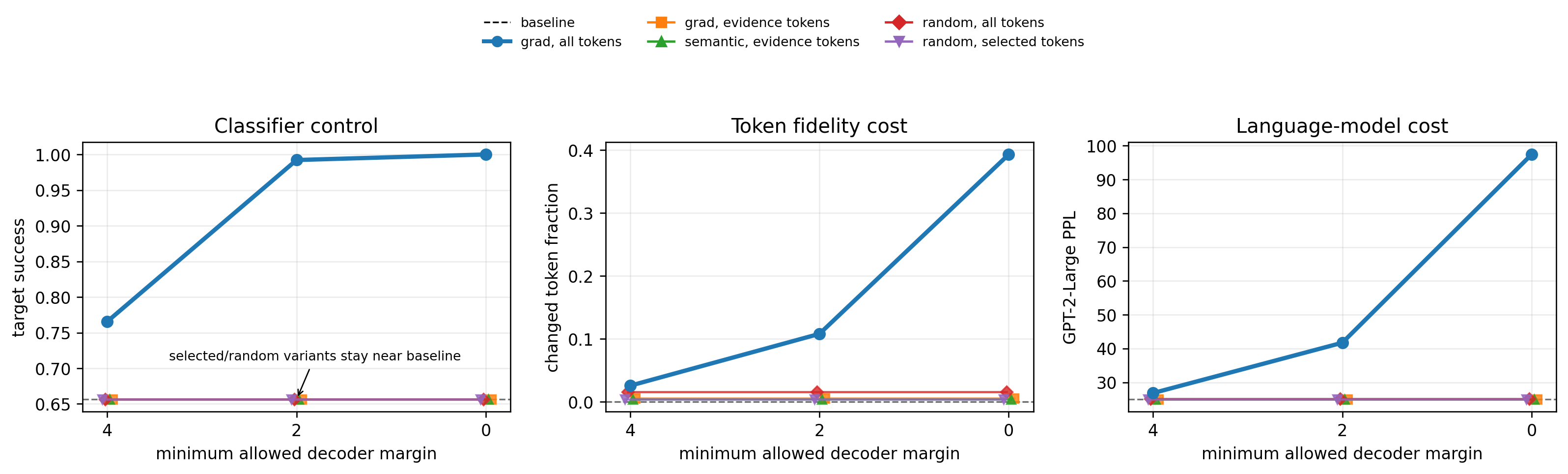}
  \caption{Decoder-gradient and token-selective BCD repairs. Left: target classifier success as the minimum allowed decoder margin is relaxed. Middle: changed-token fraction relative to the unedited decode. Right: GPT-2-Large PPL of the edited outputs. Decoder-gradient directions can control the target classifier, but the quality/fidelity cost rises sharply as the margin lower tail is relaxed. Perturbing only sentiment-evidence positions leaves almost all token decisions unchanged.}
  \label{fig:bcd-decoder-gradient}
\end{figure}

\begin{figure}[!htbp]
  \centering
  \includegraphics[width=\linewidth]{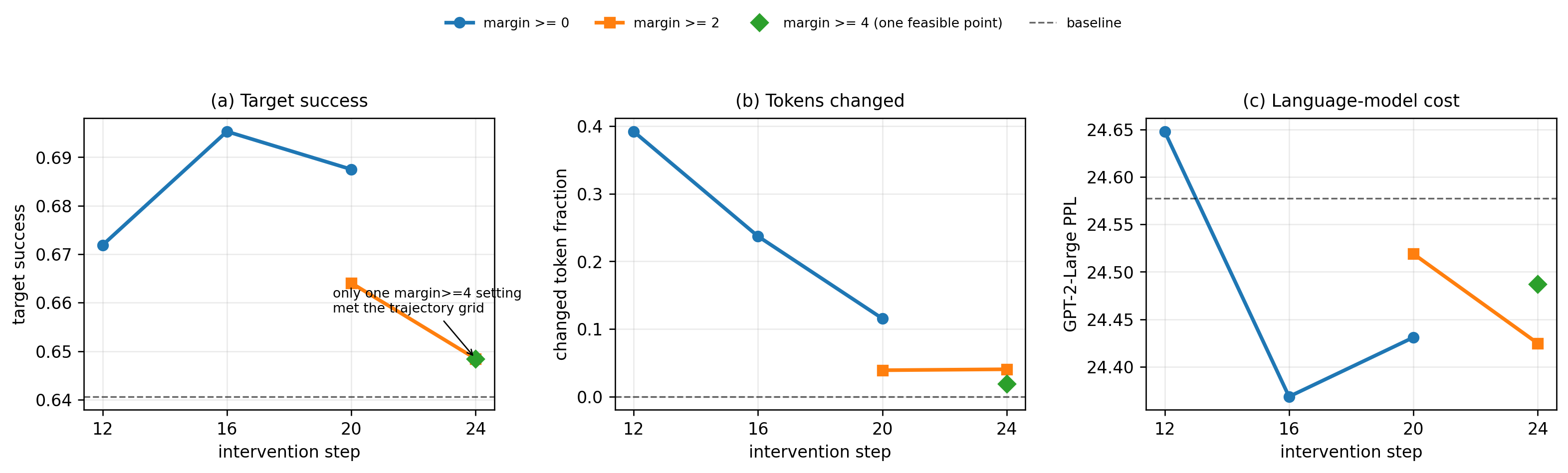}
  \caption{Trajectory-level BCD repair. Left: target sentiment success when the decoder-gradient intervention is applied at steps 12, 16, 20, and 24. Middle: fraction of tokens changed relative to the unedited decode. Right: GPT-2-Large PPL of the edited outputs; dashed lines mark the unedited baseline. Earlier intervention can slightly improve the target classifier at similar PPL, but the gain comes from changing a substantial fraction of tokens. Late or margin-conservative interventions have little effect on token decisions.}
  \label{fig:bcd-trajectory-gradient}
\end{figure}

\FloatBarrier

\parhead{Layer-wise decoder-basin sweep}
For T5~\citep{t52020}, the historical sweep extracts hidden indices 0--6 and the final hidden state, applies ELF's latent mean/std transformation, and decodes with ELF's native decoder at several corruption levels. Index 6 and the final hidden state are duplicate post-RMSNorm outputs; earlier indices have not received this terminal normalization. The figure now displays the duplicate once, while retaining both historical records. BERT~\citep{bert2019}, RoBERTa~\citep{roberta2019}, and GPT-2~\citep{gpt22019} use their native LM heads. The GPT-2 code aligns each hidden state with the \emph{next} token; the earlier description of a same-position negative control was incorrect. Native-head prediction and unmasked reconstruction are different tasks. \secref{app:layer-interface-controls} records the paired normalization follow-up. The extended IRS checks for T5-base/large, mT5, ByT5, Switch, and T5Gemma use the same fixed held-out text subset and report denoisability, order sensitivity, margin depth, and lightweight readout agreement under each tested pairing. They are intended to test interface compatibility, not to compare the general-purpose quality of pretrained encoders with different dimensions or tokenizers.

\parhead{PPL evaluation}
GPT-2-Large~\citep{gpt22019} PPL is computed with fixed truncation length 1024, using per-sample negative log likelihoods and token counts. Our default, including the new dose-grid and crossed-decoder evaluations, is corpus-level PPL: the exponential of total NLL divided by total scored tokens. Historical analyses explicitly labeled geometric PPL, including the BGEE, 4k MDP, and topic audits, instead exponentiate the mean per-text normalized NLL and retain that original protocol. The two summaries weight variable-length texts differently and must not be pooled. Per-sample PPL also enters scatter plots and Spearman correlations.

\parhead{Long-form topic audit}
The long-form boundary audit uses 1000 OpenWebText reference samples with at least 200 words, 1000 ELF-B SDE64 samples, 1000 ELF-B SDE32 samples, 512 available ELF-B ODE32 samples, 1000 MDP-B 32k linear-readout samples, and 1000 ZSBD-B 16k lookup samples. We split each document into sentence-like spans with a punctuation-and-capitalization heuristic, discard spans shorter than five words, and keep at most the first 32 spans. Each span is embedded with a locally stored sentence-embedding model using mean pooling over the final hidden states followed by L2 normalization, following the Sentence-BERT protocol~\citep{sentencebert2019}. We then report adjacent-sentence cosine, first-to-last sentence cosine, embedding dispersion, and the fraction of adjacent transitions below the 5th percentile of the OWT adjacent-cosine distribution. These metrics are used only as boundary diagnostics for document-level topic drift; they are not a replacement for human evaluation or task-specific discourse metrics.

\parhead{MAUVE and reference distribution metrics}
MAUVE~\citep{mauve2023} is computed on generated samples against OpenWebText~\citep{openwebtext2019} reference samples using a GPT-style feature extractor. JS divergence is computed over token unigram distributions. Both metrics are complementary rather than decisive: MAUVE is sensitive to sample count and feature choice, while token JS captures only shallow distributional shifts.

\parhead{Decoder calibration}
For learned decoder baselines, we generate final ELF latents, split them into train and held-out subsets, and train small token-level decoders to imitate the official decoder's argmax tokens. We evaluate official decoder, direct unembedding, learned linear decoder, two-layer MLP decoder, and three-layer MLP decoder. Entropy-matched variants sample from small decoders at a temperature chosen to approach the official decoder's logit entropy.

For the late-bypass follow-up, we use the same generated-trajectory protocol but train readouts on intermediate predicted-clean states. At step 28, a plain linear readout trained on 1024 trajectories reaches $87.6\%$ agreement and a two-layer MLP reaches $92.6\%$, leaving a visible $\approx7.4$ pp gap to the native final decode. Increasing to 2048 trajectories and using the full decoder-head bottleneck raises agreement to $95.3\%$ with PPL $23.65$, close to the official decode. This is a systems signal rather than a new decoder claim: late states are nearly readable by a lightweight head, while the tested heads leave residual native-decision disagreements.

\parhead{Decoder margin analysis}
For each final latent $h$ and official token $i$, we compute the margin between the logit of $i$ and the largest competing logit. We then corrupt $h$ with increasing noise and measure positive-margin fraction and token recovery. A subset of token positions is also used for a first-order boundary estimate $m/\|\nabla_h(g_i-g_j)\|_2$, where $j$ is the strongest competitor. This estimate is not an exact distance for the nonlinear contextual decoder, but it makes \thmref{thm:decoder-margin} operational and reveals whether many tokens are close to a decision boundary.

\parhead{Decoder-noise training ablation}
We used the official PyTorch ELF training pipeline to train short ELF-B ablations on four RTX 3090 GPUs. We varied only the decoder-input noise scale in the decoder branch, keeping architecture, data path, and training length fixed. The basin evaluation uses held-out T5 latents and the native decoder from each checkpoint. For a fair noisy-basin comparison, we report both absolute token recovery at each corruption level and retention relative to the checkpoint's own clean recovery.

\parhead{Decoder-branch frequency ablation}
We also vary the decoder-branch probability while fixing the decoder-input noise scale at $5$. This separates decoder-basin learning from flow learning. In this training setup, the pure-denoising run has no CE branch and fails native token decoding. The decoder-only run has no flow-loss branch and fails the tested end-to-end generation task despite strong handed-corruption recovery. Intermediate settings test how much decoder exposure is needed to form a robust native interface.

\parhead{Explicit margin-loss ablation}
We add a hinge-style penalty on the gap between the correct-token logit and the strongest competing logit in the decoder branch, using weights $0.05$ and $0.10$ with the official-like decoder-input noise scale $5$ and decoder-branch probability $0.2$. The goal is not to propose a new training recipe, but to test the natural hypothesis that the fragile margin tail can be repaired by directly maximizing logit margins. The evaluation uses the same basin sweep as the decoder-noise ablation.

\parhead{10k margin-loss extension}
Because a short 5k run could understate the effect of the margin objective, we extend the official-like baseline and the $0.10$ margin-loss run to 10k steps. Both use the same OWT data path, global batch size, optimizer, decoder noise scale, and decoder branch frequency. The 10k comparison is evaluated with raw weights and the same 1024-sample basin sweep. This extension tests whether the direct margin penalty eventually becomes a basin-widening objective once clean decoding and CE loss have improved.

\parhead{PCA controlled degradation}
For the PCA bottleneck experiment, we fit PCA on contextual T5 latents from OpenWebText, project to ranks $r\in\{512,256,128,64,32\}$, and reconstruct back to the native 512-dimensional ELF interface. Decoder diagnostics are computed on held-out clean T5 latents. Generation uses the official ELF-B SDE sampler with either a final-only bottleneck or an every-step bottleneck. This experiment intentionally keeps the checkpoint fixed; it tests whether the diagnostic flags a controlled interface failure without retraining.

\parhead{Cola-DLM boundary tests}
The Cola-DLM experiments use the released VAE and DiT components only to probe interface behavior. The key test corrupts clean VAE latents as $z_t=t z+(1-t)\epsilon$ and decodes them with the native VAE decoder. The resulting recovery curve measures the width of the decoder-compatible basin, not the full quality of Cola's trained prior. We also run a balanced DiT trajectory audit on four RTX 3090 GPUs: classifier-free guidance (CFG)~\citep{classifierfree2022} values $0/1/3/7$ each use 512 OpenWebText samples, 16 Euler steps, prefix length 16, block size 16, and batch size 64. The batch-size probe reached out-of-memory (OOM) at 128 and used about 18GB/GPU at 64. The audit records decoder entropy, decoded-token change, token recovery to the held-out target block, and latent cosine to the held-out target latent. These target metrics are teacher-forced diagnostics; low target recovery is not a failure of free generation, but a boundary showing that Cola's DiT enters a decoder-confident basin without reproducing the next block's teacher tokens, precisely the interface gap predicted by the margin-basin diagnostic: decoder confidence and ground-truth token alignment are distinct signals.

\section{Additional Objective Metrics}
\label{app:objective-metrics}

\parhead{Training-free and decoder-calibration controls}
The first appendix figures in this section document why the main text treats simple interventions as probes rather than mature methods. \figref{fig:one-line-controls} collects the one-line controls that use late clean-prediction averaging, margin-aware temperature, and shuffled-margin baselines. \figref{fig:large-decoder-calibration} gives the corresponding decoder-calibration view: learned readouts can move PPL, entropy, and agreement in different directions, so a low PPL point is not by itself evidence of a faithful interface. Together, the two figures justify the paper's restraint: the trajectory and margin signals are information-bearing, but local post-hoc rules and low-PPL learned decoders are not reliable replacements for the native denoising/decoding interface.

\begin{figure}[!htbp]
  \centering
  \includegraphics[width=.82\linewidth]{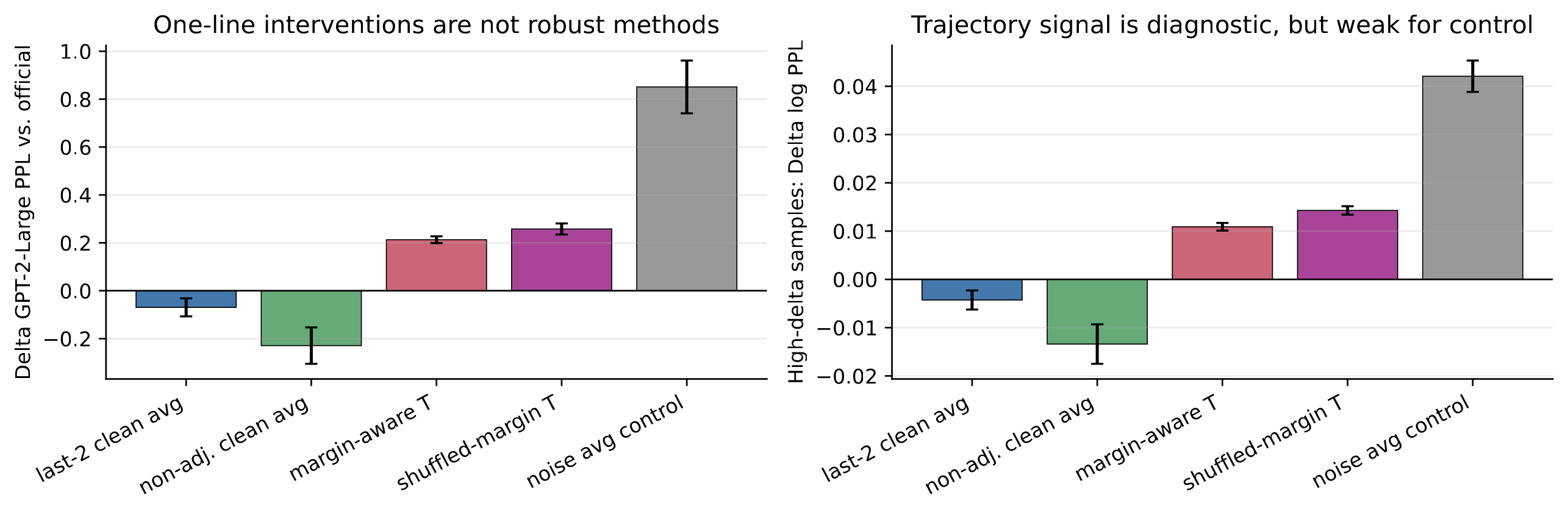}
  \caption{Training-free one-line controls across four runs. Left: late clean-prediction averaging produces small smoothing gains, while averaging with initial noise fails. Right: separating high-delta from low-delta samples shows that the trajectory signal is diagnostic, but margin-aware temperature is no better than a shuffled-margin control. The diagnostic signals are real; local rules alone are not robust samplers.}
  \label{fig:one-line-controls}
\end{figure}

\begin{figure}[!htbp]
  \centering
  \includegraphics[width=.78\linewidth]{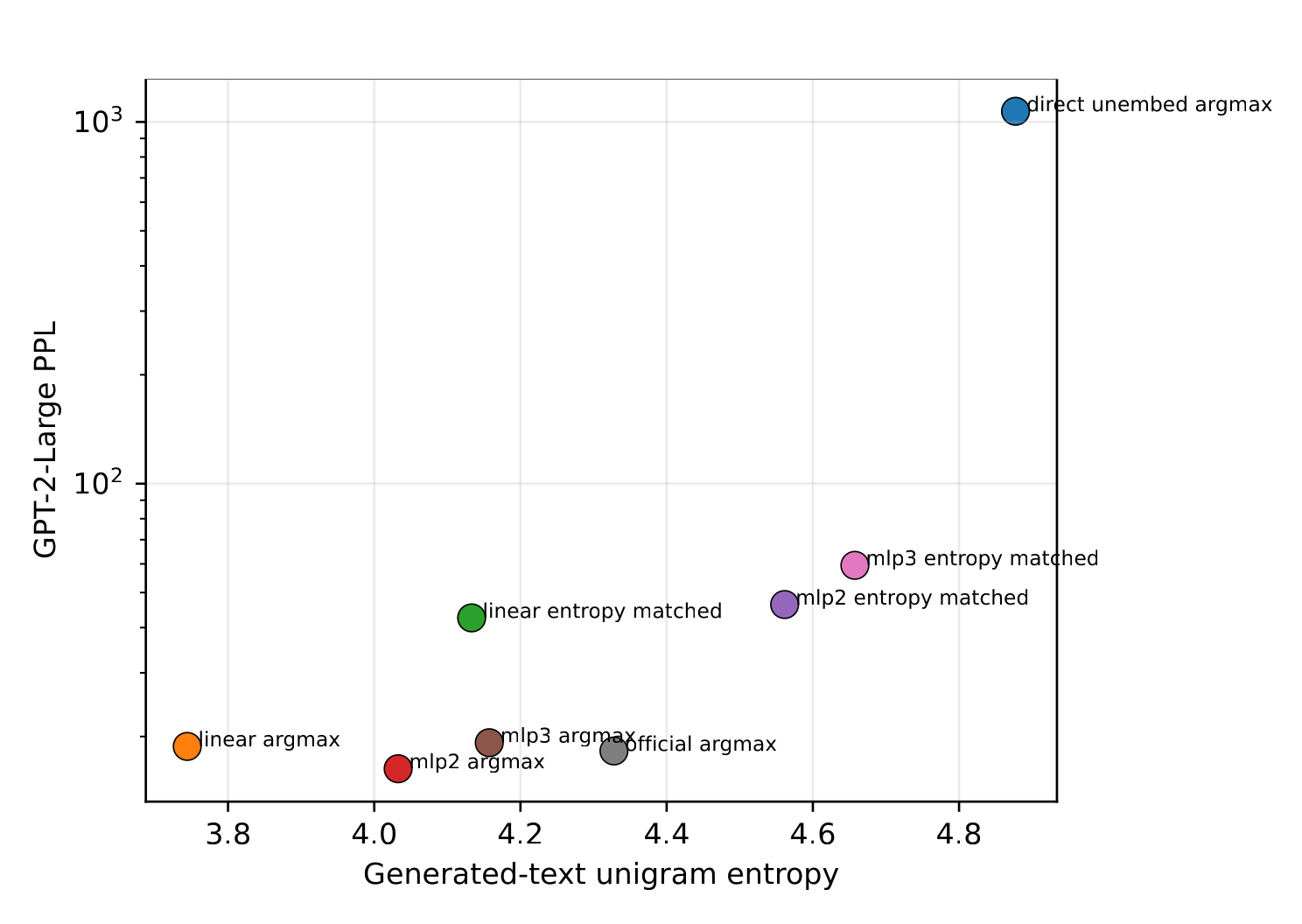}
  \caption{Large matching-set decoder calibration. Each point is a decoder variant evaluated by generated-text unigram entropy on the x-axis and GPT-2-Large PPL on the y-axis. Labels distinguish argmax and entropy-matched variants. Small decoders can reduce PPL, while PPL, entropy, and agreement to the official decoder move independently.}
  \label{fig:large-decoder-calibration}
\end{figure}

\parhead{Sampler metric frontiers}
The next group of figures records objective frontier behavior that would be hidden by a single PPL table. \figref{fig:repetition} pairs PPL with repeated n-gram statistics, \figref{fig:gamma} and \figref{fig:gamma-diversity} show how sampler noise and schedule choices move along quality-diversity frontiers, and \figref{fig:refdist} adds a reference-distribution distance axis. \figref{fig:intermediate-readability} then separates readable intermediate samples from decoder-basin compatibility. The shared lesson is that a sampler change should be interpreted as a movement on several coupled frontiers, not as a scalar improvement unless PPL, diversity, repetition, reference distance, and decoder compatibility move coherently.

\begin{figure}[!htbp]
  \centering
  \includegraphics[width=.72\linewidth]{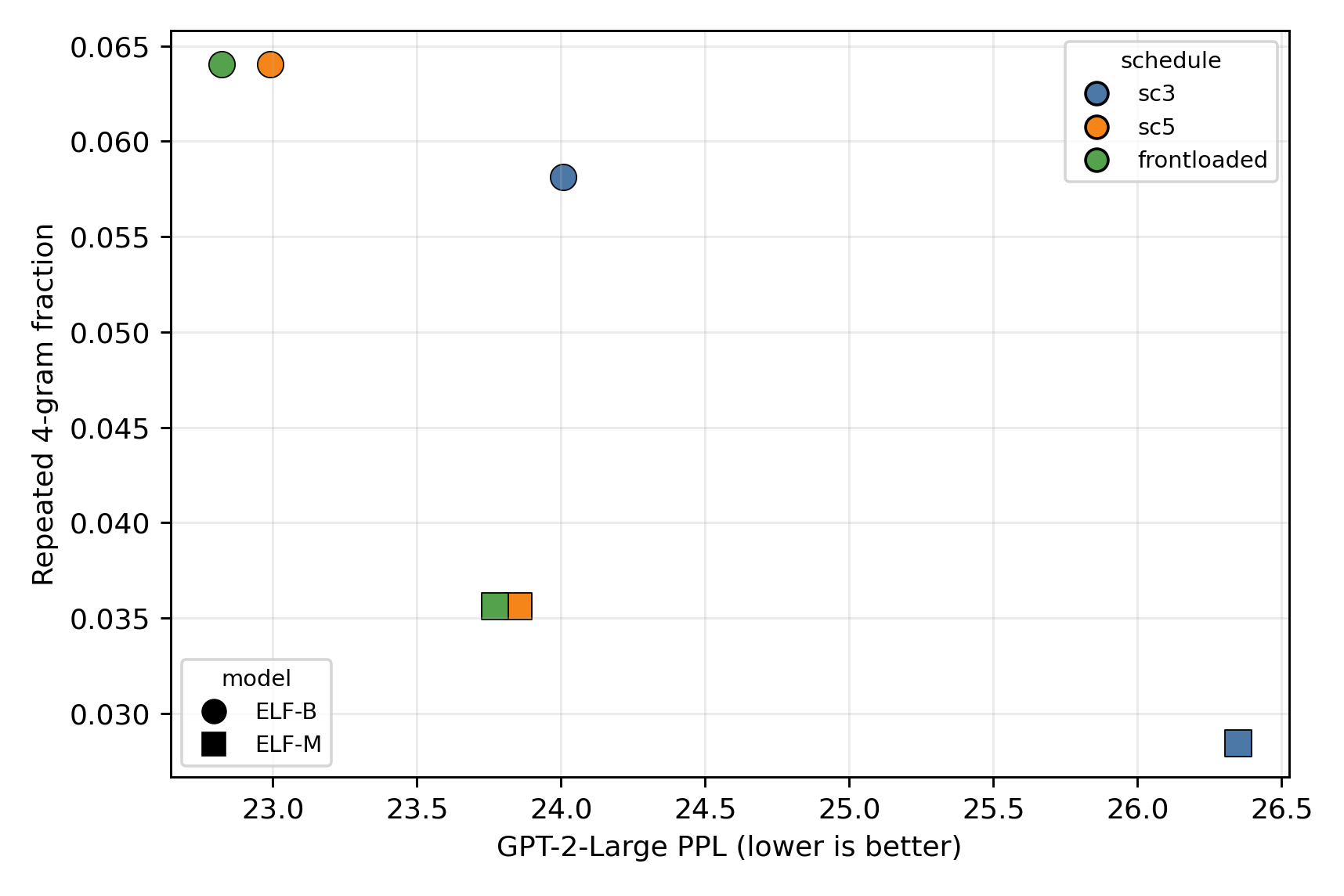}
  \caption{PPL--repetition trade-off. Each point is a sampler setting, with marker shape indicating model size and color indicating the self-conditioning schedule. The plot pairs GPT-2-Large PPL with repeated 4-gram fraction, showing that lower PPL is not automatically better if it is obtained by moving toward repetitive or low-diversity text.}
  \label{fig:repetition}
\end{figure}

\begin{figure}[!htbp]
  \centering
  \includegraphics[width=.72\linewidth]{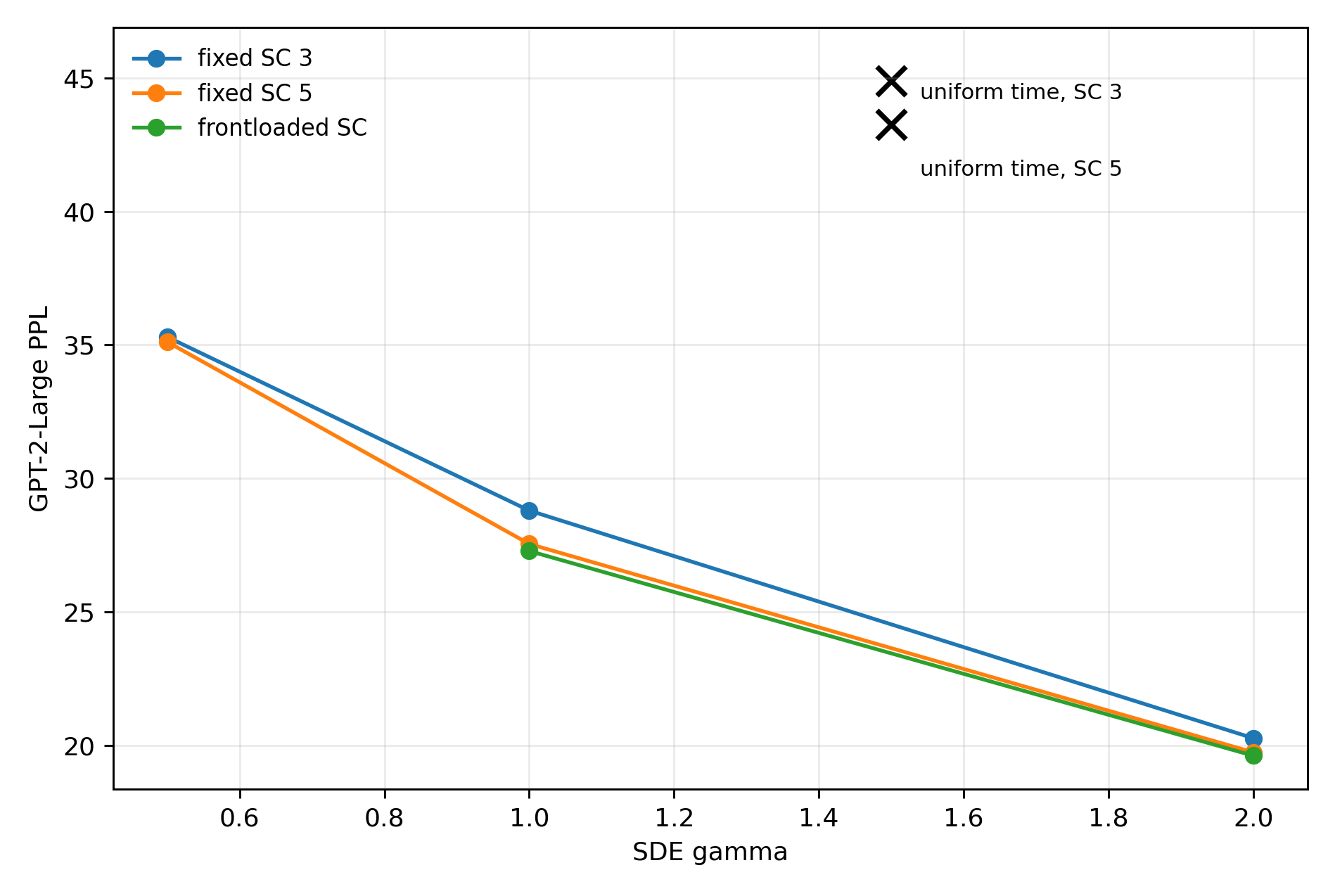}
  \caption{SDE gamma and time-schedule sweep. The curves vary sampler noise strength under fixed self-conditioning schedules, while the black crosses mark uniform-time controls. Gamma and schedule jointly control PPL rather than acting as a single monotone quality knob.}
  \label{fig:gamma}
\end{figure}

\begin{figure}[!htbp]
  \centering
  \includegraphics[width=.72\linewidth]{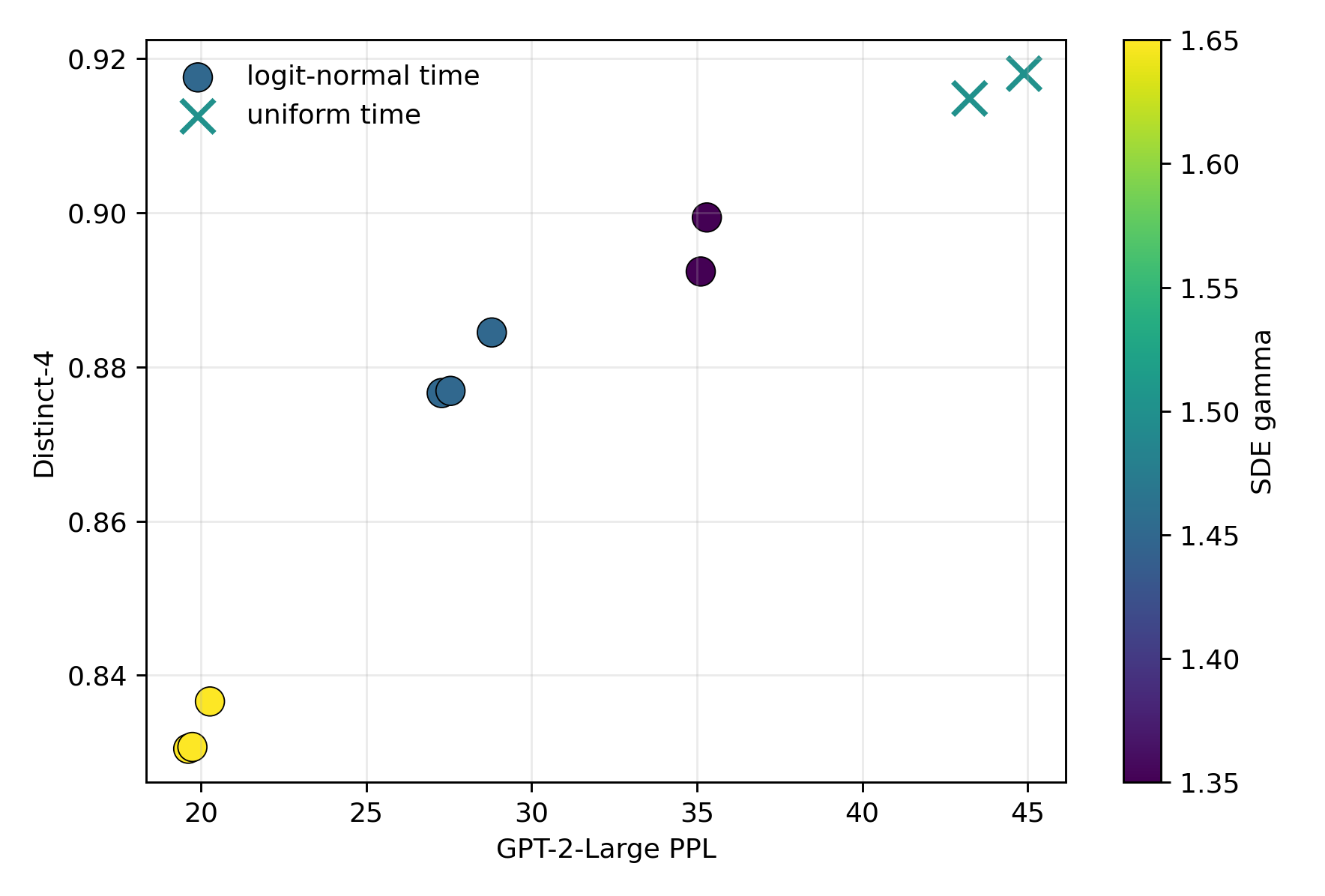}
  \caption{Gamma diversity frontier. Each point is a sampler configuration, colored by SDE gamma and marked by time schedule. Higher gamma can lower PPL while moving distinctness and repetition; the improvement should therefore be read as a frontier movement, not a free quality gain.}
  \label{fig:gamma-diversity}
\end{figure}

\begin{figure}[!htbp]
  \centering
  \includegraphics[width=.70\linewidth]{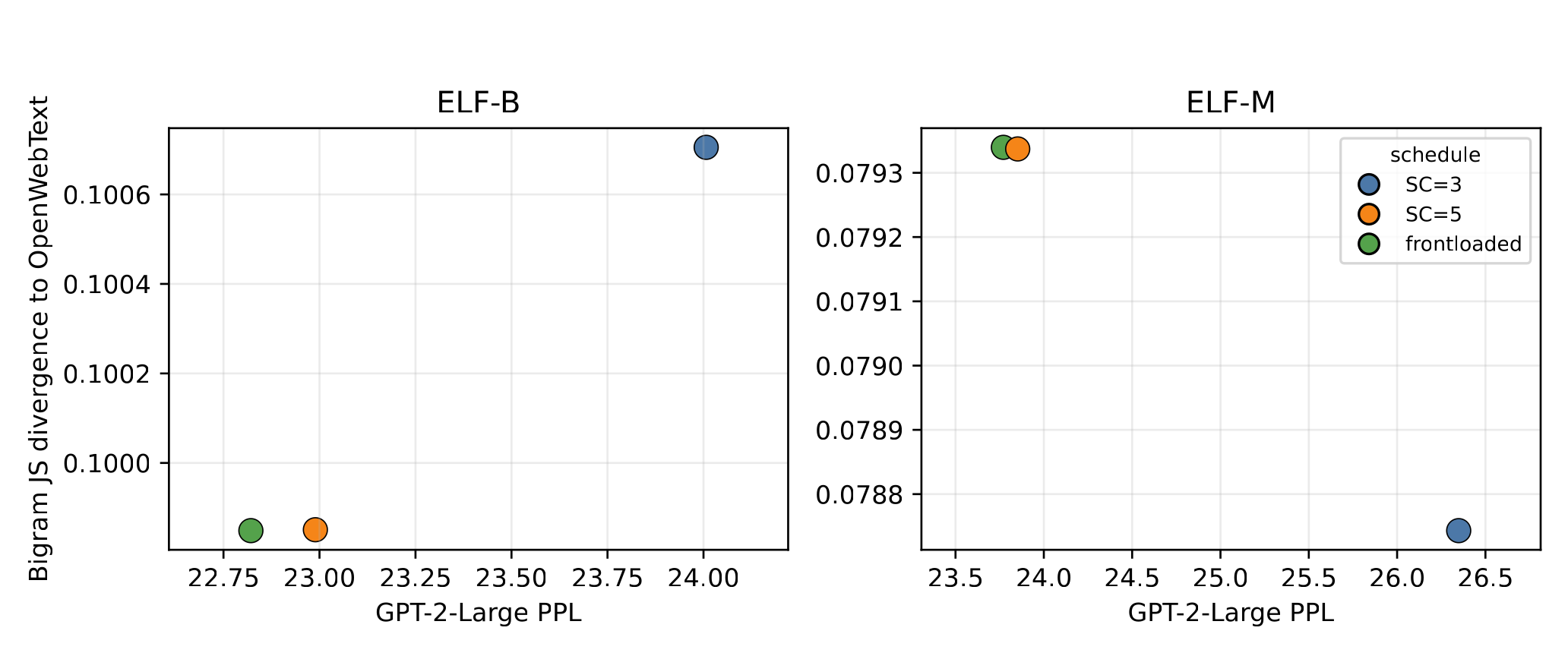}
  \caption{Reference-distribution distance. Left: ELF-B schedules plotted by GPT-2-Large PPL and bigram JS divergence to OpenWebText references. Right: the same diagnostic for ELF-M. This adds another objective frontier axis beyond fluency and diversity.}
  \label{fig:refdist}
\end{figure}

\begin{figure}[!htbp]
  \centering
  \includegraphics[width=.70\linewidth]{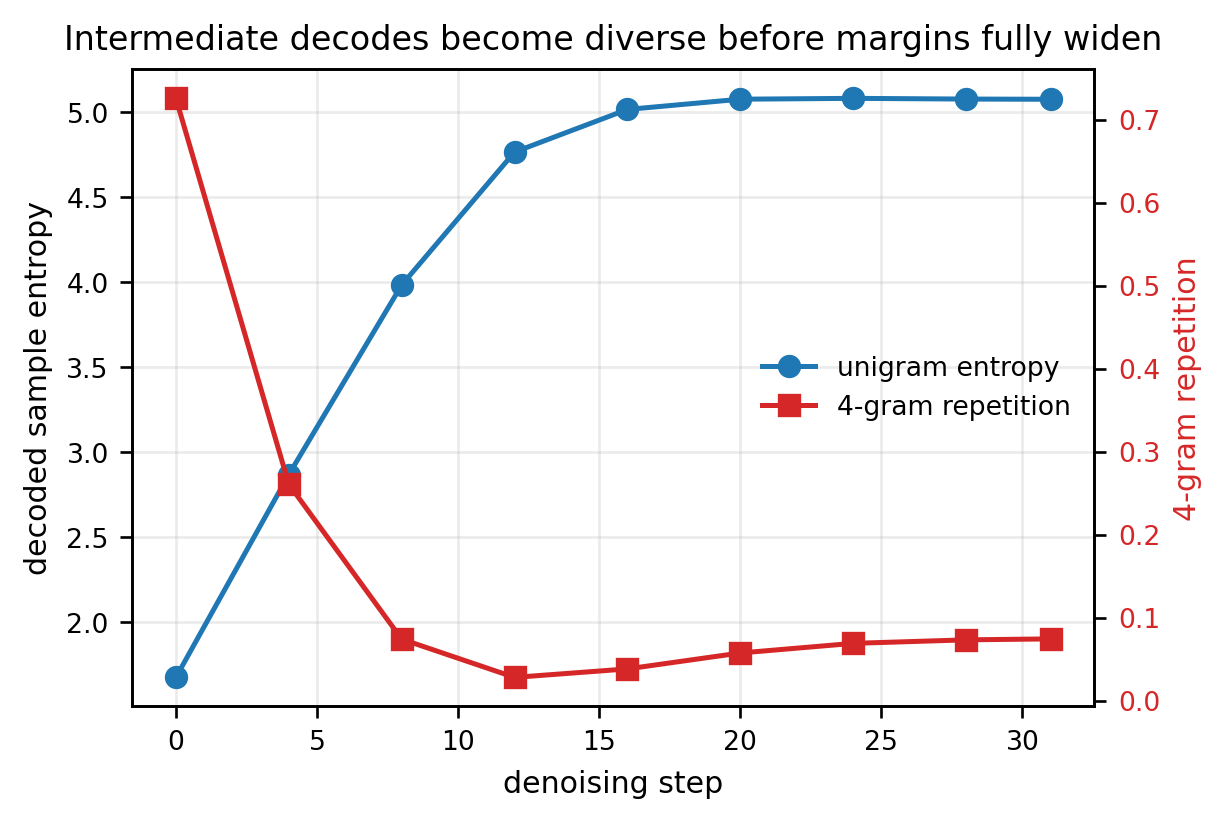}
  \caption{Intermediate decode readability. Decoded sample entropy rises and repetition collapses before the trajectory fully enters the final-token margin basin, showing that readability and basin compatibility are distinct.}
  \label{fig:intermediate-readability}
\end{figure}

\parhead{Training-side basin checks}
The final group supports the training-side discussion in \secref{sec:causal-validation}. \figref{fig:decoder-noise-training} shows that decoder-input noise widens the noisy-latent basin; \figref{fig:decoder-prob-training} separates decoder exposure from transport learning; \figref{fig:margin-widening-training}, \figref{fig:margin-10k-extension}, and \figref{fig:margin-10k-basin} show that our tested scalar margin penalty does not replace noisy decoder training in this short-run setting, and the baseline noisy-CE recipe remains the more effective practical choice. In the tested short-run grid, exposure to corrupted neighborhoods is more effective than the tested scalar margin penalty. The latter can optimize a visible logit gap without producing the widest usable off-clean basin in these comparisons; this does not exclude more effective margin objectives or longer training.

\begin{figure}[!htbp]
  \centering
  \includegraphics[width=.92\linewidth]{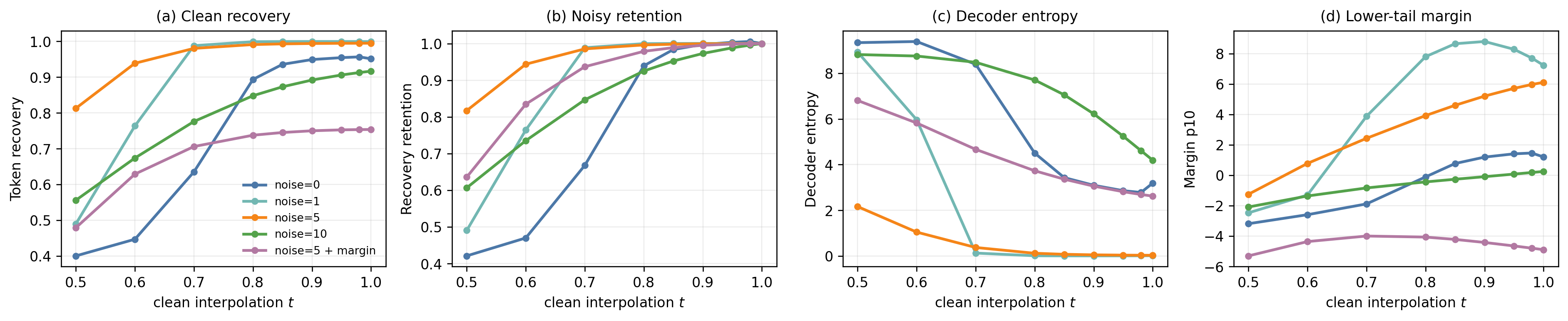}
  \caption{Official implementation decoder-noise training ablation on four RTX 3090 GPUs. The panels report (a) clean token recovery, (b) noisy-latent retention relative to each checkpoint's clean recovery, (c) decoder entropy, and (d) lower-tail native margin as the decoder-input noise scale changes. Clean-only decoder training attains high clean recovery but a narrow noisy-latent basin. An official-like decoder-input noise scale preserves clean recovery while greatly widening mid-noise retention; too much noise is non-monotonic and weakens the interface again.}
  \label{fig:decoder-noise-training}
\end{figure}

\begin{figure}[!htbp]
  \centering
  \includegraphics[width=.92\linewidth]{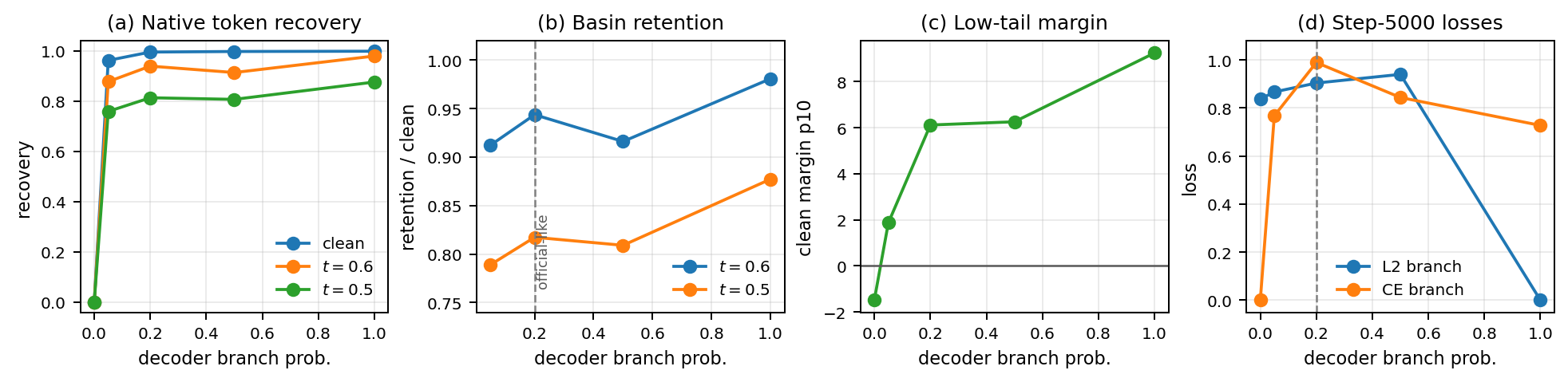}
  \caption{Decoder-branch frequency ablation with fixed decoder-input noise. From left to right, the panels report native token recovery, noisy-basin retention, lower-tail margin, and L2/CE branch losses. Pure denoising does not learn a token interface, while decoder-only training creates a broad synthetic basin but removes the flow objective. The official-like mixed objective balances transport competence and decoder-basin widening.}
  \label{fig:decoder-prob-training}
\end{figure}

\begin{figure}[!htbp]
  \centering
  \includegraphics[width=.82\linewidth]{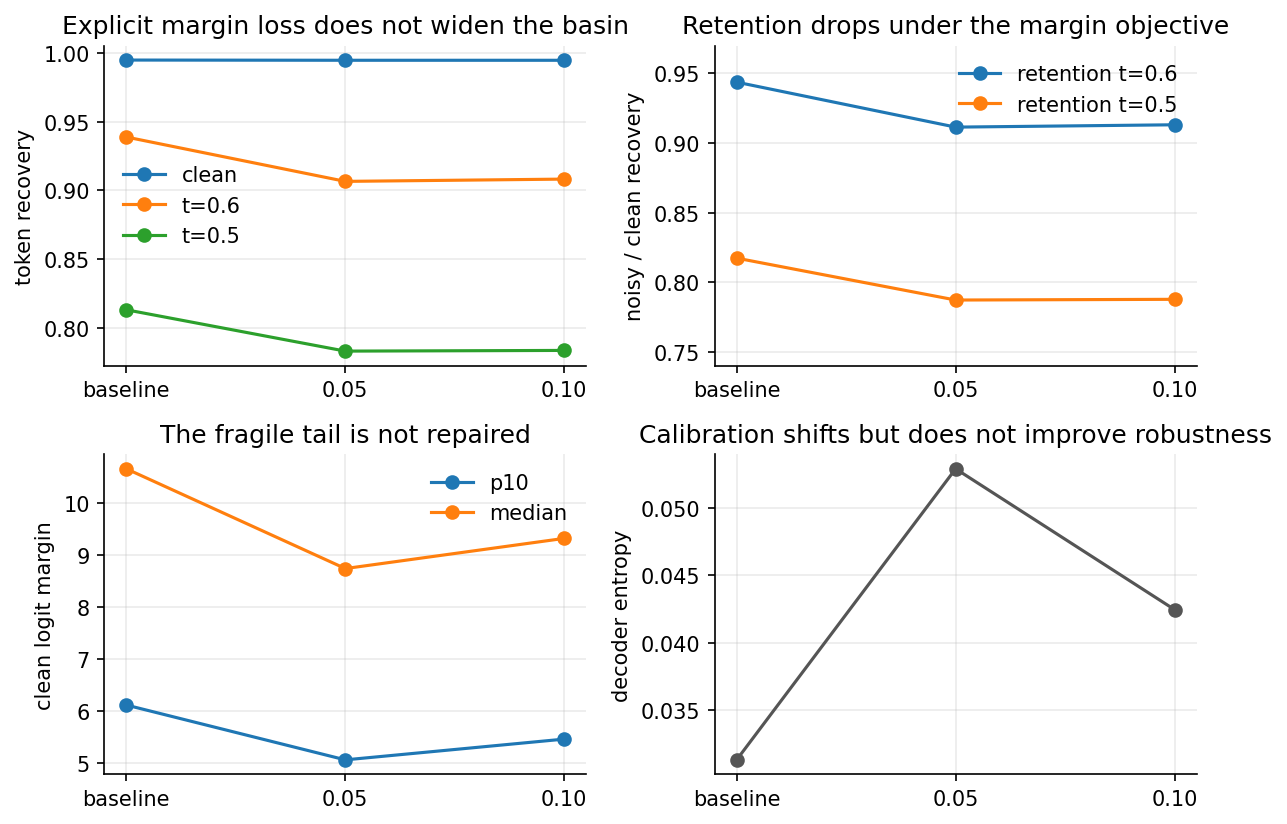}
  \caption{Explicit margin-loss ablation. Top-left: clean and corrupted token recovery. Top-right: noisy-latent retention normalized by clean recovery. Bottom-left: clean lower-tail and median margins. Bottom-right: decoder entropy. In this short ELF-B run with the official implementation, the tested scalar margin penalty preserves clean recovery but lowers retention under corruption and reduces the margin tail relative to the noisy-CE baseline.}
  \label{fig:margin-widening-training}
\end{figure}

\begin{figure}[!htbp]
  \centering
  \includegraphics[width=.82\linewidth]{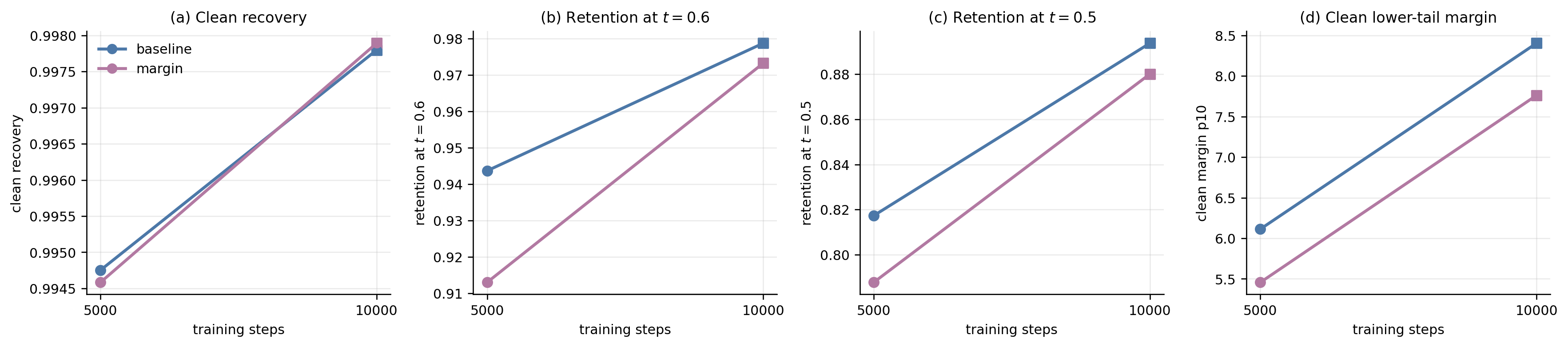}
  \caption{10k extension of the explicit margin-loss ablation. From left to right, the panels track clean recovery, retention at $t=0.6$, retention at $t=0.5$, and clean lower-tail margin as the short-run margin experiment is extended. Longer training narrows the 5k gap, but the margin-loss run still does not beat the noisy-CE baseline on noisy-latent retention or the clean margin tail. In this limited setting, the tested scalar margin penalty is not a drop-in substitute for decoder-input-noise basin training.}
  \label{fig:margin-10k-extension}
\end{figure}

\begin{figure}[!htbp]
  \centering
  \includegraphics[width=.78\linewidth]{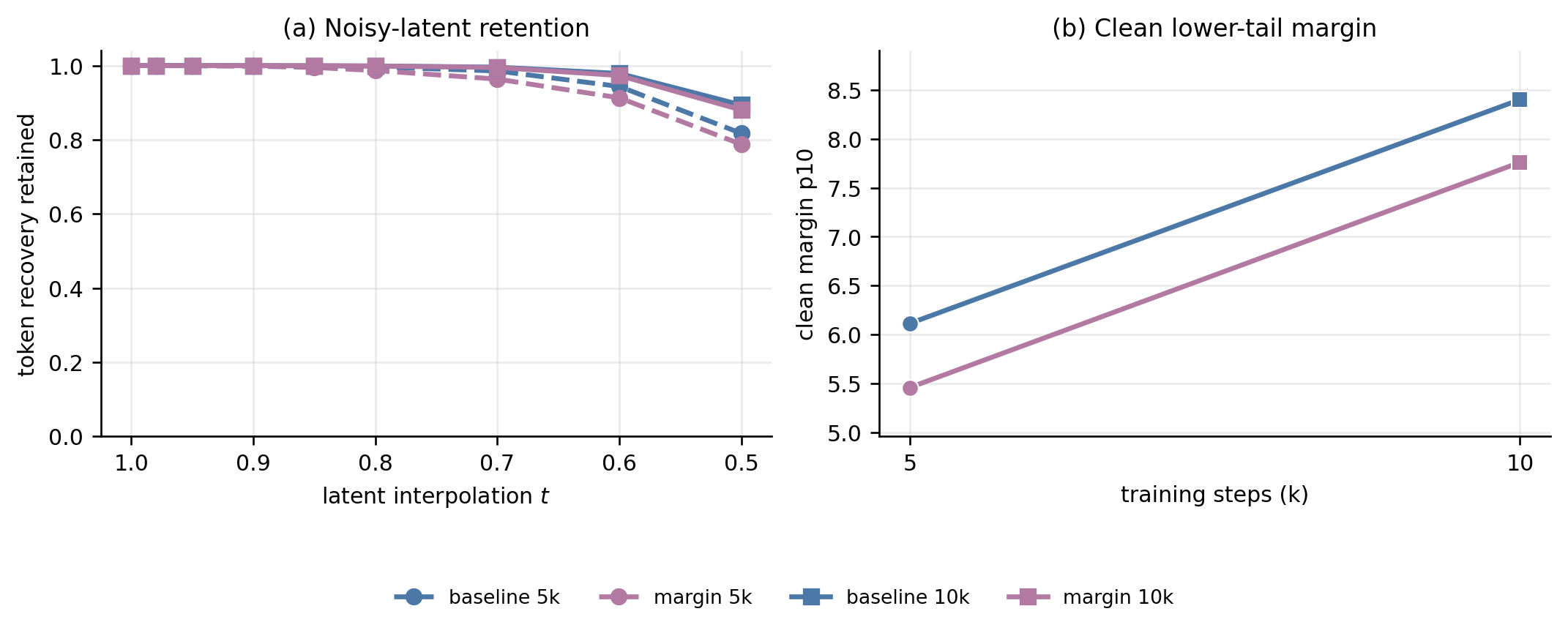}
  \caption{Zoomed basin-retention and margin check for the same 5k and 10k margin-loss extension shown in \figref{fig:margin-10k-extension}. Left: token-recovery retention under the full latent-corruption sweep, normalized by each checkpoint's clean recovery; this panel expands the selected $t=0.6$ and $t=0.5$ snapshots in the previous figure. Right: clean lower-tail decoder margin for the same objectives. Extending the margin-loss run improves its noisy-latent retention over the 5k version, but the official-like noisy-CE baseline remains the wider decoder-compatible basin.}
  \label{fig:margin-10k-basin}
\end{figure}

\section{Full Diagnostic Heatmap}
\label{app:diagnostic-heatmap}

\tabref{tab:encoder-readiness-decomposition} condenses the IRS, LDT, and encoder-diagnostic results into the four decompositions most relevant to ELF's embedding ablations: contextuality, lexical geometry, scale mismatch, and decoder compatibility. This table is not an encoder leaderboard, and it does not prove the origin of T5-small's favorable basin. It asks which observed properties make the frozen T5-small interface unusually diffusion-ready under this audit. The answer is not simply ``T5 is large'' or ``token embeddings are readable.'' T5 token embeddings have very deep lexical cells and a high linear readout score, but they lose order sensitivity. BERT and RoBERTa preserve contextual order and semantic probe accuracy, but their paired decoder-facing margins are much shallower. T5-base and T5-large are stronger encoders in the usual sense, yet they are dimension- and decoder-mismatched under this audit and their interface scores collapse. Random T5 controls show that architecture and dimensionality alone do not create a readable basin. Thus our diagnostic operationally decomposes the ELF embedding ablation rather than restating it: the successful setting is a matched contextual representation--decoder pair with lexical geometry, order sensitivity, and decoder margin all present at once.

\begin{table}[!htbp]
\centering
\scriptsize
\setlength{\tabcolsep}{3pt}
\caption{Compact decomposition of representation readiness. LDT NMSE is the held-out linear-denoising error at $t=0.5$. IRS rec. is the decoder-facing token-recovery axis in the calibrated readiness score. The table separates contextuality, lexical geometry, scale mismatch, and decoder compatibility; it is diagnostic, not a ranking of general encoder quality.}
\label{tab:encoder-readiness-decomposition}
\begin{tabular}{>{\raggedright\arraybackslash}p{0.16\linewidth}>{\raggedright\arraybackslash}p{0.23\linewidth}rrrrrr}
\toprule
State/control & Isolates & LDT NMSE$\downarrow$ & IRS rec.$\uparrow$ & Order$\uparrow$ & IRS$\uparrow$ & Margin p10$\uparrow$ & Readout$\uparrow$ \\
\midrule
T5-small contextual & Matched ELF interface & .403 & 1.000 & 1.00 & 77.3 & 15.42 & 89.8\% \\
T5 token table & Lexical geometry without contextual order & .401 & .987 & .61 & 70.1 & 88.34 & 94.1\% \\
RoBERTa-base & Contextual order with a non-ELF pair & .402 & .957 & 1.00 & 61.8 & 5.16 & 62.0\% \\
BERT-base & MLM contextual geometry & .387 & .934 & 1.00 & 54.0 & 1.55 & 59.9\% \\
T5-base & Larger T5, mismatched interface & .417 & .013 & 1.00 & 35.7 & -18.36 & 26.0\% \\
T5-large & Larger T5, mismatched interface & .414 & .018 & 1.00 & 38.1 & -19.35 & 35.1\% \\
GPT-2 & Causal-LM hidden states & .430 & .393 & .99 & 35.8 & -6.00 & 29.3\% \\
Random T5 & Architecture and dimension only & .351 & .006 & .52 & 27.1 & 0.00 & 25.2\% \\
\bottomrule
\end{tabular}
\end{table}

\tabref{tab:negative-control-summary} summarizes the covariance, shuffle, and random-initialization controls used by the heatmap. These controls are deliberately narrow: each preserves an easy explanation while disrupting a property implicated by the hypothesized interface explanation. The pattern rules out three tempting shortcuts. Second-order covariance can make MSE denoising look good without preserving token recoverability. Lexical lookup can make tokens recoverable without preserving order. Strong contextual encoders can preserve semantic and order probes while still failing the paired decoder margin. These are the control results that make the ELF ablation interpretable as an interface-pair result rather than a generic pretrained-encoder result.

\begin{table}[!htbp]
\centering
\footnotesize
\setlength{\tabcolsep}{3pt}
\caption{Negative controls used to separate covariance, lexical geometry, order, pairing, and pretrained decoder compatibility.}
\label{tab:negative-control-summary}
\begin{tabular}{>{\raggedright\arraybackslash}p{0.19\linewidth}>{\raggedright\arraybackslash}p{0.25\linewidth}>{\raggedright\arraybackslash}p{0.22\linewidth}>{\raggedright\arraybackslash}p{0.16\linewidth}}
\toprule
Control & Preserves & Breaks & Diagnostic role \\
\midrule
Covariance-matched Gaussian & First- and second-order statistics of the state space & Linguistic content and decoder alignment & Tests whether low MSE is only covariance fitting. \\
Whitened contextual states & Text pairing and contextual labels & Native anisotropy and covariance spectrum & Tests whether anisotropy alone explains recovery. \\
Token-shuffled contextual states & Token multiset and lexical identity & Positional order and compositional structure & Separates lexical lookup from contextual order. \\
Sequence-shuffled states & Marginal embedding distribution & Text--representation pairing & Tests whether semantic recovery depends on the correct sentence state. \\
Random T5 controls & Architecture, dimensionality, and tokenizer family & Pretrained geometry and decoder basin & Tests whether the interface comes from weights rather than shape. \\
\bottomrule
\end{tabular}
\end{table}

\FloatBarrier

\figref{fig:readiness-heatmap} is the full diagnostic heatmap behind the main-text readiness discussion and the decomposition in \tabref{tab:encoder-readiness-decomposition}. It shows all candidate state spaces and controls under the same five axes, making clear which failures are denoising failures, which are linguistic-recovery failures, which are order failures, which are decoder-interface failures, and which are trajectory-reliability failures. \figref{fig:cola-dit-guidance-boundary} gives the companion Cola-DLM trajectory boundary audit: stronger guidance enters a low-entropy decoder-confident basin, but this basin is not the teacher target basin measured by held-out token recovery, highlighting the distinction between confidence and alignment with a particular reference continuation, which is not a required target of unconditional sampling. Taken together, these two appendix figures state the scope boundary most explicitly: the diagnostic is not an ELF-only scorecard, but only matched tests can establish evidence for the same interface properties in another system.

\begin{figure}[!htbp]
  \centering
  \includegraphics[width=.78\linewidth]{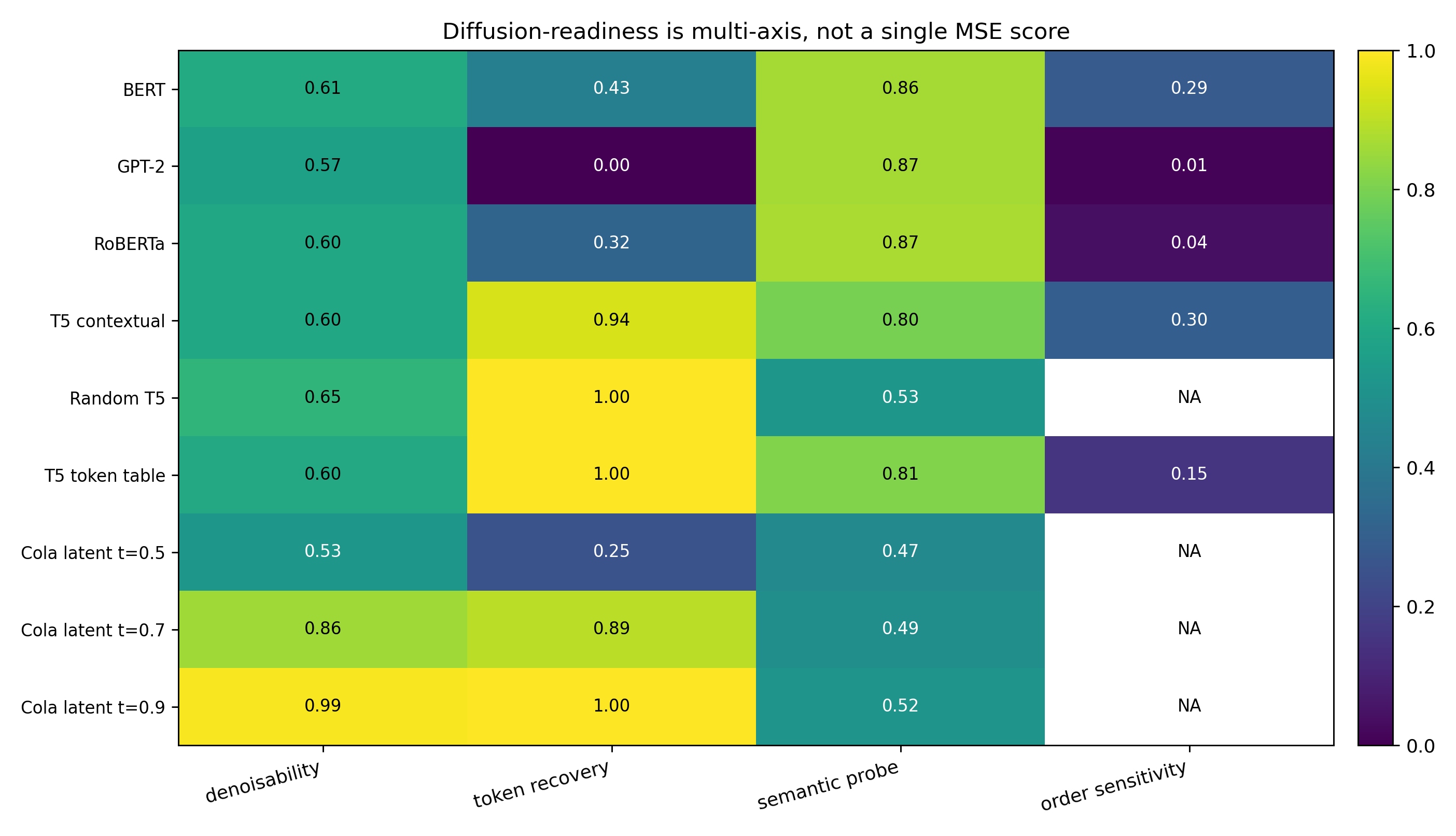}
  \caption{Detailed readiness heatmap across encoders and Cola latents. Rows correspond to candidate state spaces or controls; columns report denoisability, semantic recovery, order sensitivity, decoder-facing compatibility, and trajectory reliability. The heatmap is the appendix version of the readiness protocol summarized in the main text.}
  \label{fig:readiness-heatmap}
\end{figure}

\begin{figure}[!htbp]
  \centering
  \includegraphics[width=\linewidth]{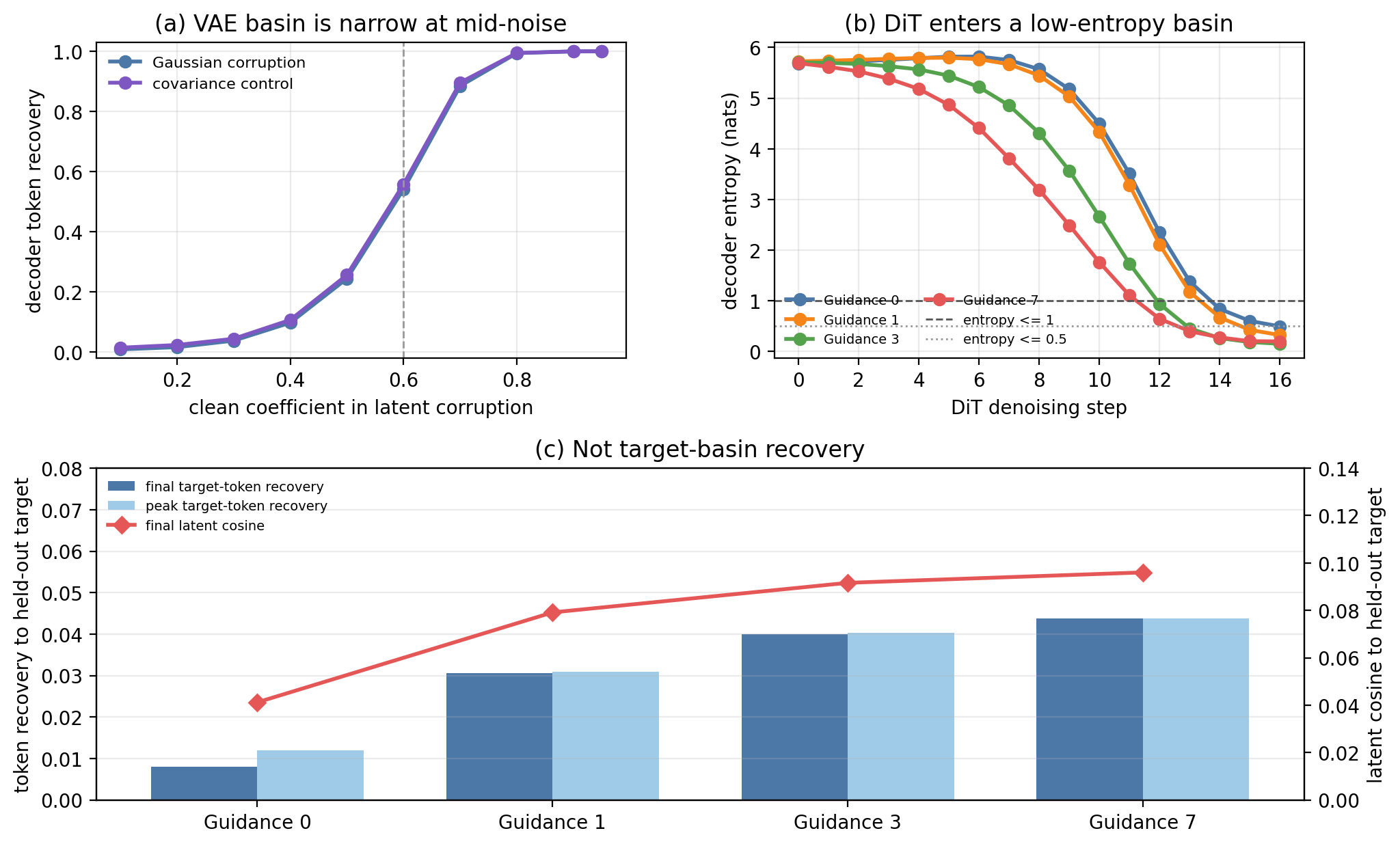}
  \caption{Balanced Cola-DLM boundary audit. Left: clean VAE latents decode reliably, but token recovery collapses under mid-noise corruption. Middle: the DiT prior transports noisy block latents into a low-entropy decoder-confident basin, with stronger guidance entering earlier. Right: final and peak recovery to the held-out target block remain below $5\%$, while latent cosine to the held-out target stays small; the trajectory reaches a decoder-confident region rather than reconstructing that particular reference continuation. This overlap diagnostic does not score free-generation quality.}
  \label{fig:cola-dit-guidance-boundary}
\end{figure}

\FloatBarrier

\FloatBarrier
\section{Margin Calibration: Depth versus Radius}
\label{app:calibration}

This appendix documents the experiment behind \secref{sec:calibration}.

\paragraph{Protocol and checkpoint scope.}
The primary audit uses the released ELF-B checkpoint on 32 synthetic clean/corrupted sequences and the eight retrained epoch-2 checkpoints on 16 cached generated final states each. Two additional $p=0.8$ epoch-5 checks retain the original epoch-2 results rather than replacing them. At corruption levels $t\in\{1,0.8,0.6\}$, the full contextual decoder is evaluated in float32. Synthetic sequences use masked 96-position contexts and three deterministic probe positions; generated states retain all 1,024 positions, including post-EOS context, with two probe positions. The index rule $(\lfloor i/2\rfloor\bmod2)$ assigns paired sequence indices to calibration or held-out halves; this is a sequence split, not an independent held-out-template claim. No predictor is fitted on the held-out half.

For each position, the winner and next two logits fix two rival margins. Perturbations use four magnitudes $(0.01,0.05,0.2,0.5)$ times $\max\{\|H\|_F,0.1\sqrt{Td}\}$ over valid positions; the common RMS floor prevents zero-state controls from receiving only zero perturbations. One direction is random, seeded independently of the gradients. The other is the negative full-context gradient of the rival with the smallest first-order radius. Thus the gradient arm deliberately stresses the linearized boundary; it is not a natural-noise robustness benchmark. Positions and magnitudes are not selected from observed flips.

\paragraph{Predictors and independent units.}
At each unperturbed probe state, let $q_j=g_i-g_j$. The implementation floors gradient norms and the normalization denominator at $10^{-12}$ to avoid division by zero. The five scores are negative raw margin $-\min_jq_j$, perturbation MSE, negative corruption level $-t$, radius-normalized magnitude $\|\Delta\|_F/\rho$ with $\rho=\min_j q_j/\|\nabla q_j\|_F$, and negative directional predicted margin $-\min_j(q_j+\langle\nabla q_j,\Delta\rangle)$. Larger scores predict a flip. AUROC pools phases, positions, and magnitudes within each checkpoint and direction; this validates ranking over the declared grid, not calibrated probabilities or fixed-magnitude performance. The released held-out set has 16 sequences and 576 perturbations per direction; each retrained held-out set has eight sequences and 192 perturbations. Perturbations from one sequence are correlated. The tables include flip counts so that rare-event and single-class outcomes remain visible.

\begin{table}[H]
\centering
\small
\setlength{\tabcolsep}{4pt}
\caption{Held-out flip-prediction AUROC along top-rival gradient directions. Bold marks the largest unrounded point estimate (exact ties both bold). ``---'' means a single observed class, not zero risk. e2/e5 identify the actual checkpoint epochs.}
\label{tab:calibration-grad}
\begin{tabular}{lrrrrrr}
\toprule
Model & Flips/$n$ & Raw & MSE & Phase & Radius & Directional \\
\midrule
Released ELF-B & 521/576 & 0.859 & 0.920 & 0.674 & \textbf{1.000} & 0.999 \\
$p=0.2$, s42, e2 & 151/192 & 0.542 & 0.794 & 0.593 & 0.865 & \textbf{0.866} \\
$p=0.2$, s43, e2 & 142/192 & 0.752 & 0.750 & 0.464 & \textbf{0.867} & 0.781 \\
$p=0.5$, s42, e2 & 164/192 & 0.544 & 0.930 & 0.598 & \textbf{0.962} & 0.962 \\
$p=0.5$, s43, e2 & 169/192 & 0.611 & 0.842 & 0.401 & 0.961 & \textbf{0.962} \\
$p=0.8$, s42, e2 & 166/192 & 0.483 & 0.727 & 0.367 & 0.806 & \textbf{0.823} \\
$p=0.8$, s43, e2 & 161/192 & 0.423 & 0.874 & 0.513 & 0.939 & \textbf{0.944} \\
$p=1.0$, s42, e2 & 192/192 & --- & --- & --- & --- & --- \\
$p=1.0$, s43, e2 & 192/192 & --- & --- & --- & --- & --- \\
$p=0.8$, s42, e5 & 176/192 & 0.594 & 0.755 & 0.455 & 0.828 & \textbf{0.841} \\
$p=0.8$, s43, e5 & 159/192 & 0.730 & 0.789 & 0.524 & 0.837 & \textbf{0.841} \\
\bottomrule
\end{tabular}
\end{table}

\begin{table}[H]
\centering
\small
\setlength{\tabcolsep}{4pt}
\caption{Held-out flip-prediction AUROC along fixed random directions. Bold marks the largest unrounded point estimate (exact ties both bold). ``---'' means a single observed class, not zero risk. e2/e5 identify the actual checkpoint epochs.}
\label{tab:calibration-random}
\begin{tabular}{lrrrrrr}
\toprule
Model & Flips/$n$ & Raw & MSE & Phase & Radius & Directional \\
\midrule
Released ELF-B & 85/576 & 0.857 & 0.811 & 0.564 & \textbf{0.959} & 0.945 \\
$p=0.2$, s42, e2 & 0/192 & --- & --- & --- & --- & --- \\
$p=0.2$, s43, e2 & 0/192 & --- & --- & --- & --- & --- \\
$p=0.5$, s42, e2 & 3/192 & 0.990 & 0.784 & 0.274 & 0.998 & \textbf{1.000} \\
$p=0.5$, s43, e2 & 0/192 & --- & --- & --- & --- & --- \\
$p=0.8$, s42, e2 & 0/192 & --- & --- & --- & --- & --- \\
$p=0.8$, s43, e2 & 3/192 & 0.997 & 0.705 & 0.839 & \textbf{1.000} & \textbf{1.000} \\
$p=1.0$, s42, e2 & 20/192 & 0.829 & 0.903 & 0.723 & \textbf{0.966} & 0.963 \\
$p=1.0$, s43, e2 & 30/192 & 0.691 & 0.842 & 0.605 & \textbf{0.906} & 0.842 \\
$p=0.8$, s42, e5 & 0/192 & --- & --- & --- & --- & --- \\
$p=0.8$, s43, e5 & 1/192 & 0.992 & 0.893 & 0.165 & \textbf{0.995} & \textbf{0.995} \\
\bottomrule
\end{tabular}
\end{table}

\paragraph{Sequence-cluster uncertainty.}
We recompute AUROC from saved perturbations and use 2,000 sequence-cluster bootstrap replicates (seed 20260921), resampling all positions and scales of a sequence together. Percentile intervals condition on the fixed checkpoint; they do not represent training-seed uncertainty. Undefined single-class replicates are excluded and counted. For gradient-direction radius AUROC, the released interval is $[0.999,1.000]$; the retrained intervals are reported below. With only eight held-out sequences, these are exploratory uncertainty estimates. In particular, a random-direction AUROC near one can rest on only one or three flips, as the preceding table makes explicit.

\begin{table}[H]
\centering
\small
\caption{Gradient-direction radius AUROC and 95\% sequence-cluster percentile intervals. Decoder-only e2 cells have 192/192 flips and undefined AUROC.}
\label{tab:calibration-ci}
\begin{tabular}{lrrr}
\toprule
Model & AUROC & 95\% interval & Defined replicates \\
\midrule
$p=0.2$, s42, e2 & 0.865 & $[0.843,0.894]$ & 2000/2000 \\
$p=0.2$, s43, e2 & 0.867 & $[0.825,0.921]$ & 2000/2000 \\
$p=0.5$, s42, e2 & 0.962 & $[0.938,0.980]$ & 2000/2000 \\
$p=0.5$, s43, e2 & 0.961 & $[0.938,0.982]$ & 2000/2000 \\
$p=0.8$, s42, e2 & 0.806 & $[0.716,0.889]$ & 2000/2000 \\
$p=0.8$, s43, e2 & 0.939 & $[0.889,0.985]$ & 2000/2000 \\
$p=0.8$, s42, e5 & 0.828 & $[0.732,0.996]$ & 1999/2000 \\
$p=0.8$, s43, e5 & 0.837 & $[0.670,0.964]$ & 2000/2000 \\
\bottomrule
\end{tabular}
\end{table}

\paragraph{Scaling sanity check and findings.}
Across all eleven audited checkpoints/conditions, the saved scaling control contains 1,728 sequence--phase--scale rows, with zero argmax changes and zero recorded margin-scaling error for $a\in\{0.5,1,2\}$. This checks an algebraic invariance rather than learning a robustness law. Radius normalization exceeds raw margin in every testable gradient cell ($0.806$--$1.000$ versus $0.423$--$0.859$). On the generated-state retraining probes alone, the radius range is $0.806$--$0.962$, and raw margin is $0.423$--$0.752$. The directional score is not uniformly second-best: MSE outranks it in some cells, and the radius and directional approximations need not agree for nonlinear boundaries. Two gradient cells contain only flips; five random cells contain none. Neither outcome is a failed experiment or a license to impute AUROC. These results support local ranking utility over the tested grid, not a global radius, a cross-model robustness ranking, or a guarantee of linguistic correctness.

% BEGIN AMPLITUDE CONTROL 20260922
\subsection{Amplitude-aware and fixed-scale controls}
\label{app:calibration-controls}

\paragraph{Controlling perturbation magnitude.}
The pooled comparison above gives radius normalization access to perturbation magnitude, whereas raw margin is unchanged across the four magnitudes for a given state. To separate this information advantage from local sensitivity, we add a post hoc CPU control on exactly the saved held-out perturbations. Let $m=\min_j q_j$ and define an amplitude-aware margin score
\[
 s_{\rm amp}=\frac{\|\Delta\|_F}{\max(m,10^{-12})},
 \qquad
 s_\rho=\frac{\|\Delta\|_F}{\max(\rho,10^{-12})}.
\]
Both scores use the same saved magnitude; only $s_\rho$ incorporates the two-rival first-order sensitivity. The amplitude-aware baseline remains dependent on logit scale, so this comparison controls access to perturbation magnitude rather than replacing the scaling argument. No checkpoint, direction, position, split, magnitude, or predictor parameter is changed or selected from these results.

\paragraph{Paired uncertainty and complete strata.}
We recompute AUROC and its paired difference $\delta=\operatorname{AUROC}(s_\rho)-\operatorname{AUROC}(s_{\rm amp})$. For the pooled difference we use 2,000 sequence-cluster bootstrap replicates (Python pseudorandom seed 20260922), drawing the same complete sequence blocks for both predictors: 16 held-out sequences for the released model and eight for each retrained model. All phases, positions, and magnitudes of a sequence remain together. The 95\% percentile intervals are pointwise and conditional on the fixed checkpoint, without multiplicity correction; undefined single-class replicates are excluded and their remaining counts are shown. We also report every pre-existing relative-RMS stratum. Fixing relative RMS removes the four-level magnitude mixture, although the absolute Frobenius norm still varies with the state's reference norm. Stratum results have only 144 perturbations for the released model or 48 for a retrained model, with the same 16 or eight sequence clusters; no additional independent sequences are created by stratification.

\begin{table}[H]
\centering
\small
\setlength{\tabcolsep}{4pt}
\caption{Amplitude-aware control on the pooled held-out grid. $A_{\rm amp}$ is AUROC for $s_{\rm amp}$; $\delta$ is radius AUROC minus this baseline. Every checkpoint/direction is retained, including single-class ``---'' entries. Defined bootstrap counts expose rare-flip uncertainty.}
\label{tab:calibration-amplitude}
\begin{tabular}{lrrrr}
\toprule
Model & $A_{\rm amp}$ & $\delta$ & 95\% interval for $\delta$ & Defined replicates \\
\midrule
\multicolumn{5}{l}{\emph{Top-rival gradient directions}} \\
Released ELF-B & 0.996 & +0.004 & $[+0.0004,+0.0076]$ & 2000/2000 \\
$p=0.2$, s42, e2 & 0.804 & +0.061 & $[+0.0322,+0.0860]$ & 2000/2000 \\
$p=0.2$, s43, e2 & 0.849 & +0.018 & $[-0.0014,+0.0434]$ & 2000/2000 \\
$p=0.5$, s42, e2 & 0.936 & +0.027 & $[-0.0063,+0.0468]$ & 2000/2000 \\
$p=0.5$, s43, e2 & 0.902 & +0.059 & $[+0.0299,+0.0933]$ & 2000/2000 \\
$p=0.8$, s42, e2 & 0.754 & +0.052 & $[+0.0311,+0.0826]$ & 2000/2000 \\
$p=0.8$, s43, e2 & 0.865 & +0.075 & $[+0.0401,+0.1087]$ & 2000/2000 \\
$p=1.0$, s42, e2 & --- & --- & --- & 0/2000 \\
$p=1.0$, s43, e2 & --- & --- & --- & 0/2000 \\
$p=0.8$, s42, e5 & 0.802 & +0.026 & $[-0.0125,+0.0675]$ & 1999/2000 \\
$p=0.8$, s43, e5 & 0.808 & +0.029 & $[-0.0247,+0.0673]$ & 2000/2000 \\
\midrule
\multicolumn{5}{l}{\emph{Fixed random directions}} \\
Released ELF-B & 0.955 & +0.004 & $[+0.0004,+0.0076]$ & 2000/2000 \\
$p=0.2$, s42, e2 & --- & --- & --- & 0/2000 \\
$p=0.2$, s43, e2 & --- & --- & --- & 0/2000 \\
$p=0.5$, s42, e2 & 0.998 & 0.000 & $[+0.0000,+0.0000]$ & 1320/2000 \\
$p=0.5$, s43, e2 & --- & --- & --- & 0/2000 \\
$p=0.8$, s42, e2 & --- & --- & --- & 0/2000 \\
$p=0.8$, s43, e2 & 1.000 & 0.000 & $[+0.0000,+0.0000]$ & 1304/2000 \\
$p=1.0$, s42, e2 & 0.958 & +0.008 & $[+0.0012,+0.0162]$ & 2000/2000 \\
$p=1.0$, s43, e2 & 0.919 & -0.013 & $[-0.0345,+0.0161]$ & 2000/2000 \\
$p=0.8$, s42, e5 & --- & --- & --- & 0/2000 \\
$p=0.8$, s43, e5 & 1.000 & -0.005 & $[-0.0157,+0.0000]$ & 1307/2000 \\
\bottomrule
\end{tabular}
\end{table}

\begin{table}[H]
\centering
\footnotesize
\setlength{\tabcolsep}{3pt}
\caption{All 88 fixed-relative-RMS cells. Each entry gives $\delta=\operatorname{AUROC}(s_\rho)-\operatorname{AUROC}(s_{\rm amp})$, followed by flips/$n$; negative values favor the amplitude-aware baseline. ``---'' denotes a single observed class, not zero risk. These are descriptive point differences without stratum-wise confidence claims.}
\label{tab:calibration-fixed-scale}
\begin{tabular}{lrrrr}
\toprule
Model & RMS $0.01$ & RMS $0.05$ & RMS $0.20$ & RMS $0.50$ \\
\midrule
\multicolumn{5}{l}{\emph{Top-rival gradient directions}} \\
Released ELF-B & +0.004 (92/144) & +0.045 (141/144) & --- (144/144) & --- (144/144) \\
$p=0.2$, s42, e2 & +0.361 (24/48) & +0.356 (37/48) & +0.083 (42/48) & --- (48/48) \\
$p=0.2$, s43, e2 & +0.173 (20/48) & +0.088 (36/48) & +0.151 (42/48) & +0.159 (44/48) \\
$p=0.5$, s42, e2 & +0.156 (22/48) & +0.348 (46/48) & --- (48/48) & --- (48/48) \\
$p=0.5$, s43, e2 & +0.265 (29/48) & +0.398 (44/48) & --- (48/48) & --- (48/48) \\
$p=0.8$, s42, e2 & +0.279 (34/48) & +0.137 (39/48) & +0.222 (45/48) & --- (48/48) \\
$p=0.8$, s43, e2 & +0.472 (22/48) & +0.615 (45/48) & +0.109 (46/48) & --- (48/48) \\
$p=1.0$, s42, e2 & --- (48/48) & --- (48/48) & --- (48/48) & --- (48/48) \\
$p=1.0$, s43, e2 & --- (48/48) & --- (48/48) & --- (48/48) & --- (48/48) \\
$p=0.8$, s42, e5 & +0.192 (38/48) & -0.023 (44/48) & +0.033 (46/48) & --- (48/48) \\
$p=0.8$, s43, e5 & +0.118 (29/48) & +0.145 (38/48) & -0.023 (44/48) & --- (48/48) \\
\midrule
\multicolumn{5}{l}{\emph{Fixed random directions}} \\
Released ELF-B & +0.007 (1/144) & -0.001 (5/144) & +0.010 (29/144) & +0.007 (50/144) \\
$p=0.2$, s42, e2 & --- (0/48) & --- (0/48) & --- (0/48) & --- (0/48) \\
$p=0.2$, s43, e2 & --- (0/48) & --- (0/48) & --- (0/48) & --- (0/48) \\
$p=0.5$, s42, e2 & --- (0/48) & --- (0/48) & 0.000 (1/48) & 0.000 (2/48) \\
$p=0.5$, s43, e2 & --- (0/48) & --- (0/48) & --- (0/48) & --- (0/48) \\
$p=0.8$, s42, e2 & --- (0/48) & --- (0/48) & --- (0/48) & --- (0/48) \\
$p=0.8$, s43, e2 & --- (0/48) & 0.000 (1/48) & 0.000 (1/48) & 0.000 (1/48) \\
$p=1.0$, s42, e2 & --- (0/48) & 0.000 (1/48) & +0.009 (5/48) & +0.023 (14/48) \\
$p=1.0$, s43, e2 & --- (0/48) & +0.076 (2/48) & +0.024 (7/48) & -0.072 (21/48) \\
$p=0.8$, s42, e5 & --- (0/48) & --- (0/48) & --- (0/48) & --- (0/48) \\
$p=0.8$, s43, e5 & --- (0/48) & --- (0/48) & --- (0/48) & -0.021 (1/48) \\
\bottomrule
\end{tabular}
\end{table}

\paragraph{Refined interpretation.}
In all nine testable pooled gradient-direction conditions, radius AUROC also exceeds the amplitude-aware baseline, but the remaining increments are $0.004$--$0.075$ and four of nine paired intervals include zero (\tabref{tab:calibration-amplitude}). At fixed relative RMS it wins in 23 of 25 testable gradient cells and loses in two, both $p=0.8$ epoch-5 follow-ups; against raw margin the corresponding counts are 23 wins, one tie, and one loss. Random-direction evidence is less uniform: among six testable pooled cells the radius has two wins, two ties, and two losses against the amplitude-aware baseline; among 16 testable fixed-scale cells it has seven wins, six ties, and three losses. Some comparisons contain only one to three flips. Thus sensitivity contributes beyond magnitude in the pooled gradient stress test, while neither a uniform advantage at every scale/direction nor a calibrated probability or certified radius is established. The gradient arm itself uses first-order geometry to choose its direction, so its advantage is not evidence of equally strong prediction under arbitrary natural perturbations.

% END AMPLITUDE CONTROL 20260922

\section{Complete Decoder-Dose and Crossed-Decoder Results}
\label{app:v3-closure}
\subsection{Version 3 closure protocol and evidence boundaries}
\label{app:v3-protocol}
The full grid comprises two initialization seeds (42 and 43), four decoder-branch probabilities $p\in\{0.2,0.5,0.8,1\}$, and five epoch checkpoints: 40 cfg3 generation cells, including all eight epoch-5 endpoints. Within each seed, runs share starting weights, optimizer settings, effective batch, and five OWT passes. Muon uses learning rate $2.5\times10^{-4}$, effective batch 48, and decoder input-noise scale 5. Decoder-branch updates replace flow updates; expected exposure is $pe$ versus $(1-p)e$ passes, not an exact realized count. Epoch-boundary continuations retain available optimizer/EMA/scheduler state but regenerate subsequent data shuffles. Training began on A800 hardware and the $p=0.8$ continuation and closure evaluation used an RTX 5090 with PyTorch 2.7.1/CUDA 12.8. Memory-saving execution settings also differ. Thus the expanded grid is matched in budget and configuration, not bitwise minibatch order or optimizer trajectory; small interior differences do not isolate branch probability alone.

Closure generation fixes a global logit-normal SDE32 time grid, $\gamma=1.5$, evaluation seed 20260906, and generator CFG 3. The complete contextual decoder reads final states with decoder CFG 3, $t=1$, and zero self-conditioning. There are 256 generated samples per cell at epochs 1--4 and 1,024 at epoch 5; the native handed-corruption audit uses 512 sequences. Historical boundary tables retain their original per-batch time-grid draws and 256-sample generation protocol. Their values are not replaced or pooled with the closure grid. PPL is the exponential of aggregate token negative log-likelihood, with empty-output accounting; it is not the arithmetic mean of per-text perplexities. Entropy and repetition are computed under the stated per-text conventions.

The calibration audit covers the released checkpoint, all eight retrained epoch-2 checkpoints, and two $p=0.8$ epoch-5 follow-ups. The crossed-decoder factorial covers epoch 2 only. These are distinct checkpoint sets; the presence of eight completed training endpoints does not imply an epoch-5 decoder-swap factorial. All newly reported numerical summaries were checked against saved measurements, with CPU recomputation of PPL, AUROC, paired bootstrap comparisons, and logit-scaling invariance. No claim is made that all historical GPU experiments were rerun for this revision.
\subsection{Full dose grid: per-epoch trajectories and endpoints}
\label{app:dose-grid}

\paragraph{Protocol.}
Grid statistics use the closure protocol: official-style SDE32 with $\gamma=1.5$, one globally shared logit-normal step grid, per-example/per-step Gaussian randomness at eval seed 20260906, generator CFG fixed at 3, 256 samples per cell at epochs 1--4 and 1,024 at epoch 5, GPT-2-Large PPL with explicit empty-output accounting. These estimates are \emph{not pooled} with the original-schedule boundary audit in \secref{app:transport-controls}: under the fixed-grid protocol the $p{=}0.2$ epoch-5 PPL is $29.7/30.4$ (seeds) versus $34.9/36.8$ under the original schedule. Every cross-$p$ and cross-epoch statement is within-protocol.

\begin{table}[H]
\centering
\footnotesize
\setlength{\tabcolsep}{3.5pt}
\caption{Dose-grid endpoints at epoch 5 (two seeds). ``--'' denotes an undefined metric for empty output. MAUVE still includes empty outputs as EOS. The fixed-grid endpoint uses cfg3 and 1,024 generated samples per cell. Generated $M_{10}$ averages per-sequence winner margins; corruption $M_{10}$ pools signed ground-truth-token margins over 512 sequences at $t{=}0.6$.}
\label{tab:dose-endpoints}
\resizebox{\linewidth}{!}{%
\begin{tabular}{lrrrrrrrr}
\toprule
& \multicolumn{2}{c}{$p{=}0.2$} & \multicolumn{2}{c}{$p{=}0.5$} & \multicolumn{2}{c}{$p{=}0.8$} & \multicolumn{2}{c}{$p{=}1.0$} \\
Metric & s42 & s43 & s42 & s43 & s42 & s43 & s42 & s43 \\
\midrule
Generated Gen-PPL $\downarrow$ & 29.72 & 30.44 & 26.94 & 26.72 & 33.80 & 29.21 & 2.91 & -- \\
Generated unigram entropy & 5.053 & 5.156 & 5.042 & 5.056 & 5.038 & 5.009 & 1.162 & 0.000 \\
Generated-state $M_{10}$ & 7.522 & 7.193 & 9.557 & 8.872 & 7.223 & 6.893 & 0.630 & 0.256 \\
Handed corruption $M_{10}$ ($t=0.6$) & 6.791 & 6.748 & 7.738 & 7.687 & 8.613 & 8.379 & 8.992 & 8.824 \\
Handed recovery ($t=0.6$) & 0.9944 & 0.9938 & 0.9960 & 0.9961 & 0.9969 & 0.9969 & 0.9970 & 0.9966 \\
MAUVE & 0.0449 & 0.0897 & 0.0989 & 0.0530 & 0.0277 & 0.0297 & 0.0041 & 0.0041 \\
Bigram repetition & 0.245 & 0.209 & 0.256 & 0.243 & 0.233 & 0.253 & 0.000 & -- \\
Empty fraction & 0.000 & 0.000 & 0.000 & 0.000 & 0.000 & 0.000 & 0.000 & 1.000 \\
\bottomrule
\end{tabular}}
\end{table}

\begin{table}[H]
\centering
\footnotesize
\setlength{\tabcolsep}{3.5pt}
\caption{Generated PPL per epoch across the dose grid (seeds side by side): 256 samples at epochs 1--4 and 1,024 at epoch 5. Sampling variation and the changed sample count limit small cross-epoch comparisons. The $p{=}0.8$ row set shows late catch-up at low flow exposure; $p{=}1.0$ collapses at every epoch.}
\label{tab:dose-trajectory}
\begin{tabular}{lrrrrrrrr}
\toprule
Epoch & \multicolumn{2}{c}{$p{=}0.2$} & \multicolumn{2}{c}{$p{=}0.5$} & \multicolumn{2}{c}{$p{=}0.8$} & \multicolumn{2}{c}{$p{=}1.0$} \\
& s42 & s43 & s42 & s43 & s42 & s43 & s42 & s43 \\
\midrule
1 & 31.62 & 32.90 & 33.38 & 32.46 & 57.78 & 56.69 & -- & -- \\
2 & 29.99 & 30.76 & 28.07 & 31.93 & 42.91 & 38.36 & -- & 3.46 \\
3 & 29.41 & 33.03 & 28.46 & 29.67 & 37.56 & 33.68 & -- & 1.07 \\
4 & 28.56 & 30.73 & 27.48 & 27.20 & 32.76 & 28.60 & 9.15 & 4.14 \\
5 & 29.72 & 30.44 & 26.94 & 26.72 & 33.80 & 29.21 & 2.91 & -- \\
\bottomrule
\end{tabular}
\end{table}

\paragraph{Flow-dose accounting and paired follow-up.}
Expected flow exposure is $(1-p)e$ and decoder exposure is $pe$. The $p=0.5$ epoch-2 and $p=0.8$ epoch-5 checkpoints both receive one expected flow pass, but respectively one and four decoder passes. This does not match learning-rate stage, realized branch counts, resumed shuffles, or optimizer history. A post hoc CPU recomputation pairs the first 256 sample IDs of each saved endpoint with the 256 epoch-2 samples; it generates no new text. For seed 42, PPL changes from $28.07$ to $33.88$ ($\Delta=+5.81$, paired sequence-bootstrap 95\% interval $[4.74,6.88]$). For seed 43 it changes from $31.93$ to $29.63$ ($\Delta=-2.30$, interval $[-3.42,-1.16]$). Entropy changes from $4.921$ to $5.040$ and from $5.032$ to $5.018$, respectively. These within-checkpoint intervals do not turn two training seeds into additional independent runs. The opposite signs preclude a consistent benefit of additional decoder exposure in this comparison. MAUVE is not reported for these 256-sample prefixes.

\subsection{Crossed source and contextual decoder}
\label{app:crossed-readout}

At epoch 2, the four source generators supply cached final states for 256 paired sample IDs per seed. We apply each of the four complete contextual decoders to each cache without regenerating its states. All eight diagonal cells reproduce their source-native decisions exactly. Off-diagonal agreement for noncollapsed source models ranges from $93.8\%$ to $97.7\%$ over source-pre-EOS positions; full-position agreement is also saved. The two tables show all 32 PPL and entropy cells, with decoder columns and source rows in dose order. Agreement is with source-native decisions, not ground-truth language.

\begin{table}[H]
\centering
\small
\caption{Crossed-decoder Gen-PPL at epoch 2. Column $D_p$ changes only the decoder. Low values for $G_1$ accompany collapse and are not quality gains.}
\label{tab:cross-ppl}
\begin{tabular}{llrrrr}
\toprule
Seed & Source & $D_{0.2}$ & $D_{0.5}$ & $D_{0.8}$ & $D_1$ \\
\midrule
42 & $G_{0.2}$ & 29.99 & 30.09 & 30.14 & 30.25 \\
42 & $G_{0.5}$ & 27.88 & 28.07 & 28.14 & 28.24 \\
42 & $G_{0.8}$ & 42.13 & 42.51 & 42.91 & 43.16 \\
42 & $G_{1}$ & 1.03 & 1.02 & 1.69 & -- \\
43 & $G_{0.2}$ & 30.76 & 30.78 & 30.85 & 30.98 \\
43 & $G_{0.5}$ & 31.74 & 31.93 & 32.00 & 32.11 \\
43 & $G_{0.8}$ & 37.53 & 37.95 & 38.36 & 38.65 \\
43 & $G_{1}$ & 2.88 & 3.07 & 4.09 & 3.46 \\
\bottomrule
\end{tabular}
\end{table}

\begin{table}[H]
\centering
\small
\caption{Mean per-text unigram entropy for the same crossed states. Empty outputs contribute zero.}
\label{tab:cross-entropy}
\begin{tabular}{llrrrr}
\toprule
Seed & Source & $D_{0.2}$ & $D_{0.5}$ & $D_{0.8}$ & $D_1$ \\
\midrule
42 & $G_{0.2}$ & 5.022 & 5.022 & 5.023 & 5.023 \\
42 & $G_{0.5}$ & 4.921 & 4.921 & 4.922 & 4.922 \\
42 & $G_{0.8}$ & 5.036 & 5.035 & 5.036 & 5.035 \\
42 & $G_{1}$ & 0.009 & 0.000 & 0.413 & 0.000 \\
43 & $G_{0.2}$ & 5.087 & 5.088 & 5.087 & 5.089 \\
43 & $G_{0.5}$ & 5.031 & 5.032 & 5.032 & 5.033 \\
43 & $G_{0.8}$ & 4.987 & 4.990 & 4.991 & 4.991 \\
43 & $G_{1}$ & 0.539 & 1.009 & 1.005 & 1.352 \\
\bottomrule
\end{tabular}
\end{table}

Using $D_{0.2}$ makes both $G_1$ sources nonempty, but seed 42 has unique fraction $3/256$ and entropy $0.009$, while seed 43 has bigram repetition $0.987$ and entropy $0.539$. Other decoder swaps also leave $G_1$ at entropy at most $1.352$. In the reverse direction, $D_1$ preserves $97.4/97.3\%$ of decisions on $G_{0.2}$ states and PPL stays near its native value ($30.25/30.98$ versus $29.99/30.76$). Thus nonemptiness and token agreement alone are insufficient, and low-diversity outputs from these sources are not rescued by the tested readouts. This completed epoch-2 factorial does not establish the unmeasured epoch-5 factorial or arbitrary-decoder impossibility.

\FloatBarrier
\section{Additional Falsification and Readout Controls}
\label{app:october-controls}

These controls reanalyze saved measurements without new model training or generation. The original tables remain intact; conditioning changes the AUROC estimand, and fixing the readout variant changes the selection protocol.

\subsection{Conditional ranking and falsification of a radius interpretation}
\label{app:calibration-falsifiability}

\paragraph{Remove phase and scale mixing together.}
The preceding fixed-scale control still pools corruption phases. We therefore perform a further CPU reanalysis of the same held-out observations, conditioning jointly on $c=(t,\text{relative RMS})$. For score $s$, define conditional AUROC
\[
 C(s)=\frac{\sum_c\sum_{i\in F_c}\sum_{j\in N_c}
 [\mathbf{1}\{s_i>s_j\}+\tfrac12\mathbf{1}\{s_i=s_j\}]}
 {\sum_c |F_c|\,|N_c|},
\]
where $F_c$ and $N_c$ contain observed flips and nonflips in that cell. Thus only same-phase, same-relative-scale pairs are compared, and defined cells are weighted by their number of such pairs. This is a different estimand from the earlier pooled AUROC. All $11\times2\times3\times4=264$ cells are retained in the audit records; single-class cells contribute no pairs, without imputed scores. There are 64/132 defined gradient cells and 28/132 defined random cells. Relative scale is fixed within a cell; absolute norm may still vary by state, and both compared scores incorporate that norm.

\tabref{tab:calibration-conditional} compares $C(s_\rho)$ with $C(s_{\rm amp})$. Paired 95\% percentile intervals use 2,000 complete-sequence cluster resamples (Python random seed 20261011), with identical resamples for both scores and 16 released-model or eight retrained-model clusters. Single-class resamples are excluded and counted. These are pointwise, unadjusted intervals conditional on fixed checkpoints. Radius normalization has the higher point estimate in all nine testable gradient conditions, with increments $0.010$--$0.478$; four intervals include zero. Random-direction results have three wins, two ties, and one loss, with all six intervals including zero. Phase/scale mixing alone therefore does not explain the gradient-grid advantage. The direction itself was constructed from the same first-order geometry, however, so this test does not remove that design coupling or establish arbitrary-direction superiority.

\begin{table}[H]
\centering
\small
\setlength{\tabcolsep}{3pt}
\caption{Conditional AUROC using only same-phase, same-relative-RMS positive--negative pairs. $A=C(s_{\rm amp})$, $R=C(s_\rho)$, and $\Delta C=R-A$. $K$ is the number of defined cells; $B$ is the number of defined paired bootstrap replicates. ``---'' means no discordant-label pairs.}
\label{tab:calibration-conditional}
\begin{tabular}{lrrrrr}
\toprule
Model & $A$ & $R$ & $\Delta C$ [95\% interval] & $K$ & $B$ \\
\midrule
\multicolumn{6}{l}{\emph{Top-rival gradient directions}} \\
Released ELF-B & 0.990 & 1.000 & $+0.010\;[+0.000,+0.018]$ & 5/12 & 2000/2000 \\
$p=0.2$, s42, e2 & 0.481 & 0.891 & $+0.411\;[+0.243,+0.570]$ & 8/12 & 2000/2000 \\
$p=0.2$, s43, e2 & 0.752 & 0.927 & $+0.175\;[+0.043,+0.309]$ & 11/12 & 2000/2000 \\
$p=0.5$, s42, e2 & 0.526 & 0.781 & $+0.255\;[-0.016,+0.459]$ & 5/12 & 2000/2000 \\
$p=0.5$, s43, e2 & 0.604 & 0.894 & $+0.291\;[+0.163,+0.505]$ & 4/12 & 2000/2000 \\
$p=0.8$, s42, e2 & 0.493 & 0.686 & $+0.193\;[+0.057,+0.385]$ & 8/12 & 2000/2000 \\
$p=0.8$, s43, e2 & 0.380 & 0.859 & $+0.478\;[+0.303,+0.690]$ & 6/12 & 2000/2000 \\
$p=1.0$, s42, e2 & --- & --- & --- & 0/12 & 0/2000 \\
$p=1.0$, s43, e2 & --- & --- & --- & 0/12 & 0/2000 \\
$p=0.8$, s42, e5 & 0.657 & 0.743 & $+0.086\;[-0.041,+0.224]$ & 8/12 & 1997/2000 \\
$p=0.8$, s43, e5 & 0.786 & 0.822 & $+0.036\;[-0.020,+0.135]$ & 9/12 & 2000/2000 \\
\midrule
\multicolumn{6}{l}{\emph{Fixed random directions}} \\
Released ELF-B & 0.905 & 0.917 & $+0.012\;[+0.000,+0.026]$ & 9/12 & 2000/2000 \\
$p=0.2$, s42, e2 & --- & --- & --- & 0/12 & 0/2000 \\
$p=0.2$, s43, e2 & --- & --- & --- & 0/12 & 0/2000 \\
$p=0.5$, s42, e2 & 1.000 & 1.000 & $+0.000\;[+0.000,+0.000]$ & 3/12 & 1329/2000 \\
$p=0.5$, s43, e2 & --- & --- & --- & 0/12 & 0/2000 \\
$p=0.8$, s42, e2 & --- & --- & --- & 0/12 & 0/2000 \\
$p=0.8$, s43, e2 & 1.000 & 1.000 & $+0.000\;[+0.000,+0.000]$ & 3/12 & 1344/2000 \\
$p=1.0$, s42, e2 & 0.800 & 0.840 & $+0.040\;[-0.025,+0.135]$ & 5/12 & 2000/2000 \\
$p=1.0$, s43, e2 & 0.783 & 0.813 & $+0.030\;[-0.028,+0.104]$ & 7/12 & 2000/2000 \\
$p=0.8$, s42, e5 & --- & --- & --- & 0/12 & 0/2000 \\
$p=0.8$, s43, e5 & 1.000 & 0.933 & $-0.067\;[-0.200,+0.000]$ & 1/12 & 1297/2000 \\
\bottomrule
\end{tabular}
\end{table}

\paragraph{An explicit test of the first-order radius threshold.}
Ranking success does not imply that $\rho$ is a stability certificate. Without tuning a threshold, we test whether a token ever flips when $\|\Delta\|_F<\rho$, equivalently $s_\rho<1$ (the saved radii are positive and above the numerical floor here). The gradient arm contains 27 flips among 198 such perturbations, across seven checkpoint conditions and 15 checkpoint-specific sequence clusters; the paired random arm has zero among its 198 corresponding perturbations (\tabref{tab:calibration-threshold}). These are descriptive counts over heterogeneous conditions, not 198 independent trials or a calibrated failure probability; sequence identifiers may also be related across checkpoints. Every violation occurs at relative RMS $0.01$. Eleven involve probe position zero and sixteen position 1023; the generated-state protocol keeps the full context, so the latter need not be pre-EOS text positions. Zero random-arm violations therefore neither certify that direction nor establish linguistic correctness.

One counterexample already occurs at an uncorrupted generated final state ($t=1$, $p=0.5$, seed 42, epoch 2, held-out sequence 10, position 0). Its perturbation norm is $6.006<\rho=7.061$ ($s_\rho=0.851$). The two-rival linearization predicts a positive margin $2.830$, yet the actual original-winner margin against the full vocabulary is $-0.809$ and the token flips. The most extreme saved violation has $s_\rho=0.625$. These observations do not contradict the ball-localized Lipschitz bound: the proxy neither bounds sensitivity throughout a neighborhood nor checks every vocabulary rival. The saved scalar outputs do not separate nonlinear-boundary effects from a new competing token. This directly establishes the distinction between a useful local ranking diagnostic and a certified stability radius.

\begin{table}[H]
\centering
\small
\setlength{\tabcolsep}{4pt}
\caption{Observed flips strictly below the first-order proxy radius. The eligible count $n$ applies separately to each direction; $G$ and $R$ count gradient- and random-direction flips, respectively. $S_G$ counts checkpoint-specific sequence clusters with gradient violations. ``---'' in the last column means no observed violation, not an infinite safe radius.}
\label{tab:calibration-threshold}
\begin{tabular}{lrrrrr}
\toprule
Model & Eligible $n$ & $G$ & $R$ & $S_G$ & Min. violating $s_\rho$ \\
\midrule
Released ELF-B & 56 & 2 & 0 & 2 & 0.958 \\
$p=0.2$, s42, e2 & 25 & 6 & 0 & 2 & 0.761 \\
$p=0.2$, s43, e2 & 34 & 8 & 0 & 3 & 0.669 \\
$p=0.5$, s42, e2 & 25 & 5 & 0 & 2 & 0.689 \\
$p=0.5$, s43, e2 & 12 & 1 & 0 & 1 & 0.870 \\
$p=0.8$, s42, e2 & 5 & 0 & 0 & 0 & --- \\
$p=0.8$, s43, e2 & 21 & 3 & 0 & 3 & 0.854 \\
$p=1.0$, s42, e2 & 0 & 0 & 0 & 0 & --- \\
$p=1.0$, s43, e2 & 0 & 0 & 0 & 0 & --- \\
$p=0.8$, s42, e5 & 2 & 0 & 0 & 0 & --- \\
$p=0.8$, s43, e5 & 18 & 2 & 0 & 2 & 0.625 \\
\midrule
Descriptive total & 198 & 27 & 0 & 15 & 0.625 \\
\bottomrule
\end{tabular}
\end{table}

\FloatBarrier
\subsection{Readout selection and a common-variant comparison}
\label{app:readout-selection}

\paragraph{Common-variant readout check and selection scope.}
The historical MDP frontier selects readout and noise-augmentation variants by their evaluation PPL; those values are descriptive selected results, not an untouched model-selection test. To check whether the qualitative finding depends on taking a different winner at each size, \tabref{tab:mdp-fixed-variant} reports the same saved $z_T$ readout and noise augmentation $0.02$ throughout. This common-variant check was specified retrospectively and does not constitute an independent replication. Native-token agreement rises from $79.61\%$ to $97.93\%$ under this fixed choice. Its PPL is not strictly monotone: $24.81$ at 16k and $24.85$ at 32k; both remain above native decoding on their respective evaluation splits.

The split is the first $75\%$ of numerically ordered generated-shard sequences for fitting and the remaining $25\%$ for evaluation. Generation summaries and exported prediction counts verify the table's splits. The 4k, 8k, 16k, and 32k caches share their earlier shards, so a smaller run's evaluation sequences enter the training prefix of a larger run. Each run retains disjoint fitting and evaluation sequence indices, but the evaluation set changes with size. The resulting curve is a comparison of existing fitting budgets, not a controlled sample-complexity or saturation law. Agreement uses every evaluation sequence; the 32k run scores PPL only on the first 4,096 of its 8,192 evaluation sequences, while the smaller runs score their full evaluation subsets. Training counts identify the fitting split; the optimizer may subsample valid token positions, with saved caps of 4.8M and 9.6M for 16k and 32k. The fixed choice concerns the input variant and augmentation level, not a claim that every fitting hyperparameter is identical. No sequence-content deduplication or new readout-seed replication is supplied by this check.

\begin{table}[!htbp]
\centering
\small
\caption{MDP with one common saved variant: final sampler state $z_T$ and training noise augmentation $0.02$. No per-size winner is selected in this table. Counts are sequences; native and linear-readout PPL score texts decoded from the same held-out latent sequences at each pool size. These are retrospective saved-run comparisons, not a fixed-test-set learning curve.}
\label{tab:mdp-fixed-variant}
\begin{tabular}{rrrrrrr}
\toprule
Total pool & Train & Agree. eval. & PPL eval. & Agreement (\%) & MDP PPL & Native PPL \\
\midrule
1,024 & 768 & 256 & 256 & 79.61 & 39.64 & 21.73 \\
4,096 & 3,072 & 1,024 & 1,024 & 94.05 & 28.35 & 23.88 \\
8,192 & 6,144 & 2,048 & 2,048 & 97.02 & 25.51 & 23.85 \\
16,384 & 12,288 & 4,096 & 4,096 & 97.57 & 24.81 & 23.49 \\
32,768 & 24,576 & 8,192 & 4,096 & 97.93 & 24.85 & 23.65 \\
\bottomrule
\end{tabular}
\end{table}

The new fits in \secref{app:confirmatory-readout} address the changing-test-set limitation separately; they do not replace these historical numerical results or use their optimization schedule.

\section{Readout repetitions on fixed tests and a native-interface control}
\label{app:confirmatory-readout}

\parhead{Protocol fixed before fitting}
This study adds new readout fits and new generations to the retrospective audit in \secref{app:readout-selection}. A local protocol file, saved before the initial fits and generations, fixes the splits, fitting seeds, update budget, validation rule, and reporting scope; it is not a public preregistration. No denoiser or native decoder is retrained. All twelve initial fits are reported, including the weaker small-budget LangFlow result. Three additional LangFlow fits belong to an explicitly post-result token-budget follow-up below. Training inputs are the sampler state before the complete contextual decoder, never a decoder hidden state.

For ELF-B, we partition the historical 32,768-sequence cache with split seed 20261013 into nested training sets of 1,024, 4,096, and 16,384 sequences, one validation set of 512, and one common test set of 1,024. No sequence belongs to more than one split. There are no empty or exactly duplicated saved native sequences in the source pool. The nested training sets contain 1,047,155, 4,188,683, and 16,754,550 valid token positions. These historical sequences were used in earlier work, so the common test is disjoint for the new fits but not globally untouched. We additionally generate 512 test sequences with independent seed 61013001 after fixing the protocol, using the released ELF-B OWT checkpoint's \texttt{ema\_params1} weights, 32 SDE steps, noise coefficient 1.5, self-conditioning guidance 3, the logit-normal grid, and a final ODE step. None duplicates a historical native token sequence exactly. The fixed and fresh tests contain 1,047,171 and 523,560 valid positions, respectively.

Each affine readout uses fitting seeds 17, 29, and 43, with 4,096 AdamW updates of 1,024 token positions sampled with replacement, learning rate $5\times10^{-4}$, weight decay $0.01$, and Gaussian input noise $0.02$ times the batch standard deviation. Parameters are float32; fitting uses bfloat16 matrix products and evaluation uses float32 with TF32 disabled. Validation native-token agreement is checked at updates 1,024, 2,048, and 4,096, with ties selecting the earlier update. All twelve initial fits select update 4,096. Tests are evaluated only after selection. Thus the comparison fixes 4,194,304 sampled token exposures, not epochs or convergence; changing the available pool changes the diversity and revisit rate of those exposures. It is not an exact replication of the older optimization schedule or an asymptotic sample-complexity estimate. In particular, the higher agreement of the new smallest-pool fit is not an improvement claim over historical MDP: fitting exposures, training counts, and test composition differ.

ELF agreement counts native nonpadding positions before the first native EOS. We separately apply each readout's predicted EOS truncation: it changes none of the nine learned-readout agreement scores on either test. Frozen cosine lookup uses the complete released T5 embedding table; its fresh-test agreement changes only from $93.2143\%$ to $93.2128\%$ under predicted EOS truncation. Neither measure evaluates token behavior beyond the native boundary. Means below average the three fitted readouts. Fit-seed standard deviations quantify optimizer/initialization variation on this one split. The 95\% intervals use 4,000 sequence-cluster bootstrap draws with seed 61013, reusing each sampled sequence across fits and compared budgets. They are conditional on these fitted models and this data split, not confidence intervals over new denoiser training seeds or new train-set partitions.

\begin{table}[!htbp]
\centering\small
\caption{Common-test ELF readout repetition. All entries are percentages; $\pm$ is the standard deviation across three fitting seeds. Brackets give conditional sequence-bootstrap 95\% intervals. Pool sizes here count training sequences only, unlike the historical total-pool labels.}
\label{tab:confirmatory-elf}
\begin{tabular}{rcc}
\toprule
Train sequences & Fixed historical test & Fresh test \\
\midrule
1,024 & $96.43\pm0.003$ [96.34, 96.52] & $96.32\pm0.010$ [96.17, 96.46] \\
4,096 & $97.10\pm0.013$ [97.02, 97.17] & $96.98\pm0.027$ [96.85, 97.11] \\
16,384 & $97.28\pm0.022$ [97.21, 97.35] & $97.17\pm0.023$ [97.04, 97.29] \\
\bottomrule
\end{tabular}
\end{table}

The 16,384-minus-1,024 paired gain is $0.851$ percentage points on the fixed historical test (95\% interval $[0.811,0.891]$) and $0.849$ on the fresh test ($[0.800,0.901]$). Both tests support a gain from the larger available fitting pool at this fixed update budget. They do not establish a universal scaling curve. Frozen cosine lookup attains $93.21\%$ on the fresh test, below all three learned readouts at the largest budget.

\begin{table}[!htbp]
\centering\small
\caption{Every scheduled fit, selected on validation only. Agreements are percentages. ELF columns use the same two tests at all budgets; LangFlow has its own independent 512-sequence test. No best-seed result is selected.}
\label{tab:confirmatory-fits}
\begin{tabular}{lrrrrr}
\toprule
Interface & Train sequences & Fit seed & Validation & Fixed test & Fresh test \\
\midrule
ELF & 1,024 & 17 & 96.30 & 96.43 & 96.31 \\
ELF & 1,024 & 29 & 96.31 & 96.43 & 96.33 \\
ELF & 1,024 & 43 & 96.31 & 96.43 & 96.32 \\
ELF & 4,096 & 17 & 97.03 & 97.11 & 97.01 \\
ELF & 4,096 & 29 & 97.02 & 97.09 & 96.97 \\
ELF & 4,096 & 43 & 97.00 & 97.09 & 96.96 \\
ELF & 16,384 & 17 & 97.19 & 97.31 & 97.19 \\
ELF & 16,384 & 29 & 97.15 & 97.27 & 97.15 \\
ELF & 16,384 & 43 & 97.18 & 97.27 & 97.16 \\
LangFlow & 1,024 & 17 & 62.97 & 62.66 & -- \\
LangFlow & 1,024 & 29 & 63.10 & 62.61 & -- \\
LangFlow & 1,024 & 43 & 63.08 & 62.59 & -- \\
\midrule
LangFlow follow-up & 8,192 & 17 & 74.64 & 74.12 & -- \\
LangFlow follow-up & 8,192 & 29 & 74.44 & 74.22 & -- \\
LangFlow follow-up & 8,192 & 43 & 74.51 & 74.11 & -- \\
\bottomrule
\end{tabular}
\end{table}

\parhead{Residual errors and conditional text scoring}
On the fresh ELF test, we retain the numeric-token definition (a digit in the tokenizer token string), report training-frequency groups, and inspect native-margin strata. The frequency threshold is a count in the largest training pool, not a test-frequency rank; it differs from the historical rare-token quintile. Native-margin groups are diagnostic oracle information. Their concentrations do not imply a deployable confidence gate. \tabref{tab:confirmatory-tail} shows that more fitting data has not removed the structured tail. All 127 valid positions whose target type was unseen in training remain wrong for all three fits.

\begin{table}[!htbp]
\centering\small
\caption{Fresh ELF residual groups for all three 16,384-sequence readouts. Agreement is mean $\pm$ fitting-seed SD in percent. Groups from different definitions overlap and must not be added. Native margins are unnormalized confidence gaps.}
\label{tab:confirmatory-tail}
\begin{tabular}{lrr}
\toprule
Group & Native positions & Agreement (\%) \\
\midrule
All & 523,560 & $97.17\pm0.02$ \\
Numeric & 4,865 & $86.08\pm0.33$ \\
Nonnumeric & 518,695 & $97.27\pm0.02$ \\
Training count 1--99 & 9,908 & $41.64\pm0.51$ \\
Training count $\geq100$ & 513,525 & $98.26\pm0.02$ \\
Unseen target type & 127 & $0.00\pm0.00$ \\
Native margin $<2$ & 8,315 & $38.13\pm0.08$ \\
Native margin $[2,8)$ & 19,310 & $69.78\pm0.16$ \\
Native margin $\geq8$ & 495,935 & $99.22\pm0.02$ \\
\bottomrule
\end{tabular}
\end{table}

For a paired text diagnostic, all readouts use the same native validity mask before T5 detokenization. GPT-2-Large scores each resulting text with length capped at 1,024 evaluator tokens, and corpus PPL is $\exp(\sum\mathrm{NLL}/\sum\mathrm{tokens})$. Detokenization can change evaluator token counts even with a common native boundary. No text has zero scored tokens. This oracle-boundary comparison tests conditional readout fidelity, not autonomous generation quality, and is not pooled with older PPL protocols.

\begin{table}[!htbp]
\centering\small
\caption{Conditional text scoring on the same 512 fresh ELF states. Every largest-budget fit is included. Native boundaries and evaluator truncation are fixed; lower PPL alone does not establish a better generator.}
\label{tab:confirmatory-ppl}
\begin{tabular}{lrr}
\toprule
Readout & GPT-2-Large PPL & Scored evaluator tokens \\
\midrule
Native decoder & 24.71 & 485,131 \\
Linear, seed 17 & 26.27 & 486,112 \\
Linear, seed 29 & 26.33 & 486,128 \\
Linear, seed 43 & 26.28 & 486,073 \\
Frozen cosine lookup & 33.17 & 507,699 \\
\bottomrule
\end{tabular}
\end{table}

\parhead{LangFlow tests an architectural alternative}
The released LangFlow LM1B checkpoint has a learned 768-dimensional embedding table and a 30,522-token BERT vocabulary. We generate 1,024 training, 256 validation, and 512 test sequences of 128 positions at 64 sampler evaluations, with seed 61013002. All positions follow the released sampling convention; unlike ELF, no EOS mask is applied. The three partitions contain no exactly duplicated native token sequences within or across partitions. Readout fitting and validation follow the same update protocol, but the number of available training positions is only 131,072. Thus this is an interface-specific feasibility check, not a matched-token-budget comparison to ELF.

The native logits have the implementation-level decomposition
\[
\ell_\gamma(z,s)=R_\gamma(z,s)+c(\gamma)zE^\top,
\qquad
c(\gamma)=\exp\!\left(\frac{\operatorname{softplus}(-\gamma)-\gamma}{2}\right),
\]
where $E$ is the checkpoint's layer-normalized learned embedding table and $s$ is the sampler's self-conditioning input. We capture the actual backbone output $R_\gamma$ and verify exact reconstruction of native logits with the implementation's floating-point casts. The frozen lookup uses float32 dot products before the native skip's bfloat16 cast; it is not a bitwise ablation of the quantized skip. At five phases on the same 512 test trajectories, we compare the frozen dot-product lookup $\arg\max zE^\top$ and backbone-only $\arg\max R_\gamma$ to the full native decision at that phase. This is a fixed-state logit ablation: the trajectories and self-conditioning remain those of the complete model. It does not measure generation after removing a term throughout sampling, nor agreement to the eventual final token at earlier phases.

\begin{table}[!htbp]
\centering\small
\caption{LangFlow native-interface controls. Each row uses 65,536 test positions. Intervals are conditional 95\% sequence-bootstrap intervals in percent. Learned rows average three fits per budget; the 8,192-sequence row is a post-result follow-up. Phase rows require no new fitting.}
\label{tab:confirmatory-langflow}
\begin{tabular}{rcc}
\toprule
Phase index & Native embedding lookup & Backbone term only \\
\midrule
0 & 0.01 [0.01, 0.02] & 94.06 [93.88, 94.24] \\
16 & 0.89 [0.82, 0.96] & 35.83 [35.30, 36.38] \\
32 & 3.43 [3.29, 3.58] & 44.78 [44.09, 45.47] \\
48 & 11.36 [11.12, 11.60] & 52.93 [52.31, 53.56] \\
63 & 89.69 [89.41, 89.96] & 55.86 [55.32, 56.39] \\
\midrule
\multicolumn{2}{l}{Learned map, 1,024 train sequences} & 62.62 [62.26, 62.99] \\
\multicolumn{2}{l}{Learned map, 8,192 train sequences} & 74.15 [73.79, 74.50] \\
\bottomrule
\end{tabular}
\end{table}

At the initial fitting budget, the affine map reaches $62.62\%\pm0.037$ percentage points on this LangFlow test. In an explicitly post-fit coverage diagnostic, target types seen at least 100 times during training occupy $51.02\%$ of test positions and have $99.04\%$ mean agreement. Types seen 1--99 times occupy $44.79\%$ with $26.99\%$ agreement; unseen types occupy $4.19\%$ with zero agreement. This concentration prevents interpreting the low overall score as proof of intrinsic nonlinear decoding difficulty or as a ceiling for larger fitting budgets. Frozen lookup has access to the checkpoint's complete learned vocabulary geometry, an advantage the independently initialized affine map must learn from its finite matching set.

The phase audit gives a narrower, useful boundary: at the final phase, native embedding lookup agrees on $89.69\%$ of positions and the backbone term alone on $55.86\%$, whereas at the initial phase the ordering is reversed ($0.014\%$ versus $94.06\%$). The full native decision still differs from float32 lookup on roughly one in ten final positions. These same-phase comparisons expose an explicit linear contribution to late readability; they do not independently reproduce ELF's frozen T5 geometry, prove a term unnecessary for transport, or rank the two systems.

\parhead{Post-result token-budget follow-up}
The initial LangFlow fitting pool contains eight times fewer token positions than the smallest ELF pool. After inspecting the initial results, a separate protocol freezes an extension to 8,192 LangFlow training sequences: the original 1,024 plus 7,168 generated with seed 61013003, batch size 16, and the same native sampler. The resulting 1,048,576 positions approximately match ELF's 1,047,155-position smallest pool. Validation and test files are unchanged byte for byte; optimizer settings, update budget, validation rule, and three fitting seeds are unchanged. Each additional fit restarts from scratch using the original seed list; it does not continue the smaller-pool weights. This is a result-motivated diagnostic extension on a previously inspected test, not an independent confirmation or a preregistered cross-architecture comparison. Its conditional intervals do not adjust for the choice to run this follow-up.

At this larger available pool, mean LangFlow agreement is $74.15\%\pm 0.060$ percentage points, with conditional interval $[73.79,74.50]\%$. The paired change from the original three fits is $11.53$ percentage points ($[11.29,11.75]$). The six LangFlow fits and both fitting budgets are all retained. Even this larger fit remains below the $89.69\%$ frozen-lookup baseline. Increasing the available pool also changes exposure frequencies under fixed compute, so the gain does not separately identify vocabulary coverage, optimization, and representation effects. The unchanged test has 3,172 native special-token positions ($4.84\%$, all \texttt{[CLS]}). As a post-fit mask sensitivity check, excluding these positions gives 60.72\% and 72.84\% mean agreement at the two budgets; the all-position metric remains primary. The enlarged training set has no exact native-sequence duplicates or overlap with validation or test. This within-interface pool intervention measures how sensitive the learned probe is to finite matching data; it does not isolate tokenizer, architecture, or dataset effects between ELF and LangFlow.

\parhead{Label lineage and reproducibility}
As a training-split lineage check, we decode the first 32 newly assigned historical training states with the currently released ELF EMA decoder. It agrees with the stored labels on 32,707 of 32,720 valid positions ($99.9603\%$), but is not bitwise identical. The 13 mismatches have current top-two logit gaps between approximately $7.6\times10^{-6}$ and $0.1224$; the check alone does not identify their cause. Historical results retain their saved native targets, while fresh-test targets are computed directly in the new generation run. The ancillary package includes the frozen protocol, split identifiers, checkpoint and selected-readout hashes, validation histories, every per-sequence agreement count, phase controls, tail counts, and per-sequence evaluator NLL/token sums. A CPU verifier reproduces the aggregate agreements, seed variation, paired intervals, and corpus PPL from these scalar records. Original generation and fitting scripts are included as provenance; rerunning them also requires the local checkpoints and latent caches, which are not bundled.

\section{Prospective Readout Prediction and Equal-Error Repair}
\label{app:prospective-readout}

\parhead{Protocol, calibration, and transfer}
The local protocol, calibration coefficients, readout weights, and source checkpoints were hashed before any new test generation. This is a prospectively frozen local test, not a publicly registered study. Calibration reuses the already inspected 512-sequence ELF-B set from \secref{app:confirmatory-readout}; its generation seed is 61013001. All three validation-selected affine readouts trained on 16,384 ELF-B sequences are frozen. Their fitting seeds are 17, 29, and 43. Test generation seeds are 61013201 and 61013202, with 256 sequences per seed and condition, batch size four, and current native labels computed during generation. ELF-B and ELF-M use 32 SDE steps, $\gamma=1.5$, and self-conditioning CFG $3.0$; the released PD student uses eight steps, $\gamma=1.5$, and CFG $2.5$. All use their native logit-normal grids and the final ODE update. These are within-interface checkpoint transfers; checkpoint size, training, and sampling differences are not separately identified.

Let $v_{i\ell}$ denote the native valid-position mask, $y_{i\ell}$ the native token, and $\widehat y_{i\ell}$ a frozen readout decision. The target is the difference of mean evaluator NLLs for the two detokenized texts,
\[
q_i=\frac{\operatorname{NLL}(\widehat y_i)}{n(\widehat y_i)}-
\frac{\operatorname{NLL}(y_i)}{n(y_i)}.
\]
Here $n$ counts scored GPT-2 tokens after truncation, not T5 positions. The baseline covariates are the native-position disagreement rate, log native valid length, and native mean evaluator NLL. Three added covariates are the number of disagreements in each specified group divided by the total native valid length: target-token training count below 100 (including unseen types), a digit in the native token string, and native top-two margin below 2. Groups overlap. Training frequencies come only from the frozen 16,384-sequence ELF-B fitting pool. For each readout fit, ridge regression uses calibration-only means and standard deviations, an unpenalized intercept, and fixed penalty $\alpha=10$. There is no hyperparameter search, test-based selection, or target-checkpoint refitting.

\begin{table}[!htbp]
\centering\small
\caption{Frozen prediction of excess per-sequence NLL. MAE is averaged across the three fixed readouts. Positive reduction favors adding tail-specific features. Per-condition intervals are conditional 95\%; the pooled primary interval is 97.5\% to account for the two prespecified primary tests.}
\label{tab:prospective-prediction}
\begin{tabular}{lrrr}
\toprule
Condition & Aggregate MAE & Tail MAE & Reduction [interval] \\
\midrule
Pooled & 0.0483 & 0.0395 & $+0.0088$ [$ +0.0076, +0.0100 $] \\
B32 & 0.0331 & 0.0306 & $+0.0025$ [$ +0.0017, +0.0034 $] \\
M32 & 0.0700 & 0.0432 & $+0.0268$ [$ +0.0244, +0.0292 $] \\
PD8 & 0.0417 & 0.0446 & $-0.0029$ [$ -0.0049, -0.0010 $] \\
\bottomrule
\end{tabular}
\end{table}

\parhead{Matched repair and uncertainty}
For readout seed 17 only, each repair corrects $\lfloor e_i/2\rfloor$ of the $e_i$ disagreements in sequence $i$ to their native targets. The rare-first policy orders erroneous positions by ascending target training count; numeric-first prioritizes a digit in the native target token; low-margin-first orders native margins. Random tie breaking and three uniform random selections are fixed by the protocol. The NumPy seed sequence combines the generation seed, condition index, sequence index, policy index, and constant 61013. Random arms A/B/C correspond to policy indices 0/1/2 (internal file labels 11/23/37); these file labels are identifiers, not standalone random seeds. These are oracle edits: all policies know the native errors, and all achieve exactly identical per-sequence agreement. They cannot be interpreted as inference-time routers or as edits toward ground-truth factual answers. The low-margin priority uses the native decoder, whereas the historical routing experiment uses confidence of the learned linear readout and a different correction budget; these signals and interventions must not be conflated.

The OPT-125M secondary scoring protocol was added after generation but before inspecting primary scoring outcomes. It scores the exact same ten arms without new prediction fitting or output selection. All OPT intervals are descriptive 95\% intervals. This is evaluator sensitivity, not an independent confirmation of the primary predictor. 

Primary prediction inference averages the absolute prediction errors across the three fixed readouts within each sequence. Repair contrasts compare each targeted arm with the mean of the three random arms on the same sequence. We use 4,000 paired sequence-cluster bootstrap draws with seed 61013203, equal condition/seed weighting, and reuse sampled indices across checkpoint conditions sharing a generation seed. This common-random-number correction was documented before test generation as an amendment to the initial stratified-bootstrap wording. The two primary tests---pooled prediction MAE reduction and pooled rare-first repair advantage---use Bonferroni 97.5\% two-sided intervals. Other intervals are descriptive 95\% intervals without simultaneous-coverage claims. All intervals are conditional on the fixed random-repair realizations as well as the selected models. None represents uncertainty over new denoiser training seeds or new calibration partitions.

\begin{table}[!htbp]
\centering\small
\caption{Equal-error repair: mean per-sequence NLL of random repair minus targeted repair. Positive values favor targeted repair. All repair arms have identical native agreement within each sequence. Pooled rare-first is the prespecified primary contrast; remaining contrasts are secondary.}
\label{tab:prospective-repair}
\begin{tabular}{llrr}
\toprule
Condition & Priority & NLL difference & Interval \\
\midrule
Pooled & Rare & $+0.00956$ & [$ +0.00790, +0.01122 $] \\
Pooled & Numeric & $-0.00002$ & [$ -0.00110, +0.00110 $] \\
Pooled & Low margin & $-0.01814$ & [$ -0.01957, -0.01681 $] \\
B32 & Rare & $+0.00172$ & [$ -0.00010, +0.00360 $] \\
B32 & Numeric & $+0.00008$ & [$ -0.00155, +0.00184 $] \\
B32 & Low margin & $-0.01748$ & [$ -0.01927, -0.01567 $] \\
M32 & Rare & $+0.00687$ & [$ +0.00417, +0.00963 $] \\
M32 & Numeric & $-0.00063$ & [$ -0.00272, +0.00151 $] \\
M32 & Low margin & $-0.01947$ & [$ -0.02204, -0.01704 $] \\
PD8 & Rare & $+0.02009$ & [$ +0.01741, +0.02269 $] \\
PD8 & Numeric & $+0.00047$ & [$ -0.00162, +0.00262 $] \\
PD8 & Low margin & $-0.01748$ & [$ -0.01996, -0.01510 $] \\
\midrule
\multicolumn{4}{l}{Secondary OPT-125M (all intervals descriptive 95\%)} \\
Pooled & Rare & $+0.00871$ & [$ +0.00720, +0.01020 $] \\
Pooled & Numeric & $-0.00081$ & [$ -0.00196, +0.00040 $] \\
Pooled & Low margin & $-0.02172$ & [$ -0.02319, -0.02029 $] \\
B32 & Rare & $+0.00196$ & [$ +0.00007, +0.00382 $] \\
B32 & Numeric & $-0.00077$ & [$ -0.00245, +0.00108 $] \\
B32 & Low margin & $-0.01901$ & [$ -0.02085, -0.01714 $] \\
M32 & Rare & $+0.00813$ & [$ +0.00533, +0.01096 $] \\
M32 & Numeric & $-0.00069$ & [$ -0.00286, +0.00141 $] \\
M32 & Low margin & $-0.02130$ & [$ -0.02409, -0.01863 $] \\
PD8 & Rare & $+0.01603$ & [$ +0.01319, +0.01881 $] \\
PD8 & Numeric & $-0.00095$ & [$ -0.00331, +0.00143 $] \\
PD8 & Low margin & $-0.02484$ & [$ -0.02755, -0.02225 $] \\
\bottomrule
\end{tabular}
\end{table}

\parhead{Scores and measurement scope}
Every output uses the same native validity mask before T5 detokenization. GPT-2-Large uses float16 weights with float32 losses and at most 1,024 evaluator tokens. Corpus PPL exponentiates summed NLL divided by summed scored tokens, whereas prediction and repair inference use per-sequence mean NLL. Different outputs can have different evaluator tokenizations and truncation coverage; the measurements concern these scored texts, not equal-position teacher-forced probabilities. Common-scorable sequence counts, token counts, exact-overlap checks, frozen coefficients, all predictions, and scalar score records are included in the ancillary evidence. Exact overlap checks do not assess semantic near-duplicates. Native agreement and this external text score are distinct outcomes; neither is factual correctness or document planning.

\begin{table}[!htbp]
\centering\small
\caption{Conditional corpus PPL on all 512 sequences per checkpoint condition. The three learned maps are frozen ELF-B fits. Repair arms use fit seed 17 and identical per-sequence correction counts. All prespecified random arms are retained.}
\label{tab:prospective-ppl}
\begin{tabular}{lrrr}
\toprule
Readout & B32 & M32 & PD8 \\
\midrule
Native & 24.27 & 26.50 & 22.43 \\
Linear, seed 17 & 25.82 & 30.37 & 24.67 \\
Linear, seed 29 & 25.81 & 30.40 & 24.66 \\
Linear, seed 43 & 25.84 & 30.36 & 24.61 \\
Random repair A & 25.37 & 29.04 & 24.45 \\
Random repair B & 25.37 & 29.06 & 24.50 \\
Random repair C & 25.36 & 29.06 & 24.44 \\
Rare-first repair & 25.33 & 28.86 & 23.99 \\
Numeric-first repair & 25.36 & 29.07 & 24.45 \\
Low-margin-first repair & 25.82 & 29.63 & 24.90 \\
\midrule
\multicolumn{4}{l}{Secondary OPT-125M} \\
Native & 22.21 & 24.31 & 20.42 \\
Linear, seed 17 & 23.77 & 28.14 & 23.10 \\
Linear, seed 29 & 23.77 & 28.14 & 23.12 \\
Linear, seed 43 & 23.75 & 28.14 & 23.09 \\
Random repair A & 23.28 & 26.78 & 22.60 \\
Random repair B & 23.29 & 26.78 & 22.63 \\
Random repair C & 23.27 & 26.74 & 22.58 \\
Rare-first repair & 23.24 & 26.55 & 22.25 \\
Numeric-first repair & 23.30 & 26.78 & 22.62 \\
Low-margin-first repair & 23.73 & 27.34 & 23.18 \\
\bottomrule
\end{tabular}
\end{table}

The portable NumPy verifier refits the calibration regressions, reproduces the frozen predictions and paired intervals, checks equal repair counts from saved per-sequence records, and recomputes corpus PPL. This scalar audit does not rerun the GPU models. Local token arrays additionally verify the actual edited positions and the unchanged nonselected outputs; their hashes are recorded, while large latent arrays and source checkpoints are not bundled.

\parhead{Released-checkpoint provenance}
A live inventory of the official \href{https://huggingface.co/embedded-language-flows}{ELF model organization} on October 13, 2026 returned 13 repositories. The newest repository update was the June 26 release of ELF-B-PD-owt-torch, whose student is included in PD8. All three checkpoints used here exactly match their current official LFS SHA-256 values; the API snapshots and comparisons are included in the ancillary evidence. Modification dates identify repository updates, not new independent training runs. This provenance check did not change the frozen test protocol.

\parhead{Interpreting the student failure}
A post-result breakdown distinguishes absolute prediction from ranking. On PD8, adding tail features increases mean prediction bias from $+0.0026$ to $+0.0233$ nats (prediction minus measured excess NLL), while descriptive Spearman correlation rises from $0.124$ to $0.262$. These summaries pool the three fixed readout records per sequence and carry no independent-sample significance claim. Better ordering therefore coexists with worse absolute MAE. This does not identify the cause of the shift: the checkpoint and sampler change together, and no post-test recalibration is performed.

\clearpage
\section{Contextual Stability at a Fixed Focal Latent}
\label{app:context-boundary}

The affine readout acts on one position at a time, whereas the native decoder reads the complete state. High agreement on naturally generated endpoints therefore leaves a specific question open: if a focal latent is held fixed while its context changes, is its native decision also fixed? This is a question about the decoder mapping, not about generating a useful endpoint or proving that a commitment policy accelerates sampling.

\subsection{Separating focal and contextual perturbations}

\begin{proposition}[Blockwise form of the margin condition]
\label{prop:context-margin}
Let $Z=(z_1,\ldots,z_T)$ be a full decoder input and let $a_i$ be the unique winning token at output position $i$. For each rival $j\ne a_i$, define the pairwise logit gap $q_{ij}(Z)=f_{i,a_i}(Z)-f_{i,j}(Z)$. Assume these gaps are continuously differentiable on a neighborhood of the segment $Z+s\Delta$, $s\in[0,1]$, and choose finite bounds
\[
L_{i\leftarrow k,j}\ \geq\ \sup_{s\in[0,1]}
\bigl\|\nabla_{z_k}q_{ij}(Z+s\Delta)\bigr\|_2.
\]
Then
\begin{equation}
q_{ij}(Z+\Delta)\ \geq\ q_{ij}(Z)
-\sum_{k=1}^{T}L_{i\leftarrow k,j}\|\Delta_k\|_2.
\label{eq:context-margin}
\end{equation}
If the right-hand side is strictly positive for every rival, the native decision at position $i$ is preserved. Setting $\Delta_i=0$ removes only the focal term; the terms with $k\ne i$ remain.
\end{proposition}
\begin{proof}
Apply the fundamental theorem of calculus to $q_{ij}(Z+s\Delta)$. Its derivative is $\sum_k\langle\nabla_{z_k}q_{ij}(Z+s\Delta),\Delta_k\rangle$. Integrating and applying Cauchy--Schwarz to each term gives the inequality. Strictly positive gaps preserve the unique winner.
\end{proof}

This is the standard margin--sensitivity argument~\citep{hein2017formal}, with the full input of \thmref{thm:decoder-margin} written in position blocks; it is not a new general robustness theorem. The constants must hold along the entire perturbed segment, not just at its starting point. We do not estimate them or obtain certificates from the following stress tests. A two-position example makes the distinction explicit: with logits $(z_1+bz_2,0)$, $b\ne0$, the input $(m,0)$ has margin $m>0$, but changing only $z_2$ to $-2m/b$ reverses the decision. By contrast, a token-wise affine readout $Wz_i+c$ is exactly invariant to all changes outside $i$.

Positive global logit scaling is a separate invariance (\corref{cor:scale}). Scaling logits by $a>0$ multiplies both gaps and their sensitivity bounds by $a$, leaving the sufficient inequality and the argmax unchanged. A fixed raw-margin threshold has no such invariance. This does not invalidate a threshold calibrated for a fixed checkpoint; it limits interpretations that compare uncalibrated thresholds across interfaces or treat them as intrinsic distances.

\subsection{Frozen context-only intervention protocol}

We reuse the first 64 sequences from each of the two generation seeds in each prospective-test condition (B32, M32, and PD8): 384 previously analyzed endpoints in total. These are new interventions with a protocol frozen before their outcomes, not new prospective generations or a new test of the quality predictor. Each sequence contributes eight equally spaced native-valid probe positions, selected without reference to confidence or error, giving 3,072 probe instances. All eight focal vectors remain bitwise unchanged throughout an intervention. A cyclic next-sequence donor supplies a context displacement; an independent Gaussian direction is normalized to the same Frobenius norm. Both are zeroed at the eight protected positions. We apply each direction at amplitudes $0.1$, $0.25$, $0.5$, and $1.0$ to all other positions of the 1,024-position decoder canvas, without restricting context changes to native-valid positions. At amplitude $1$, the donor arm replaces every unprotected vector with its donor counterpart.

For each sequence we run a fresh baseline, a separate zero-displacement control, and all eight nonzero interventions, totaling 3,840 native forward passes. The native decode configuration and the affine readout fit (seed 17) remain fixed. Flip rates refer to the fresh native baseline at the same protected positions; they do not measure disagreement with ground-truth text. We retain all prespecified conditions and amplitudes. Intervals use 4,000 bootstrap draws over whole sequences, stratified by the two original generation seeds, averaging the eight probes within each sequence. They are descriptive conditional 95\% intervals, not inference over independently trained denoisers. Input hashes, the frozen protocol, per-sequence outputs, and a portable scalar verifier accompany the release.

\begin{figure}[!htbp]
\centering
\includegraphics[width=\linewidth]{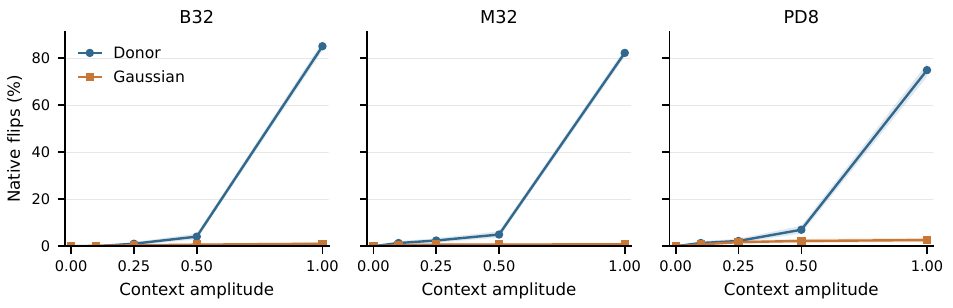}
\caption{Native decisions under context-only perturbations. Each panel uses 128 reused endpoint sequences and eight protected probes per sequence. The donor and Gaussian directions have matched Frobenius norms within each sequence, and neither changes a probe vector. Shading shows conditional 95\% sequence-bootstrap intervals. The unchanged token-wise affine readout is an exact control; these curves measure native decoder sensitivity on synthetic contexts, not text quality or the effect of a deployed commitment policy.}
\label{fig:context-boundary}
\end{figure}

\begin{table}[!htbp]
\centering\small
\caption{Native flip rates at fixed focal vectors (percent; conditional 95\% intervals). Every condition has 1,024 probes from 128 sequences. The upper block contains all prespecified full-canvas arms. The lower block is a post-result follow-up restricted to shared native-valid context. The zero arm repeats the baseline computation. Degenerate intervals indicate no observed flips, not certified zero risk.}
\label{tab:context-boundary}
\begin{tabular}{llccc}
\toprule
Context direction & Amplitude & B32 & M32 & PD8 \\
\midrule
Zero & 0 & 0.00 [0.00, 0.00] & 0.00 [0.00, 0.00] & 0.00 [0.00, 0.00] \\
Donor & 0.1 & 0.00 [0.00, 0.00] & 1.37 [0.68, 2.15] & 1.37 [0.68, 2.15] \\
Donor & 0.25 & 1.07 [0.49, 1.76] & 2.44 [1.56, 3.42] & 2.25 [1.37, 3.22] \\
Donor & 0.5 & 4.10 [2.93, 5.37] & 4.98 [3.71, 6.35] & 7.03 [5.47, 8.69] \\
Donor & 1 & 85.06 [83.20, 86.91] & 82.23 [80.57, 83.79] & 74.90 [72.07, 77.64] \\
Gaussian & 0.1 & 0.20 [0.00, 0.49] & 0.49 [0.10, 0.98] & 0.78 [0.29, 1.37] \\
Gaussian & 0.25 & 0.39 [0.10, 0.78] & 0.68 [0.20, 1.27] & 1.76 [0.98, 2.54] \\
Gaussian & 0.5 & 0.68 [0.20, 1.17] & 0.68 [0.20, 1.27] & 2.25 [1.37, 3.12] \\
Gaussian & 1 & 1.07 [0.49, 1.66] & 0.88 [0.39, 1.46] & 2.64 [1.76, 3.61] \\
\midrule
\multicolumn{5}{l}{Post-result follow-up: shared native-valid context only} \\
Donor & 1 & 83.50 [81.84, 85.16] & 80.86 [79.39, 82.32] & 73.73 [71.09, 76.37] \\
Gaussian & 1 & 1.37 [0.68, 2.15] & 1.17 [0.59, 1.86] & 2.64 [1.66, 3.61] \\
\bottomrule
\end{tabular}
\end{table}

All 3,072 protected vectors and all corresponding affine decisions remain unchanged, and the separate zero control has no native flips. At donor amplitude $0.25$, native flip rates are $1.07\%$, $2.44\%$, and $2.25\%$ for B32, M32, and PD8, respectively. At full donor displacement they reach $85.06\%$, $82.23\%$, and $74.90\%$. \tabref{tab:context-boundary} also retains the norm-matched Gaussian controls; amplitude is a multiplier of a sequence-specific displacement, not a common absolute distance across checkpoints. The fresh baseline matches all 3,072 saved native probe labels. The full-amplitude Gaussian controls yield only $1.07\%$, $0.88\%$, and $2.64\%$ flips, respectively. Thus displacement norm alone does not determine this response. Full donor displacement is large, averaging $1.31$--$1.32$ times the full baseline-state norm; it is not an infinitesimal robustness test. These measurements establish sensitivity of the native mapping to the prescribed context changes while the token-wise decision is fixed.

\parhead{Post-result validity-mask check}
After the full-canvas outcomes, we froze a follow-up on the same 384 endpoints, displacing only positions valid in both the base and its cyclic donor, excluding the protected probes. A fresh baseline, zero control, donor arm, and newly drawn norm-matched Gaussian arm add 1,536 native forwards. The fresh decisions match the original baseline and the zero controls again have no flips. Donor rates remain $83.50\%$, $80.86\%$, and $73.73\%$ (B32/M32/PD8), with all follow-up intervals in \tabref{tab:context-boundary}. This is explicitly a post-result check, not part of the original frozen sweep.

The mask audit also exposes a coverage limit: none of these 384 endpoints contains an EOS token. The native-valid mask excludes only a few padding-token positions, and the follow-up allows an average of $1014.7$ context positions to change, close to the full-canvas maximum of $1016$. It therefore checks padding-mask sensitivity on nearly full-length outputs; it cannot validate behavior across EOS boundaries or substantially shorter contexts. The motivation to check post-EOS editing was not evaluable on this subset. Mixed donor states remain synthetic even after applying the validity mask.

\begin{table}[!htbp]
\centering\small
\caption{Algebraic logit-scale control on the same baseline probes. Entries are the percentage passing a fixed raw-margin threshold of 8 after scaling all logits by $a$. Every argmax is unchanged by positive scaling; no model is retrained and no new trajectory is sampled.}
\label{tab:context-scale}
\begin{tabular}{lrrr}
\toprule
Condition & $a=0.5$ & $a=1$ & $a=2$ \\
\midrule
B32 & 73.54 & 92.48 & 96.09 \\
M32 & 65.23 & 95.21 & 97.17 \\
PD8 & 75.29 & 89.45 & 92.68 \\
\bottomrule
\end{tabular}
\end{table}

\tabref{tab:context-scale} is a numerical illustration of the exact scaling invariance, not an additional empirical robustness result. It shows why a margin threshold should be calibrated for its actual logit scale; it does not claim that the listed checkpoints differ only by temperature.

\subsection{What this boundary does and does not establish}

The intervention separates token-wise readability on the observed endpoint distribution from context independence of the native mapping. It does not measure the benefit of context on natural text: the displaced states can be off distribution, and even the donor condition combines vectors from different sequences. Nor does it show that contextual dependence caused the earlier TEC failure. A freezing policy changes the denoising trajectory, conditioning history, and later readouts together; identifying its failure mechanism requires a matched trajectory intervention. The present result only tests whether a focal latent alone fixes the decoder output under these prescribed context changes.

The descriptive intervals resample base-sequence summaries. Because the cyclic donor scheme reuses each state as a neighboring donor, those summaries are not fully independent; this bootstrap does not preserve all donor-induced dependence and is not asserted to provide nominal population coverage.

The same scope discipline applies to the other negative controls. A mismatched encoder with the old decoder does not rank encoders after matched retraining. Removing feature dimensions by post-hoc PCA does not test learned sequence compression with a compatible decoder. Disabling self-conditioning at evaluation does not isolate its role during training. Our particular hard-freezing rule does not evaluate trained asynchronous or position-dependent commitment. Finally, a failed post-hoc sampler shortcut does not rule out calibrated timestep placement. These distinctions preserve the empirical failures while leaving architectural and training improvements open.

\clearpage
\section{Separating Layer Depth from the Readout Interface}
\label{app:layer-interface-controls}

\parhead{Why this follow-up was needed}
The raw dips in \figref{fig:layer-basin} suggested a possible loss of readability at T5 layer 5 and GPT-2 layer 10. Auditing the code identified two confounds before attributing them to information content: terminal normalization changes the scale and coordinate calibration of the state presented to the readout, and the historical T5 plot displayed the same final output twice. The earlier caption also mislabeled the left and right panels. These are corrections to the interpretation and presentation of the historical measurements; their source rows remain available. At corrupted points, the two T5 final-state aliases used separate noise draws; the revised plot retains the historical final-output curve (record ID 100), rather than averaging those repetitions or treating them as different depths.

\parhead{Post-hoc protocol and sample accounting}
This is an exploratory follow-up prompted by the historical figure, not a pre-registered confirmation of its interpretation. We froze the implementation and protocol before these new decoder runs. For each model we use 256 texts: the first 128 eligible OpenWebText texts overlap the historical audit, and filtered indices 512--639 provide 128 additional texts disjoint from its native-head 512-text subset. Eligibility requires at least 20 whitespace-separated words, in sorted parquet-file order. We retain text indices and hashes in the portable evidence. This disjointness is relative to the historical audit, not a claim that the texts were absent from pretraining.

T5-small uses a float32 encoder on the native 1024-token canvas, followed by the released ELF-B checkpoint 95085, EMA1 weights, native decoder, and CFG 3. GPT-2 uses float32 computation and at most 128 tokens with shifted next-token targets. Networks are frozen; no readout is fitted. Every comparison within a model uses the same texts and valid targets. There are 512 model--text pairs and 19,456 per-sequence measurement rows; repeated rows are not independent samples. At $t=0.7$, T5 variants share Gaussian noise draws, with noise scaled by each transformed full-canvas state's standard deviation, as in the historical corruption protocol. This holds the noise convention fixed, not the absolute noise amplitude across differently scaled states.

\parhead{Four state transformations}
Let $H=(h_i)$ be a hidden-state canvas, and let $\gamma$ be the learned terminal normalization gain. The T5 terminal transformation is
\[
 N(h_i)=\gamma\odot\frac{h_i}{\sqrt{d^{-1}\|h_i\|_2^2+\epsilon}}.
\]
We compare raw $H$, the actual terminal normalization, one global-RMS matching rule $cH/\operatorname{RMS}(H)$, and per-token RMS matching $ch_i/\operatorname{RMS}(h_i)$, where $c=\operatorname{RMS}(\gamma)$. The two RMS controls use a $10^{-12}$ denominator floor and no labels or fitted scalar. For T5 all transformations precede ELF's fixed latent scaling (mean 0, standard deviation 0.2). The global denominator includes the padded canvas. Per-token RMS preserves each vector's direction and omits the learned coordinate gains. The global rule is one explicit control, not a search over all possible scalar calibrations.

For GPT-2, the actual terminal LayerNorm also centers each vector and applies learned gain and bias. Its bias-free linear vocabulary head gives an exact control: for any $a>0$, $W(ah)=aWh$, so token argmax is unchanged and a target margin scales by $a$. Across 3,328 sequence--layer comparisons, positive global scaling changes no argmax; the largest numerical discrepancy from margin scaling is $4.9\times10^{-4}$. Thus a large change in a negative raw margin alone need not imply any changed decision. This statement is specific to this bias-free head, not the nonlinear contextual ELF decoder.

\parhead{Terminal-state identity checks}
Hooks capture the T5 state before its final RMSNorm and the GPT-2 state before \texttt{ln\_f}. Applying the actual normalization reconstructs the corresponding native final hidden state exactly in every checked sequence. Each final hidden-state tuple entry also equals \texttt{last\_hidden\_state} exactly. We therefore distinguish T5 layer 5, pre-norm output of block 6, and native final output; analogously, GPT-2 pre-norm block 12 is distinct from the already normalized hidden index 12. The follow-up samples layer 11 as well, which was absent from the historical GPT-2 plot.

\begin{table}[!htbp]
\centering\small
\caption{Paired normalization controls on the 128 additional texts. Recovery is the mean per-sequence percentage; margin is the mean of per-sequence target-margin p10 values. This aggregation differs from historical batch-pooled p10. T5 uses unmasked reconstruction through ELF; GPT-2 uses next-token prediction. Rounded 100.00\% entries describe observed samples, not a population guarantee. All layers, both text subsets, noise levels, and individual rows are retained in the evidence files.}
\label{tab:layer-interface-controls}
\begin{tabular}{llrrr}
\toprule
State & Transformation & Clean recovery & Clean margin & Recovery, $t=0.7$ \\
\midrule
T5 layer 5 & Raw & 78.82 & -284.38 & 31.33 \\
 & Global RMS & 64.90 & -2.73 & 31.20 \\
 & Per-token RMS & 99.97 & 12.34 & 99.97 \\
 & Terminal norm & 100.00 & 15.03 & 100.00 \\
\midrule
T5 pre-norm 6 & Raw & 61.77 & -856.33 & 9.27 \\
 & Global RMS & 27.03 & -9.24 & 12.14 \\
 & Per-token RMS & 99.91 & 11.11 & 99.94 \\
 & Terminal norm & 100.00 & 15.44 & 100.00 \\
\midrule
GPT-2 layer 10 & Raw & 22.54 & -126.63 & -- \\
 & Global RMS & 22.54 & -22.85 & -- \\
 & Terminal norm & 25.94 & -12.43 & -- \\
\midrule
GPT-2 pre-norm 12 & Raw & 12.49 & -609.32 & -- \\
 & Global RMS & 12.49 & -76.83 & -- \\
 & Terminal norm & 38.58 & -5.84 & -- \\
\bottomrule
\end{tabular}
\end{table}

\parhead{What explains the T5 dip}
On the additional texts, terminal normalization improves layer-5 recovery by $21.18$ percentage points (paired 95\% interval $[20.66,21.71]$). Per-token RMS alone already recovers almost all targets, whereas the tested global rule does not. Applying the actual terminal normalization to the pre-norm block-6 state recovers the native final output, improving recovery by $38.22$ points ($[37.57,38.88]$). These same-state interventions show that much of the apparent readability collapse is an interface-normalization effect; they do not demonstrate disappearance and re-creation of lexical information. Results at $t=0.7$ support the same interpretation under the stated relative-noise convention.

A subsequent descriptive encoder-only pass on the additional 128 texts examines why global and per-token scaling differ. At layer 5, the mean within-sequence ordinary-token RMS median is $66.26$, while its mean 99th percentile is $1685.73$; before block 6's final norm these are $85.53$ and $3513.65$. A single canvas scale therefore encounters strongly heterogeneous token magnitudes. EOS accounts for only $0.36\%$ of valid-position squared energy at layer 5, and padding for $3.18\%$ of full-canvas energy: neither dominates that deep-state energy. This argues against an EOS/padding-only explanation. The distributional audit is descriptive and does not identify the semantic cause of the high-norm ordinary positions.

\parhead{What explains the GPT-2 dip}
At layer 10, global scaling reduces the magnitude of mean target-margin p10 while leaving recovery at $22.54\%$; actual terminal LayerNorm increases recovery by $3.40$ points ($[2.63,4.20]$). Layer 11 has raw recovery $20.96\%$ and normalized recovery $32.27\%$. At the captured pre-norm block-12 state, actual LayerNorm raises recovery from $12.49\%$ to the native $38.58\%$ ($26.09$ points, interval $[25.16,27.04]$). However, applying the same norm early can reduce recovery: layer 5 changes from $11.64\%$ to $7.15\%$. The full curves retain these adverse outcomes. The final jump cannot be attributed solely to the last block's representational content; nor does the result justify normalizing every intermediate state as a universal improvement.

\begin{figure}[!htbp]
\centering
\includegraphics[width=\linewidth]{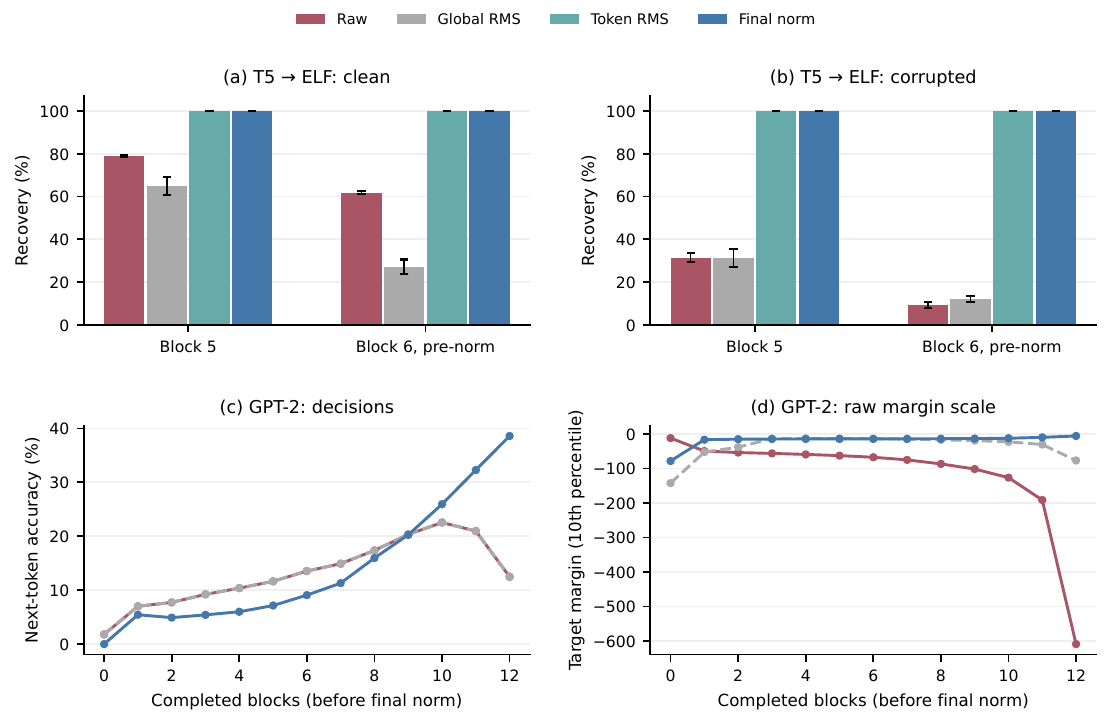}
\caption{New post-hoc normalization controls on the additional 128 texts. Top: T5 layer 5 and pre-norm block 6, decoded through the same ELF interface, at $t=1$ and $t=0.7$. Bars show mean recovery with conditional sequence-bootstrap intervals. Bottom: GPT-2 recovery and mean per-sequence margin p10 at every captured pre-terminal depth, including pre-norm block 12. Global rescaling overlaps raw recovery despite changing margin magnitude; actual LayerNorm helps late layers but can hurt early ones. Native final outputs are represented by normalizing the captured final pre-norm states. No new generation quality is measured.}
\label{fig:layer-normalization-controls}
\end{figure}

\parhead{Uncertainty, precision, and scope}
Intervals use 4,000 paired sequence-bootstrap resamples within each text subset with seed 6102201. They are conditional descriptive intervals, not coverage claims over independently trained models or arbitrary corpora. The historical-overlap subset exhibits the same qualitative normalization pattern and remains separately tabulated in the CSVs. A limited precision check repeats raw T5 layer 5 and native final layer 6 on the first eight historical texts with a float16 encoder: recovery is unchanged, while mean margin-p10 differences are $0.047$ and $0.013$ respectively. This small check does not establish precision invariance more generally. The terminal-normalization identity checks, exact scalar control, adverse early-layer outcomes, and both text subsets are included in the portable evidence. Together they constrain the interpretation of readability; they do not establish a new sampler, better unconditional generation, or a complete account of every layer's behavior.
\FloatBarrier

\section{Independent AURORA Readout Replication}
\label{app:aurora-replication}

The local protocol was frozen before formal data collection; it is not a public preregistration. Mechanical pilot sequences were excluded from every fitting, validation, calibration, and test pool. The released EMA denoiser and shared text autoencoder were downloaded at pinned revisions and checked against the upstream SHA-256 digests. The memory-bounded adapter preserves the released model and sampler operations. Native labels use the released decoder\textquotesingle s \texttt{argmax}, rather than the potentially different first index of \texttt{topk} on tied BF16 logits. The public inference entry and the adapter agree exactly in the recorded single-sequence pilot.

\parhead{Fitting and calibration}
All development sequences use the released SDE-DPM++ sampler with 64 diffusion updates per 16-token block and self-conditioning guidance 4.0. We generate 512-token prefixes (32 blocks); the checkpoint\textquotesingle s training and default maximum is 1,024. S16, S32, and S64 denote updates per block, not whole-sequence function evaluations; compute is not matched to ELF. Decoder evaluation is performed one sequence at a time to bound memory. Each affine fit uses seeds 17, 29, or 43, 4,096 AdamW updates of 1,024 sampled token positions, learning rate $5\times10^{-4}$, weight decay 0.01, and input noise 0.02 times the batch feature standard deviation. Native EOS positions are excluded from fitting. Checkpoints at updates 1,024, 2,048, and 4,096 are selected by validation lexical agreement, with earlier updates breaking ties. Calibration fits one baseline and one augmented ridge predictor per selected readout, using $\alpha=10$, calibration-only standardization, and an unpenalized intercept.

\parhead{Scoring and matched edits}
OWT samples contain document separators, so first-EOS truncation would change the task. Every arm retains the same native EOS separators and is decoded and re-tokenized by its evaluator. A sequence belongs to that evaluator\textquotesingle s common scorable set only when all ten arms contain 2--1,024 evaluator tokens; no arm is truncated. The native decoder is an oracle comparator, not a semantic truth label. Repairs use readout seed 17 and correct exactly $\lfloor e/2\rfloor$ of its lexical disagreements, using ascending native training count, numeric-first, ascending native margin, or three uniform random selections. All six arms therefore retain identical source-token error counts within each sequence. The underlying latent trajectory and its contextual decoder inputs are held fixed.

The common scorable sets retain 1,536 of 1,536 test sequences for GPT-2-Large and 1,536 for OPT-125M. Excluded indices, every per-sequence score, and exact native-output overlap counts are retained in the evidence. Native lexical agreement uses all 1,536 generated sequences.

\parhead{Inference and scope}
Test seeds are 61022401 and 61022402, with 256 sequences per seed at each of 16, 32, and 64 steps per block. Paired sequence bootstrap intervals use 4,000 draws, preserve common-seed dependence across step conditions, and average the three fixed readout errors within each sequence. The two primary pooled contrasts use 97.5\% intervals; all other intervals are descriptive 95\% intervals. The primary scorer is GPT-2-Large; OPT-125M evaluates exactly the same output edits. Both use BF16 inference and float32 cross entropy. These intervals condition on the released model, the matching pool, and the fitted readouts. No result establishes a universal ceiling for affine fitting or an architecture-isolated quality ranking.

\begin{figure}[!htbp]
\centering
\includegraphics[width=\linewidth]{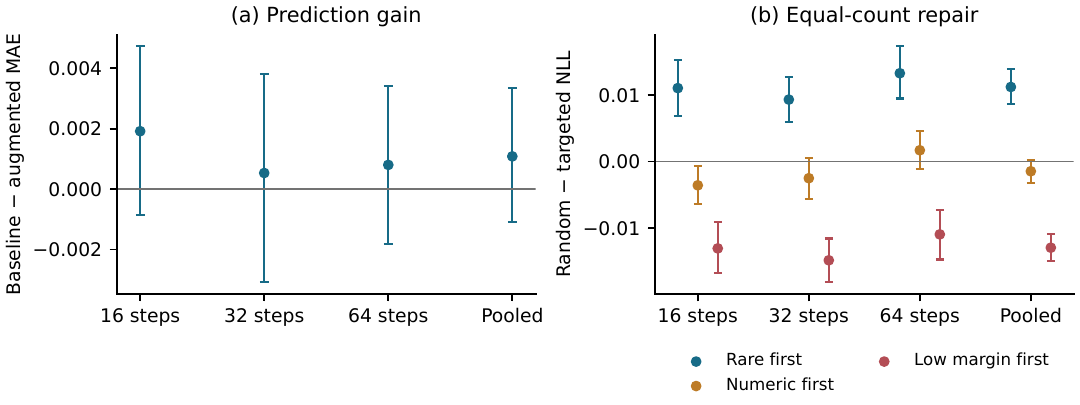}
\caption{Independent AURORA replication on new generated states, scored by GPT-2-Large. Step counts denote updates per 16-token block. Positive values favor (a) the composition-aware predictor or (b) targeted repair over the mean of three random repair arms. Each repair arm corrects exactly the same number of native disagreements within its sequence. Pooled primary contrasts use 97.5\% paired intervals; the other intervals are descriptive 95\%. These are conditional evaluator-score effects, not factuality or deployed speedups.}
\label{fig:aurora-replication}
\end{figure}

\begin{table}[!htbp]
\centering\small
\caption{Frozen prediction on AURORA. The frequency-matched sensitivity substitutes the ELF relative-frequency cutoff for the primary training-count cutoff. All conditions are retained.}
\label{tab:aurora-prediction}
\begin{tabular}{lrrrr}
\toprule
Condition & Sequences & Baseline MAE & Tail MAE & Freq.-matched MAE \\
\midrule
S16 & 512 & 0.0622 & 0.0603 & 0.0606 \\
S32 & 512 & 0.0539 & 0.0534 & 0.0534 \\
S64 & 512 & 0.0546 & 0.0538 & 0.0539 \\
all & 1536 & 0.0569 & 0.0558 & 0.0559 \\
\bottomrule
\end{tabular}
\end{table}

\begin{table}[!htbp]
\centering\small
\caption{Random-minus-targeted repair contrasts in nats per evaluator token. Brackets are paired intervals; pooled GPT-2-Large rare-first uses 97.5\%, and other entries use descriptive 95\%. Positive values favor the targeted policy.}
\label{tab:aurora-repair}
\begin{tabular}{llll}
\toprule
Condition & Policy & GPT-2-Large & OPT-125M \\
\midrule
S16 & rare & +0.0110 [+0.0069, +0.0152] & +0.0128 [+0.0087, +0.0170] \\
S16 & numeric & -0.0036 [-0.0064, -0.0007] & -0.0022 [-0.0050, +0.0006] \\
S16 & lowmargin & -0.0131 [-0.0167, -0.0091] & -0.0118 [-0.0155, -0.0080] \\
S32 & rare & +0.0093 [+0.0060, +0.0127] & +0.0113 [+0.0081, +0.0146] \\
S32 & numeric & -0.0025 [-0.0056, +0.0006] & -0.0020 [-0.0047, +0.0006] \\
S32 & lowmargin & -0.0148 [-0.0182, -0.0116] & -0.0116 [-0.0145, -0.0087] \\
S64 & rare & +0.0133 [+0.0095, +0.0173] & +0.0137 [+0.0102, +0.0174] \\
S64 & numeric & +0.0017 [-0.0011, +0.0046] & +0.0012 [-0.0016, +0.0040] \\
S64 & lowmargin & -0.0109 [-0.0147, -0.0072] & -0.0096 [-0.0132, -0.0059] \\
all & rare & +0.0112 [+0.0086, +0.0139] & +0.0126 [+0.0104, +0.0149] \\
all & numeric & -0.0015 [-0.0032, +0.0002] & -0.0010 [-0.0026, +0.0006] \\
all & lowmargin & -0.0129 [-0.0149, -0.0109] & -0.0110 [-0.0129, -0.0091] \\
\bottomrule
\end{tabular}
\end{table}

\begin{table}[!htbp]
\centering\small
\caption{Every independently fitted AURORA readout and its pooled text score. S16--S64 columns give lexical agreement (\%), excluding native EOS positions. PPL is token-weighted on each evaluator\textquotesingle s common scorable set, with native EOS fixed across arms. These agreements do not identify an intrinsic decoder-complexity ceiling; full fitting-coverage and condition-specific scores accompany the evidence.}
\label{tab:aurora-fits}
\begin{tabular}{lrrrrrr}
\toprule
Readout & Update & S16 & S32 & S64 & GPT-2-L PPL & OPT PPL \\
\midrule
Native & -- & 100 & 100 & 100 & 44.30 & 38.88 \\
Affine 17 & 4096 & 94.43 & 96.16 & 96.62 & 48.54 & 42.50 \\
Affine 29 & 4096 & 94.46 & 96.19 & 96.57 & 48.46 & 42.46 \\
Affine 43 & 4096 & 94.45 & 96.20 & 96.57 & 48.57 & 42.50 \\
\bottomrule
\end{tabular}
\end{table}

\parhead{Regularization controls}
This sensitivity audit was added after the local protocol freeze, before inspection of the new AURORA test-quality outcomes; it is not part of the original confirmatory contrasts. ELF outcomes were already known, so its control comparisons are retrospective. All predictors are fitted only on the saved calibration set. Duplicating each of the three baseline features once halves its effective ridge penalty without adding information. Adding three copies of the baseline error rate quarters that feature\textquotesingle s penalty while leaving the other penalties unchanged. We compare the original augmented predictor against both controls and separately compare unpenalized baseline and augmented fits (standardized OLS with fixed \texttt{rcond=1e-10}). Every condition is retained; intervals here are descriptive paired 95\% intervals.

On pooled ELF, duplicate-all-baseline minus augmented MAE is $+0.00878 [+0.00774, +0.00982]$ nats; duplicate-error-baseline minus augmented MAE is $+0.00877 [+0.00773, +0.00980]$. Under OLS, baseline minus augmented MAE is $+0.00724 [+0.00603, +0.00848]$. The gain remains positive on B/M, whereas the distilled-student failure persists: its OLS contrast is $-0.00874 [-0.01116, -0.00612]$. These controls rule out the tested feature-duplication explanation for the pooled ELF gain without extending it to every checkpoint.

On pooled AURORA, duplicate-all-baseline minus augmented MAE is $+0.00108 [-0.00078, +0.00301]$ nats; duplicate-error-baseline minus augmented MAE is $+0.00108 [-0.00077, +0.00302]$. Under OLS, baseline minus augmented MAE is $+0.00079 [-0.00120, +0.00285]$. All corresponding step-specific intervals also span zero. The controls therefore leave the AURORA prediction gain unresolved. Coefficients, per-sequence errors, algebraic penalty-equivalence checks, and all step/checkpoint-specific contrasts for both architectures accompany the evidence.

\clearpage
\section{Paired Transfer to Conditional Summarization}
\label{app:task-transfer}

\parhead{Cohort and frozen transfer}
We select 1,024 distinct articles from the standard XSum test split by a seeded shuffle, applying the released AURORA input-eligibility rule and excluding exact source-document or ID matches to the XSum training split. Filling the cohort rejects 54 ineligible examples and one training overlap; this is an eligible paired cohort, not a full-test leaderboard reproduction. Eight validation documents are reserved for mechanical pilots; two per architecture check public decoding and input isolation. The cohort, three existing OWT readouts per interface, sampling settings, repair rules, and two primary contrasts are frozen locally before formal generation. This local record is not a public preregistration. Generation bases 61022801 and 61024801 each add the document index, avoiding overlap between the two RNG-seed ranges. The human reference is used for eligibility and evaluation only. This original repair study uses no task-specific readout fitting or predictor calibration. The separate routing test in \secref{app:native-free-routing} subsequently reuses these observed outputs to fit a task-specific risk model and evaluates it on disjoint articles; it does not refit the affine readouts.

ELF-B-XSum uses the released EMA1 checkpoint, 64-step logit-normal ODE, CFG 2, self-conditioning scale 1, a T5 source prefix of at most 1,024 tokens, a full 1,088-position canvas, and 64 target positions. AURORA-XSum uses the released checkpoint 300000 and shared autoencoder, its public conditional sampler with 32 ODE updates per 16-token block, CFG 2 in the released sigma window $[0.9,1.0]$, repaint, and a nominal 64-token target budget. Repainting a partial prefix block can round the generated canvas to 64--79 tokens. Source prefixes follow the released GPT-2 preparation with a 960-token limit. Full contextual native decoding precedes target slicing; affine maps see decoder-input latents. Source truncation and tokenization differ, so this is not a matched-compute comparison.

\parhead{Mechanical compatibility and audit history}
The current Transformers T5 encoder cannot accept the release's three-dimensional mask. A two-dimensional source-key mask gives each source query the same allowed keys; target-query encodings are discarded by the conditioning mask. Changing all non-source placeholder IDs leaves source encodings unchanged in the pilot. Native ELF decoding matches the public routine exactly. AURORA uses the public sampler unchanged and checks captured token labels against its returned tokens. A runtime assertion after 73 saved AURORA generations exposed the official \texttt{first\_eos} convention: leading EOS tokens are skipped before searching for a body-ending EOS. Before any formal task scores existed, we archived the original freeze, recorded this boundary correction, and resumed the same seeds. All saved trajectories are retained. SentencePiece 0.2.1 was subsequently added to the private environment to load the evaluator's tokenizer before any SummaC score; the upstream default fast tokenizer is retained. The obsolete \texttt{truncation\_strategy} keyword is mapped to the modern \texttt{truncation} argument, preserving the explicit source-only truncation strategy. These implementation events and original files accompany the evidence.

\parhead{Matched interventions and stopping}
Using affine seed 17, each arm changes exactly $\lfloor e/2\rfloor$ native lexical disagreements: ascending original OWT target count, numeric first, ascending native margin, or three seeded random permutations. Ties are randomized by fixed policy-specific seeds. Leading AURORA EOS positions and ELF PAD positions are excluded from lexical counts. The operational non-PAD/non-EOS mask retains 12 ELF unknown-token positions, all agreeing with affine seed 17; these positions do not enter the repair budget. Main arms share the native start/end and clamp nonlexical positions identically. An interior special token predicted by a readout is skipped in text decoding rather than used to shorten this common boundary. The separate autonomous affine arm applies the released stopping convention to its own output. Because the original readout fits excluded EOS labels, this is a stress test of unmodified lexical probes, not a comparison to a readout trained to terminate. Native-token edit counts also need not equal word-level or evaluator-token edit counts. Native labels are oracle targets, not human truth; no corrected token feeds back into the sampler. The displayed random arms A/B/C correspond to the archived arm identifiers random11/random23/random37; these suffixes are identifiers, while the frozen RNG derives each permutation from the generation seed, document, policy index, and family. All 4,096 generations enter the primary task metrics, including zero-edit and empty-output cases. The predeclared editable subset requires at least two disagreements.

\parhead{Evaluation and inference}
ROUGE-1/2/L F1 uses stemming. BERTScore uses RoBERTa-large layer 17, no IDF or baseline rescaling, the official slow tokenizer, and FP32. SummaC-ZS uses the released VitaminC/MNLI ALBERT-xlarge model, sentence granularity, at most 100 source sentences, 500-token sentence-pair inputs, and FP32. Its default score averages, over summary sentences, maximum source entailment minus maximum source contradiction. This signed score lies in $[-1,1]$ and is not a calibrated probability of factual correctness; negative absolute scores are valid. Empty or too-short sentence sets receive the upstream zero score. These metrics respectively measure overlap, contextual reference similarity, and an NLI-based source-consistency proxy; none replaces human evaluation. BERTScore token lengths and metric truncation flags are retained. GPT-2-Large NLL is a descriptive auxiliary measure: prepend one EOS to summary text, use BF16 inference and FP32 losses, and compare per-summary token means. Each architecture's ten main arms share the scorable set; pooled NLL further requires all four family/seed replicas. Autonomous NLL uses its own eligibility and is not a common-set causal comparison.

The paired-document bootstrap resamples the same 1,024 articles across two seeds and two architectures (4,000 draws, seed 61022901). The pooled primary estimand weights families and seeds equally. The two rare-first-minus-random primary contrasts use Bonferroni 97.5\% intervals; all other intervals are descriptive 95\%. Editable-subset intervals condition on the realized edit budget and use a generation-weighted ratio within document clusters. All ROUGE contrasts, every fitting seed, all output texts, exact edit masks, and scalar metric traces remain in the ancillary evidence. Checkpoint and task are both changed relative to OWT; these are not repeated denoiser-training experiments.

\begin{figure}[!htbp]
\centering
\includegraphics[width=\linewidth]{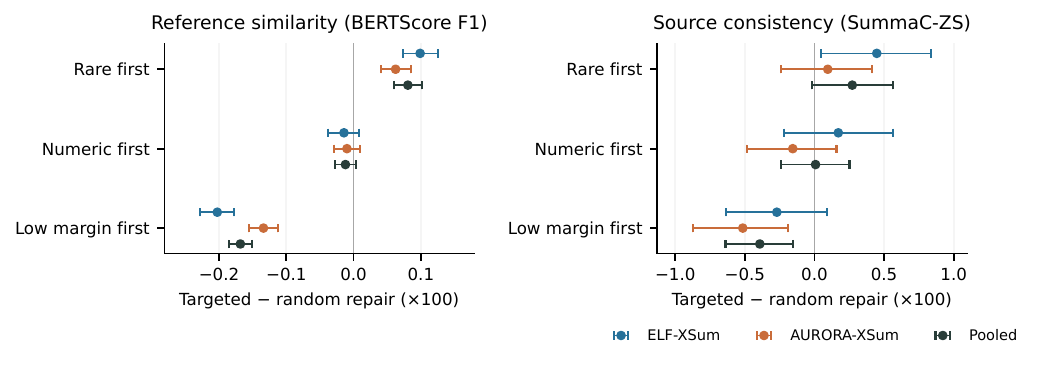}
\caption{Matched repairs on the same XSum documents. Each point compares one targeted policy with the mean of three random policies at exactly the same correction count. Positive values favor targeted repair. Pooled rare-first intervals are the two prespecified 97.5\% intervals; the others are descriptive 95\%. All generations, including zero-edit cases, contribute. Automatic reference and source consistency scores are shown on a $100\times$ scale.}
\label{fig:task-transfer}
\end{figure}

\begin{table}[!htbp]
\centering\scriptsize
\setlength{\tabcolsep}{3pt}
\caption{Original task-transfer arms, averaged over two generation seeds. Agreement is micro-averaged native lexical agreement (percent); ROUGE and BERTScore F1 and SummaC-ZS are multiplied by 100. NLL is in nats per scored token, averaged over summaries. Native-boundary arms share the NLL-eligible set within each generation; the own-stop arm uses its own eligibility. The three affine maps are the previously frozen OWT fits.}
\label{tab:task-scores}
\begin{tabular}{lrrrrrrr}
\toprule
Arm & Agree. & R1 & R2 & RL & BERT & SummaC & NLL \\
\midrule
\multicolumn{8}{l}{\textbf{ELF-XSum}} \\
Native & 100.00 & 34.42 & 11.01 & 26.31 & 88.57 & -39.38 & 4.482 \\
Affine 17 & 89.06 & 31.51 & 9.21 & 24.31 & 87.71 & -40.48 & 5.020 \\
Affine 29 & 88.93 & 31.48 & 9.18 & 24.26 & 87.72 & -40.18 & 5.008 \\
Affine 43 & 89.00 & 31.56 & 9.26 & 24.31 & 87.72 & -40.68 & 5.007 \\
Random A & 93.65 & 32.73 & 9.90 & 25.13 & 88.06 & -39.87 & 4.819 \\
Random B & 93.65 & 32.66 & 9.89 & 25.06 & 88.05 & -39.90 & 4.823 \\
Random C & 93.65 & 32.70 & 9.82 & 25.09 & 88.06 & -40.02 & 4.826 \\
Rare first & 93.65 & 33.31 & 10.20 & 25.52 & 88.15 & -39.48 & 4.777 \\
Numeric first & 93.65 & 32.66 & 9.88 & 25.09 & 88.04 & -39.76 & 4.839 \\
Low margin first & 93.65 & 31.91 & 9.46 & 24.60 & 87.85 & -40.20 & 4.937 \\
Affine 17, own stop & -- & 31.52 & 9.21 & 24.32 & 87.12 & -40.10 & 4.951 \\
\midrule
\multicolumn{8}{l}{\textbf{AURORA-XSum}} \\
Native & 100.00 & 36.90 & 13.81 & 28.96 & 89.72 & -40.33 & 3.508 \\
Affine 17 & 90.11 & 33.43 & 11.26 & 26.39 & 88.79 & -41.81 & 4.200 \\
Affine 29 & 90.12 & 33.39 & 11.26 & 26.38 & 88.78 & -41.48 & 4.202 \\
Affine 43 & 90.15 & 33.49 & 11.31 & 26.45 & 88.81 & -41.24 & 4.195 \\
Random A & 94.08 & 34.73 & 12.12 & 27.35 & 89.14 & -41.23 & 3.967 \\
Random B & 94.08 & 34.78 & 12.12 & 27.38 & 89.16 & -40.89 & 3.960 \\
Random C & 94.08 & 34.73 & 12.12 & 27.33 & 89.16 & -40.91 & 3.957 \\
Rare first & 94.08 & 35.03 & 12.35 & 27.56 & 89.22 & -40.91 & 3.922 \\
Numeric first & 94.08 & 34.63 & 12.08 & 27.26 & 89.14 & -41.17 & 3.973 \\
Low margin first & 94.08 & 34.36 & 11.81 & 27.10 & 89.02 & -41.53 & 4.048 \\
Affine 17, own stop & -- & 28.02 & 9.04 & 20.54 & 86.52 & -38.31 & 4.574 \\
\bottomrule
\end{tabular}
\end{table}

\begin{table}[!htbp]
\centering\scriptsize
\setlength{\tabcolsep}{3pt}
\caption{Targeted-minus-random task effects and paired-document intervals. BERTScore and SummaC-ZS use a $100\times$ scale. NLL reverses the sign (random minus targeted), so positive always favors targeted repair. Only the two pooled rare-first task intervals are primary 97.5\% intervals; all other intervals are descriptive 95\%.}
\label{tab:task-effects}
\begin{tabular}{llrrr}
\toprule
Group & Policy & BERTScore F1 & SummaC-ZS & NLL \\
ELF & Rare first & +0.099 [+0.073, +0.126] & +0.448 [+0.049, +0.837] & +0.0455 [+0.0315, +0.0602] \\
ELF & Numeric first & -0.014 [-0.037, +0.008] & +0.171 [-0.217, +0.565] & -0.0161 [-0.0292, -0.0035] \\
ELF & Low margin first & -0.203 [-0.228, -0.178] & -0.270 [-0.634, +0.093] & -0.1144 [-0.1284, -0.1007] \\
AURORA & Rare first & +0.062 [+0.041, +0.085] & +0.095 [-0.243, +0.413] & +0.0393 [+0.0256, +0.0535] \\
AURORA & Numeric first & -0.010 [-0.029, +0.009] & -0.156 [-0.487, +0.158] & -0.0112 [-0.0237, +0.0007] \\
AURORA & Low margin first & -0.134 [-0.156, -0.113] & -0.515 [-0.873, -0.190] & -0.0863 [-0.0997, -0.0733] \\
Pooled & Rare first & +0.081 [+0.060, +0.101] & +0.271 [-0.018, +0.561] & +0.0425 [+0.0327, +0.0529] \\
Pooled & Numeric first & -0.012 [-0.028, +0.003] & +0.008 [-0.238, +0.252] & -0.0136 [-0.0225, -0.0049] \\
Pooled & Low margin first & -0.168 [-0.186, -0.150] & -0.393 [-0.639, -0.152] & -0.1005 [-0.1105, -0.0909] \\
\bottomrule
\end{tabular}
\end{table}

\begin{table}[!htbp]
\centering\scriptsize
\setlength{\tabcolsep}{3pt}
\caption{Coverage across 2,048 generations per architecture. Editable means at least two native lexical disagreements. No end means no body-ending native EOS within the canvas; lead EOS counts outputs with skipped leading EOS. Earlier/later compare affine and native stop positions. NLL gives the common scorable count for the ten native-boundary arms. All rows remain in ROUGE, BERTScore, and SummaC analyses.}
\label{tab:task-coverage}
\begin{tabular}{lrrrrrrr}
\toprule
Family & Editable & Empty & No end & Lead EOS & Earlier & Later & NLL \\
ELF & 1593 & 1 & 7 & 0 & 0 & 2041 & 2047 \\
AURORA & 1414 & 0 & 2 & 6 & 0 & 2046 & 2048 \\
\bottomrule
\end{tabular}
\end{table}

The pooled test finds a positive effect on reference similarity and an unresolved effect on source consistency. In particular, the pooled source-consistency interval includes zero, and AURORA's rare-first source-consistency interval also includes zero. Thus the reference-similarity result does not establish a source-consistency benefit across architectures. The full per-family and per-policy intervals in \tabref{tab:task-effects} delimit any pooled conclusion. Readability, agreement, language-model likelihood, reference similarity, and source consistency are distinct measured outcomes. The intervention identifies consequences of choosing among native corrections under fixed trajectories; it does not establish a learned or inexpensive repair algorithm, nor a human-quality gain.

\section{Residual Consequences, Calibration, and Context Controls}
\label{app:deep-mechanism}

This extension follows inspection of the original prospective outcomes. Its accounting and support diagnostics are post hoc; its new models and policies use locally recorded rules before their own evaluation on previously examined cohorts. None is presented as a new independent confirmation. The original positive, null, and adverse results remain unchanged.

% Exploratory appendix insertion. This block is self-contained and does not
% replace or reclassify the original confirmatory task tests.
\subsection{What contributes to matched-repair score differences?}
\label{sec:deep-repair-accounting}

The rare-first advantage could reflect token-frequency preferences, effects on
subsequent text, or changes in the evaluator's tokenization. We distinguish
these explanations using the cached XSum outputs and GPT-2-Large token losses.
This is a \emph{post hoc accounting analysis}: no new output, repair policy, or
outcome-dependent policy selection is introduced. All intervals below are
\emph{descriptive} 95\% paired-document bootstrap intervals using the original
4,000 resamples, without multiplicity correction. The pooled NLL analysis uses
exactly the original 1,023 documents with all four architecture--seed replicas
scorable (4,092 generations).

\paragraph{An exact decomposition.}
Let $b$ denote the affine baseline text and $r$ a repaired text, with evaluator
token counts $n_b,n_r>0$, token losses $\ell_i(b),\ell_j(r)$, and
$S(r)=\sum_j\ell_j(r)$. A deterministic monotone exact-token alignment gives
matched pairs $M$ and unmatched index sets $U_b,U_r$. Then
\begin{align}
 \overline\ell(b)-\overline\ell(r)
 &=D(b,r)+C(b,r)+L(b,r),\label{eq:deep-repair-identity}\\
 D(b,r)&=\frac{\sum_{i\in U_b}\ell_i(b)-\sum_{j\in U_r}\ell_j(r)}{n_b},\nonumber\\
 C(b,r)&=\frac{\sum_{(i,j)\in M}[\ell_i(b)-\ell_j(r)]}{n_b},\nonumber\\
 L(b,r)&=S(r)\left(\frac1{n_b}-\frac1{n_r}\right).\nonumber
\end{align}
Here $D$ is the unmatched-token term, $C$ the loss change on aligned unchanged
tokens, and $L$ the denominator correction. Subtracting the mean of the three
random arms from a targeted arm's gain preserves the identity. We use
\texttt{SequenceMatcher(autojunk=False)} on evaluator token IDs; repeated tokens
can admit other alignments. Thus the terms are loss accounting, not separately
randomized mediators. Equal native-token repair budgets need not imply equal
evaluator-token edit counts.

\begin{table}[t]
\centering
\scriptsize
\setlength{\tabcolsep}{3pt}
\begin{tabular}{lccc}
\toprule
Rare-first minus random gain & Pooled & ELF & AURORA\\
\midrule
Total NLL gain & $+0.0425\;[+0.0327,+0.0529]$ & $+0.0455\;[+0.0315,+0.0602]$ & $+0.0393\;[+0.0256,+0.0535]$ \\
Unmatched-token term $D$ & $+0.0138\;[+0.0059,+0.0215]$ & $+0.0079\;[-0.0045,+0.0198]$ & $+0.0196\;[+0.0094,+0.0296]$ \\
Unchanged-token term $C$ & $+0.0220\;[+0.0157,+0.0284]$ & $+0.0244\;[+0.0164,+0.0321]$ & $+0.0194\;[+0.0097,+0.0293]$ \\
Normalization term $L$ & $+0.0066\;[+0.0039,+0.0095]$ & $+0.0132\;[+0.0085,+0.0185]$ & $+0.0002\;[-0.0018,+0.0021]$ \\
Reference unigram gain & $-0.0275\;[-0.0303,-0.0247]$ & $-0.0326\;[-0.0374,-0.0279]$ & $-0.0224\;[-0.0255,-0.0194]$ \\
NLL minus unigram gain & $+0.0700\;[+0.0596,+0.0806]$ & $+0.0781\;[+0.0637,+0.0929]$ & $+0.0616\;[+0.0471,+0.0765]$ \\
\bottomrule
\end{tabular}
\caption{Post hoc repair accounting in nats per evaluator token; positive
values favor rare-first. Brackets give descriptive 95\% intervals. The first
three component rows sum to the total point estimate; their interval endpoints
need not sum. Family analyses retain 2,047 scorable ELF and 2,048 AURORA
generations, each spanning 1,024 documents; the pooled analysis requires the
common 1,023-document cohort. The unigram distribution is a fixed external
reference prior, not the scorer's training distribution.}
\label{tab:deep-repair-accounting}
\end{table}

The unchanged-token term is positive in both architectures
(Table~\ref{tab:deep-repair-accounting}). In the pooled analysis, most of this
term occurs within one to four evaluator tokens after an edit:
$+0.0203\,[0.0143,0.0263]$ nats/token. Output edits therefore affect the scorer's
assessment of subsequent unchanged text. This finding concerns autoregressive
\emph{scorer context}; it does not identify a causal mechanism of diffusion
trajectories, representation learning, or factual correctness.

\paragraph{A fixed unigram null.}
For the GPT-2 vocabulary, define
\begin{equation}
 q(v)=\frac{c_v+1}{\sum_u c_u+|\mathcal V|},\qquad
 U(t)=-\frac1{n_t}\sum_{i=1}^{n_t}\log q(t_i),
 \label{eq:deep-repair-unigram}
\end{equation}
using the original AURORA OWT readout-training counts $c_v$, with additive
smoothing fixed at one. The same reference prior scores both families after
evaluator tokenization. Rare-first increases $U$ while decreasing contextual
NLL in both families. The pooled gain in $\overline\ell-U$ is
$+0.0700\,[0.0596,0.0806]$. Consequently, preference for frequent tokens under
this particular unigram prior cannot account for the benefit. This arithmetic
contrast is not a calibrated quality metric and does not eliminate every
possible scorer bias.

\paragraph{Alignment and denominator sensitivity.}
We additionally require the baseline and \emph{all six} repaired texts to have
the same evaluator-token length, and split losses by exact same-position token
equality. This removes the denominator term and alignment choice. In the
resulting 2,987-generation, 996-document pooled subset, the total rare-first gain
is $+0.0282\,[0.0187,0.0376]$, including an unchanged-position term of
$+0.0167\,[0.0101,0.0235]$. This term is positive separately for ELF
($+0.0164\,[0.0076,0.0247]$; 1,155 generations) and AURORA
($+0.0168\,[0.0072,0.0265]$; 1,834 generations). These family subsets retain
733 and 972 documents, respectively. Because eligibility depends on the
resulting texts, this sensitivity analysis does not replace the primary cohort.

We further require all six policies to change the \emph{same number} of
same-position evaluator-token IDs. This rule was recorded after the preceding
results were known and before extracting edit counts. It retains 2,619
generations from 958 pooled documents, including 1,532 native-editable cases
and 1,531 with positive evaluator edit dose; zero-dose cases remain included.
The pooled total gain is $+0.0178\,[0.0088,0.0270]$ and the unchanged-position
term is $+0.0135\,[0.0068,0.0205]$. The ELF total is unresolved
($+0.0108\,[-0.0018,0.0234]$; 966 generations, 644 documents), whereas AURORA
retains a positive total ($+0.0217\,[0.0091,0.0346]$; 1,655 generations,
911 documents). The unchanged-position terms remain positive separately:
ELF $+0.0133\,[0.0049,0.0218]$, AURORA $+0.0135\,[0.0039,0.0233]$.
This stricter output-dependent subset addresses evaluator edit dose, while
retaining the weaker ELF total and the same post hoc limitations.

\paragraph{Other policies and selection differences.}
Table~\ref{tab:deep-repair-controls} retains the other two targeted policies;
all three random arms enter through their prespecified mean. Low-margin-first
improves the unigram cost yet worsens contextual NLL, with adverse
unchanged-token terms in both architectures. Numeric-first has a negative
pooled NLL contrast; the AURORA-only interval includes zero. Thus neither
correcting an equal number of native disagreements nor improving the unigram
cost guarantees a better contextual score.

\begin{table}[t]
\centering
\scriptsize
\setlength{\tabcolsep}{3pt}
\begin{tabular}{lccc}
\toprule
Pooled targeted-minus-random gain & NLL & Unchanged-token term & Unigram cost\\
\midrule
Numeric-first & $-0.0136\;[-0.0225,-0.0049]$ & $-0.0061\;[-0.0120,-0.0006]$ & $+0.0010\;[-0.0015,+0.0034]$ \\
Low-margin-first & $-0.1005\;[-0.1105,-0.0909]$ & $-0.0387\;[-0.0452,-0.0325]$ & $+0.0133\;[+0.0107,+0.0161]$ \\
\bottomrule
\end{tabular}
\caption{Other existing repair policies on the same 1,023-document cohort.
Units and descriptive intervals follow Table~\ref{tab:deep-repair-accounting}.
All policy outputs and family-level decompositions are retained in the
accompanying evidence.}
\label{tab:deep-repair-controls}
\end{table}

Rarity is not isolated as the causal selection feature. Among editable
generations, rare-first selects more positions where the native token occurs
in the supplied source prefix but the affine alternative does not
($+9.69\,[8.26,11.09]$ percentage points versus random); the corresponding
reference statistic---native token present and affine alternative absent---is $+7.14\,[5.80,8.53]$ points. Native
margins are also larger within each architecture. These token-occurrence and
confidence differences co-vary with rarity; occurrence is not entailment, and
raw margins are not comparable across architectures. The accounting result
therefore narrows two explanations---a specified unigram null and effects
restricted to edited-token self-scores---while leaving the causal contribution
of individual selection features and the original unresolved factual-consistency
result unchanged.

\subsection{Separating prediction ranking from absolute calibration}
\label{app:deep-prediction-diagnosis}

The prospective prediction failures warrant a more specific diagnosis. We
therefore analyze the saved calibration features and frozen test predictions
\emph{after observing the original outcomes}. All analyses in this subsection
are exploratory; they do not replace the prespecified tests or turn their
negative results into successes. No original readout, predictor, test sequence,
or outcome is changed. We distinguish association with the ordering of
sequence-level excess NLL from calibration of its absolute value.

\paragraph{An exact error decomposition.}
Let $r_i=\hat q_i-q_i$, where $q_i$ is affine minus native evaluator NLL per
token under the common native boundary. For each fixed readout,
\begin{equation}
 \operatorname{MSE}(\hat q,q)
 =\bar r^{\,2}+\frac{1}{n}\sum_{i=1}^{n}(r_i-\bar r)^2.
 \label{eq:deep-prediction-decomposition}
\end{equation}
We compute these terms separately for the three readouts and then average.
Table~\ref{tab:deep-prediction-calibration} shows that the PD failure differs
from the unresolved AURORA gain. On PD, composition features increase signed
bias from $+0.00261$ to $+0.02330$ nats. The signed differences (composition minus baseline) are
$\Delta\operatorname{MSE}=+4.05\times10^{-4}$,
$\Delta\bar r^{\,2}=+5.78\times10^{-4}$, and
$\Delta\operatorname{Var}(r)=-1.73\times10^{-4}$. Their descriptive 95\% paired intervals, in units of
$10^{-4}$, are $[1.42,6.79]$, $[4.08,7.62]$, and $[-3.70,0.15]$ for the
signed total, squared-bias, and centered-variance differences, respectively.
Thus the observed deterioration is accounted for by a shift in prediction
level, partly offset by an uncertain improvement in centered residuals.
The signed-bias direction repeats in both generation seeds. Mean per-readout
Spearman correlation increases from $0.122$ to $0.266$, a descriptive ranking
association. Removing the target-set mean error would be an oracle operation,
not a deployable recalibration method.

\begin{table}[t]
\centering
\small
\begin{tabular}{lrrrr}
\toprule
& \multicolumn{2}{c}{MAE} & \multicolumn{2}{c}{Signed bias} \\
Condition & Baseline & Composition & Baseline & Composition \\
\midrule
ELF-B32 & 0.03309 & 0.03059 & $-0.00102$ & $-0.00053$ \\
ELF-M32 & 0.07002 & 0.04322 & $-0.06490$ & $-0.02422$ \\
ELF-PD8 & 0.04170 & 0.04461 & $+0.00261$ & $+0.02330$ \\
AURORA, pooled & 0.05693 & 0.05585 & $+0.00483$ & $+0.00294$ \\
\bottomrule
\end{tabular}
\caption{Post hoc diagnosis of the unchanged frozen predictors. Bias is
prediction minus observed excess NLL; all entries are in nats per evaluator
token. Each condition contains 512 generated sequences; pooled AURORA contains
1,536. Readout fits are averaged within this diagnostic. The original
prospective intervals and failures remain the inferential reference.}
\label{tab:deep-prediction-calibration}
\end{table}

\paragraph{Support stress and limited incremental information.}
For ELF fit 17, PD's mean disagreement rate is $3.09$ calibration standard
deviations above the ELF-B calibration mean; $51.6\%$ of its sequences exceed
the calibration 99th-percentile Mahalanobis radius. This radius uses the six
calibration-standardized features with $0.001I$ covariance regularization.
These are support diagnostics, not an identification of a population shift or
the causal effect of distillation: checkpoint, training, step count, and
guidance change jointly.

AURORA has small average bias and little established incremental predictive
gain. Its original rare-error feature is nearly redundant with aggregate
error rate: calibration correlation is approximately $0.995$, and only about
$1\%$ of its variance remains after linear adjustment for baseline features.
The preexisting relative-frequency-matched control raises this residual
fraction to $16$--$19\%$, yet its prediction benefit also remains unresolved.
Threshold saturation is therefore a limitation of the original feature,
not a sufficient explanation of the failed prediction hypothesis.
Nor is the ELF predictor simply a positive rare-word penalty: its standardized
rare-error coefficients vary in sign across fits, whereas low-margin-error
coefficients range from $-0.0266$ to $-0.0238$ after adjustment. These are
conditional associations among correlated features, not causal coefficients.

\paragraph{A structural constraint does not by itself repair transfer.}
Under an identical output boundary, zero token disagreements imply zero
excess NLL. We test two new calibration-only predictors enforcing this identity:
$\hat q=e\,g(z)$, where $e$ is disagreement rate and $g$ is a ridge-linear
function of length and native NLL; the composition version additionally uses
rare, numeric, and low-margin fractions among disagreements. The two models,
fixed penalty $10$, calibration-only scaling, and five-fold calibration
evaluation were specified before computing their new test outcomes; the old
test results had already been observed. The design and target are divided by $\sqrt{\operatorname{mean}_{\rm cal}(e^2)}$ before applying the fixed ridge penalty. No target-set fitting or per-condition
tuning is used. The constraint fails to improve transfer: composition-model
MAE worsens from $0.04461$ to $0.09137$ on PD and from $0.05585$ to $0.05840$
on pooled AURORA. It also performs worse in calibration cross-validation.
We retain this negative result and make no algorithmic improvement claim.

The resulting distinction is narrower and more informative than universal
quality prediction: residual composition, absolute calibration, and
equal-count intervention effects are different empirical questions. Repair
ordering can replicate while a sequence-level forecast remains miscalibrated
or adds insufficient information. These oracle-feature diagnostics do not
identify the full mechanism of generation.

\subsection{Which known disagreements are worth correcting?}
\label{app:deep-margin-policies}

A low native margin can mark disagreement-prone positions without ranking the
consequences of disagreements that are already known. To separate these
questions, we freeze three additional oracle repair rules on the existing
4,096 XSum outputs before scoring their new texts. This is an exploratory
extension designed after observing the original study, not a new confirmatory
test. Every rule corrects exactly $\lfloor e/2\rfloor$ native/affine
disagreements per output, with the original native boundary. All outputs are
retained, including 1,089 with zero correction budget.

\paragraph{Opposite margin ranking and adjacent-margin blocking.}
High-margin-first corrects the largest native top-one/top-two logit gaps,
opposite to the original low-margin-first rule. The other rules sort mismatches
by ascending native margin and form disjoint adjacent pairs. Each pair
contributes its lower-count native token to one repair arm and its higher-count
token to the other. Count ties use fixed random opposite choices; margin ties
use seeded random ordering. When $e$ is odd, the highest-margin mismatch is
unpaired in both arms. No threshold, caliper, pair exclusion, or policy tuning
is introduced. Counts retain their original interface-specific definitions.

The comparison comprises 4,723 pairs. This is \emph{adjacent-margin blocking},
not exact confidence matching. Mean absolute within-pair margin gaps are
$15.43$ for ELF and $12.50$ for AURORA, in their respective native-logit units.
Lower-count selections still have larger mean margins by $4.70$ and $3.48$,
and greater native-token presence in the actual source prefix by $11.31$ and
$8.75$ percentage points. Count ties occupy $2.79\%$ and $27.46\%$ of pairs.
All residual position, count, margin, and source-overlap differences are
retained. Consequently, a benefit within these blocks cannot identify a
rarity effect independent of confidence or source overlap.

\paragraph{Conditional repair consequences.}
Table~\ref{tab:deep-margin-results} reports the results. The lower-count
member yields a pooled BERTScore gain of $0.0788$ points with paired 97.5\%
interval $[0.0505,0.1081]$, and an NLL benefit of $0.0267$ nats with interval
$[0.0085,0.0436]$. These are the two locally fixed diagnostic contrasts.
ELF's count-pair NLL contrast remains unresolved
($0.0142$, descriptive 95\% interval $[-0.0070,0.0338]$).
High-margin-first improves over random repair by $0.1603$ BERTScore points
($[0.1422,0.1789]$) and $0.1073$ nats ($[0.0973,0.1175]$), with positive
descriptive 95\% intervals in each architecture. Its pooled advantages over
low-margin-first are $0.3285$ BERTScore points and $0.2078$ nats.

\begin{table}[t]
\centering
\small
\begin{tabular}{llrr}
\toprule
Repair contrast & Interface & $\Delta$BERTScore & NLL benefit \\
\midrule
Paired lower vs. higher count & ELF & 0.0911 & 0.0142 \\
& AURORA & 0.0665 & 0.0392 \\
& Pooled & 0.0788 & 0.0267 \\
\midrule
High margin vs. random & ELF & 0.1985 & 0.1315 \\
& AURORA & 0.1221 & 0.0830 \\
& Pooled & 0.1603 & 0.1073 \\
\midrule
High margin vs. rare-first\textsuperscript{*} & ELF & 0.0996 & 0.0859 \\
& AURORA & 0.0597 & 0.0436 \\
& Pooled & 0.0796 & 0.0648 \\
\bottomrule
\end{tabular}
\caption{Equal-budget native-oracle repairs on unchanged generated states.
Positive values favor the first rule. BERTScore F1 is multiplied by 100;
NLL benefit is comparator minus selected-rule NLL per evaluator token.
The bootstrap unit is the common document, paired across both generation
seeds and architectures. BERTScore retains 1,024 documents. NLL retains the
1,023-document intersection across all required arms and replicas; document
719 is excluded because an ELF output is empty. The starred comparison was
added after observing the high-margin/random result and is descriptive.}
\label{tab:deep-margin-results}
\end{table}

After inspecting the high-margin/random result, we separately record a
descriptive comparison against the original rare-first rule, without changing
the two primary diagnostic contrasts. High-margin-first improves pooled
BERTScore by $0.0796$ points ($[0.0617,0.0985]$) and NLL by $0.0648$ nats
($[0.0546,0.0752]$); all four architecture-specific 95\% intervals are
positive. This comparison was not selected before the preceding results.

These interventions distinguish \emph{error incidence} from \emph{correction
consequence}: conditional on a native/affine disagreement, restoring a
high-margin native decision can matter more than restoring an uncertain one.
Frequency is therefore one correlated ranking, rather than a demonstrated
optimal correction rule. The experiment requires native labels and margins;
affine-only confidence has not been validated as a substitute. It neither
provides deployable acceleration nor contradicts confidence-based remasking
methods that address different decisions. New policies are evaluated only by
BERTScore and NLL, so the results establish neither factuality nor human
quality improvements.

\begin{figure}[t]
  \centering
  \includegraphics[width=\linewidth]{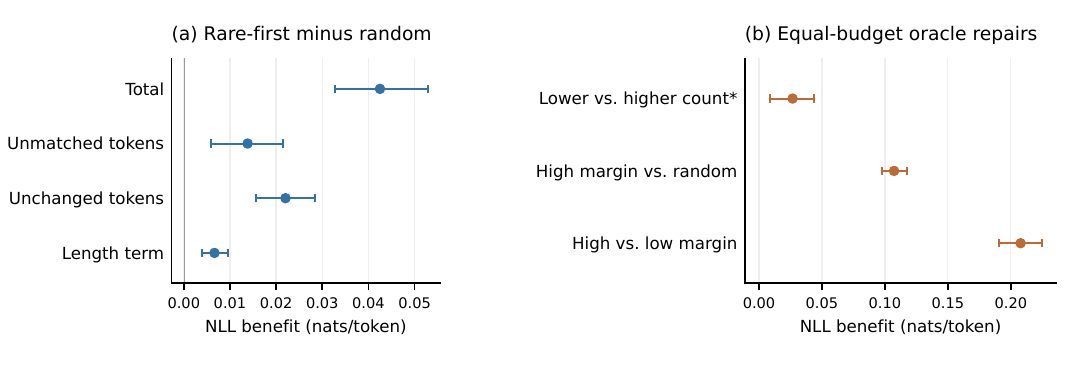}
  \caption{Exploratory diagnostics on paired XSum generations. (a) Exact accounting of rare-first's NLL advantage over mean random repair: unmatched-token, aligned unchanged-token, and length-normalization terms sum to the total. (b) Equal-budget native-oracle policies. The starred count contrast uses lower- versus higher-count members of adjacent-margin pairs; residual margin and source-overlap imbalance remains. All estimates use the same 1,023-document NLL cohort, pairing generation seeds and architectures. Bars show paired-document bootstrap intervals: 95\% descriptive intervals except 97.5\% for the locally fixed count contrast. Positive values favor the first policy. These are scorer consequences of output edits, not an identification of the generative mechanism.}
  \label{fig:deep-residual-diagnostics}
\end{figure}

\subsection{An equal-parameter context intervention and its scaling control}
\label{sec:context_adapter_followup}
The context-swap stress test establishes that the native decoder can use other
positions; it does not establish that a cheap context feature improves the
readout. We therefore freeze each interface's seed-17 affine head and fit
\[
 \ell_i=W x_i+b+UV\bigl((q_i-\mu)\oslash s\bigr)+c,
 \qquad \operatorname{rank}(UV)\leq 8.
\]
The three equal-parameter arms supply $q_i$ as the focal state, the mean of up to
four preceding full-canvas states, or a context feature from another sequence.
Context excludes the focal position and uses a zero vector at position zero;
on XSum it can include available source-prefix states. No native labels or
margins select these features. The focal arm remains affine in $x_i$ and
controls additional fitting rather than a larger function class.
All arms share the sampled training/validation pools, update budgets, optimizer,
and two fit seeds. Native-agreement validation selects among steps
$0,256,512,1024$, with the unchanged head eligible at step zero. Neither XSum
nor the reused OWT tests select weights or checkpoints.
All entries below use the seed-17 base and macro agreement. ELF OWT pools
1,536 generations from B32/M32/PD8; its $95.307\%$ baseline is therefore a
different estimand from the earlier ELF-B-only $97.17\%$ three-fit micro
agreement. AURORA OWT pools 1,536 generations across three step budgets.
XSum likewise uses seed-17 macro agreement, rather than the earlier three-fit
micro aggregate. These differences are not changes to the original results.

\begin{table}[t]
\centering
\small
\setlength{\tabcolsep}{4pt}
\begin{tabular}{llrrrr}
\toprule
Scale & Interface / test & Base & Focal & Context & Shuffled\\
\midrule
Coordinate & ELF / OWT & 95.307 & 95.213 & 95.314 & 95.289 \\
Coordinate & ELF / XSum & 89.095 & 89.001 & 89.065 & 89.068 \\
Coordinate & AURORA / OWT & 95.738 & 95.725 & 95.687 & 95.697 \\
Coordinate & AURORA / XSum & 90.121 & 63.638 & 27.556 & 27.488 \\
Global & ELF / OWT & 95.307 & 95.238 & 95.310 & 95.301 \\
Global & ELF / XSum & 89.095 & 89.067 & 89.045 & 89.085 \\
Global & AURORA / OWT & 95.738 & 95.725 & 95.682 & 95.711 \\
Global & AURORA / XSum & 90.121 & 90.095 & 90.037 & 90.096 \\
\bottomrule
\end{tabular}
\caption{Macro native lexical agreement (\%) for equal-parameter context
controls. Each adapted arm averages two selected fits; the unchanged seed-17
head is the base. Both exploratory rounds, including the original scaling
failure, are retained. These are fidelity measurements under the native output
boundary, not task correctness.}
\label{tab:context_adapter_means}
\end{table}

\begin{table}[t]
\centering
\small
\setlength{\tabcolsep}{3pt}
\begin{tabular}{llccc}
\toprule
Scale & Interface / test & \shortstack{Context$-$\\focal} & \shortstack{Context$-$\\shuffled} & \shortstack{Context$-$\\base}\\
\midrule
Coordinate & ELF / OWT & \shortstack{+0.101\\{[+0.091, +0.110]}} & \shortstack{+0.024\\{[+0.017, +0.031]}} & \shortstack{+0.007\\{[-0.0004, +0.014]}} \\
Coordinate & ELF / XSum & \shortstack{+0.064\\{[+0.009, +0.119]}} & \shortstack{-0.003\\{[-0.046, +0.039]}} & \shortstack{-0.030\\{[-0.073, +0.015]}} \\
Coordinate & AURORA / OWT & \shortstack{-0.038\\{[-0.048, -0.028]}} & \shortstack{-0.010\\{[-0.021, +0.001]}} & \shortstack{-0.051\\{[-0.062, -0.041]}} \\
Coordinate & AURORA / XSum & \shortstack{-36.082\\{[-36.662, -35.492]}} & \shortstack{+0.069\\{[-0.879, +1.041]}} & \shortstack{-62.565\\{[-63.311, -61.827]}} \\
Global & ELF / OWT & \shortstack{+0.071\\{[+0.063, +0.080]}} & \shortstack{+0.008\\{[-0.0002, +0.017]}} & \shortstack{+0.003\\{[-0.006, +0.012]}} \\
Global & ELF / XSum & \shortstack{-0.022\\{[-0.073, +0.032]}} & \shortstack{-0.040\\{[-0.093, +0.017]}} & \shortstack{-0.049\\{[-0.104, +0.008]}} \\
Global & AURORA / OWT & \shortstack{-0.043\\{[-0.052, -0.034]}} & \shortstack{-0.029\\{[-0.038, -0.019]}} & \shortstack{-0.055\\{[-0.065, -0.045]}} \\
Global & AURORA / XSum & \shortstack{-0.058\\{[-0.105, -0.010]}} & \shortstack{-0.059\\{[-0.107, -0.013]}} & \shortstack{-0.085\\{[-0.135, -0.035]}} \\
\bottomrule
\end{tabular}
\caption{Exploratory context-adapter controls. Differences are percentage
points with descriptive 95\% paired bootstrap intervals (4,000 draws).
Fits are averaged within a generation; XSum resamples the 1,024 shared documents
across generation seeds, and OWT preserves sequence-index pairing across
conditions within each generation seed. ELF OWT pools B32/M32/PD8; AURORA OWT
pools its three step budgets. The coordinate-scaled round and the subsequent
global-scale control are both retained. Empty denominators remain in coverage
counts but have undefined accuracy. Native fidelity is not task correctness.}
\label{tab:context_adapter_controls}
\end{table}

\paragraph{Separating a scaling failure from a context effect.}
The initial per-coordinate training standardization causes AURORA's nonzero
adapters to fail severely on XSum. An independent CPU audit reproduces all
cached predictions on 96 OWT and 64 XSum positions. A single coordinate (86)
has nearly constant OWT training values and the only standard deviation floored
to $10^{-4}$, but varies substantially on XSum. The raw XSum focal/context RMS
medians are $0.608/0.476$; after standardization they are $98.84/87.37$, with
this coordinate accounting for more than $99.98\%$ of standardized energy.
The elementary identity
\[
 r(q)= [UV]_{:k}\frac{q_k-\mu_k}{s_k}+r_{-k}(q),
 \qquad
 \|r(q+\delta)-r(q)\|_2
 \leq \|UV\operatorname{diag}(s^{-1})\|_2\|\delta\|_2
\]
isolates this amplification. Subtracting only that coordinate's residual-logit
contribution restores the two context adapters from $32.81/25.00\%$ to
$87.50/87.50\%$ agreement on the inspected XSum positions, matching the base.
This selected, limited-sample ablation diagnoses the numerical failure; it is
not an independently validated model or a task-quality result.

We then run a disclosed second-round control that retains the training means
and divides every feature coordinate by the single training RMS scale
$s=\sqrt{D^{-1}\sum_j\operatorname{Var}_{\mathrm{train}}(q_j)}$.
The same global rule applies to both interfaces and all arms, with no special
handling of coordinate 86; all other fitting and evaluation settings remain
fixed. The rule was chosen after observing first-round AURORA XSum failure.
Its results therefore remain exploratory despite training-only statistics and
validation-only checkpoint selection. The global control removes the collapse;
the direct contrasts in Table~\ref{tab:context_adapter_controls} determine
whether this limited context representation adds value over focal fitting,
shuffled context, and the unchanged head.
With global scaling, the ELF OWT context arm exceeds focal refitting by +0.071 [+0.063, +0.080] percentage points, but its contrast with the unchanged head is only +0.003 [-0.006, +0.012] points. Thus its apparent advantage over the focal arm does not establish an improvement over the original readout. AURORA context agreement remains slightly below both the focal and unchanged heads on OWT and XSum. The tested context feature does not yield a consistent cross-interface gain.

\paragraph{Scope.}
These interventions do not establish a reliable cross-interface improvement
in native fidelity, a new generative algorithm, or a complete causal account of
the diffusion trajectory. A context arm beating a degraded focal refit does
not by itself improve over the unchanged head. Conversely, a null result for
four preceding states and rank eight does not rule out richer context models.
The native output boundary is retained, and no task-quality or latency benefit
is inferred from these agreement measurements. The evidence instead separates
the usefulness of a tested context feature from adaptation-induced instability
and reinforces the need to audit feature scaling under checkpoint/task shift.

\section{Selecting Native Decoding Without Native Test Outputs}
\label{app:native-free-routing}

The known-error interventions do not specify how to make a decision before
seeing the native output. This separate test asks whether affine-only features
can identify entire outputs for which native decoding has greater task-score
value. Both endpoints are computed for paired evaluation; the experiment is a
counterfactual routing audit, not an executed acceleration pipeline.

\paragraph{Development and disjoint test articles.}
The previously examined 1,024-article XSum cohort, with its two generation seeds,
is now explicitly reused as development data for a task-specific risk model.
The OWT affine readouts remain unchanged. Before fresh generation, we freeze
the feature definitions, fitted coefficients, cohort, policies, endpoints,
and analysis code locally. This is a recorded prospective test, not a public
preregistration. From the same pinned XSum test split, a seeded shuffle selects
256 new articles after excluding prior test/pilot IDs and exact document
hashes, training-set overlap, and within-cohort duplicates. Filling the new
cohort rejects 16 overlap candidates and eight ineligible candidates under the
same released AURORA eligibility rule. References affect eligibility and
evaluation, but never generation or test-time risk features.

Generation bases 61032301 and 61034301 each add the new article index. They are
disjoint from the original seed ranges. Two architectures and two seeds yield
1,024 new generations, using the unchanged checkpoints and sampler settings of
\secref{app:task-transfer}. Each native or affine output uses its \emph{own}
stopping convention: first EOS for ELF, or the first body-ending EOS after
skipping leading EOS tokens for AURORA; absent a terminating EOS, it uses the
canvas end. There is no native lexical mask, boundary clamp, or corrected
token feedback in the affine arm. These remain unmodified lexical probes
whose original fitting omitted EOS targets, not independently optimized
terminating decoders.

\paragraph{Features and frozen task calibration.}
For the seed-17 affine logits at every generated target position, let $g_i$ be
the top-two logit gap and $H_i$ the categorical entropy in nats. The eight
features are the mean and 10th percentile of $\log(1+g_i)$, mean $H_i$, mean
$\log(1+c_i)$ for the original OWT count of the predicted token, predicted
digit-containing-token fraction, whether any predicted token is EOS, the
affine stopping index divided by canvas length, and canvas length. Summaries
use the full generated target canvas, including positions after an affine
EOS. The canvas is available from generation before native decoding; no native
labels, native logit margins, native--affine disagreements, native stopping
information, references, or external language-model scores enter these features.
Presence of any EOS is distinct
from the terminating-EOS indicator when leading AURORA EOS tokens are skipped.

For each architecture separately, ridge regression predicts
$y=B_{\mathrm{native}}-B_{\mathrm{affine}}$, where $B$ is independently stopped
BERTScore F1 multiplied by 100. It minimizes summed squared error plus
$10\|\beta\|_2^2$, with an unpenalized intercept and development-only feature
standardization. No hyperparameter search or model selection is performed.
Five document-grouped folds describe development performance only. A second
ridge uses just the three stopping/canvas features. In development, EOS
presence and stopping fraction are constant for both families; canvas length
is also constant for ELF. This control is therefore a constant predictor for
ELF and a canvas-length predictor for AURORA, not a learned stopping model.
An advantage over it cannot rule out the other features indirectly predicting
native output length.

\paragraph{Matched native-call budget and estimands.}
Within each architecture/seed group, the primary policy routes the 128 of 256
outputs with largest predicted native-minus-affine BERTScore gain. Ties use a
fixed hash of family, seed, and document ID. The stopping-feature model, mean
pairwise disagreement of the three frozen affine fits, and negative mean
$\log(1+g_i)$ provide three fixed descriptive alternatives with the same call
count. No policy is chosen after test outcomes. This is batch ranking at a
fixed budget, not a calibrated online threshold.

For one group of $N=256$ outputs, let $r_i$ indicate native selection and
$q=1/2$. The exact random-policy
expectation for uniformly selecting half the outputs gives the paired contrast
\[
 \Delta_B=\frac{1}{N}\sum_i(r_i-q)
       (B_{i,\mathrm{native}}-B_{i,\mathrm{affine}}).
\]
Thus no favorable random seed is selected. The primary endpoint is pooled
$\Delta_B$, averaging both seeds and architectures equally within each article.
A document-paired bootstrap uses 4,000 draws and seed 61032329, with a single
primary 95\% interval. All other intervals are descriptive. Fitted models and
realized batch selections are held fixed across resamples, so intervals do
not include fitting uncertainty or uncertainty from reranking a different
batch. The analogous positive NLL benefit replaces the last factor by
$L_{i,\mathrm{affine}}-L_{i,\mathrm{native}}$. NLL uses the document intersection
where both endpoints are scorable in all four family/seed replicas. Selections
remain those of the full cohort, and the random marginal remains $1/2$ even
when the eligible subset has a different selected fraction. Empty outputs,
termination rates, all-native/all-affine levels, and every policy are retained.
Evaluator settings are those of \secref{app:task-transfer}.

\paragraph{Cost interpretation.}
Equal native-call counts do not imply equal FLOPs or wall time: native
contextual decoding can depend on prefix length. We record each terminal
native forward, all three affine heads with risk features, and an isolated
replay of the single affine head with its features and ridge prediction on
the same latent. The replay must reproduce the feature vector exactly.
These shared-GPU timing intervals include scheduling and, for the affine
path, host synchronization gaps. The complete sampling wall time also includes
native decoding and state capture. In ELF, denoising and native decoding share
weights, so skipping the terminal forward does not remove the denoiser or its
resident weights. In the tested AURORA path, one terminal native decode follows
latent generation and does not control earlier stopping. These instruments
describe stage costs; they do not measure an end-to-end routed speedup.
The native-forward timer ends when logits are returned: it excludes argmax,
subsequent logit processing, text decoding, and CPU state capture, and is not
the cost of the complete native endpoint.

\paragraph{Results and heterogeneity.}
All 256 articles are scorable for both endpoints in all four architecture/seed
replicas; no NLL exclusion is needed and no output is empty. The primary pooled
BERTScore contrast is positive, as are both descriptive family intervals
(\tabref{tab:routing-effects}). Pooled NLL remains unresolved: its ELF point
estimate is negative, whereas its AURORA interval is positive. The full model
also does not establish an AURORA BERTScore advantage over the stopping/canvas
control. For NLL, the full-minus-control contrast is negative on ELF and
positive on AURORA. These results do not establish a uniformly superior policy
across interfaces or metrics. Three-fit disagreement has a larger pooled NLL
point benefit than the primary ridge; no policy is selected retrospectively.

\begin{table}[!htbp]
\centering
\small
\caption{Whole-output selection at exactly 50\% native calls. Positive effects
favor the named policy over the analytical same-budget random expectation.
BERTScore is in $100\times$ F1 points; NLL benefit is in nats per evaluator
token. The full-minus-control row directly compares ridge with stopping/canvas.
Only pooled ridge BERTScore is primary; all other 95\% paired-document
intervals are descriptive. Each interval uses 256 article clusters, with both
generation seeds paired.}
\label{tab:routing-effects}
\begin{tabular}{lcc}
\toprule
Policy & BERTScore gain & NLL benefit \\
\midrule
\multicolumn{3}{l}{\textbf{Pooled}} \\
Ridge (primary) & $+0.2613\,[+0.1395,+0.3813]$ & $+0.0252\,[-0.0155,+0.0661]$ \\
Three-fit disagreement & $+0.2034\,[+0.0850,+0.3278]$ & $+0.0459\,[+0.0057,+0.0878]$ \\
Low affine gap & $+0.1493\,[+0.0393,+0.2662]$ & $-0.0023\,[-0.0448,+0.0415]$ \\
Stopping/canvas & $+0.0094\,[-0.1076,+0.1261]$ & $-0.0062\,[-0.0456,+0.0331]$ \\
Full minus stopping/canvas & $+0.2519\,[+0.0958,+0.4139]$ & $+0.0314\,[-0.0248,+0.0893]$ \\
\midrule
\multicolumn{3}{l}{\textbf{ELF}} \\
Ridge (primary) & $+0.3006\,[+0.1705,+0.4408]$ & $-0.0270\,[-0.0745,+0.0210]$ \\
Three-fit disagreement & $+0.2915\,[+0.1739,+0.4175]$ & $+0.0222\,[-0.0251,+0.0697]$ \\
Low affine gap & $+0.2102\,[+0.0725,+0.3517]$ & $-0.0417\,[-0.0900,+0.0052]$ \\
Stopping/canvas & $+0.0414\,[-0.0658,+0.1505]$ & $+0.0408\,[+0.0007,+0.0796]$ \\
Full minus stopping/canvas & $+0.2592\,[+0.1048,+0.4238]$ & $-0.0678\,[-0.1259,-0.0092]$ \\
\midrule
\multicolumn{3}{l}{\textbf{AURORA}} \\
Ridge (primary) & $+0.2220\,[+0.0251,+0.4161]$ & $+0.0774\,[+0.0100,+0.1458]$ \\
Three-fit disagreement & $+0.1154\,[-0.0634,+0.3015]$ & $+0.0696\,[+0.0118,+0.1311]$ \\
Low affine gap & $+0.0884\,[-0.1137,+0.2918]$ & $+0.0371\,[-0.0279,+0.1053]$ \\
Stopping/canvas & $-0.0225\,[-0.2443,+0.2002]$ & $-0.0532\,[-0.1278,+0.0206]$ \\
Full minus stopping/canvas & $+0.2446\,[-0.0458,+0.5301]$ & $+0.1306\,[+0.0298,+0.2339]$ \\
\bottomrule
\end{tabular}
\end{table}

\paragraph{Absolute levels and lexical overlap.}
\tabref{tab:routing-levels} retains all-native and all-affine endpoints,
the exact random expectation, and every fixed policy. All-native decoding
retains higher pooled BERTScore and ROUGE and lower NLL than the primary
50\% policy. ROUGE is a separately recorded descriptive reporting extension:
its reporting rule was recorded before fresh ROUGE outcomes were inspected,
and it neither changes the frozen primary analysis nor adds a significance
test. Ridge-minus-random pooled ROUGE-1/2/L differences are
$+0.759/+0.448/+0.548$ points. The corresponding ELF differences are
$+0.277/+0.092/+0.103$ and AURORA differences are
$+1.241/+0.804/+0.993$. These automatic overlap and reference-similarity
effects do not establish source factuality or human preference.

\begin{table}[!htbp]
\centering
\small
\caption{Absolute scores on the same 256 fresh articles. BERTScore and
ROUGE-1/2/L are F1 multiplied by 100 (higher is better); NLL is the mean
per-output GPT-2-Large token NLL (lower is better). Native-call percentage is
fixed within each architecture/seed group. All outputs use their own stopping
boundary. The full cohort is NLL-eligible. ROUGE levels are descriptive.}
\label{tab:routing-levels}
\begin{tabular}{lrrrrrr}
\toprule
Policy & Calls (\%) & BERT & NLL & R1 & R2 & RL \\
\midrule
\multicolumn{7}{l}{\textbf{Pooled}} \\
All native & 100 & 89.147 & 4.074 & 35.857 & 12.524 & 28.139 \\
All affine & 0 & 86.760 & 4.865 & 29.578 & 9.071 & 22.579 \\
Random expectation & 50 & 87.954 & 4.469 & 32.718 & 10.798 & 25.359 \\
Ridge (primary) & 50 & 88.215 & 4.444 & 33.477 & 11.245 & 25.906 \\
Three-fit disagreement & 50 & 88.157 & 4.423 & 33.122 & 10.989 & 25.632 \\
Low affine gap & 50 & 88.103 & 4.472 & 33.026 & 10.866 & 25.550 \\
Stopping/canvas & 50 & 87.963 & 4.476 & 33.004 & 10.867 & 25.513 \\
\midrule
\multicolumn{7}{l}{\textbf{ELF}} \\
All native & 100 & 88.644 & 4.553 & 34.857 & 11.310 & 27.212 \\
All affine & 0 & 87.153 & 5.076 & 31.540 & 9.434 & 24.959 \\
Random expectation & 50 & 87.899 & 4.815 & 33.198 & 10.372 & 26.086 \\
Ridge (primary) & 50 & 88.199 & 4.842 & 33.476 & 10.464 & 26.188 \\
Three-fit disagreement & 50 & 88.190 & 4.793 & 33.659 & 10.488 & 26.302 \\
Low affine gap & 50 & 88.109 & 4.857 & 33.334 & 10.331 & 26.146 \\
Stopping/canvas & 50 & 87.940 & 4.774 & 33.359 & 10.477 & 26.230 \\
\midrule
\multicolumn{7}{l}{\textbf{AURORA}} \\
All native & 100 & 89.650 & 3.594 & 36.857 & 13.738 & 29.066 \\
All affine & 0 & 86.368 & 4.653 & 27.616 & 8.708 & 20.198 \\
Random expectation & 50 & 88.009 & 4.124 & 32.237 & 11.223 & 24.632 \\
Ridge (primary) & 50 & 88.231 & 4.046 & 33.478 & 12.027 & 25.624 \\
Three-fit disagreement & 50 & 88.124 & 4.054 & 32.584 & 11.490 & 24.963 \\
Low affine gap & 50 & 88.097 & 4.087 & 32.718 & 11.401 & 24.954 \\
Stopping/canvas & 50 & 87.986 & 4.177 & 32.650 & 11.256 & 24.797 \\
\bottomrule
\end{tabular}
\end{table}

\paragraph{Risk levels, output boundaries, and costs.}
Fresh risk MAE is $1.3335$ BERTScore points for ELF and $1.3118$ for AURORA;
mean signed biases are $-0.0152$ and $-0.0033$, and Spearman correlations are
$0.3012$ and $0.2660$. These are descriptive diagnostics, not evidence that
individual gains are accurately calibrated. The fixed-budget test depends
on ranking and does not validate an absolute decision threshold.

Native outputs terminate at EOS in $506/512$ ELF and $509/512$ AURORA runs.
Neither unmodified affine endpoint produces a terminating EOS. Mean evaluator
token counts are $25.86$ native versus $28.07$ affine for ELF, and $24.90$
versus $66.98$ for AURORA. Stopping indices refer to canvas coordinates and
need not equal visible text lengths, since text decoding suppresses special
tokens. The stopping/canvas control therefore addresses only its three
observed variables; other features may still encode native-length information.
This experiment does not identify a length-independent causal benefit or test
an affine head optimized for termination.
The selected/unselected mean source-prefix lengths are $646.2/342.6$ for ELF
and $436.9/483.2$ for AURORA. Selection may also reflect source/context-length
proxies; the result concerns the combined feature set, not an isolated causal
contribution of confidence or token frequency.

Mean/median native-forward intervals are $10.37/9.74$ ms for ELF and
$9.66/8.00$ ms for AURORA. The isolated single-head/features/ridge replay
uses $5.25/1.81$ and $5.35/2.51$ ms; all three heads together use
$16.49/13.68$ and $16.67/12.35$ ms. The mean sampling wall times, including
native decoding and capture, are $1.269$ and $2.903$ s. Merely adding each
selected-native forward proxy to the mean single-head replay yields
$10.39$ and $10.21$ ms per output, compared with $10.37$ and $9.66$ ms for
all native forwards. Even this rough stage sum does not support a speed gain.
Shared scheduling, host synchronization gaps, separate replay, and omitted
postprocessing prevent interpreting these sums as deployed end-to-end latency.

The complete evidence retains all 1,024 compact token records, independently
stopped texts, metric caches and NLL traces, frozen development fits, and
selection/interval inputs. A separate verifier refits all four ridge models
and reconstructs all budgets and intervals. Full-state integrity checks cover
all 1,024 local latent records; four independent CPU checks reproduce affine
argmax tokens exactly and all eight features within $1.46\times10^{-6}$.
The portable numerical evidence package omits large latents and pretrained weights, so
its checks do not rerun neural generation or metric inference.
\FloatBarrier

\section{Compressed and Directional Interfaces: Boundary Tests}
\label{app:new-interfaces}

The main replication asks whether a token-wise affine map preserves a contextual decoder's native decisions on generated ELF and AURORA states. Two additional public interfaces test the limits of that question. CALM~\citep{shao2025calm} changes the unit to a compressed four-token chunk and uses an autoregressive prior. LLaCO~\citep{jia2026llaco} changes the corruption geometry and inserts a learned directional adapter before its contextual backbone. Neither is an architecture-isolated benchmark competitor. Primary protocols were fixed before their checkpoint measurements; subsequent fitting controls and competitor audits are explicitly marked exploratory, with readout selection confined to validation data. We retain failures as well as gains.

\subsection{CALM: readout fidelity after four-token compression}
\label{app:calm-compression}

We use the released 75.8M-parameter autoencoder, which compresses four tokens into a 128-dimensional latent and reconstructs them through residual MLPs and a tied unembedding. Its weights are bitwise identical to the corresponding component in the released CALM-M checkpoint. We select 4,096 fitting, 256 validation, and 512 test documents from a local OpenWebText shard, deduplicated by full-document hash, and take 256 tokens per document. These splits are disjoint for this probe experiment; we do not establish exclusion from the original model's pretraining data.

Each of four position-specific affine classifiers receives the same complete chunk latent and predicts over the full 128,257-token vocabulary. Targets are the native decoder's decisions on the identical state. The initialization regresses native final decoder features on the latent and composes that affine map with the fixed unembedding. Three independent AdamW fits then optimize full-vocabulary cross-entropy for 2,048 updates; the highest validation agreement among updates 0/512/1,024/2,048 selects each fit. No test result selects a checkpoint. This finite-budget protocol measures the fitted maps, not the best achievable affine approximation.

\begin{table}[!htbp]
\centering
\footnotesize
\caption{CALM boundary audit on the same 512 test documents. Percentages; brackets are descriptive 95\% document-bootstrap intervals. Native/reference measures reconstruction, whereas affine/native measures decision fidelity on identical states. Exact chunk requires all four decisions to match. All three primary optimization runs select their shared ridge initialization, so the selected map is identical across runs; the interval does not represent independent fitting variation.}
\label{tab:calm-compressed-readout}
\begin{tabular}{lrrr}
\toprule
State & Native/reference & Affine/native [95\% CI] & Exact chunk \\
\midrule
Posterior sample & 99.88 & 90.84 [90.32, 91.26] & 70.31 \\
Posterior mean & 99.88 & 94.37 [93.89, 94.74] & 81.14 \\
Mean $+\,0.25\epsilon$ & 99.87 & 92.69 [92.20, 93.09] & 76.20 \\
Posterior $\times0.5$ & 71.72 & 69.28 [68.79, 69.73] & 24.02 \\
Posterior $\times2$ & 99.85 & 92.63 [92.14, 93.02] & 75.52 \\
\bottomrule
\end{tabular}
\end{table}

The primary validation-selected map recovers 90.84\% of native decisions on sampled posterior states, while native reconstruction is 99.88\%. All three unrestricted CE runs select update zero. This outcome motivates a fit-adequacy control rather than an intrinsic nonlinear-decoding claim.

The follow-up is exploratory because the first test was already inspected. It compares the original ridge map, a lower-learning-rate unrestricted affine fit, and four affine latent-to-feature maps with the released unembedding fixed. The latter fit 264,192 parameters instead of 66,180,612; frozen vocabulary weights still contribute storage and inference computation. Learning rate and checkpoint selection use only the original validation documents. The selected fits are then frozen before evaluating the next 512 documents in the original seeded permutation; exact hashes and token sequences verify that these documents are disjoint from all original splits.

\begin{table}[!htbp]
\centering
\footnotesize
\caption{Fit-adequacy follow-up on 512 additional documents, encoded after model selection. Agreement and paired gain are percentages/percentage points with descriptive 95\% document-bootstrap intervals, averaged across three fits for each fitted family. Native/reference reconstruction is 99.94\%. This tests readout fitting, not generation quality or a nonlinear expressivity lower bound.}
\label{tab:calm-fit-adequacy}
\begin{tabular}{lrr}
\toprule
Readout & Native agreement [95\% CI] & Gain over ridge [95\% CI] \\
\midrule
Original ridge & 90.73 [90.25, 91.17] & 0.00 [0.00, 0.00] \\
Unrestricted, lower LR & 90.74 [90.25, 91.18] & 0.00 [-0.03, 0.04] \\
Fixed-unembedding affine & 93.33 [92.88, 93.71] & 2.60 [2.48, 2.72] \\
\bottomrule
\end{tabular}
\end{table}

We also use the official CALM-M generative head to sample one next-chunk latent at temperature 1 for each of 63 teacher prefixes in each of the original 512 test documents (32,256 latent predictions). The original ridge map agrees with the native decoder on 41.75\% [41.39, 42.15] of decisions; the validation-selected fixed-unembedding maps reach 43.48\% [43.10, 43.88], a paired change of 1.73 points [1.62, 1.84]. This additional transfer test uses the same state and native label for every readout. It is a conditional one-step interface test, not a free-running CALM benchmark or a comparison against a unique correct continuation.

An additional exploratory control changes the fitting population while holding the affine family, initialization, AdamW settings, seeds, minibatch indices, and 2,048-update budget fixed. Both arms use the same 4,096 fitting documents and 63 aligned chunk positions per document; one uses supplied posterior states and the other samples prior states from the corresponding teacher prefixes. Both select their checkpoint on the same 256-document prior-state validation pool. After all six fits are frozen, we generate 32,256 new prior states from the 512 held-out documents used in the posterior follow-up above. The documents are reused for this new state condition; they are not a third independent document cohort.

\begin{table}[!htbp]
\centering
\footnotesize
\caption{Matched fitting-population control on newly sampled CALM-M prior states. Both arms use fixed-unembedding affine readouts and the same prior-state validation pool. Three-fit mean native agreement is reported with descriptive 95\% document-bootstrap intervals. Temperature 1 and teacher prefixes are unchanged.}
\label{tab:calm-generated-fit}
\begin{tabular}{lr}
\toprule
Fitting population & Native agreement [95\% CI] \\
\midrule
Supplied posterior states & 43.21 [42.77, 43.64] \\
Sampled prior states & 43.88 [43.47, 44.30] \\
\bottomrule
\end{tabular}
\end{table}

Matching the fitting population increases prior-state agreement by 0.68 percentage points [0.52, 0.83]. This modest gain does not close the observed gap under the declared budget. All six fits select the last evaluated checkpoint and are still improving there, so the residual is not an affine expressivity lower bound or a demonstrated property of compression itself. The control establishes the effect of this fitting-data substitution, not a full causal account of the generated states.

Using an affine translator before a frozen unembedding is an established probe design, exemplified by the tuned lens~\citep{belrose2023tunedlens}. Our four-position compressed-state control uses native argmax targets and omits a terminal normalization so that the effective readout stays affine. The contribution of this follow-up is the matched fitting and transfer evidence, not a claim to invent vocabulary-tied probes.

The CALM-M transfer failure survives implementation checks: the official loader reproduces all 234 stored checkpoint tensors, and the direct one-step path matches the official temperature-1 sampling method on a pilot. Matched-batch checks under CALM's pinned dependencies reproduce the stored encoder outputs and all 196,608 validation/test native labels. A post hoc comparison of four paired chunk positions per test document finds similar mean latent norms (10.64 for supplied posterior states; 10.56 for prior predictions), but smaller mean native logit gaps (15.68 versus 4.54). Relative distance to the posterior mean obtained by re-encoding the native decoded chunk rises from 0.393 to 0.816. These are observational signatures of different decoder-facing states, not identification of the cause of the transfer gap. They do not support a simple global scale explanation on their own. The experiment uses the base temperature-1 head; CALM's lower-temperature candidate-collision selection, free-running fluency, and published BrierLM protocol are separate settings that we do not evaluate.

The posterior-sample, posterior-mean, and additional-noise conditions separate clean reconstruction from fidelity under a specified perturbation. Positive radial scaling is a separate paired intervention that changes norm while preserving direction. CALM is not trained to make its entire decoder scale invariant, so a directional rule cannot be transferred merely because both models use continuous states. Chunk size, vocabulary, training distribution, and probe budget also differ from ELF and AURORA; the table does not isolate compression as the cause of any numerical gap.

\takeawaybox{Boundary finding}{CALM readout gaps remain budget-dependent}{%
Improving the affine fit helps posterior-state fidelity but does not close the tested prior-state gap. The temperature-1 teacher-prefix fits are still improving, so the residual is not an expressivity lower bound.%
}{takeaway:calm}

\subsection{LLaCO: directional identity is not the soft interface}
\label{app:llaco-directional}

We extract the released LLaCO-8B geometry snapshot and learned adapter parameters from its distributed checkpoint. The component has 131,072 codewords of dimension 4,096, inverse temperature $\beta=\exp(3)$, and learned token biases. The released adapter uses
\[
 p_v(s)=\operatorname{softmax}_v\!\left(\beta\,u_v^\top\frac{s}{\|s\|}+b_v\right),
 \qquad a(s)=\sum_v p_v(s)e_v,\qquad u_v=\frac{e_v}{\|e_v\|}.
\]
This soft embedding is supplied to the time-conditioned contextual backbone. The final token is selected from that backbone's logits, not from cosine lookup or the adapter posterior mode. Our measurements below concern the released component under forward corruption, not full-model generation quality or native-decoder replacement.

We sample 1,024 distinct, non-special, nonzero source codewords, reserve 256 for development, and report the remaining 768. Each source shares a random direction across the uniform-endpoint and mask-biased-endpoint trajectories, using the upstream spherical interpolation. All vocabulary entries remain competitors. Raw Euclidean comparisons restore the original source radius along the spherical path so that the clean state equals its own embedding. In a separate noise-free radial control, scaling source embeddings by $0.5$ reduces Euclidean source recovery from $100\%$ to $0\%$, while cosine and normalized-adapter decisions remain unchanged. Removing special and unused competitors in a post hoc check still gives $0\%$ Euclidean recovery. This reproduces the norm sensitivity motivating LLaCO; it does not establish that an arbitrary contextual decoder should discard radial information.

\begin{table}[!htbp]
\centering
\footnotesize
\caption{LLaCO component measurements on 768 paired held-out source codewords. All entries after $\gamma$ are percentages. Cosine/source and adapter/source measure recovery of the supplied token. Mode agreement compares two component readouts, not the contextual model. Max mass is the mean largest adapter posterior probability. Under mask-biased endpoints at $\gamma=0.8/0.9$, both modes are always mask; their agreement is not lexical recovery.}
\label{tab:llaco-interface-boundary}
\begin{tabular}{lrrrrr}
\toprule
Endpoint & $\gamma$ & Cosine/source & Adapter/source & Mode agreement & Max mass \\
\midrule
Uniform & 0.8 & 100.00 & 100.00 & 100.00 & 0.31 \\
Uniform & 0.9 & 100.00 & 3.39 & 3.39 & 0.03 \\
Uniform & 0.95 & 73.31 & 0.00 & 0.00 & 0.03 \\
Mask-biased & 0.5 & 100.00 & 100.00 & 100.00 & 87.22 \\
Mask-biased & 0.8 & 0.00 & 0.00 & 100.00 & 9.57 \\
Mask-biased & 0.9 & 0.00 & 0.00 & 100.00 & 16.54 \\
Mask-biased & 1 & 0.00 & 0.00 & 100.00 & 20.71 \\
\bottomrule
\end{tabular}
\end{table}

The winner identity matters. Under the mask-biased endpoint, the cosine and adapter winners at $\gamma=0.8$ and $0.9$ are the mask token in every test case. At $\gamma=0.8$, the source still ranks first among lexical competitors in every case; removing the mask therefore changes the meaning of the retention statistic. Under the uniform endpoint at $\gamma=0.9$, cosine retains every source token, but the learned adapter mode retains only $3.39\%$. Learned bias and the chosen endpoint distribution must therefore be retained in an interface diagnosis. These observations complement LLaCO's use of logarithmic rank decay instead of treating a top-1 trace as a sufficient account of the soft state.

\takeawaybox{Boundary finding}{LLaCO consumes more than the winning token}{%
Matching the most likely token does not preserve the soft embedding passed to the backbone. In the tested trajectories, matching modes can coexist with diffuse mass or selection of the mask token.%
}{takeaway:llaco}

\subsection{A bounded adapter optimization and its failure boundary}
\label{app:llaco-adapter-cost}

We test an implementation improvement and a simple approximation separately. Caching the FP32 raw and normalized codebooks permits an algebraically equivalent dense expectation. Against the unmodified streamed implementation on the same four states, the maximum FP32 component discrepancy is $1.05\!\times\!10^{-7}$ across the four parity cases; BF16 output comparisons are also retained. This cache costs an additional 4\,GiB. Its speed benefit is a memory--latency trade-off, not a free replacement for the original memory-conserving implementation.

The approximation retains the shortest leading set $S$ with posterior mass at least $1-\varepsilon$ if it has at most 64 entries, and otherwise computes the exact dense expectation. We retain both fixed budgets $\varepsilon=10^{-3}$ and $10^{-2}$. With omitted mass $\delta$ and $R=\max_v\|e_v\|$, the renormalized retained expectation $a_S$ satisfies the elementary bound
\[
 \|a-a_S\|=\delta\|a_{S^c}-a_S\|\leq 2\delta R.
\]
This is a standard convex-combination inequality, not a new robustness theorem. It bounds an adapter-input perturbation; it does not certify the contextual decoder's tokens or task performance.

\begin{table}[!htbp]
\centering
\footnotesize
\caption{Adapter-only latency in milliseconds on one RTX 5090, batch four, FP32 accumulation: median of 20 synchronized repetitions after five warmups on the same states. Selection and fallback work are included. Dense and sparse implementations cache both FP32 codebooks (4\,GiB additional memory). These shared-machine measurements exclude the 8B backbone, KV cache, and full generation; neither end-to-end speedup nor unchanged generated text is established.}
\label{tab:llaco-adapter-runtime}
\begin{tabular}{rrrrr}
\toprule
$\gamma$ & Upstream streamed & Dense exact & Sparse $10^{-3}$ & Sparse $10^{-2}$ \\
\midrule
0 & 42.64 & 3.37 & 1.85 & 1.93 \\
0.9 & 42.63 & 3.37 & 3.65 & 3.64 \\
0.98 & 42.66 & 3.39 & 3.64 & 3.64 \\
1 & 42.65 & 3.39 & 3.64 & 3.64 \\
\bottomrule
\end{tabular}
\end{table}

On clean test codewords, the $10^{-3}$ and $10^{-2}$ budgets permit sparse evaluation in $99.22\%$ and $100\%$ of cases, respectively. At every tested nonzero mask-biased corruption level, both budgets fall back to the exact expectation. Their selection overhead then makes them slower than the optimized exact baseline. All observed approximation errors obey the stated bound up to the separately recorded FP32 tolerance. Agreement of the most likely token is consequently insufficient to infer concentration of the soft interface, and a local clean-state speed gain does not establish useful trajectory acceleration.

These extensions sharpen the scope of the main story: report the state unit, the decoder actually used, and the information lost by the proposed substitution. CALM tests fitted decision recovery after compression. LLaCO distinguishes geometric identity, learned soft input, and contextual output, and supplies a component-level engineering opportunity with an explicit memory cost. A general generation algorithm would additionally require full-backbone evaluation, independently stopped outputs, matched total compute, and task-quality evidence. None is inferred from these boundary experiments.

\clearpage
\section{Cross-Family Evidence Map}
\label{app:cross-family-evidence}

This map summarizes existing measurements; it introduces no new experiment or aggregate score. Only ELF and AURORA carry the replicated matched repair and conditional-task comparison. Other rows answer narrower questions under their own native interfaces. In particular, CALM is an autoregressive boundary case, and the LLaCO rows concern released components rather than full 8B generation. Numerical values across rows use different targets, fitting conditions, and denominators and should not be ranked against one another.

\begin{table}[!htbp]
\centering
\footnotesize
\setlength{\tabcolsep}{4pt}
\renewcommand{\arraystretch}{1.15}
\caption{What each family contributes to the argument. ``Native'' refers to the original operation for that interface, not semantic ground truth. The last column links to the full protocols, null results, and limitations. All values summarize retained experiments.}
\label{tab:cross-family-evidence}
\begin{tabular}{@{}>{\raggedright\arraybackslash}p{0.13\linewidth}>{\raggedright\arraybackslash}p{0.25\linewidth}>{\raggedright\arraybackslash}p{0.41\linewidth}>{\raggedright\arraybackslash}p{0.14\linewidth}@{}}
\toprule
Family & Object and comparison & Finding relevant to simplification & Evidence \\
\midrule
ELF & Generated token states; repeated fits, matched repairs, fresh selection & $97.17\%$ native agreement leaves consequential misses. Equal-count policies differ in pooled evaluator effects; selection benefit does not establish acceleration. & \secref{sec:prospective-readout}; \secref{sec:native-free-routing} \\
\addlinespace
AURORA-LM & Independent contextual decoder; same repair questions & $95.74\%$ native agreement and the repair ordering replicate. Tokenizer and fitting pool differ from ELF; architecture effects are not isolated. & \secref{sec:aurora-replication}; \secref{sec:task-transfer} \\
\addlinespace
LangFlow & Native linear skip plus contextual term; fitting-pool control & Learned agreement changes from $62.62\%$ to $74.15\%$; frozen lookup reaches $89.69\%$. Probe budget and the native vocabulary shortcut matter. & \secref{app:confirmatory-readout} \\
\addlinespace
CALM & Four-token latent; posterior and conditional prior states & Fixed-unembedding fitting improves posterior fidelity. Matched prior-state fitting gives only $43.21\%\to43.88\%$ under the tested budget; convergence is unestablished. & \secref{app:calm-compression} \\
\addlinespace
Cola-DLM & VAE reconstruction under declared corruption & High clean recovery does not establish corrupted-state recovery. Overlap of a sampled continuation with a particular teacher target is a different diagnostic. & \secref{sec:cola-boundary} \\
\addlinespace
CoBit & Bit/code interface; proxy and real-code noise sweeps & Similar bit-recovery trends support testing the native state unit; they do not replicate lexical repair or prove ELF's trajectory mechanism. & \secref{sec:scale-sampler-architecture} \\
\addlinespace
LLaCO & Directional geometry and learned soft adapter & Top-1 agreement can hide diffuse input or mask selection. Exact caching improves component latency at a memory cost; tested sparse budgets fall back under noise. & \secref{app:llaco-directional}; \secref{app:llaco-adapter-cost} \\
\addlinespace
Encoder swaps & T5 variants, Switch, mT5, ByT5, T5Gemma & Dimensions, tokenizers, normalization, or decoder pairings change. These compatibility controls do not rank the encoders or add independent repair replications. & \secref{sec:limitations} \\
\bottomrule
\end{tabular}
\end{table}

\parhead{What can be carried to a new model}
The transferable element is the comparison: same-state fidelity, matched intervention consequences, held-out prediction, and measured total cost. The numerical agreement threshold, a margin cutoff, the rare-first ordering, and the best fitting recipe are not assumed to transfer. A new decoder or state unit requires its own targets and controls. This distinction lets the broad model coverage inform future work without turning heterogeneous boundary checks into a universal law.

\end{document}